\documentclass[10pt,english,journal,cspaper,compsoc]{IEEEtran}
\usepackage[T1]{fontenc}
\usepackage[latin9]{inputenc}
\usepackage[letterpaper]{geometry}
\usepackage{color}
\usepackage{array}
\usepackage{float}
\usepackage{textcomp}
\usepackage{amsthm}
\usepackage{amsmath}
\usepackage{amssymb}
\usepackage{graphicx}

\makeatletter

\newcommand{\lyxmathsym}[1]{\ifmmode\begingroup\def\b@ld{bold}
  \text{\ifx\math@version\b@ld\bfseries\fi#1}\endgroup\else#1\fi}

\providecommand{\tabularnewline}{\\}
\newcommand{\lyxdot}{.}

\floatstyle{ruled}
\newfloat{algorithm}{tbp}{loa}
\providecommand{\algorithmname}{Algorithm}
\floatname{algorithm}{\protect\algorithmname}

  \theoremstyle{plain}
  \newtheorem*{thm*}{\protect\theoremname}
  \theoremstyle{plain}
  \newtheorem*{lem*}{\protect\lemmaname}
  \theoremstyle{plain}
  \newtheorem*{cor*}{\protect\corollaryname}

\ifCLASSOPTIONcompsoc
\else
\fi

\ifCLASSINFOpdf
\else
\fi
\usepackage{slashed}

\usepackage{psfrag}
\usepackage{yafp}
\usepackage{comment}

\usepackage{tikz}
\usepackage{tkz-graph}
\usetikzlibrary{shapes.arrows,chains}
\usetikzlibrary{calc}
\usetikzlibrary{arrows}

\@ifundefined{showcaptionsetup}{}{%
 \PassOptionsToPackage{caption=false}{subfig}}
\usepackage{subfig}
\makeatother

\usepackage{babel}
  \providecommand{\corollaryname}{Corollary}
  \providecommand{\lemmaname}{Lemma}
  \providecommand{\theoremname}{Theorem}

\begin{document}

\title{Learning Graphical Model Parameters with Approximate Marginal Inference }

\author{Justin Domke, NICTA \& Australia National University}
\maketitle
\begin{abstract}
Likelihood based-learning of graphical models faces challenges of
computational-complexity and robustness to model mis-specification.
This paper studies methods that fit parameters directly to maximize
a measure of the accuracy of predicted marginals, taking into account
both model and inference approximations at training time. Experiments
on imaging problems suggest marginalization-based learning performs
better than likelihood-based approximations on difficult problems
where the model being fit is approximate in nature.
\end{abstract}

\IEEEpeerreviewmaketitle

\section{Introduction}

\IEEEPARstart{G}{raphical} models are a standard tool in image
processing, computer vision, and many other fields. Exact
inference and inference are often intractable, due to the high treewidth of the
graph.

Much previous work involves approximations of the likelihood. (Section
\ref{sec:Loss-Functions}). In this paper, we suggest that parameter
learning can instead be done using ``marginalization-based'' loss
functions. These directly quantify the quality of the \emph{predictions}
of a given marginal inference algorithm. This has two major advantages.
First, approximation errors in the inference algorithm are taken into
account while learning. Second, this is robust to model mis-specification.

The contributions of this paper are, first, the general framework
of marginalization-based fitting as implicit differentiation. Second,
we show that the parameter gradient can be computed by ``perturbation''--
that is, by re-running the approximate algorithm twice with the parameters
perturbed slightly based on the current loss. Third, we introduce
the strategy of ``truncated fitting''. Inference algorithms are
based on optimization, where one iterates updates until some convergence
threshold is reached. In truncated fitting, algorithms are derived
to fit the marginals produced after a \emph{fixed number} of updates,
with no assumption of convergence. We show that this leads to significant
speedups. We also derive a variant of this that can apply to likelihood
based learning. Finally, experimental results confirm that marginalization
based learning gives better results on difficult problems where inference
approximations and model mis-specification are most significant.

\section{Setup}

\subsection{Markov Random Fields}

Markov random fields are probability distributions that may be written
as
\begin{equation}
p({\bf x})=\frac{1}{Z}\prod_{c}\psi({\bf x}_{c})\prod_{i}\psi(x_{i}).\label{eq:MRF_def}
\end{equation}
This is defined with reference to a graph, with one node for each
random variable. The first product in Eq. \ref{eq:MRF_def} is over
the set of \emph{cliques} $c$ in the graph, while the second is over
all individual variables. For example, the graph 

\begin{center}
\begin{tikzpicture}[scale=1.2]
\tikzstyle{every node}=[draw,shape=circle];
\node (x1) at (0,  0)  {$x_1$};
\node (x2) at (1,  0)  {$x_2$};
\node (x3) at (2, .4)  {$x_3$};
\node (x4) at (3, .4)  {$x_4$};
\node (x5) at (2,-.4)  {$x_5$};
\node (x6) at (3,-.4)  {$x_6$};
\draw (x1) -- (x2)
      (x2) -- (x3)
      (x2) -- (x5)
      (x3) -- (x5)
      (x3) -- (x4)
(x5) -- (x6);
\end{tikzpicture}
\par\end{center}

\noindent corresponds to the distribution\vspace{-10pt}

\begin{align*}
p({\bf x})= & \frac{1}{Z}\psi(x_{1},x_{2})\psi(x_{2},x_{3},x_{5})\psi(x_{3},x_{4})\psi(x_{5},x_{6})\\
 & \times\psi(x_{1})\psi(x_{2})\psi(x_{3})\psi(x_{4})\psi(x_{5})\psi(x_{6}).
\end{align*}

Each function $\psi({\bf x}_{c})$ or $\psi(x_{i})$ is positive,
but otherwise arbitrary. The factor $Z$ ensures normalization.

The motivation for these types of models is the Hammersley\textendash{}Clifford
theorem \cite{SpatialInteractionAndTheStatistical}, which gives specific
conditions under which a distribution can be written as in Eq. \ref{eq:MRF_def}.
Those conditions are that, first, each random variable is conditionally
independent of all others, given its immediate neighbors and, secondly,
that each configuration ${\bf x}$ has nonzero probability. Often,
domain knowledge about conditional independence can be used to build
a reasonable graph, and the factorized representation in an MRF reduces
the curse of dimensionality encountered in modeling a high-dimensional
distribution.

\subsection{Conditional Random Fields}

One is often interested in modeling the conditional probability of
${\bf x}$, given observations ${\bf y}$. For such problems, it is
natural to define a Conditional Random Field \cite{ConditionalRandomFields}
\[
p({\bf x}|{\bf y})=\frac{1}{Z({\bf y})}\prod_{c}\psi({\bf x}_{c},{\bf y})\prod_{i}\psi(x_{i},{\bf y}).
\]

Here, $\psi({\bf x}_{c},{\bf y})$ indicates that the value for a
particular configuration ${\bf x}_{c}$ depends on the input ${\bf y}$.
In practice, the form of this dependence is application dependent.

\subsection{Inference Problems\label{sub:Inference-Problems}}

Suppose we have some distribution $p({\bf x}|{\bf y})$, we are given
some input ${\bf y}$, and we need to guess a single output vector
${\bf x}^{*}$. What is the best guess?

The answer clearly depends on the meaning of ``best''. One framework
for answering this question is the idea of a Bayes estimator \cite{ModelDistortionsInBayesian}.
One must specify some utility function $U({\bf x},{\bf x}')$, quantifying
how ``happy'' one is to have guessed ${\bf x}$ if the true output
is ${\bf x}'$. One then chooses ${\bf x}^{*}$ to maximize the expected
utility
\[
{\bf x}^{*}=\arg\max_{{\bf x}}\sum_{{\bf x}'}p({\bf x}'|{\bf y})U({\bf x},{\bf x}').
\]

One natural utility function is an indicator function, giving one
for the exact value ${\bf x}'$, and zero otherwise. It is easy to
show that for this utility, the optimal estimate is the popular Maximum
a Posteriori (MAP) estimate. 
\begin{thm*}
If $U({\bf x},{\bf x}')=I[{\bf x}={\bf x}'],$ then
\[
{\bf x}^{*}=\arg\max_{{\bf x}}p({\bf x}|{\bf y}).
\]

\end{thm*}
Little can be said in general about if this utility function truly
reflects user priorities. However, in high-dimensional applications,
there are reasons for skepticism. First, the actual maximizing probability
$p({\bf x}^{*}|{\bf y})$ in a MAP estimate might be extremely small,
so much so that astronomical numbers of examples might be necessary
before one could expect to exactly predict the true output. Second,
this utility does not distinguish between a prediction that contains
only a single error at some component $x_{j}$, and one that is entirely
wrong.

An alternative utility function, popular for imaging problems, quantifies
the Hamming distance, or the \emph{number of components} of the output
vector that are correct. Maximizing this results in selecting the
most likely value for each component independently.
\begin{thm*}
If $U({\bf x},{\bf x}')=\sum_{i}I[x_{i}=x_{i}'],$ then

\noindent 
\begin{equation}
x_{i}^{*}=\arg\max_{x_{i}}p(x_{i}|{\bf y}).\label{eq:MPM-inference}
\end{equation}

\end{thm*}
This appears to have been originally called Maximum Posterior Marginal
(MPM) inference \cite{ProbabilisticSolutionOfIllPosedProblems}, though
it has been reinvented under other names \cite{TrainingCRFsForMaximumLabelwiseAccuracy}.
From a computational perspective, the main difficulty is not performing
the trivial maximization in Eq. \ref{eq:MPM-inference}, but rather
computing the marginals $p(x_{i}|{\bf y})$. The marginal-based loss
functions introduced in Section \ref{sub:Marginal-based-Loss-Functions}
can be motivated by the idea that at test time, one will use an inference
method similar to MPM where one in concerned only with the accuracy
of the marginals.

The results of MAP and MPM inference will be similar if the distribution
$p({\bf x}|{\bf y})$ is heavily ``peaked'' at a single configuration
${\bf x}$. Roughly, the greater the entropy of $p({\bf x}|{\bf y})$,
the more there is to be gained in integrating over all possible configurations,
as MPM does. A few papers have experimentally compared MAP and MPM
inference \cite{kumar_exploiting,MeasuringUnvertaintyInGraphCutSolutions}.

\subsection{Exponential Family}

The exponential family is defined by
\[
p({\bf x};\boldsymbol{\theta})=\exp\bigl(\boldsymbol{\theta}\cdot{\bf f}({\bf x})-A(\boldsymbol{\theta})\bigr),
\]

\noindent where $\boldsymbol{\theta}$ is a vector of parameters,
${\bf f}({\bf x})$ is a vector of sufficient statistics, and the
log-partition function
\begin{equation}
A(\boldsymbol{\theta})=\log\sum_{{\bf x}}\exp\boldsymbol{\theta}\cdot{\bf f}({\bf x}).\label{eq:log-partition}
\end{equation}

\noindent ensures normalization. Different sufficient statistics ${\bf f}({\bf x})$
define different distributions. The exponential family is well understood
in statistics. Accordingly, it is useful to note that a Markov random
field (Eq. \ref{eq:MRF_def}) is a member of the exponential family,
with sufficient statistics consisting of indicator functions for each
possible configuration of each clique and each variable \cite{WainwrightJordanMonster},
namely,

\[
{\bf f}({\bf X})=\{I[{\bf X}_{c}={\bf x}_{c}]|\forall c,{\bf x}_{c}\}\cup\{I[X_{i}=x_{i}]|\forall i,x_{i}\}.
\]

It is useful to introduce the notation $\theta({\bf x}_{c})$ to refer
to the component of $\boldsymbol{\theta}$ corresponding to the indicator
function $I[{\bf X}_{c}={\bf x}_{c}],$ and similarly for $\theta(x_{i})$.
Then, the MRF in Eq. \ref{eq:MRF_def} would have $\psi({\bf x}_{c})=e^{\theta({\bf x}_{c})}$
and $\psi(x_{i})=e^{\theta(x_{i})}$. Many operations on graphical
models can be more elegantly represented using this exponential family
representation.

A standard problem in the exponential family is to compute the mean
value of ${\bf f}$,

\vspace{-5pt}
\[
\boldsymbol{\mu}(\boldsymbol{\theta})=\sum_{{\bf x}}p({\bf x};\boldsymbol{\theta}){\bf f}({\bf x}),
\]
called the ``mean parameters''. It is easy to show these are equal
to the gradient of the log-partition function.

\vspace{-5pt}
\begin{equation}
\frac{dA}{d\boldsymbol{\theta}}=\boldsymbol{\mu}(\boldsymbol{\theta}).\label{eq:dA_dtheta_eq_mu}
\end{equation}

For an exponential family corresponding to an MRF, computing $\boldsymbol{\mu}$
is equivalent to computing all the marginal probabilities. To see
this, note that, using a similar notation for indexing $\boldsymbol{\mu}$
as for $\boldsymbol{\theta}$ above,\vspace{-5pt}

\[
\boldsymbol{\mu}({\bf x}_{c};\boldsymbol{\theta})=\sum_{{\bf X}}p({\bf X};\boldsymbol{\theta})I[{\bf X}_{c}={\bf x}_{c}]=p({\bf x}_{c};\boldsymbol{\theta}).
\]

Conditional distributions can be represented by thinking of the parameter
vector $\boldsymbol{\theta}({\bf y};\boldsymbol{\gamma})$ as being
a function of the input ${\bf y}$, where $\boldsymbol{\gamma}$ are
now the free parameters rather than $\boldsymbol{\theta}$. (Again,
the nature of the dependence of $\boldsymbol{\theta}$ on ${\bf y}$
and $\boldsymbol{\gamma}$ will vary by application.) Then, we have
that
\begin{equation}
p({\bf x}|{\bf y};\boldsymbol{\gamma})=\exp\bigl(\boldsymbol{\theta}({\bf y};\boldsymbol{\gamma})\cdot{\bf f}({\bf x})-A(\boldsymbol{\theta}({\bf y};\boldsymbol{\gamma}))\bigr),\label{eq:conditiona-efam}
\end{equation}
sometimes called a curved conditional exponential family.

\subsection{Learning\label{sub:Learning}}

The focus of this paper is learning of model parameters from data.
(Automatically determining graph \emph{structure} remains an active
research area, but is not considered here.) Specifically, we take
the goal of learning to be to minimize the empirical risk
\begin{equation}
R(\boldsymbol{\theta})=\sum_{\hat{{\bf x}}}L\bigl(\boldsymbol{\theta},\hat{{\bf x}}\bigr),\label{eq:empirical_risk}
\end{equation}
where the summation is over all examples $\hat{{\bf x}}$ in the dataset,
and the loss function $L(\boldsymbol{\theta},\hat{{\bf x}})$ quantifies
how well the distribution defined by the parameter vector $\boldsymbol{\theta}$
matches the example $\hat{{\bf x}}$. Several loss functions are considered
in Section \ref{sec:Loss-Functions}.

We assume that the empirical risk will be fit by some gradient-based
optimization. Hence, the main technical issues in learning are which
loss function to use and how to compute the gradient $\frac{dL}{d\boldsymbol{\theta}}$.

In practice, we will usually be interested in fitting conditional
distributions. Using the notation from Eq. \ref{eq:conditiona-efam},
we can write this as

\[
R(\boldsymbol{\gamma})=\sum_{(\hat{{\bf y}},\hat{{\bf x}})}L\bigl(\boldsymbol{\theta}(\hat{{\bf y}},\boldsymbol{\gamma}),\hat{{\bf x}}\bigr).
\]

Note that if one has recovered $\frac{dL}{d\boldsymbol{\theta}},$
$\frac{dL}{d\boldsymbol{\gamma}}$ is immediate from the vector chain
rule as

\begin{equation}
\frac{dL}{d\boldsymbol{\gamma}}=\frac{d{\bf \boldsymbol{\theta}}^{T}}{d\boldsymbol{\gamma}}\frac{dL}{d\boldsymbol{\theta}}.\label{eq:conditional_model_chainrule}
\end{equation}

Thus, the main technical problems involved in fitting a conditional
distribution are similar to those for a generative distribution: One
finds $\boldsymbol{\theta}=\boldsymbol{\theta}(\hat{{\bf y}},\boldsymbol{\gamma})$,
computes the $L$ and $\frac{dL}{d\boldsymbol{\theta}}$ on example
$\hat{{\bf x}}$ exactly as in the generative case, and finally recovers
$\frac{dL}{d\boldsymbol{\gamma}}$ from Eq. \ref{eq:conditional_model_chainrule}.
So, for simplicity, ${\bf y}$ and $\boldsymbol{\gamma}$ will largely
be ignored in the theoretical developments below.

\section{Variational Inference}

This section reviews approximate methods for computing marginals,
with notation based on Wainwright and Jordan \cite{WainwrightJordanMonster}.
For readability, all proofs in this section are postponed to Appendix
A.

The relationship between the marginals and the log-partition function
in Eq. \ref{eq:dA_dtheta_eq_mu} is key to defining approximate marginalization
procedures. In Section \ref{sub:Exact-Variational-Principle}, the
exact variational principle shows that the (intractable) problem of
computing the log-partition function can be converted to a (still
intractable) optimization problem. To derive a tractable marginalization
algorithm one approximates this optimization, yielding some approximate
log-partition function $\tilde{A}(\boldsymbol{\theta})$. The approximate
marginals are then taken as the \emph{exact} gradient of $\tilde{A}$.

We define the reverse mapping $\boldsymbol{\theta}(\boldsymbol{\mu})$
to return some parameter vector that yields that marginals $\boldsymbol{\mu}$.
While this will in general not be unique \cite[sec. 3.5.2]{WainwrightJordanMonster},
any two vectors that produce the same marginals $\boldsymbol{\mu}$
will also yield the same distribution, and so $p({\bf x};\boldsymbol{\theta}(\boldsymbol{\mu}))$
is unambiguous.

\subsection{Exact Variational Principle\label{sub:Exact-Variational-Principle}}
\begin{thm*}
[Exact variational principle]The log-partition function can also
be represented as

\begin{equation}
A(\boldsymbol{\theta})=\max_{\boldsymbol{\mu}\in\mathcal{M}}\boldsymbol{\theta}\cdot\boldsymbol{\mu}+H(\boldsymbol{\mu}),\label{eq:A-variational}
\end{equation}
where 
\[
\mathcal{M}=\{\boldsymbol{\mu}':\exists\boldsymbol{\theta},\boldsymbol{\mu}'=\boldsymbol{\mu}(\boldsymbol{\theta})\}
\]
is the marginal polytope, and 
\[
H(\boldsymbol{\mu})=-\sum_{{\bf x}}p({\bf x};\boldsymbol{\theta}(\boldsymbol{\mu}))\log p({\bf x};\boldsymbol{\theta}(\boldsymbol{\mu}))
\]

\noindent is the entropy.
\end{thm*}
In treelike graphs, this optimization can be solved efficiently. In
general graphs, however, it is intractable in two ways. First, the
marginal polytope $\mathcal{M}$ becomes difficult to characterize.
Second, the entropy is intractable to compute.

Applying Danskin's theorem to Eq. \ref{eq:A-variational} yields that

\begin{equation}
\boldsymbol{\mu}(\boldsymbol{\theta})=\frac{dA}{d\boldsymbol{\theta}}=\underset{\boldsymbol{\mu}\in\mathcal{M}}{\arg\max}\,\boldsymbol{\theta}\cdot\boldsymbol{\mu}+H(\boldsymbol{\mu}).\label{eq:mu-variational}
\end{equation}

Thus, the partition function (Eq. \ref{eq:A-variational}) and marginals
(Eq. \ref{eq:mu-variational}) can both be obtained from solving the
same optimization problem. This close relationship between the log-partition
function and marginals is heavily used in the derivation of approximate
marginalization algorithms. To compute approximate marginals, first,
derive an approximate version of the optimization in Eq. \ref{eq:A-variational}.
Next, take the exact gradient of this approximate partition function.
This strategy is used in both of the approximate marginalization procedures
considered here: mean field and tree-reweighted belief propagation.

\subsection{Mean Field}

The idea of mean field is to approximate the exact variational principle
by replacing $\mathcal{M}$ with some tractable subset $\mathcal{F}\subset\mathcal{M}$,
such that $\mathcal{F}$ is easy to characterize, and for any vector
$\boldsymbol{\mu}\in\mathcal{F}$ we can exactly compute the entropy.
To create such a set $\mathcal{F}$, instead of considering the set
of mean vectors obtainable from \emph{any} parameter vector (which
characterizes $\mathcal{M}$), consider a subset of \emph{tractable}
parameter vectors. The simplest way to achieve this to restrict consideration
to parameter vectors $\boldsymbol{\theta}$ with $\theta({\bf x}_{c})=0$
for all factors $c$.

\[
\mathcal{F}=\{\boldsymbol{\mu}':\exists\boldsymbol{\theta},\boldsymbol{\mu}'=\boldsymbol{\mu}(\boldsymbol{\theta}),\,\forall c,\,\theta({\bf x}_{c})=0\}.
\]

It is not hard to see that this corresponds to the set of \emph{fully-factorized}
distributions. Note also that this is (in non-treelike graphs) a non-convex
set, since it has the same convex hull as $\mathcal{M}$, but is a
proper subset. So, the mean field partition function approximation
is based on the optimization

\begin{equation}
\tilde{A}(\boldsymbol{\theta})=\max_{\boldsymbol{\mu}\in\mathcal{F}}\boldsymbol{\theta}\cdot\boldsymbol{\mu}+H(\boldsymbol{\mu}),\label{eq:A-meanfield}
\end{equation}

\noindent with approximate marginals corresponding to the maximizing
vector $\boldsymbol{\mu}$, i.e.

\begin{equation}
\tilde{\boldsymbol{\mu}}(\boldsymbol{\theta})=\arg\max_{\boldsymbol{\mu}\in\mathcal{F}}\boldsymbol{\theta}\cdot\boldsymbol{\mu}+H(\boldsymbol{\mu}).\label{eq:mu-meanfield}
\end{equation}

Since this is maximizing the same objective as the exact variational
principle, but under a more restricted constraint set, clearly $\tilde{A}(\boldsymbol{\theta})\leq A(\boldsymbol{\theta}).$

Here, since the marginals are coming from a fully-factorized distribution,
the exact entropy is available as
\begin{equation}
H(\boldsymbol{\mu})=-\sum_{i}\sum_{x_{i}}\mu(x_{i})\log\mu(x_{i}).\label{eq:meanfield-entropy}
\end{equation}

The strategy we use to perform the maximization in Eq. \ref{eq:A-meanfield}
is block-coordinate ascent. Namely, we pick a coordinate $j$, then
set $\mu(x_{j})$ to maximize the objective, leaving $\mu(x_{i})$
fixed for all $i\not=j$. The next theorem formalizes this.
\begin{thm*}
[Mean Field Updates]A local maximum of Eq. \ref{eq:A-meanfield}
can be reached by iterating the updates

\[
\mu(x_{j})\leftarrow\frac{1}{Z}\exp\bigl(\theta(x_{j})+\sum_{c:j\in c}\sum_{{\bf x}_{c\backslash j}}\theta({\bf x}_{c})\prod_{i\in c\backslash j}\mu(x_{i})\bigr),
\]

\noindent where $Z$ is a normalizing factor ensuring that ${\displaystyle \sum_{x_{j}}\mu(x_{j})=1}$. 
\end{thm*}

\subsection{Tree-Reweighted Belief Propagation }

Whereas mean field replaced the marginal polytope with a subset, tree-reweighted
belief propagation (TRW) replaces it with a superset, $\mathcal{L}\supset\mathcal{M}$.
This clearly can only increase the value of the approximate log-partition
function. However, a further approximation is needed, as the entropy
remains intractable to compute for an arbitrary mean vector $\boldsymbol{\mu}$.
(It is not even defined for $\boldsymbol{\mu}\not\in\mathcal{M}.$)
Thus, TRW further approximates the entropy with a tractable upper
bound. Taken together, these two approximations yield a tractable
upper bound on the log-partition function.

Thus, TRW is based on the optimization problem

\begin{equation}
\tilde{A}(\boldsymbol{\theta})=\max_{\boldsymbol{\mu}\in\mathcal{L}}\boldsymbol{\theta}\cdot\boldsymbol{\mu}+\tilde{H}(\boldsymbol{\mu}).\label{eq:A-TRW}
\end{equation}

\noindent Again, the approximate marginals are simply the maximizing
vector $\boldsymbol{\mu}$, i.e.,

\begin{equation}
\tilde{\boldsymbol{\mu}}(\boldsymbol{\theta})=\arg\max_{\boldsymbol{\mu}\in\mathcal{L}}\boldsymbol{\theta}\cdot\boldsymbol{\mu}+\tilde{H}(\boldsymbol{\mu}).\label{eq:mu-TRW}
\end{equation}

The relaxation of the local polytope used in TRW is the \emph{local
polytope}, 

\begin{equation}
\mathcal{L}=\{\boldsymbol{\mu}:\sum_{{\bf x}_{c\backslash i}}\mu({\bf x}_{c})=\mu(x_{i}),\,\sum_{x_{i}}\mu(x_{i})=1\}.\label{eq:local-polytope}
\end{equation}

\noindent Since any valid marginal vector must obey these constraints,
clearly $\mathcal{M}\subset\mathcal{L}$. However, $\mathcal{L}$
in general also contains unrealizable vectors (though on trees $\mathcal{L}=\mathcal{M})$.
Thus, the marginal vector returned by TRW may, in general, be inconsistent
in the sense that no joint distribution yields those marginals.

The entropy approximation used by TRW is 
\begin{equation}
\tilde{H}(\mu)=\sum_{i}H(\mu_{i})-\sum_{c}\rho_{c}I(\mu_{c}),\label{eq:TRW-entropy}
\end{equation}

\noindent where $H(\mu_{i})=-\sum_{x_{i}}\mu(x_{i})\log\mu(x_{i})$
is the univariate entropy corresponding to variable $i$, and
\begin{equation}
I(\mu_{c})=\sum_{{\bf x}_{c}}\mu({\bf x}_{c})\log\frac{\mu({\bf x}_{c})}{\prod_{i\in c}\mu(x_{i})}\label{eq:TRW-mutualinfo}
\end{equation}
is the mutual information corresponding to the variables in the factor
$c$. The motivation for this approximation is that if the constants
$\rho_{c}$ are selected appropriately, this gives an upper bound
on the true entropy.
\begin{thm*}
[TRW Entropy Bound]Let $Pr(\mathcal{G})$ be a distribution over
tree structured graphs, and define $\rho_{c}=Pr(c\in\mathcal{G}).$
Then, with $\tilde{H}$ as defined in Eq. \ref{eq:TRW-entropy}, 
\[
\tilde{H}(\boldsymbol{\mu})\geq H(\boldsymbol{\mu}).
\]

\end{thm*}

Thus, TRW is maximizing an upper bound on the exact variational principle,
under an expanded constraint set. Since both of these changes can
only increase the maximum value, we have that $\tilde{A}(\boldsymbol{\theta})\geq A(\boldsymbol{\theta})$.

Now, we consider how to actually compute the approximate log-partition
function and associated marginals. Consider the message-passing updates

\begin{equation}
m_{c}(x_{i})\propto\sum_{{\bf x}_{c\backslash i}}e^{\frac{1}{\rho_{c}}\theta({\bf x}_{c})}\prod_{j\in c\backslash i}e^{\theta(x_{j})}\frac{\prod_{d:j\in d}m_{d}(x_{j})^{\rho_{d}}}{m_{c}(x_{j})},\label{eq:TRW-msgs}
\end{equation}
where ``$\propto$'' is used as an assignment operator to means
assigning after normalization.
\begin{thm*}
[TRW Updates]Let $\rho_{c}$ be as in the previous theorem. Then,
if the updates in Eq. \ref{eq:TRW-msgs} reach a fixed point, the
marginals defined by 
\begin{eqnarray*}
\mu({\bf x}_{c}) & \propto & e^{\frac{1}{\rho_{c}}\theta({\bf x}_{c})}\prod_{i\in c}e^{\theta(x_{i})}\frac{\prod_{d:i\in d}m_{d}(x_{i})^{\rho_{d}}}{m_{c}(x_{i})},\\
\mu(x_{i}) & \propto & e^{\theta(x_{i})}\prod_{d:i\in d}m_{d}(x_{i})^{\rho_{d}}
\end{eqnarray*}

\noindent constitute the global optimum of Eq. \ref{eq:A-TRW}.
\end{thm*}
So, if the updates happen to converge, we have the solution. Meltzer
et al. show \cite{meltzer_et_al} that on certain graphs made up of
\emph{monotonic chains}, an appropriate ordering of messages does
assure convergence. (The proof is essentially that under these circumstances,
message passing is equivalent to coordinate ascent in the dual.)

TRW simplifies into loopy belief propagation by choosing $\rho_{c}=1$
everywhere, though the bounding property is lost.

\section{Loss Functions\label{sec:Loss-Functions}}

For space, only a representative sample of prior work can be cited.
A recent review \cite{StructuredLearningAndPredictionInComputerVision}
is more thorough.

Though, technically, a ``loss'' should be minimized, we continue
to use this terminology for the likelihood and its approximations,
where one wishes to maximize.

For simplicity, the discussion below is for the generative setting.
Using the same loss functions for training a conditional model is
simple (Section \ref{sub:Learning}).

\subsection{The Likelihood and Approximations\label{sub:The-Likelihood}}

The classic loss function would be the likelihood, with
\begin{equation}
L(\boldsymbol{\theta},{\bf x})=\log p({\bf x};\boldsymbol{\theta})=\boldsymbol{\theta}\cdot{\bf f}({\bf x})-A(\boldsymbol{\theta}).\label{eq:likelihood}
\end{equation}

This has the gradient
\begin{equation}
\frac{dL}{d\boldsymbol{\theta}}={\bf f}({\bf x})-\boldsymbol{\mu}(\boldsymbol{\theta}).\label{eq:likelihood-gradient}
\end{equation}

One argument for the likelihood is that it is efficient; given a correct
model, as data increases it converges to true parameters at an asymptotically
optimal rate \cite{MathematicalMethodsOfStatistics}.

Some previous work uses tree structured graphs where marginals may
be computed exactly \cite{OnParameterLearninginCRFbased}. Of course,
in high-treewidth graphs, the likelihood and its gradient will be
intractable to compute exactly, due to the presence of the log-partition
function $A(\boldsymbol{\theta})$ and marginals $\boldsymbol{\mu}(\boldsymbol{\theta})$.
This has motivated a variety of approximations. The first is to approximate
the marginals $\boldsymbol{\mu}$ using Markov chain Monte Carlo \cite{LearningFlexibleFeatures,MCMCML}.
This can lead to high computational expense (particularly in the conditional
case, where different chains must be run for each input). Contrastive
Divergence \cite{OnConstrastiveDivergenceLearning} further approximates
these samples by running the Markov chain for only a few steps, but
started at the data points \cite{FieldsOfExperts}. If the Markov
chain is run long enough, these approaches can give an arbitrarily
good approximation. However, Markov chain parameters may need to be
adjusted to the particular problem, and these approaches are generally
slower than those discussed below.

\subsubsection{Surrogate Likelihood}

A seemingly heuristic approach would be to replace the marginals in
Eq. \ref{eq:likelihood-gradient} with those from an approximate inference
method. This approximation can be quite principled if one thinks instead
of approximating the log-partition function in the likelihood itself
(Eq. \ref{eq:likelihood}). Then, the corresponding approximate marginals
will emerge as the \emph{exact} gradient of this surrogate loss. This
``surrogate likelihood'' \cite{EstimatingTheWrong} approximation
appears to be the most widely used loss in imaging problems, with
marginals approximated by either mean field \cite{EfficientlyLearningRandomFields,RandomFieldModelForIntegration},
TRW \cite{LearningToCombineBottomUpAndTopDown} or LBP \cite{ExploitingInferenceForApproximate,FigureGroundAssignment,AcceleratedTrainingofCRFs,LearningProbabilisticModels,SceneUnderstandingWithDiscriminative}.
However, the terminology of ``surrogate likelihood'' is not widespread
and in most cases, only the gradient is computed, meaning the optimization
cannot use line searches.

If one uses a log-partition approximation that provides a bound on
the true log-partition function, the surrogate likelihood will then
bound the true likelihood. Specifically, mean field based surrogate
likelihood is an upper bound on the true likelihood, while TRW-based
surrogate likelihood is a lower bound.

\subsubsection{Expectation Maximization}

In many applications, only a subset of variables may be observed.
Suppose that we want to model ${\bf x}=({\bf z},{\bf h})$ where ${\bf z}$
is observed, but ${\bf h}$ is hidden. A natural loss function here
is the expected maximization (EM) loss

\[
L(\boldsymbol{\theta},{\bf z})=\log p({\bf z};\boldsymbol{\theta})=\log\sum_{{\bf h}}p({\bf z},{\bf h};\boldsymbol{\theta}).
\]

\noindent It is easy to show that this is equivalent to
\begin{equation}
L(\boldsymbol{\theta},{\bf z})=A(\boldsymbol{\theta},{\bf z})-A(\boldsymbol{\theta}),\label{eq:EM_loss}
\end{equation}
where $A(\boldsymbol{\theta},{\bf z})=\log\sum_{{\bf h}}\exp\boldsymbol{\theta}\cdot{\bf f}({\bf z},{\bf h})$
is the log-partition function with ${\bf z}$ ``clamped'' to the
observed values. If all variables are observed $A(\boldsymbol{\theta},{\bf z})$
reduces to $\boldsymbol{\theta}\cdot{\bf f}({\bf z})$.

If on substitutes a variational approximation for $A(\boldsymbol{\theta},{\bf z})$,
a ``variational EM'' algorithm \cite[Sec. 6.2.2]{WainwrightJordanMonster}
can be recovered that alternates between computing approximate marginals
and parameter updates. Here, because of the close relationship to
the surrogate likelihood, we designate ``surrogate EM'' for the
case where $A(\boldsymbol{\theta},{\bf z})$ and $A(\boldsymbol{\theta})$
may both be approximated and the learning is done with a gradient-based
method. To obtain a bound on the true EM loss, care is required. For
example, lower-bounding $A(\boldsymbol{\theta},{\bf z})$ using mean
field, and upper-bounding $A(\boldsymbol{\theta})$ using TRW means
a lower-bound on the true EM loss. However, using the same approximation
for both $A(\boldsymbol{\theta})$ and $A(\boldsymbol{\theta},{\bf z})$
appears to work well in practice \cite{SceneSegmentationWithCRFsLearned}.

\subsubsection{Saddle-Point Approximation}

A third approximation of the likelihood is to search for a ``saddle-point''.
Here, one approximates the gradient in Eq. \ref{eq:likelihood-gradient}
by running a (presumably approximate) MAP inference algorithm, and
then imagining that the marginals put unit probability at the approximate
MAP solution, and zero elsewhere \cite{LearningConditionalRandomFieldsForStereo,UsingCombinationOfStatisticalModelsAndMultilevel,ExploitingInferenceForApproximate}.
This is a heuristic method, but it can be expected to work well when
the estimated MAP solution is close to the true MAP and the conditional
distribution $p({\bf x}|{\bf y})$ is strongly ``peaked''.

\subsubsection{Pseudolikelihood}

Finally, there are two classes of likelihood approximations that do
not require inference. The first is the classic pseudolikelihood \cite{StatisticalAnalysis},
where one uses
\[
L(\boldsymbol{\theta},{\bf x})=\sum_{i}\log p(x_{i}|{\bf x}_{-i};\boldsymbol{\theta}).
\]

This can be computed efficiently, even in high treewidth graphs, since
conditional probabilities are easy to compute. Besag \cite{StatisticalAnalysis}
showed that, under certain conditions, this will converge to the true
parameter vector as the amount of data becomes infinite. The pseudolikelihood
has been used in many applications \cite{MultiscaleConditionalRandomFieldsFor,DiscriminativeRandomFields}.
Instead of the probability of individual variables given all others,
one can take the probability of patches of variables given all others,
sometimes called the ``patch'' pseudolikelihood \cite{LearningInGibbsianFieldsHowAccurate}.
This interpolates to the exact likelihood as the patches become larger,
though some type of inference is generally required.

\subsubsection{Piecewise Likelihood}

More recently, Sutton and McCallum \cite{PiecewiseTrainingForUndirectedModels}
suggested the piecewise likelihood. The idea is to approximate the
log-partition function as a sum of log-partition functions of the
different ``pieces`` of the graph. There is flexibility in determining
which pieces to use. In this paper, we will use pieces consisting
of each clique and each variable, which worked better in practice
than some alternatives. Then, one has the surrogate partition function
\begin{eqnarray*}
\tilde{A}(\boldsymbol{\theta}) & = & \sum_{c}A_{c}(\boldsymbol{\theta})+\sum_{i}A_{i}(\boldsymbol{\theta}),\\
A_{c}(\boldsymbol{\theta}) & = & \log\sum_{{\bf x}_{c}}e^{\theta({\bf x}_{c})},\,\,\,\, A_{i}(\boldsymbol{\theta})=\log\sum_{x_{i}}e^{\theta(x_{i})}.
\end{eqnarray*}

It is not too hard to show that $A(\boldsymbol{\theta})\leq\tilde{A}(\boldsymbol{\theta})$.
In practice, it is sometimes best to make some heuristic adjustments
to the parameters after learning to improve test-time performance
\cite{RobustModelBasedSceneInterpretation,TextonBoostForImageUnderstanding}.

\subsection{Marginal-based Loss Functions\label{sub:Marginal-based-Loss-Functions}}

Given the discussion in Section \ref{sub:The-Likelihood}, one might
conclude that the likelihood, while difficult to optimize, is an ideal
loss function since, given a well-specified model, it will converge
to the true parameters at asymptotically efficient rates. However,
this conclusion is complicated by two issues. First, of course, the
maximum likelihood solution is computationally intractable, motivating
the approximations above.

A second issue is that of \emph{model mis-specification}. For many
types of complex phenomena, we will wish to fit a model that is approximate
in nature. This could be true because the conditional independencies
asserted by the graph do not exactly hold, or because the parametrization
of factors is too simplistic. These approximations might be made out
of ignorance, due to a lack of knowledge about the domain being studied,
or deliberately because the true model might have too many degrees
of freedom to be fit with available data.

In the case of an approximate model, no ``true'' parameters exist.
The idea of marginal-based loss functions is to instead consider how
the model will be used. If one will compute marginals at test-time
-- perhaps for MPM inference (Section \ref{sub:Inference-Problems})
-- it makes sense to maximize the accuracy of these predictions. Further,
if one will use an approximate inference algorithm, it makes sense
to optimize the accuracy of the \emph{approximate} marginals. This
essentially fits into the paradigm of empirical risk minimization
\cite{EmpiricalRiskMinimizationofGraphicalModel,LearningConvexInference}.
The idea of training a probabilistic model using an alternative loss
to the likelihood goes back at least to Bahl et al. in the late 1980s
\cite{BahlEtAl}.

There is reason to think the likelihood is somewhat robust to model
mis-specification. In the infinite data limit, it finds the ``closest''
solution in the sense of KL-divergence since, if $q$ is the true
distribution, then
\begin{eqnarray*}
KL(q||p) & = & \text{const.}-\underset{q}{\mathbb{E}}\log p({\bf x};\boldsymbol{\theta}).
\end{eqnarray*}

\subsubsection{Univariate Logistic Loss}

The univariate logistic loss \cite{AlternativeObjectiveFunctionFor}
is defined by
\[
L(\boldsymbol{\theta},{\bf x})=-\sum_{i}\log\mu(x_{i};\boldsymbol{\theta}),
\]
where we use the notation $\mu$ to indicate that the loss is implicitly
defined with respect to the marginal predictions of some (possibly
approximate) algorithm, rather than the true marginals. This measures
the mean accuracy of all univariate marginals, rather than the joint
distribution. This loss can be seen as empirical risk minimization
of the KL-divergence between the true marginals and the predicted
ones, since
\begin{eqnarray*}
\sum_{i}KL(q_{i}||\mu_{i}) & = & \sum_{i}\sum_{x_{i}}q(x_{i})\log\frac{q(x_{i})}{\mu(x_{i};\boldsymbol{\theta})}\\
 & = & \text{const.}-\underset{q}{\mathbb{E}}\sum_{i}\log\mu(x_{i};\boldsymbol{\theta}).
\end{eqnarray*}
If defined on exact marginals, this is a type of composite likelihood
\cite{CompositeLikelihoods}.

\subsubsection{Smoothed Univariate Classification Error}

Perhaps the most natural loss in the conditional setting would be
the univariate classification error,

\[
L(\boldsymbol{\theta},{\bf x})=\sum_{i}S\bigl(\max_{x_{i}'\not=x_{i}}\mu(x_{i};\boldsymbol{\theta})-\mu(x_{i};\boldsymbol{\theta})\bigr),
\]
where $S(\cdot)$ is the step function. This exactly measures the
number of components of ${\bf x}$ that would be incorrectly predicted
if using MPM inference. Of course, this loss is neither differentiable
nor continuous, which makes it impractical to optimize using gradient-based
methods. Instead Gross et al. \cite{TrainingCRFsForMaximumLabelwiseAccuracy}
suggest approximating with a sigmoid function $S(t)=(1+\exp(\text{\textminus\ensuremath{\alpha}}t))^{-1}$,
where $\alpha$ controls approximation quality.

There is evidence \cite{EmpiricalRiskMinimizationofGraphicalModel,TrainingCRFsForMaximumLabelwiseAccuracy}
that the smoothed classification loss can yield parameters with lower
univariate classification error under MPM inference. However, our
experience is that it is also more prone to getting stuck in local
minima, making experiments difficult to interpret. Thus, it is not
included in the experiments below. Our experience with the univariate
quadratic loss \cite{Domke} is similar.

\subsubsection{Clique Losses}

Any of the above univariate losses can be instead taken based on cliques.
For example, the clique logistic loss is

\[
L(\boldsymbol{\theta},{\bf x})=-\sum_{c}\log\mu({\bf x}_{c};\boldsymbol{\theta}),
\]

\noindent which may be seen as empirical risk minimization of the
mean KL-divergence of the true clique marginals to the predicted ones.
An advantage of this with an exact model is consistency. Simple examples
show cases where a model predicts perfect univariate marginals, despite
the joint distribution being very inaccurate. However, if all clique
marginals are correct, the joint must be correct, by the standard
moment matching conditions for the exponential family \cite{WainwrightJordanMonster}.

\subsubsection{Hidden variables}

Marginal-based loss functions can accommodate hidden variables by
simply taking the sum in the loss over the \emph{observed} variables
only. A similar approach can be used with the pseudolikelihood or
piecewise likelihood.

\subsection{Comparison with Exact Inference}

\begin{figure}

\includegraphics[bb=31bp 179bp 558bp 575bp,clip,scale=0.165]{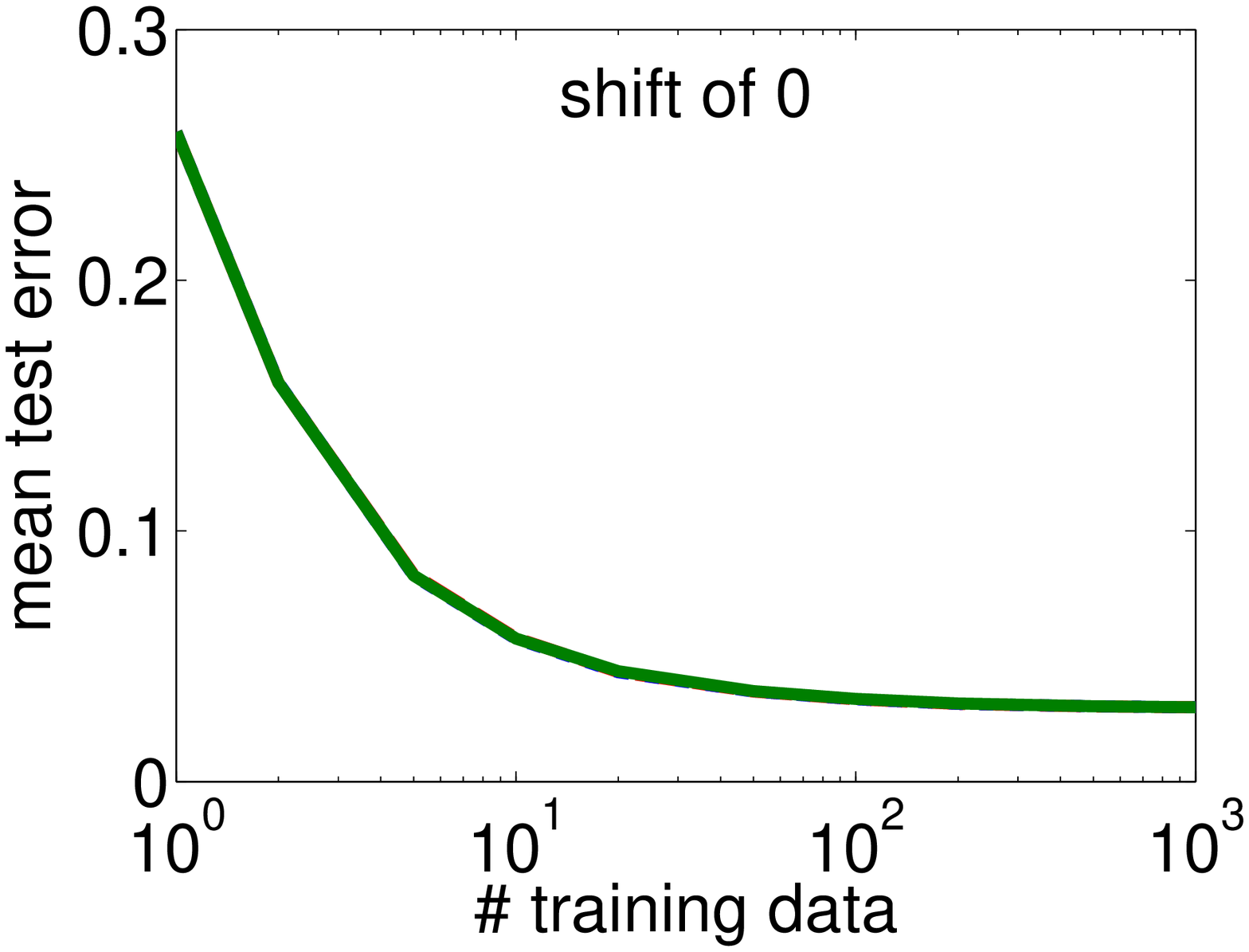}\includegraphics[bb=51bp 179bp 558bp 575bp,clip,scale=0.165]{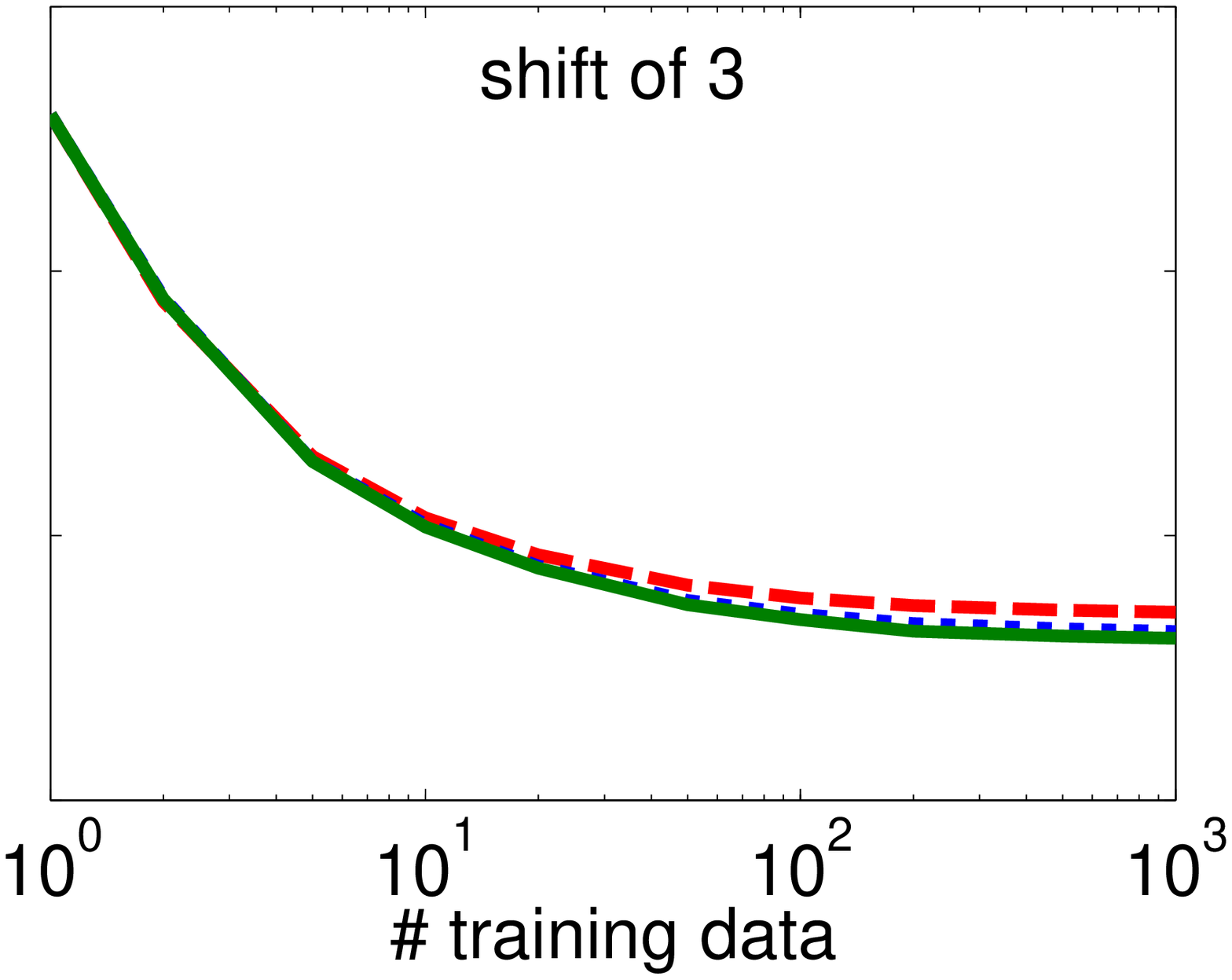}\includegraphics[bb=51bp 179bp 558bp 575bp,clip,scale=0.165]{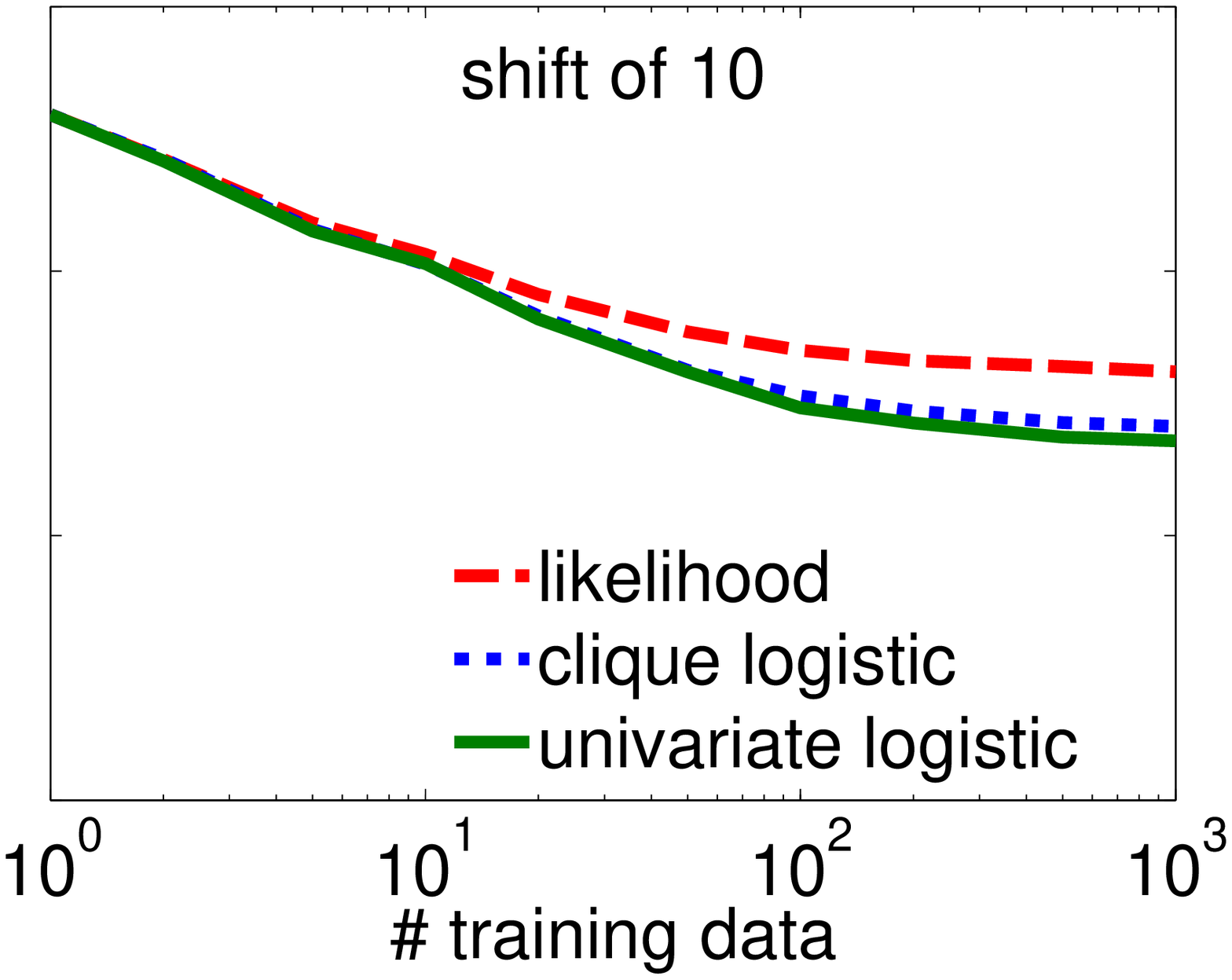} 

\caption{Mean test error of various loss functions trained with exact inference.
In the case of a well-specified model (shift of zero), the likelihood
performs essentially identically to the marginal-based loss functions.
However, when mis-specification is introduced, quite different estimates
result. \label{fig:exactinference-means-smaller}}
\end{figure}
\begin{figure*}
\psfrag{Pygivenx      }{$p(y_i=1|{\bf x})$}
\psfrag{i}{$i$}

\begin{centering}
\includegraphics[scale=0.4]{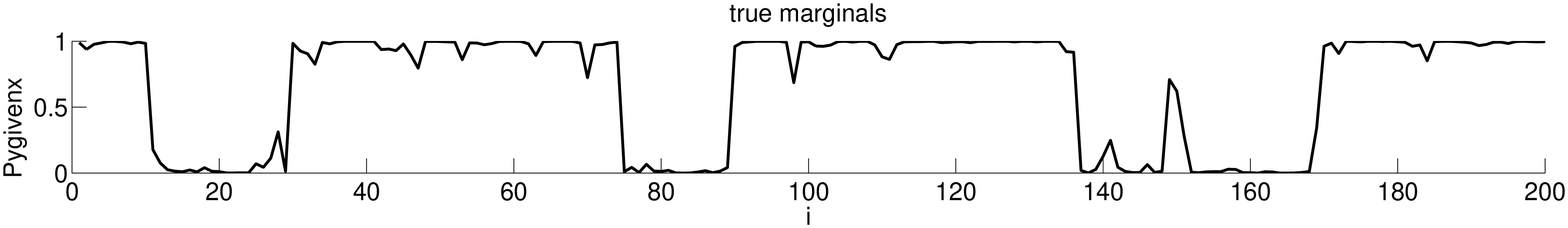}
\par\end{centering}

\begin{centering}
\vspace{1pt}

\par\end{centering}

\begin{centering}
\includegraphics[scale=0.4]{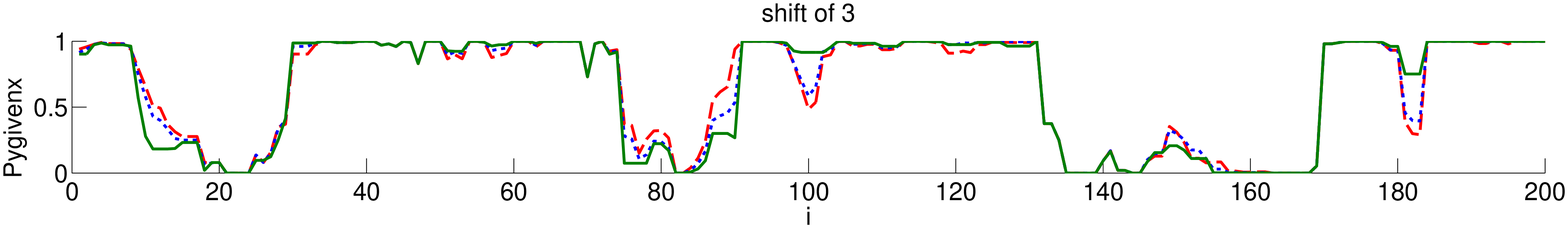}
\par\end{centering}

\begin{centering}
\vspace{1pt}

\par\end{centering}

\begin{centering}
\includegraphics[scale=0.4]{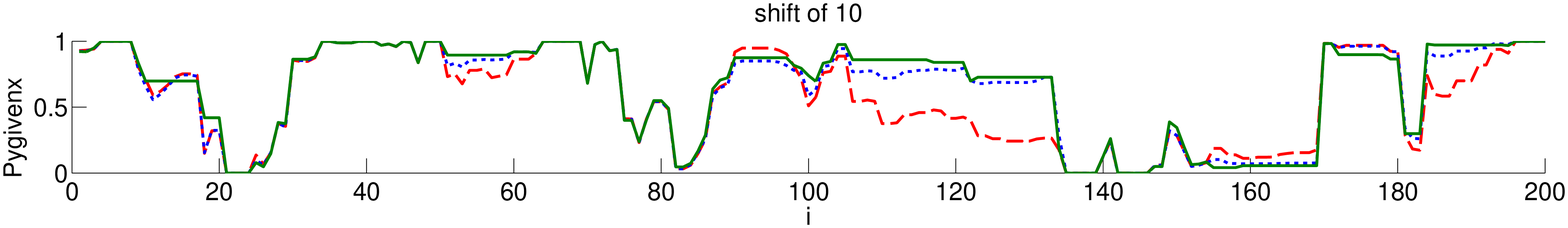}
\par\end{centering}

\caption{Exact and predicted marginals for an example input. Predicted marginals
are trained using 1000 data. With low shifts, all loss functions lead
to accurate predicted marginals. However, the univariate and clique
logistic loss are more resistant to the effects of model mis-specification.
Legends as in Fig. \ref{fig:exactinference-means-smaller}.\label{fig:exactinference-example} }
\end{figure*}
To compare the effects of different loss functions in the presence
of model mis-specification, this section contains a simple example
where the graphical model takes the following ``chain'' structure: 

\begin{center}
\begin{tikzpicture}[scale=1.2]
\node at (3.5, -.5) {$\hdots$};
\tikzstyle{every node}=[draw,shape=circle];
\node (x1) at (0,  0)  {$x_1$};
\node (x2) at (1,  0)  {$x_2$};
\node (x3) at (2,  0)  {$x_3$};
\node (x4) at (3,  0)  {$x_4$};
\node (x200) at (4,  0)  {$x_n$};
\node (y1) at (0, -1)  {$y_1$};
\node (y2) at (1, -1)  {$y_2$};
\node (y3) at (2, -1)  {$y_3$};
\node (y4) at (3, -1)  {$y_4$};
\node (y200) at (4, -1)  {$y_n$};
\draw (x1) -- (x2)
      (x2) -- (x3)
      (x3) -- (x4)
      (x1) -- (y1)
      (x2) -- (y2)
      (x3) -- (y3)
      (x4) -- (y4)
      (y1) -- (y2)
      (y2) -- (y3)
      (y3) -- (y4)
      (x200) -- (y200);
\end{tikzpicture}
\par\end{center}

Here, exact inference is possible, so comparison is not complicated
by approximate inference.

All variables are binary. Parameters are generated by taking $\theta(x_{i})$
randomly from the interval $[-1,+1]$ for all $i$ and $x_{i}$. Interaction
parameters are taken as $\theta(x_{i},x_{j})=t$ when $x_{i}=x_{j}$,
and $\theta(x_{i},x_{j})=-t$ when $x_{i}\not=x_{j}$, where $t$
is randomly chosen from the interval $[-1,+1]$ for all $(i,j)$.
Interactions $\theta(y_{i},y_{j})$ and $\theta(x_{i},y_{i})$ are
chosen in the same way.

To systematically study the effects of differing ``amounts'' of
mis-specification, after generating data, we apply various circular
shifts to ${\bf x}$. Thus, the data no longer corresponds exactly
the the structure of the graphical model being fit.

Thirty-two different random distributions were created. For each,
various quantities of data were generated by Markov chain Monte Carlo,
with shifts introduced after sampling. The likelihood was fit using
the closed-form gradient (Sec. \ref{sub:The-Likelihood}), while the
logistic losses were trained using a gradient obtained via backpropagation
(Sec. \ref{sec:Truncated-Fitting}). Fig. \ref{fig:exactinference-means-smaller} 
shows the mean test error (estimated on 1000 examples), while Fig.
\ref{fig:exactinference-example} shows example marginals. We see
that the performance of all methods deteriorates with mis-specification,
but the marginal-based loss functions are more resistant to these
effects.

\subsection{MAP-Based Training}

Another class of methods explicitly optimize the performance of MAP
inference \cite{DiscriminativeModelsForMultiClassObjectLayout,LearningCRFsUsingGraphCuts,HierarchicalImageRegionLabeling,SceneSegmentationViaLowDimensional,SceneUnderstandingWithDiscriminative}.
 This paper focuses on applications that use marginal inference,
and that may need to accommodate hidden variables, and so concentrates
on likelihood and marginal-based losses.

\section{Implicit Fitting}

We now turn to the issue of how to train high-treewidth graphical
models to optimize the performance of a marginal-based loss function,
based on some approximate inference algorithm. Now, computing the
value of the loss for any of the marginal-based loss functions is
not hard. One can simply run the inference algorithm and plug the
resulting marginal into the loss. However, we also require the gradient
$\frac{dL}{d\boldsymbol{\theta}}$.

Our first result is that the loss gradient can be obtained by solving
a sparse linear system. Here, it is useful to introduce notation to
distinguish the loss $L$, defined in terms of the parameters $\boldsymbol{\theta}$
from the loss $Q$, defined directly in terms of the marginals $\boldsymbol{\mu}$.
(Note that though the notation suggests the application to marginal
inference, this is a generic result.)
\begin{thm*}
Suppose that
\begin{equation}
\boldsymbol{\mu}(\boldsymbol{\theta}):=\underset{\boldsymbol{\mu}:B\boldsymbol{\mu}={\bf d}}{\arg\max}\,\,\boldsymbol{\theta}\cdot\boldsymbol{\mu}+H(\boldsymbol{\mu}).\label{eq:implicit_optimization}
\end{equation}
Define $L(\boldsymbol{\theta},{\bf x})=Q(\boldsymbol{\mu}(\boldsymbol{\theta}),{\bf x}).$
Then, letting $D=\frac{d^{2}H}{d\boldsymbol{\mu}d\boldsymbol{\mu}^{T}},$
\[
\frac{dL}{d\boldsymbol{\theta}}=\bigl(D^{-1}B^{T}(BD^{-1}B^{T})^{-1}BD^{-1}-D^{-1}\bigr)\frac{dQ}{d\boldsymbol{\mu}}.
\]

\end{thm*}
A proof may be found in Appendix B. This theorem states that, essentially,
once one has computed the predicted marginals, the gradient of the
loss with respect to marginals $\frac{dQ}{d\boldsymbol{\mu}}$ can
be transformed into the gradient of the loss with respect to parameters
$\frac{dL}{d\boldsymbol{\theta}}$ through the solution of a sparse
linear system.

The optimization in Eq. \ref{eq:implicit_optimization} takes place
under linear constraints, which encompasses the local polytope used
in TRW message-passing (Eq. \ref{eq:local-polytope}). This theorem
does not apply to mean field, as $\mathcal{F}$ is not a linear constraint
set when viewed as a function of both clique and univariate marginals.\textcolor{cyan}{}

In any case, the methods developed below are simpler to use, as they
do not require explicitly forming the constraint matrix $B$ or solving
the linear system.

\section{Perturbation}

This section observes that variational methods have a special structure
that allows derivatives to be calculated without explicitly forming
or inverting a linear system. We have, by the vector chain rule, that
\begin{equation}
\frac{dL}{d\boldsymbol{\theta}}=\frac{d\boldsymbol{\mu}^{T}}{d\boldsymbol{\theta}}\frac{dQ}{d\boldsymbol{\mu}}.\label{eq:loss_chainrule}
\end{equation}

A classic trick in scientific computing is to efficiently compute
Jacobian-vector products by finite differences. The basic result is
that, for any vector ${\bf v}$,
\[
\frac{d\boldsymbol{\mu}}{d\boldsymbol{\theta}^{T}}{\bf v}=\lim_{r\rightarrow0}\frac{1}{r}\bigl(\boldsymbol{\mu}(\boldsymbol{\theta}+r{\bf v})-\boldsymbol{\mu}(\boldsymbol{\theta})\bigr),
\]
which is essentially just the definition of the derivative of $\boldsymbol{\mu}$
in the direction of ${\bf v}$. Now, this does not immediately seem
helpful, since Eq. \ref{eq:loss_chainrule} requires $\frac{d\boldsymbol{\mu}^{T}}{d\boldsymbol{\theta}}$,
not $\frac{d\boldsymbol{\mu}}{d\boldsymbol{\theta}^{T}}$. However,
with variational methods, these are symmetric. The simplest way to
see this is to note that
\[
\frac{d\boldsymbol{\mu}}{d\boldsymbol{\theta}^{T}}=\frac{d}{d\boldsymbol{\theta}^{T}}\left(\frac{dA}{d\boldsymbol{\theta}}\right)=\frac{dA}{d\boldsymbol{\theta}d\boldsymbol{\theta}^{T}}.
\]

\noindent Domke \cite{ImplicitDifferentiationByPerturbation} lists
conditions for various classes of entropies that guarantee that $A$
will be differentiable.

Combining the above three equations, the loss gradient is available
as the limit

\begin{equation}
\frac{dL}{d\boldsymbol{\theta}}=\lim_{r\rightarrow0}\frac{1}{r}\bigl(\boldsymbol{\mu}(\boldsymbol{\theta}+r\frac{dQ}{d\boldsymbol{\mu}})-\boldsymbol{\mu}(\boldsymbol{\theta})\bigr).\label{eq:implicit_limit}
\end{equation}

In practice, of course, the gradient is approximated using some finite
$r$. The simplest approximation, one-sided differences, simply takes
a single value of $r$ in Eq. \ref{eq:implicit_limit}, rather than
a limit. More accurate results at the cost of more calls to inference,
are given using two-sided differences, with
\[
\frac{dL}{d\boldsymbol{\theta}}\approx\frac{1}{2r}\bigl(\boldsymbol{\mu}(\boldsymbol{\theta}+r\frac{dQ}{d\boldsymbol{\mu}})-\boldsymbol{\mu}(\boldsymbol{\theta}-r\frac{dQ}{d\boldsymbol{\mu}})\bigr),
\]
which is accurate to order $o(r^{2}).$ Still more accurate results
are obtained with ``four-sided'' differences, with
\begin{multline*}
\frac{dL}{d\boldsymbol{\theta}}\approx\frac{1}{12r}\bigl(-\boldsymbol{\mu}(\boldsymbol{\theta}+2r\frac{dQ}{d\boldsymbol{\mu}})+8\boldsymbol{\mu}(\boldsymbol{\theta}+r\frac{dQ}{d\boldsymbol{\mu}})\\
-8\boldsymbol{\mu}(\boldsymbol{\theta}-r\frac{dQ}{d\boldsymbol{\mu}})+\boldsymbol{\mu}(\boldsymbol{\theta}-2r\frac{dQ}{d\boldsymbol{\mu}})\bigr),
\end{multline*}
which is accurate to order $o(r^{4})$ \cite{ApproximateSolutionMethods}.

Alg. \ref{alg:Calculating-loss-derivatives-2sided} shows more explicitly
how the loss gradient could be calculated, using two-sided differences.

The issue remains of how to calculate the step size $r$. Each of
the approximations above becomes exact as $r\rightarrow0$. However,
as $r$ becomes very small, numerical error eventually dominates.
To investigate this issue experimentally, we generated random models
on a $10\times10$ binary grid, with each parameter $\theta(x_{i})$
randomly chosen from a standard normal, while each interaction parameter
$\theta(x_{i},x_{j})$ was chosen randomly from a normal with a standard
deviation of $s$. In each case, a random value ${\bf x}$ was generated,
and the ``true'' loss gradient was estimated by standard (inefficient)
2-sided finite differences, with inference re-run after each component
of $\boldsymbol{\theta}$ is perturbed independently. To this, we
compare one, two, and four-sided perturbations. In all cases, the
step size is, following Andrei \cite{AcceleratedConjugateGradientAlgorithm},
taken to be $r=m\epsilon^{\frac{1}{3}}\bigl(1+||\boldsymbol{\theta}||_{\infty}\bigr)/||\frac{dQ}{d\boldsymbol{\mu}}||_{\infty},$
where $\epsilon$ is machine epsilon, and $m$ is a multiplier that
we will vary. Note that the optimal power of $\epsilon$ will depend
on the finite difference scheme; $\frac{1}{3}$ is optimal for two-sided
differences \cite[Sec. 8.1]{NocedalWright}. All calculations take
place in double-precision with inference run until marginals changed
by a threshold of less than $10^{-15}$. Fig. \ref{fig:Perurbation-Sizes}
shows that using many-sided differences leads to more accuracy, at
the cost of needing to run inference more times to estimate a single
loss gradient. In the following experiments, we chose two-sided differences
with a multiplier of 1 as a reasonable tradeoff between accuracy,
simplicity, and computational expense.

Welling and Teh used sensitivity of approximate beliefs to parameters
to approximate joint probabilities of non-neighboring variables \cite{WellingTeh}.

\begin{algorithm}
\begin{enumerate}
\item Do inference. ${\displaystyle \boldsymbol{\mu}^{*}\leftarrow\arg\max_{\boldsymbol{\mu}\in\mathcal{M}}\boldsymbol{\theta}\cdot\boldsymbol{\mu}+H(\boldsymbol{\mu})}$
\item At $\boldsymbol{\mu}^{*}$, calculate the gradient ${\displaystyle \frac{dQ}{d\boldsymbol{\mu}}}$.\medskip{}

\item Calculate a perturbation size $r$.
\item Do inference on perturbed parameters.

\begin{enumerate}
\item ${\displaystyle \boldsymbol{\mu}^{+}\leftarrow\arg\max_{\boldsymbol{\mu}\in\mathcal{M}}(\boldsymbol{\theta}+r\frac{dQ}{d\boldsymbol{\mu}})\cdot\boldsymbol{\mu}+H(\boldsymbol{\mu})}$
\item ${\displaystyle \boldsymbol{\mu}^{-}\leftarrow\arg\max_{\boldsymbol{\mu}\in\mathcal{M}}(\boldsymbol{\theta}-r\frac{dQ}{d\boldsymbol{\mu}})\cdot\boldsymbol{\mu}+H(\boldsymbol{\mu})}$
\end{enumerate}
\item Recover full derivative as~~~${\displaystyle \frac{dL}{d\boldsymbol{\theta}}\leftarrow\frac{1}{2r}(\boldsymbol{\mu}^{+}-\boldsymbol{\mu}^{-})}$.
\end{enumerate}
\caption{Calculating $\frac{dL}{d\boldsymbol{\theta}}$ by perturbation (two-sided).\label{alg:Calculating-loss-derivatives-2sided}}
\end{algorithm}
\begin{figure}


\begin{centering}
\includegraphics[bb=18bp 179bp 540bp 604bp,clip,scale=0.18]{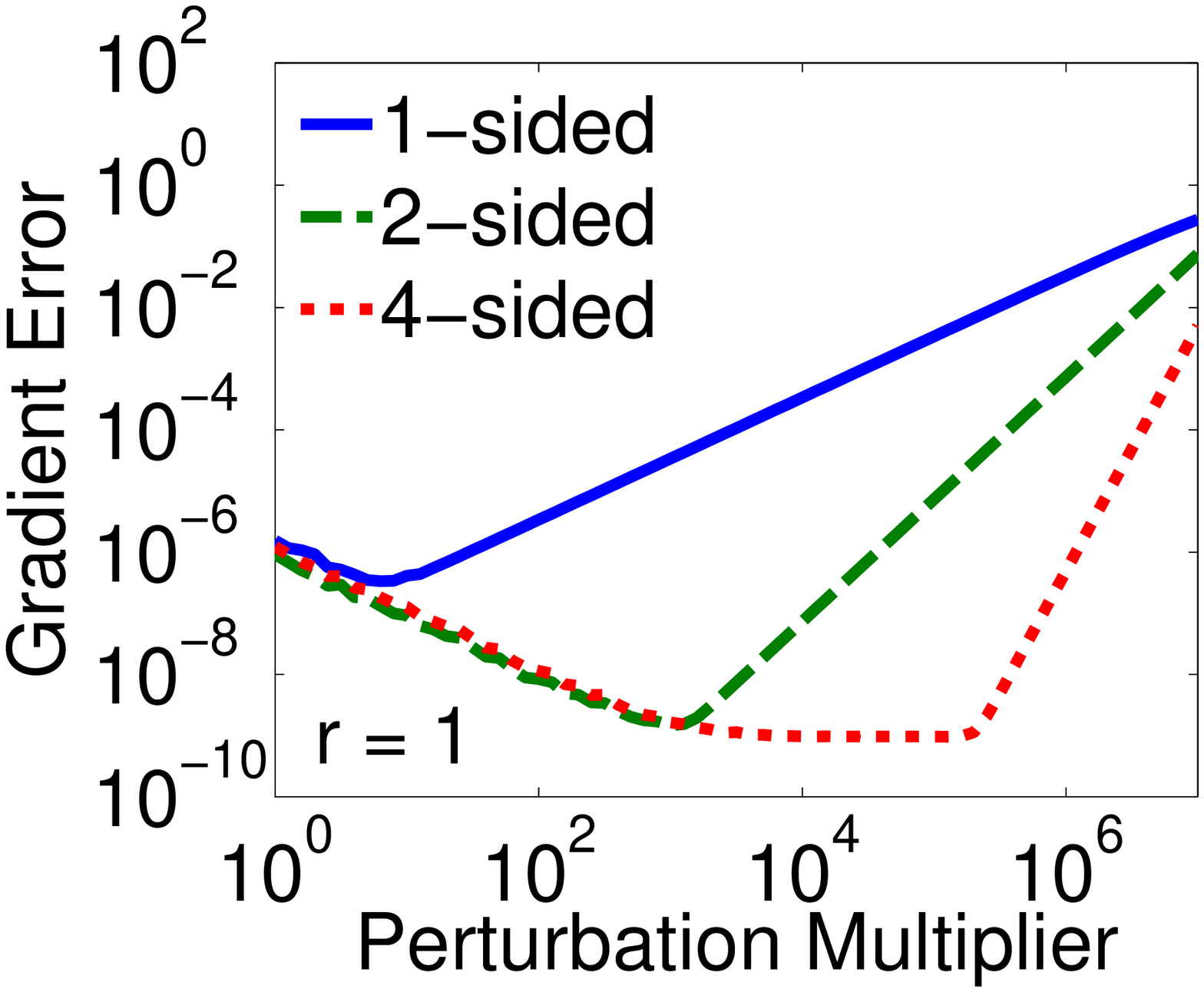}\includegraphics[bb=110bp 179bp 540bp 604bp,clip,scale=0.18]{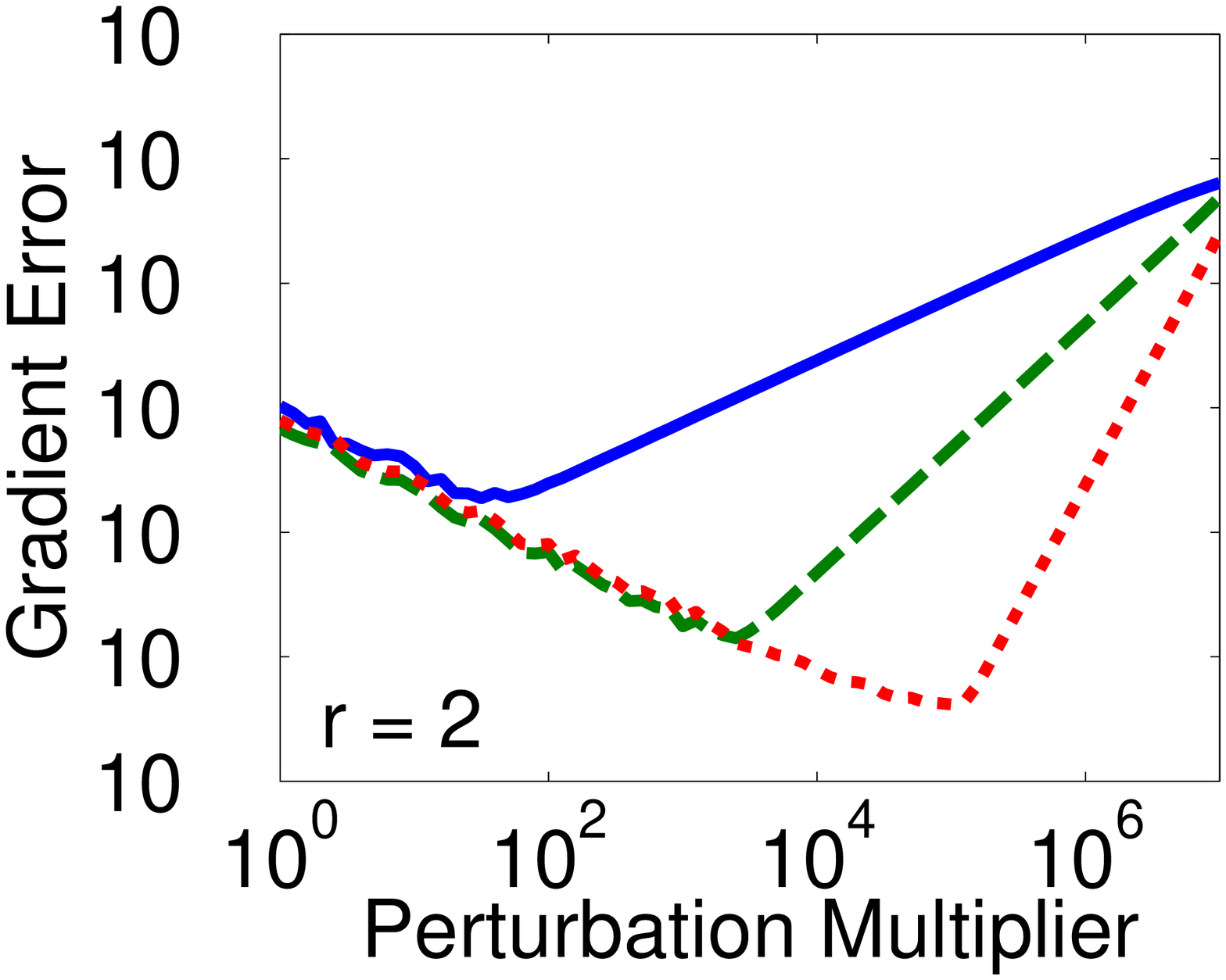}\includegraphics[bb=110bp 179bp 540bp 604bp,clip,scale=0.18]{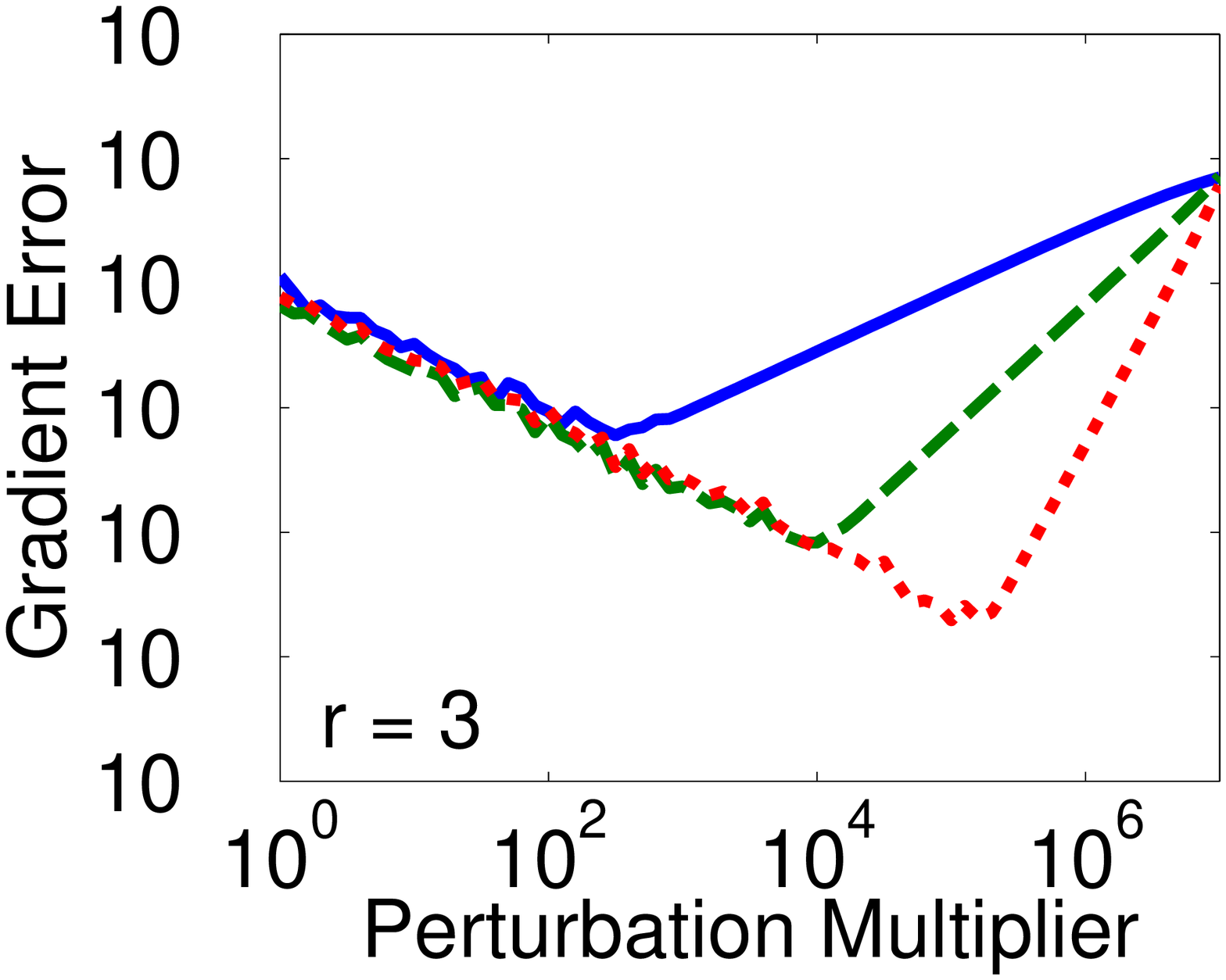}
\par\end{centering}

\begin{centering}
\includegraphics[bb=18bp 179bp 540bp 604bp,clip,scale=0.18]{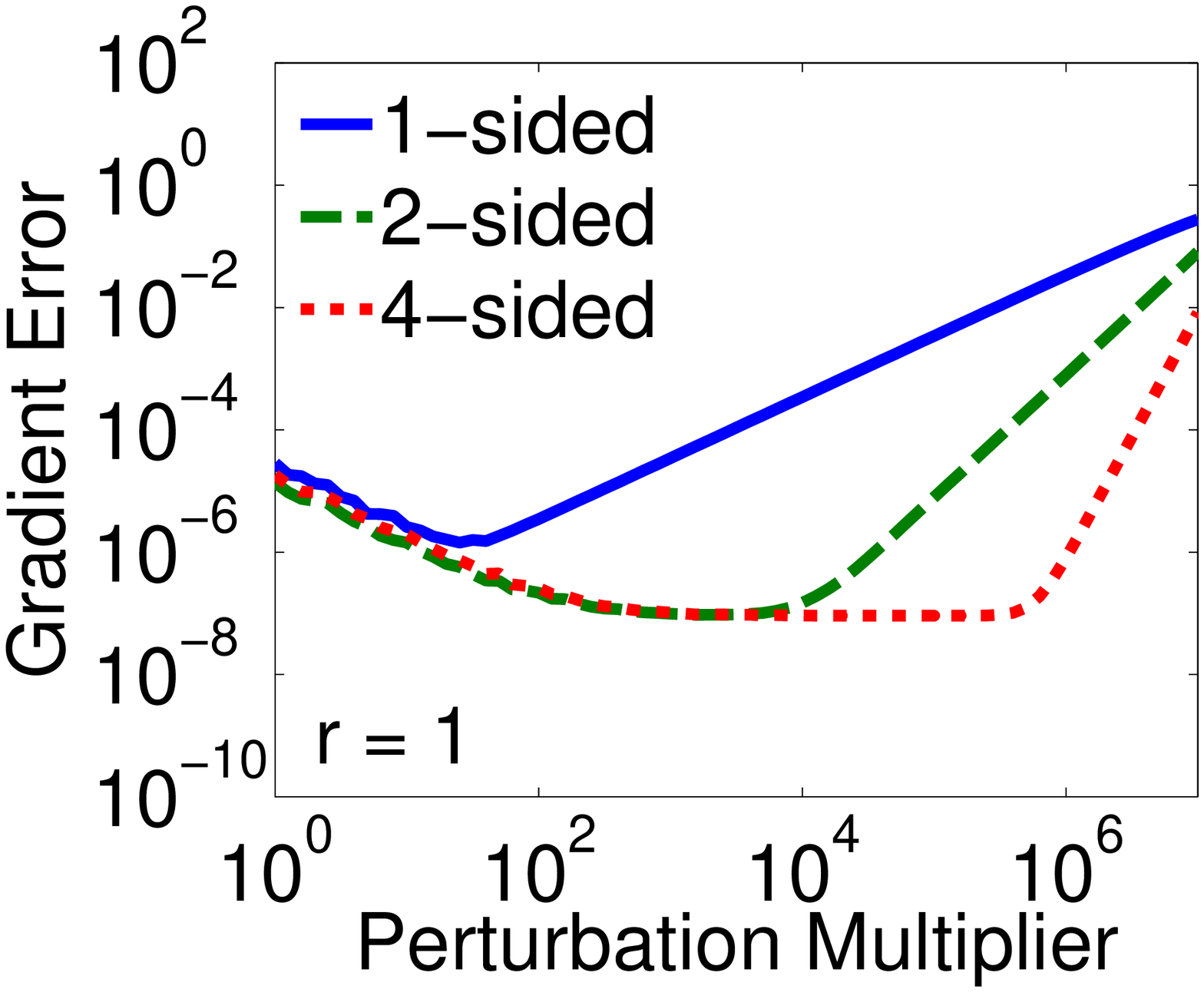}\includegraphics[bb=110bp 179bp 540bp 604bp,clip,scale=0.18]{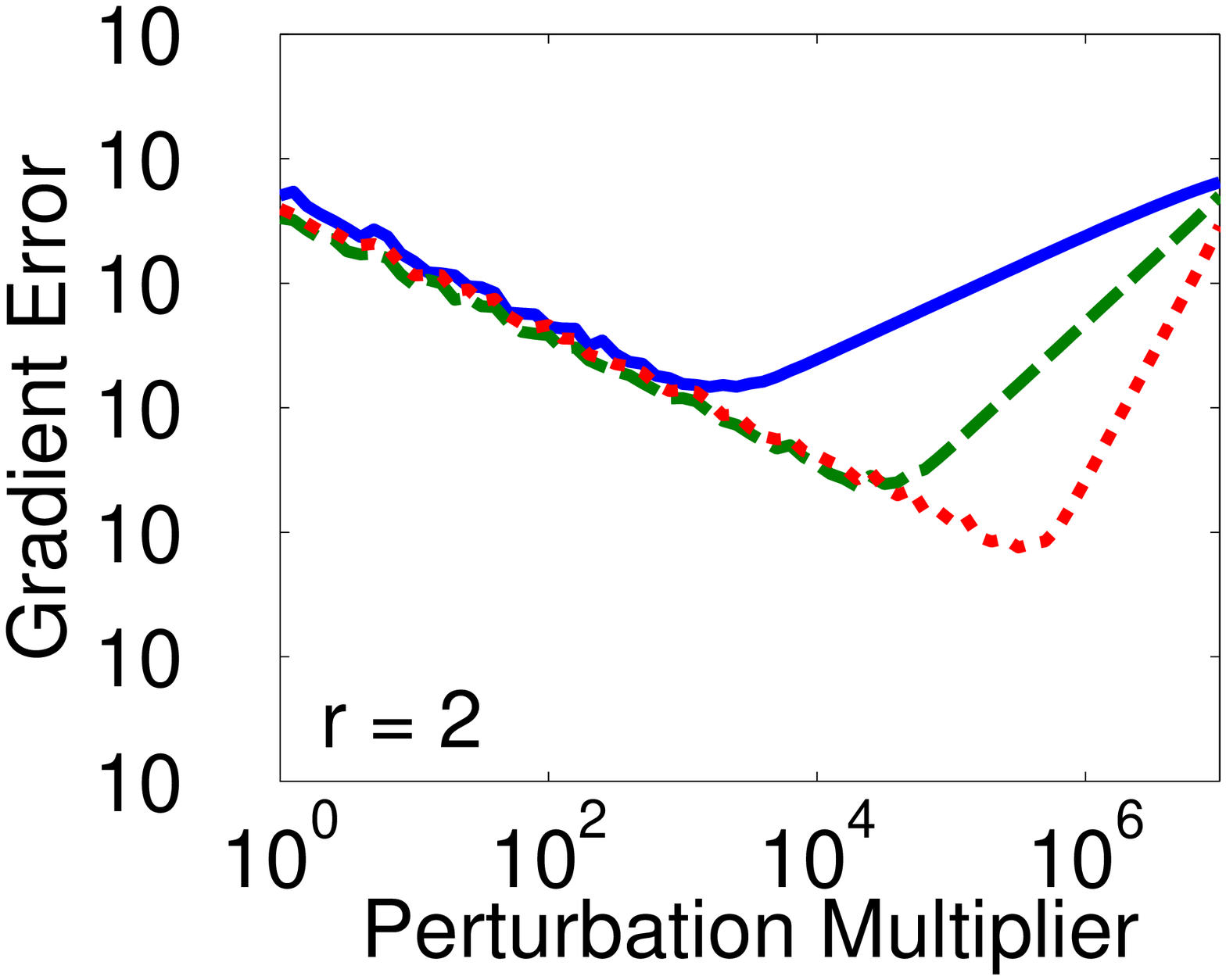}\includegraphics[bb=110bp 179bp 540bp 604bp,clip,scale=0.18]{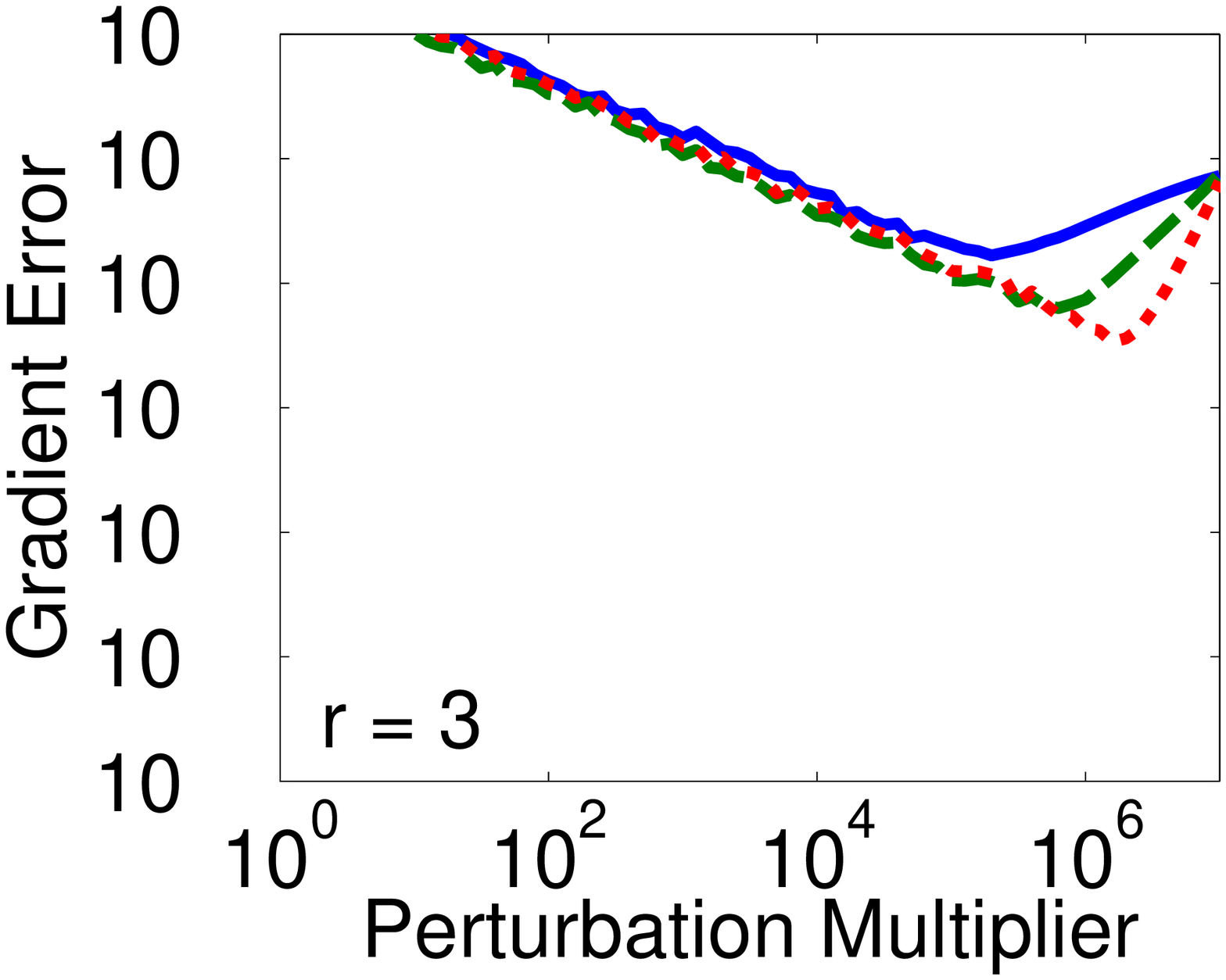}
\par\end{centering}

\centering{}\caption{An evaluation of perturbation multipliers $m$. Top: TRW. Bottom:
Mean field. Two effects are in play here: First, for too small a perturbation,
numerical errors dominate. Meanwhile, for too large a perturbation,
approximation errors dominate. We see that using 2- or 4-sided differences
differences reduce approximation error, leading to better results
with larger perturbations.\label{fig:Perurbation-Sizes}}
\end{figure}

\section{Truncated Fitting\label{sec:Truncated-Fitting}}

The previous methods for computing loss gradients are derived under
the assumption that the inference optimization is solved exactly.
In an implementation, of course, some convergence threshold must be
used. 

Different convergence thresholds can be used in the learning stage
and at test time. In practice, we have observed that too loose a threshold
in the learning stage can lead to a bad estimated risk gradient, and
learning terminating with a bad search direction. Meanwhile, a loose
threshold can often be used at test time with few consequences. Usually,
a difference of $10^{-3}$ in estimated marginals has little practical
impact, but this can still be enough to prevent learning from succeeding
\cite{DomkeCVPR2011}.

It seems odd that the learning algorithm would spend the majority
of computational effort exploring tight convergence levels that are
irrelevant to the practical performance of the model. Here, we \emph{define}
the learning objective in terms of the approximate marginals obtained
after a fixed number of iterations. To understand this, one may think
of the inference process not as an optimization, but rather as a large,
nonlinear function. This clearly leads to a well-defined objective
function. Inputting parameters, applying the iterations of either
TRW or mean field, computing predicted marginals, and finally a loss
are all differentiable operations. Thus, the loss gradient is efficiently
computable, at least in principle, by reverse-mode automatic differentiation
(autodiff), an approach explored by Stovanov et al. \cite{EmpiricalRiskMinimizationofGraphicalModel,stoyanov-eisner-2012-naacl}.
In preliminary work, we experimented with autodiff tools, but found
these to be unsatisfactory for our applications for two reasons. Firstly,
these tools impose a computational penalty over manually derived gradients.
Secondly, autodiff stores all intermediate calculations, leading to
large memory requirements. The methods derived below use less memory,
both in terms of constant factors and big-O complexity. Nevertheless,
some of these problems are issues with current \emph{implementations}
of reverse-mode autodiff, avoidable in theory.

Both mean field and TRW involve steps where we first take a product
of a set of terms, and then normalize. We define a ``backnorm''
operator, which is useful in taking derivatives over such operations,
by
\[
\text{backnorm}({\bf g},{\bf c})={\bf c}\odot({\bf g}-{\bf g}\cdot{\bf c}).
\]

\noindent This will be used in the algorithms here. More discussion
on this point can be found in Appendix C.

\subsection{Back Mean Field}

\begin{algorithm}
\begin{enumerate}
\item \setlength{\itemsep}{2pt}Initialize $\mu$ uniformly.
\item \noindent Repeat ${\displaystyle N}$ times for all $j$:

\begin{enumerate}
\item \noindent \setlength{\itemsep}{5pt}Push the marginals $\mu_{j}$
onto a stack.
\item \noindent ${\mu(x_{j})\propto\exp\bigl(\theta(x_{j})+\underset{c:j\in c}{\sum}\,\underset{{\bf x}_{c\backslash j}}{\sum}\theta({\bf x}_{c})\underset{i\in c\backslash j}{\prod}\mu(x_{i})\bigr)}$
\end{enumerate}
\item \noindent Compute $L$, $\overleftarrow{\mu}(x_{j})=\frac{dL}{d\mu(x_{j})}$
and $\overleftarrow{\mu}({\bf x}_{c})=\frac{dL}{d\mu({\bf x}_{c})}$.
\item \noindent Initialize $\overleftarrow{\theta}(x_{i})\leftarrow0,\,\overleftarrow{\theta}({\bf x}_{c})\leftarrow0.$
\item \noindent Repeat $N$ times for all $j$ (in reverse order):

\begin{enumerate}
\item \noindent \setlength{\itemsep}{5pt}${\displaystyle \overleftarrow{\nu_{j}}\,\,\,\,\,\,\,\,\,\,\leftarrow\text{backnorm}(\overleftarrow{\mu_{j}},\mu_{j})}$
\item \noindent $\overleftarrow{\theta}(x_{j})\leftarrow\overleftarrow{\theta}(x_{j})+\overleftarrow{\nu}(x_{j})$
\item \noindent $\overleftarrow{\theta}({\bf x}_{c})\leftarrow\overleftarrow{\theta}({\bf x}_{c})+\overleftarrow{\nu}(x_{j})\underset{i\in c\backslash j}{\prod}\mu(x_{i})$\hfill{}${\displaystyle \forall c:j\in c}$
\item \noindent $\overleftarrow{\mu}(x_{i})\leftarrow\overleftarrow{\mu}(x_{i})+\underset{{\bf x}_{c\backslash i}}{\sum}\overleftarrow{\nu}(x_{j})\theta({\bf x}_{c})\underset{k\in c\backslash\{i,j\}}{\prod}\mu(x_{k})$\vspace{5pt}\\
\hphantom{1pt}\hfill{}${\displaystyle \forall c:j\in c,\,\forall i\in c\backslash j}$
\item Pull marginals $\mu_{j}$ from the stack.
\item $\overleftarrow{\mu_{j}}(x_{j})\leftarrow0$
\end{enumerate}
\end{enumerate}
\caption{Back Mean Field\label{alg:Back-Mean-Field}}
\end{algorithm}

The first backpropagating inference algorithm, back mean field, is
shown as Alg. \ref{alg:Back-Mean-Field}. The idea is as follows:
Suppose we start with uniform marginals, run $N$ iterations of mean
field, and then-- regardless of if mean field has converged or not--
take predicted marginals and plug them into one of the marginal-based
loss functions. Since each step in this process is differentiable,
this specifies the loss as a differentiable function of model parameters.
We want the exact gradient of this function.
\begin{thm*}
After execution of back mean field, 
\[
\overleftarrow{\theta}(x_{i})=\frac{dL}{d\theta(x_{i})}\text{ and }\overleftarrow{\theta}({\bf x}_{c})=\frac{dL}{d\theta({\bf x}_{c})}.
\]

\end{thm*}
A proof sketch is in Appendix C. Roughly speaking, the proof takes
the form of a mechanical differentiation of each step of the inference
process.

Note that, as written, back mean field only produces univariate marginals,
and so cannot cope with loss functions making use of clique marginals.
However, with mean field, the clique marginals, are simply the product
of  univariate marginals: $\mu({\bf x}_{c})=\prod_{i\in c}\mu(x_{i})$.
Hence, any loss defined on clique marginals can equivalently be defined
on univariate marginals.

\subsection{Back TRW}

\begin{algorithm}
\hspace{-10pt}%
\begin{minipage}[t]{1\columnwidth}%
\begin{enumerate}
\item \setlength{\itemsep}{2pt}Initialize $m$ uniformly.
\item \noindent Repeat ${\displaystyle N}$ times for all pairs $(c,i)$,
with $i\in c$:

\begin{enumerate}
\item \noindent Push the messages $m_{c}(x_{i})$ onto a stack.
\item \noindent $m_{c}(x_{i})\propto\sum_{{\bf x}_{c\backslash i}}e^{\frac{1}{\rho_{c}}\theta({\bf x}_{c})}\prod_{j\in c\backslash i}e^{\theta(x_{j})}\frac{\underset{d:j\in d}{\prod}m_{d}(x_{j})^{\rho_{d}}}{m_{c}(x_{j})}$
\end{enumerate}
\item \hspace{-8pt}%
\begin{tabular}{>{\centering}p{0.3in}>{\centering}p{0.01in}c}
$\mu({\bf x}_{c})$ & $\propto$ & $e^{\frac{1}{\rho_{c}}\theta({\bf x}_{c})}\prod_{i\in c}e^{\theta(x_{i})}\frac{\underset{d:i\in d}{\prod}m_{d}(x_{i})^{\rho_{d}}}{m_{c}(x_{i})}$\tabularnewline
\end{tabular}\hfill{}$\forall c$
\item \hspace{-8pt}%
\begin{tabular}{>{\centering}p{0.3in}>{\centering}p{0.01in}c}
$\mu(x_{i})$ & $\propto$ & $e^{\theta(x_{i})}\prod_{d:i\in d}m_{d}(x_{i})^{\rho_{d}}$\tabularnewline
\end{tabular}\hfill{}$\forall i$
\item \noindent Compute $L$, $\overleftarrow{\mu}(x_{i})=\frac{dL}{d\mu(x_{i})}$
and $\overleftarrow{\mu}({\bf x}_{c})=\frac{dL}{d\mu({\bf x}_{c})}$.
\item For all $c,$

\begin{enumerate}
\item \setlength{\itemsep}{5pt}\hspace{-8pt}%
\begin{tabular}{>{\centering}p{0.3in}>{\centering}p{0.01in}c}
$\overleftarrow{\nu}({\bf x}_{c})$ & $\leftarrow$ & $\text{backnorm}(\overleftarrow{\mu_{c}},\mu_{c})$\tabularnewline
\end{tabular}
\item \hspace{-8pt}%
\begin{tabular}{>{\centering}p{0.3in}>{\centering}p{0.01in}c}
$\overleftarrow{\theta}({\bf x}_{c})$ & $\overset{+}{\leftarrow}$ & $\frac{1}{\rho_{c}}\overleftarrow{\nu}({\bf x}_{c})$\tabularnewline
\end{tabular}
\item \hspace{-8pt}%
\begin{tabular}{>{\centering}p{0.3in}>{\centering}p{0.01in}c}
$\overleftarrow{\theta}(x_{i})$ & $\overset{+}{\leftarrow}$ & $\sum_{{\bf x}_{c\backslash i}}\overleftarrow{\nu}({\bf x}_{c})$\tabularnewline
\end{tabular}\hfill{}$\forall i\in c$
\item \hspace{-8pt}%
\begin{tabular}{>{\centering}p{0.3in}>{\centering}p{0.01in}c}
$\overleftarrow{m_{d}}(x_{i})$ & $\overset{+}{\leftarrow}$ & $\frac{\rho_{d}-I_{c=d}}{m_{d}(x_{i})}\sum_{{\bf x}_{c\backslash i}}\overleftarrow{\nu}$\tabularnewline
\end{tabular}\hfill{}$\forall i\in c,\forall d:i\in d$
\end{enumerate}
\item For all $i$,

\begin{enumerate}
\item \setlength{\itemsep}{5pt}\hspace{-8pt}%
\begin{tabular}{>{\centering}p{0.3in}>{\centering}p{0.01in}c}
$\overleftarrow{\nu}(x_{i})$ & $\leftarrow$ & $\text{backnorm}(\overleftarrow{\mu_{i}},\mu_{i})$\tabularnewline
\end{tabular}
\item \hspace{-8pt}%
\begin{tabular}{>{\centering}p{0.3in}>{\centering}p{0.01in}c}
$\overleftarrow{\theta}(x_{i})$ & $\overset{+}{\leftarrow}$ & $\overleftarrow{\nu}(x_{i})$\tabularnewline
\end{tabular}
\item \hspace{-8pt}%
\begin{tabular}{>{\centering}p{0.3in}>{\centering}p{0.01in}c}
$\overleftarrow{m_{d}}(x_{i})$ & $\overset{+}{\leftarrow}$ & $\rho_{d}\frac{\overleftarrow{\nu}(x_{i})}{m_{d}(x_{i})}$\tabularnewline
\end{tabular}\hfill{}$\forall d:i\in d$
\end{enumerate}
\item \noindent Repeat $N$ times for all pairs $(c,i)$ (in reverse order)

\begin{enumerate}
\item \noindent \setlength{\itemsep}{5pt}\hspace{-8pt}%
\begin{tabular}{>{\centering}p{0.3in}>{\centering}p{0.01in}c}
$s({\bf x}_{c})$ & $\leftarrow$ & $e^{\frac{1}{\rho_{c}}\theta({\bf x}_{c})}\prod_{j\in c\backslash i}e^{\theta(x_{j})}\frac{\underset{d:j\in d}{\prod}m_{d}(x_{j})^{\rho_{d}}}{m_{c}(x_{j})}$\tabularnewline
\end{tabular}
\item \hspace{-8pt}%
\begin{tabular}{>{\centering}p{0.3in}>{\centering}p{0.01in}c}
$\overleftarrow{\nu}(x_{i})$ & $\leftarrow$ & $\text{backnorm}(\overleftarrow{m_{ci}},m_{ci})$\tabularnewline
\end{tabular}
\item \hspace{-8pt}%
\begin{tabular}{>{\centering}p{0.3in}>{\centering}p{0.01in}c}
$\overleftarrow{\theta}({\bf x}_{c})$ & $\overset{+}{\leftarrow}$ & $\frac{1}{\rho_{c}}s({\bf x}_{c})\frac{\overleftarrow{\nu}(x_{i})}{m_{c}(x_{i})}$\tabularnewline
\end{tabular}
\item \hspace{-8pt}%
\begin{tabular}{>{\centering}p{0.3in}>{\centering}p{0.01in}c}
$\overleftarrow{\theta}(x_{j})$ & ${\displaystyle \overset{+}{\leftarrow}}$ & $\sum_{{\bf x}_{c\backslash j}}s({\bf x}_{c})\frac{\overleftarrow{\nu}(x_{i})}{m_{c}(x_{i})}$\tabularnewline
\end{tabular}\hfill{}$\forall j\in c\backslash i$
\item \hspace{-8pt}%
\begin{tabular}{>{\centering}p{0.3in}>{\centering}p{0.01in}c}
$\overleftarrow{m_{d}}(x_{j})$ & $\overset{+}{\leftarrow}$ & $\frac{\rho_{d}-I_{c=d}}{m_{d}(x_{j})}\sum_{{\bf x}_{c\backslash j}}s({\bf x}_{c})\frac{\overleftarrow{\nu}(x_{i})}{m_{c}(x_{i})}$\tabularnewline
\end{tabular}\\
\hphantom{1pt}\hfill{}$\forall j\in c\backslash i,\forall d:j\in d$\vspace{-5pt}
\item Pull messages $m_{c}(x_{i})$ from the stack.
\item $\overleftarrow{m_{c}}(x_{i})\leftarrow0$\end{enumerate}
\end{enumerate}
\end{minipage}

\caption{Back TRW. \label{alg:Back-TRW}}
\end{algorithm}

Next, we consider truncated fitting with TRW inference. As above,
we will assume that some fixed number $N$ of inference iterations
have been run, and we want to define and differentiate a loss defined
on the current predicted marginals. Alg. \ref{alg:Back-TRW} shows
the method.
\begin{thm*}
After execution of back TRW, 
\[
\overleftarrow{\theta}(x_{i})=\frac{dL}{d\theta(x_{i})}\text{ and }\overleftarrow{\theta}({\bf x}_{c})=\frac{dL}{d\theta({\bf x}_{c})}.
\]

\end{thm*}
Again, a proof sketch is in Appendix C.

If one uses pairwise factors only, uniform appearance probabilities
of $\rho=1$, removes all reference to the stack, and uses a convergence
threshold in place of a fixed number of iterations, one obtains essentially
Eaton and Ghahramani's back belief propagation \cite[ extended version, Fig. 5]{Eaton_Ghahramani}.
Here, we refer to the general strategy of using full (non-truncated)
inference as ``backpropagation'', either with LBP, TRW, or mean
field.

\subsection{Truncated Likelihood \& Truncated EM}

Applying the truncated fitting strategies to any of the marginal-based
loss functions is simple. Applying it to the likelihood or EM loss,
however, is not so straightforward. The reason is that these losses
(Eqs. \ref{eq:likelihood} and \ref{eq:EM_loss}) are defined, not
in terms of predicted \emph{marginals}, but in terms of \emph{partition
functions}. Nevertheless, we wish to compare to these losses in the
experiments below. As we found truncation to be critical for speed,
we instead derive a variant of truncated fitting.

The basic idea is to define a ``truncated partition function''.
This is done by taking the predicted marginals, obtained after a fixed
number of iterations, and plugging them into the entropy approximations
used either for mean field (Eq. \ref{eq:meanfield-entropy}) or TRW
(Eq. \ref{eq:TRW-entropy}). The approximate entropy $\tilde{H}$
is then used in defining a truncated partition function as\vspace{-5pt}
\[
\tilde{A}(\boldsymbol{\theta})=\boldsymbol{\theta}\cdot\boldsymbol{\mu}(\boldsymbol{\theta})-\tilde{H}(\boldsymbol{\mu}(\boldsymbol{\theta})).
\]
As we will see below, with too few inference iterations, using this
approximation can cause the surrogate likelihood to diverge. To see
why, imagine an extreme case where \textit{zero} inference iterations
are used. This results in the loss $L(\boldsymbol{\theta},{\bf x})=\boldsymbol{\theta}\cdot({\bf f}({\bf x})-\boldsymbol{\mu}^{0})+\tilde{H}(\boldsymbol{\mu}^{0}),$
where $\boldsymbol{\mu}^{0}$ are the initial marginals. As long as
the mean of ${\bf f}({\bf x})$ over the dataset is not equal to $\boldsymbol{\mu}^{0}$,
arbitrary loss can be achieved. With hidden variables, $A(\boldsymbol{\theta},{\bf z})$
is defined similarly, but with the variables ${\bf z}$ ``clamped''
to the observed values. (Those variables will play no role in determining
$\boldsymbol{\mu}(\boldsymbol{\theta})$).

\section{Experiments}

These experiments consider three different datasets with varying complexity.
In all cases, we try to keep the features used relatively simple.
This means some sacrifice in performance, relative to using sophisticated
features tuned more carefully to the different problem domains. However,
given that our goal here is to gauge the relative performance of the
different algorithms, we use simple features for the sake of experimental
transparency.

We compare marginal-based learning methods to the surrogate likelihood/EM,
the pseudolikelihood and piecewise likelihood. These comparisons were
chosen because, first, they are the most popular in the literature
(Sec. \ref{sec:Loss-Functions}). Second, the surrogate likelihood
also requires marginal inference, meaning an ``apples to apples''
comparison using the same inference method. Third, these methods can
all cope with hidden variables, which appear in our third dataset.

In each experiment, an ``independent'' model, trained using univariate
features only with logistic loss was used to initialize others. The
smoothed classification loss, because of more severe issues with local
minima, was initialized using surrogate likelihood/EM.

\subsection{Setup}

All experiments here will be on vision problems, using a pairwise,
4-connected grid. Learning uses the L-BFGS optimization algorithm.
The values $\theta$ are linearly parametrized in terms of unary and
edge features. Formally, we will fit two matrices, $F$ and $G$,
which determine all unary and edge features, respectively. These can
be expressed most elegantly by introducing a bit more notation. Let
$\boldsymbol{\theta}_{i}$ denote the set of parameter values $\theta(x_{i})$
for all values $x_{i}$. If ${\bf u}({\bf y},i)$ denotes the vector
of unary features for variable $i$ given input image ${\bf y}$,
then
\[
\boldsymbol{\theta}_{i}=F{\bf u}({\bf y},i).
\]

\noindent Similarly, let $\boldsymbol{\theta}_{ij}$ denote the set
of parameter values $\theta(x_{i},x_{j})$ for all $x_{i},x_{j}$.
If ${\bf v}({\bf y},i,j)$ is the vector of edge features for pair
$(i,j)$, then
\[
\boldsymbol{\theta}_{ij}=G{\bf v}({\bf y},i,j).
\]

Once $\frac{dL}{d\boldsymbol{\theta}}$ has been calculated (for
whichever loss and method), we can easily recover the gradients with
respect to $F$ and $G$ by\vspace{-10pt}

\[
\frac{dL}{dF}=\sum_{i}\frac{dL}{d\boldsymbol{\theta}_{i}}{\bf u}({\bf y},i)^{T},\,\,\,\,\frac{dL}{dG}=\sum_{ij}\frac{dL}{d\boldsymbol{\theta}_{ij}}{\bf v}({\bf y},i,j)^{T}.
\]

\subsection{Binary Denoising Data}

\begin{table}[h]
\begin{centering}
\setlength{\tabcolsep}{5pt}%
\begin{tabular}{|c|cc|cc|cc|}
\hline 
 & \multicolumn{2}{c|}{$n=1.25$} & \multicolumn{2}{c|}{$n=1.5$} & \multicolumn{2}{c|}{$n=5$}\tabularnewline
\hline 
Loss & Train & Test & Train & Test & Train & Test\tabularnewline
\hline 
surrogate likelihood & .149 & .143 & .103 & .097 & .031 & .030\tabularnewline
\hline 
univariate logistic & .133 & .126 & .102 & .096 & .031 & .030\tabularnewline
\hline 
clique logistic & .132 & .126 & .102 & .096 & .031 & .030\tabularnewline
\hline 
univariate quad & .132 & .126 & .102 & .096 & .031 & .030\tabularnewline
\hline 
smooth class. $\alpha=5$ & .136 & .129 & .105 & .099 & .032 & .030\tabularnewline
\hline 
smooth class. $\alpha=15$ & .132 & .126 & .102 & .096 & .031 & .030\tabularnewline
\hline 
smooth class. $\alpha=50$ & .132 & .125 & .102 & .096 & .030 & .030\tabularnewline
\hline 
\hline 
pseudo-likelihood & .207 & .204 & .117 & .112 & .032 & .030\tabularnewline
\hline 
piecewise & .466 & .481 & .466 & .481 & .063 & .058\tabularnewline
\hline 
independent & .421 & .424 & .367 & .368 & .129 & .129\tabularnewline
\hline 
\end{tabular}
\par\end{centering}

\caption{Binary denoising error rates for different noise levels $n$. All
methods use TRW inference with back-propagation based learning with
a threshold of $10^{-4}$.\label{tab:Binary-Denoising-error}}
\end{table}
\begin{figure}[h]
\begin{centering}
\vspace{-10pt}
\renewcommand{\tabcolsep}{1pt}%
\begin{tabular}{>{\centering}m{0.57in}>{\centering}m{0.9in}>{\centering}m{0.9in}>{\centering}m{0.9in}}
 & $n=1.25$ & $n=1.5$ & $n=5$\tabularnewline
{\small input} & \includegraphics[width=1\linewidth]{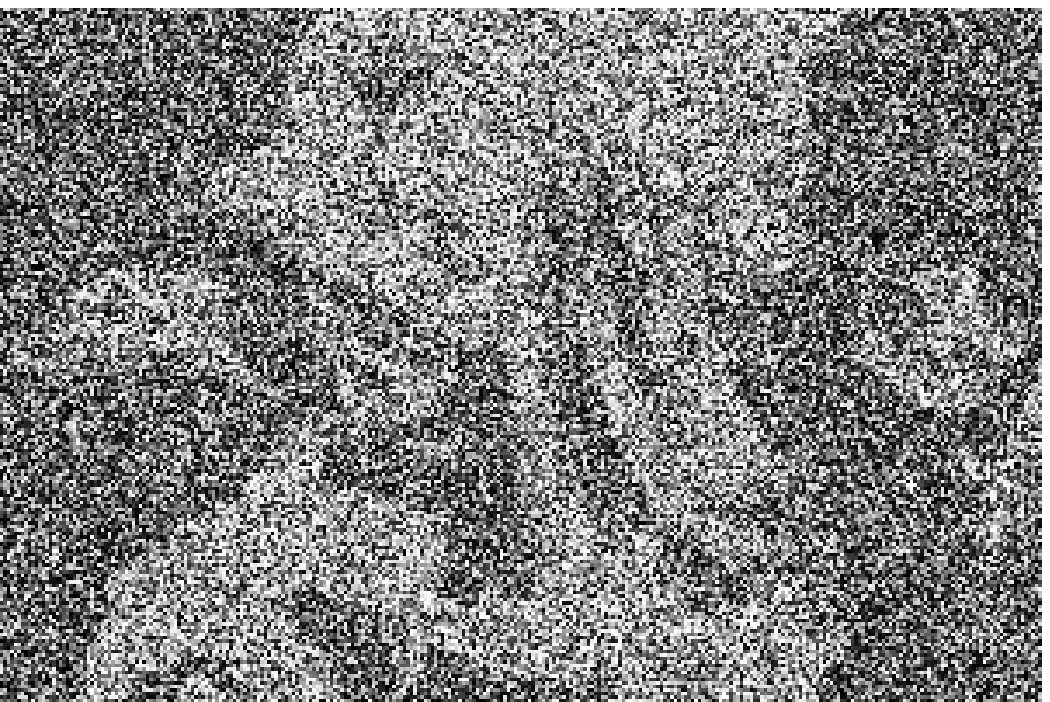} & \includegraphics[width=1\linewidth]{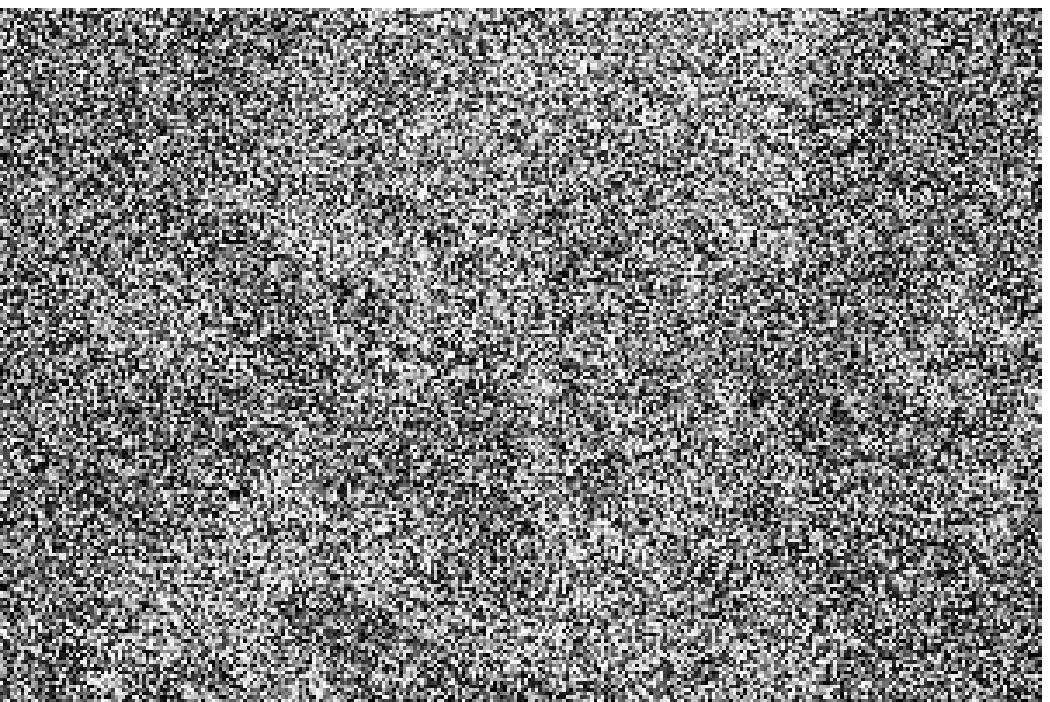} & \includegraphics[width=1\linewidth]{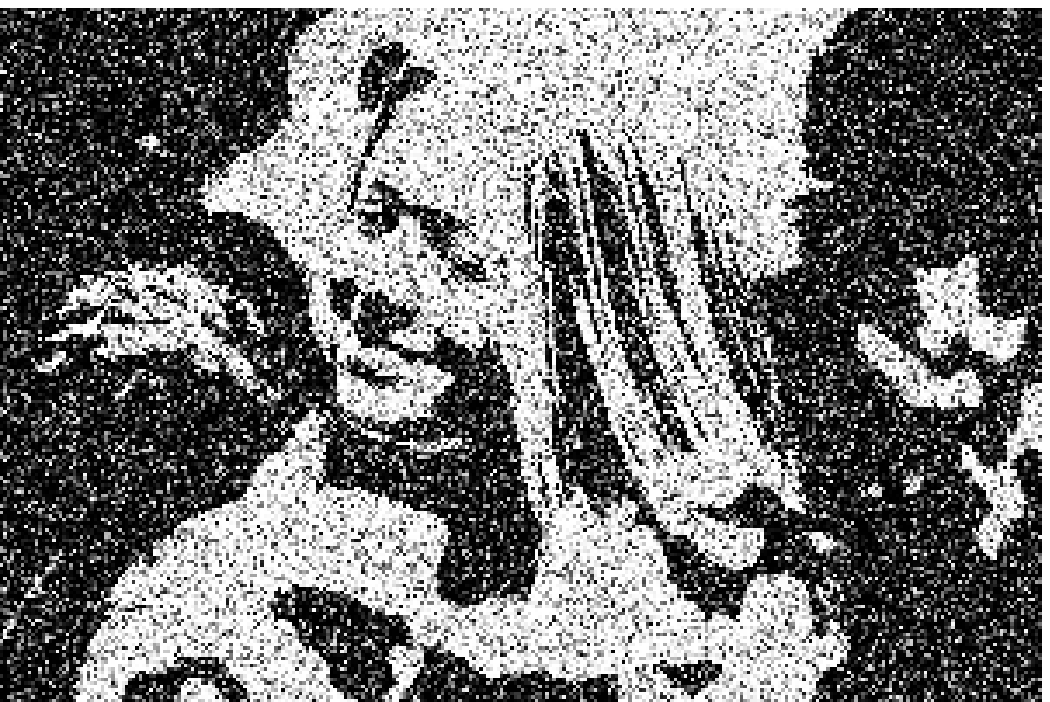}\tabularnewline
{\small surrogate likelihood} & \includegraphics[width=1\linewidth]{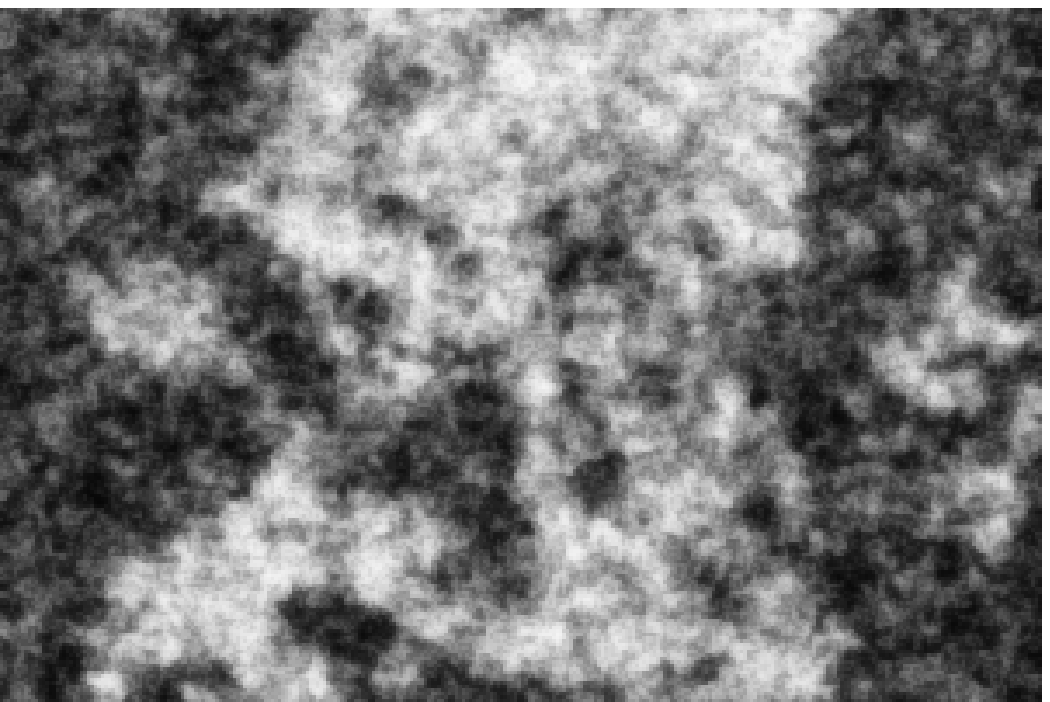} & \includegraphics[width=1\linewidth]{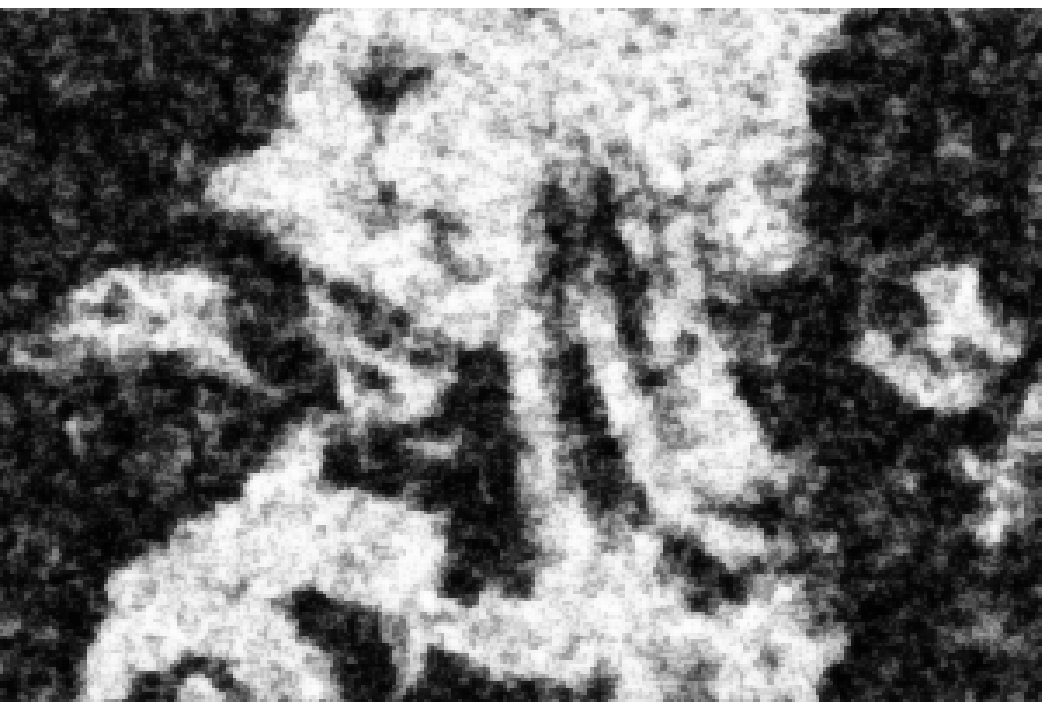} & \includegraphics[width=1\linewidth]{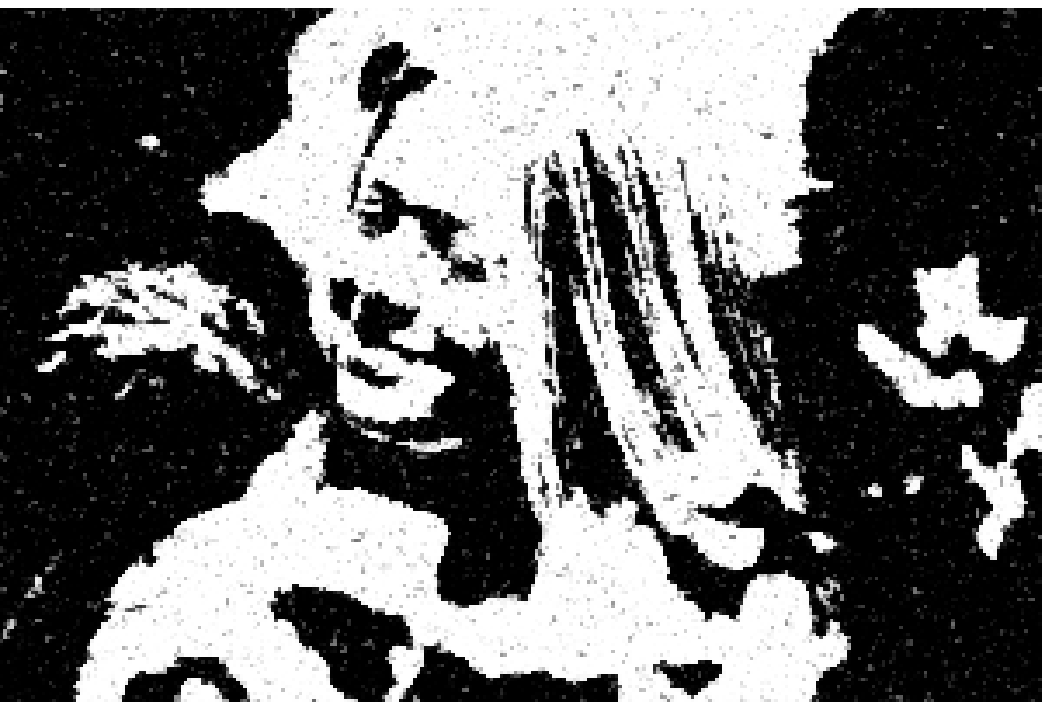}\tabularnewline
{\small univariate logistic} & \includegraphics[width=1\linewidth]{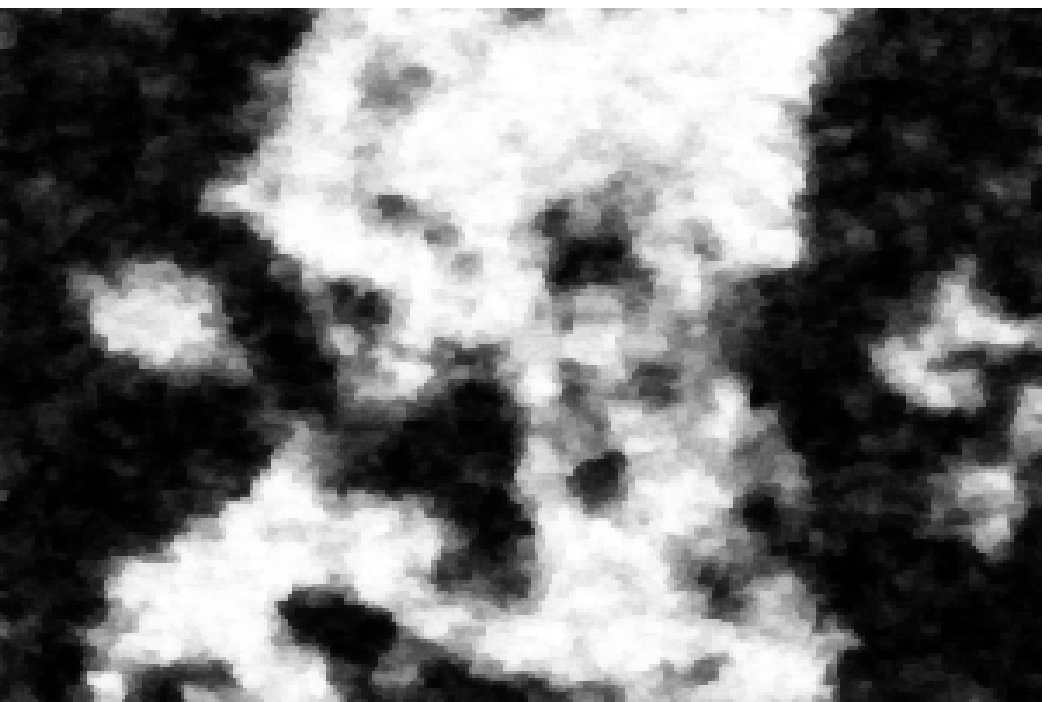} & \includegraphics[width=1\linewidth]{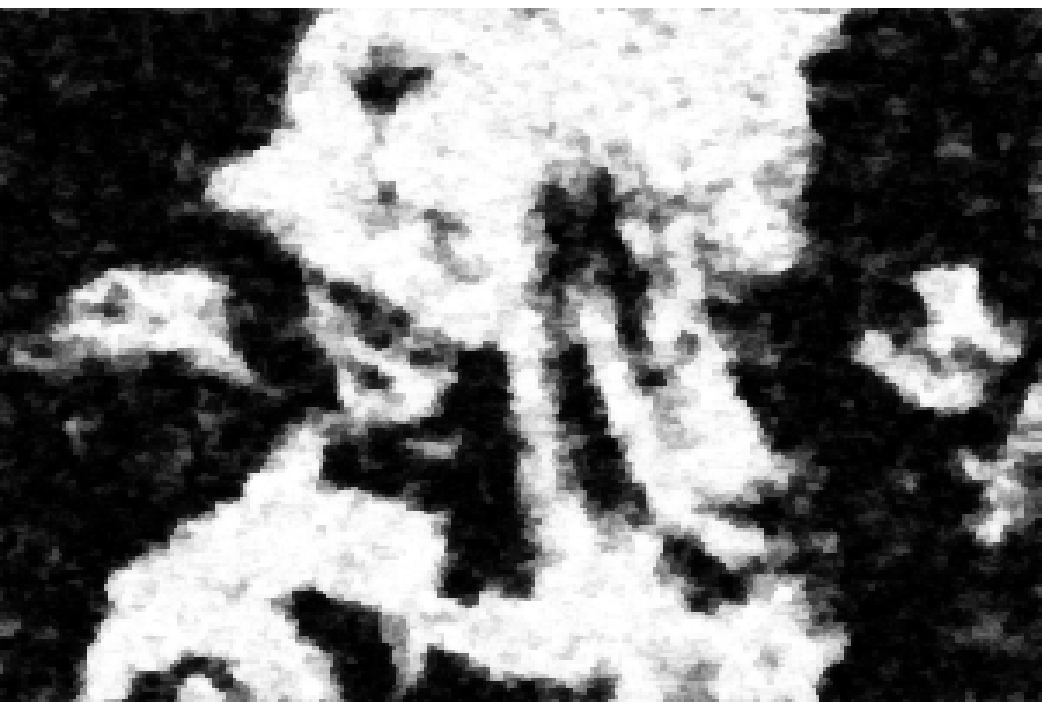} & \includegraphics[width=1\linewidth]{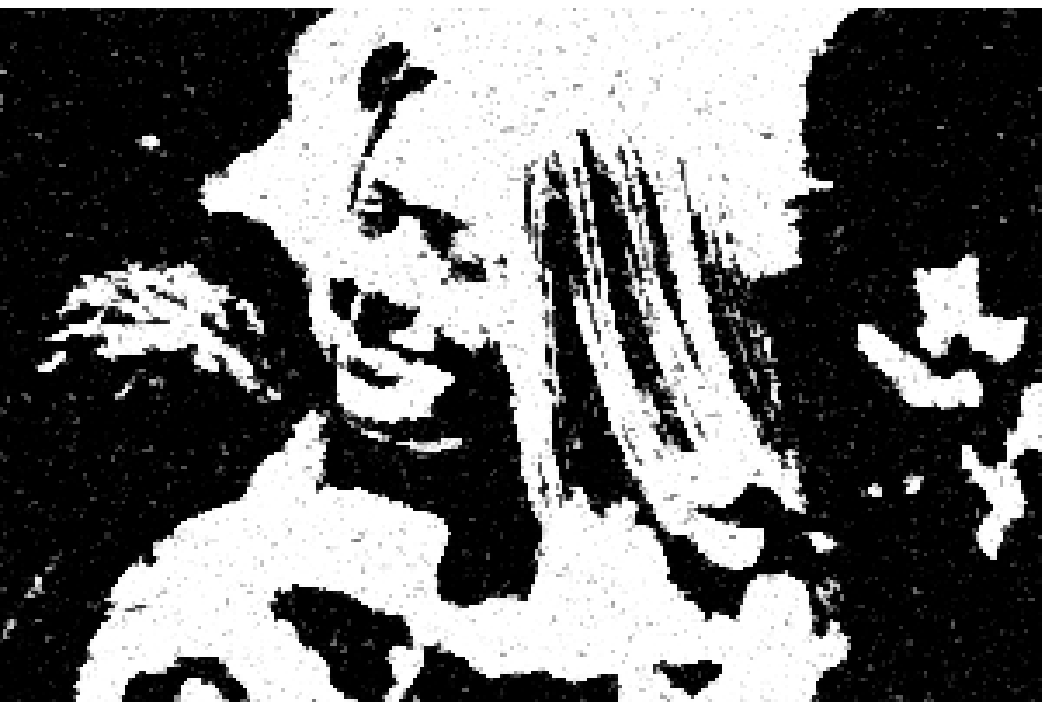}\tabularnewline
{\small clique logistic} & \includegraphics[width=1\linewidth]{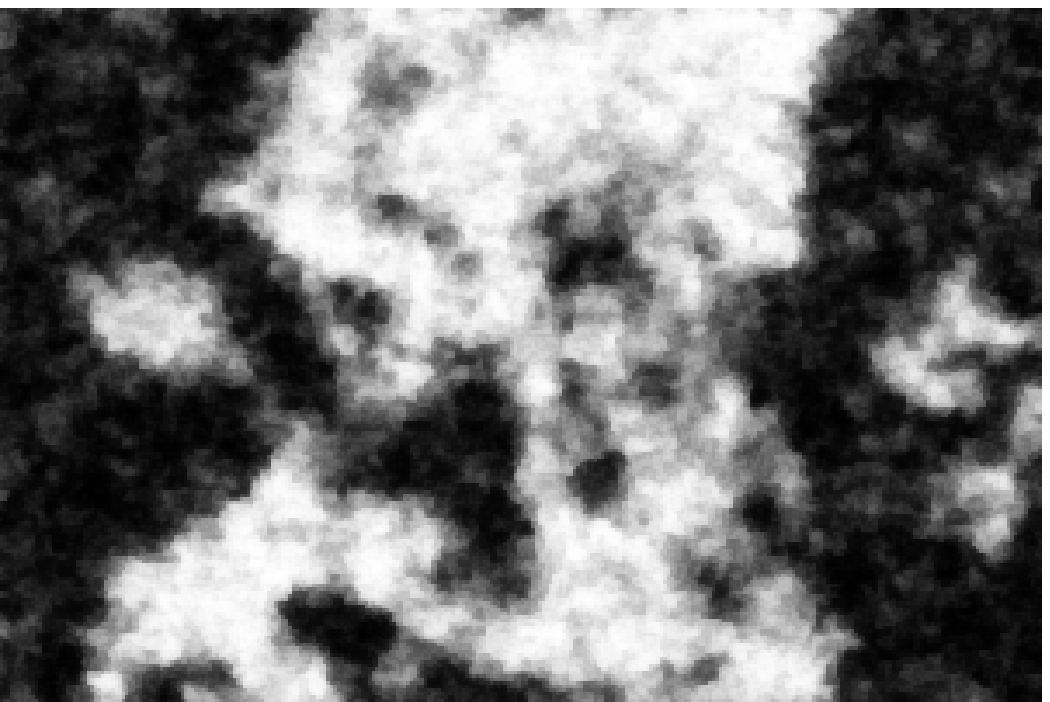} & \includegraphics[width=1\linewidth]{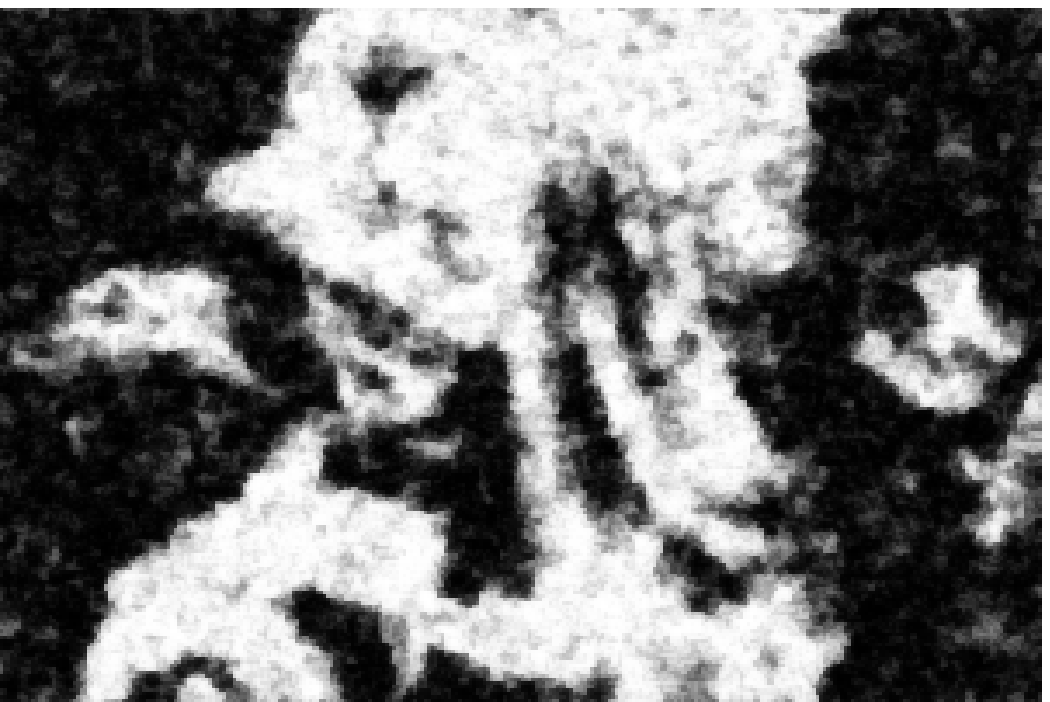} & \includegraphics[width=1\linewidth]{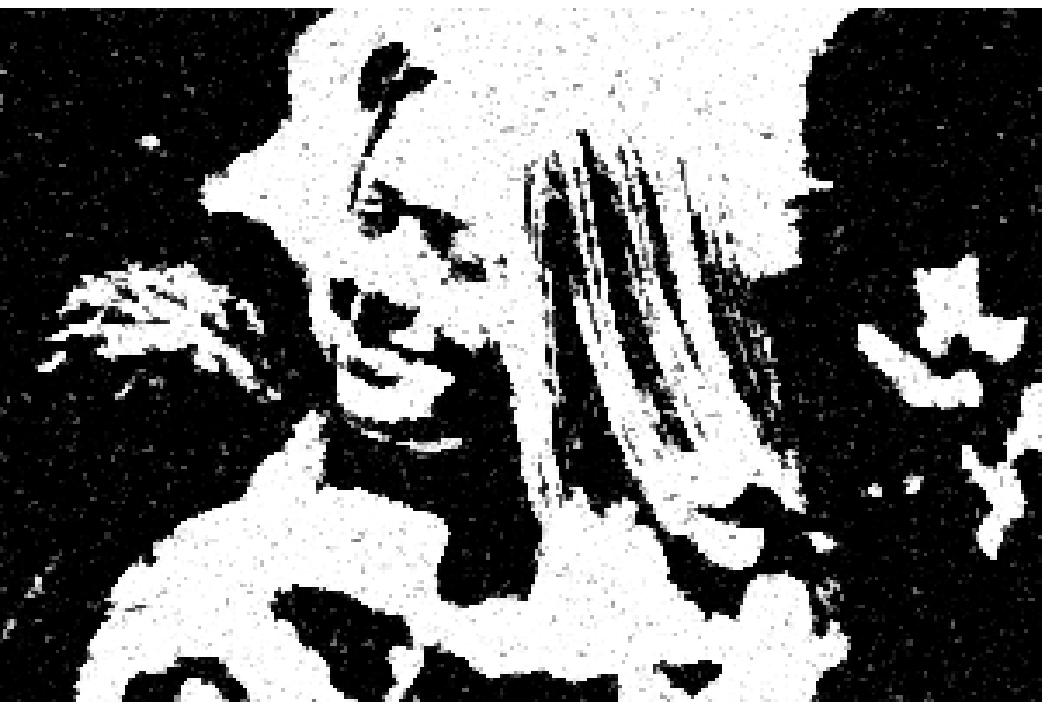}\tabularnewline
sm. class $\alpha=50$ & \includegraphics[width=1\linewidth]{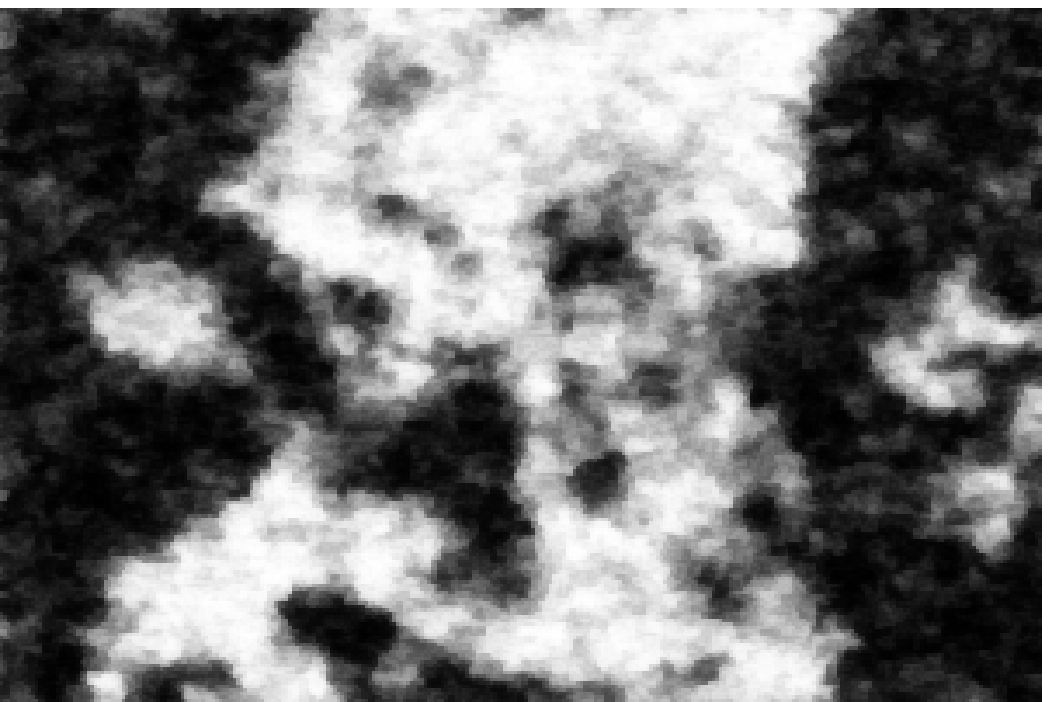} & \includegraphics[width=1\linewidth]{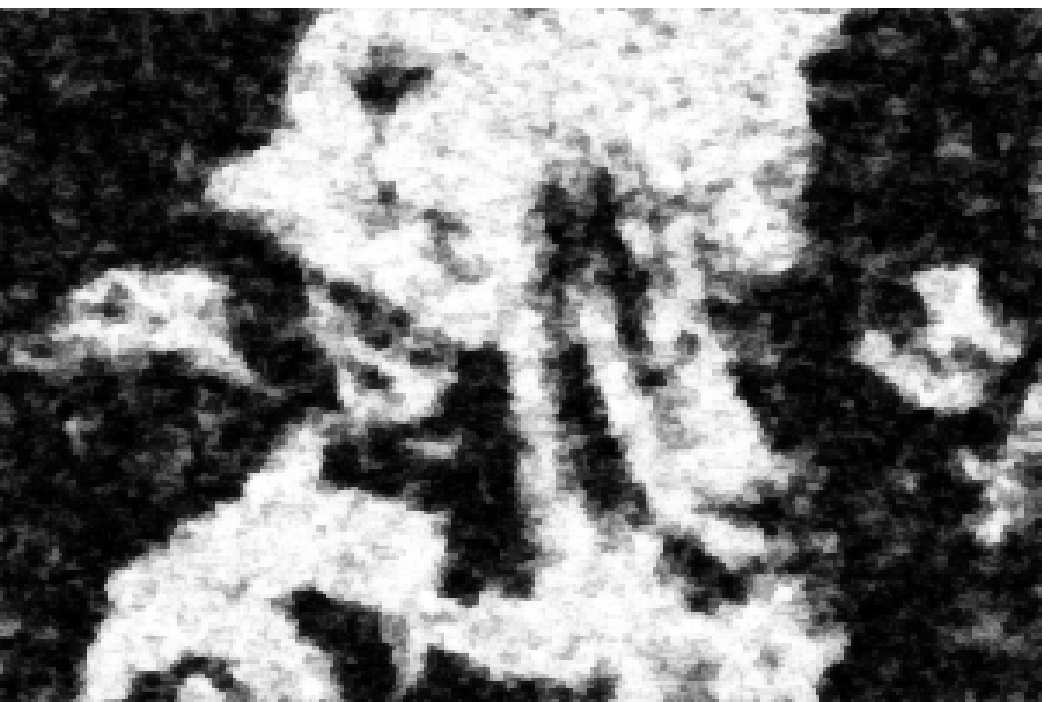} & \includegraphics[width=1\linewidth]{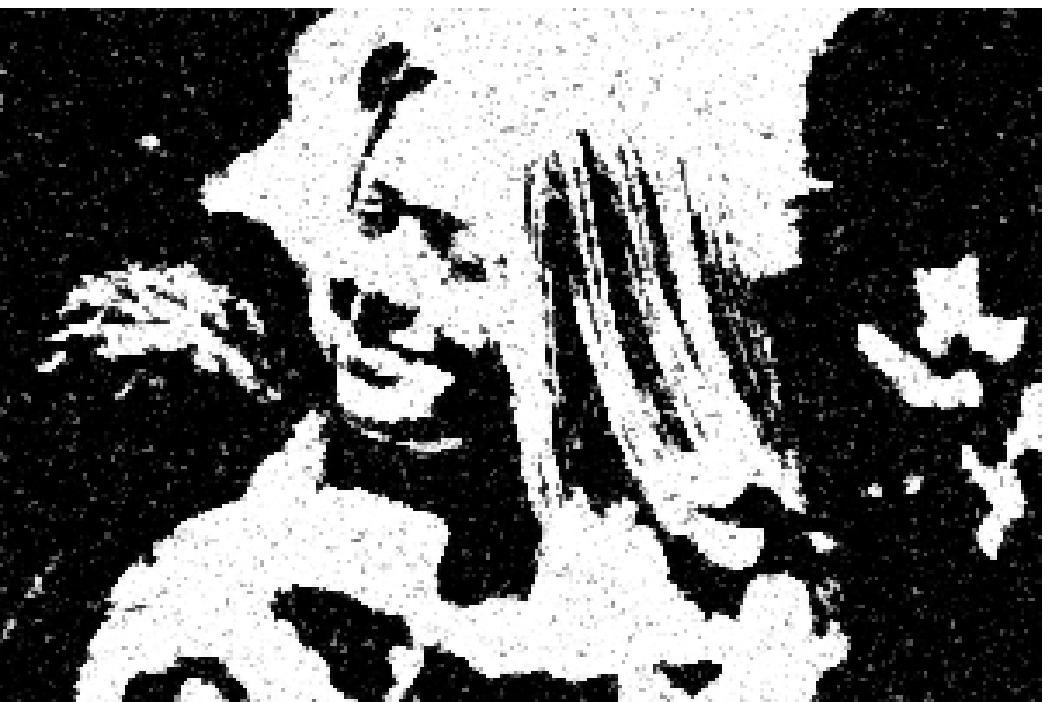}\tabularnewline
{\small pseudo-likelihood} & \includegraphics[width=1\linewidth]{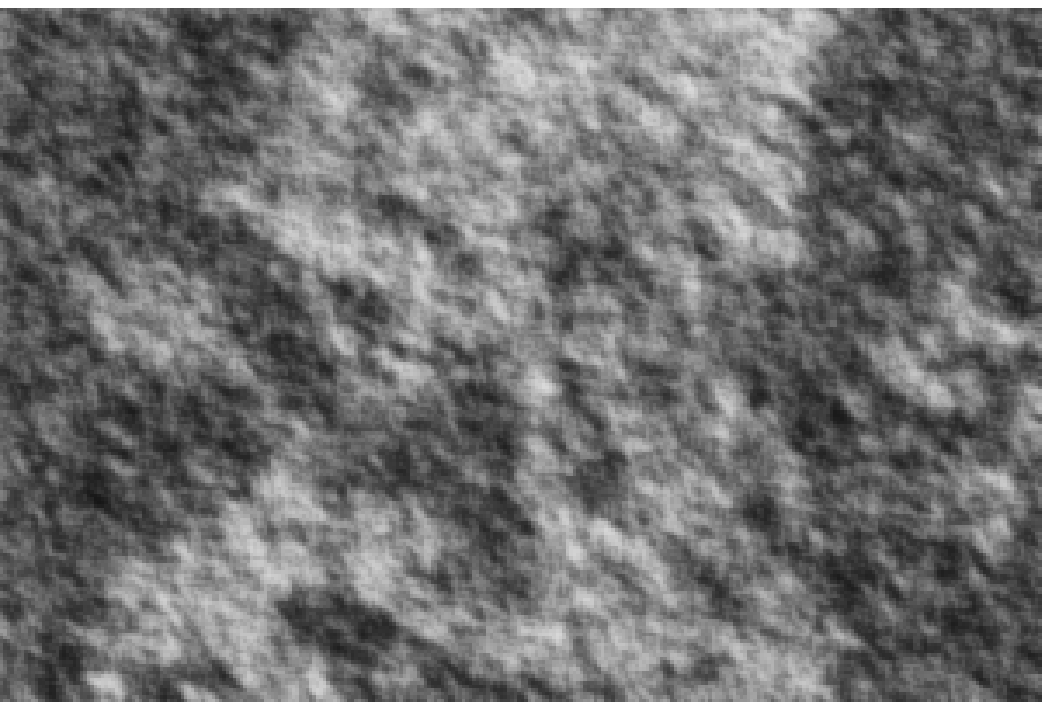} & \includegraphics[width=1\linewidth]{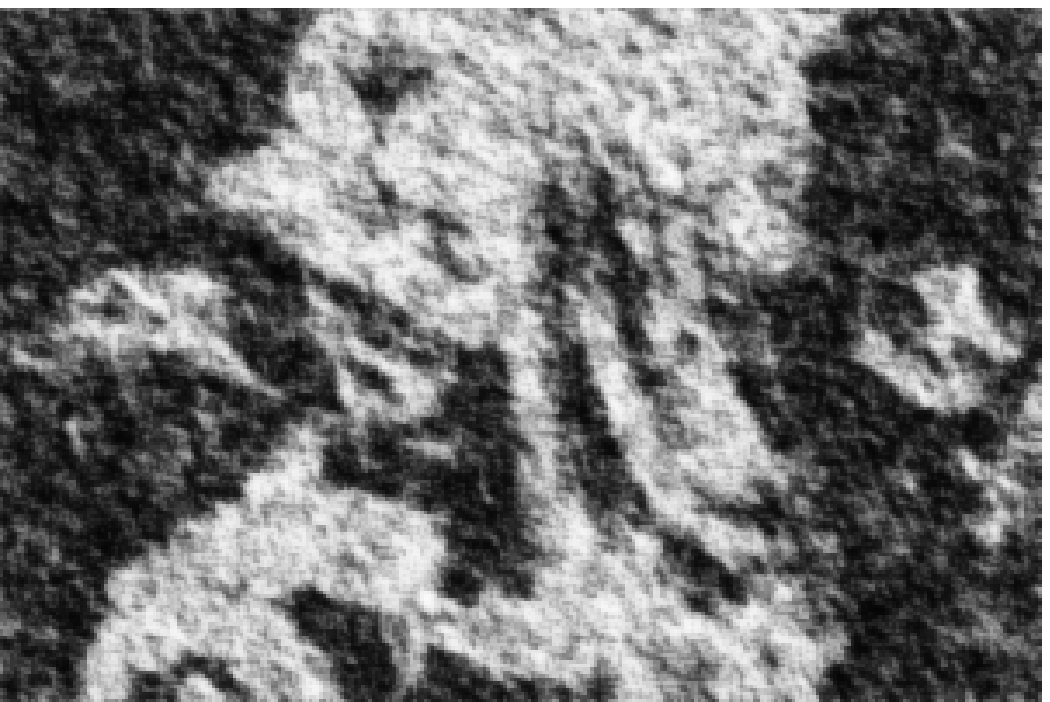} & \includegraphics[width=1\linewidth]{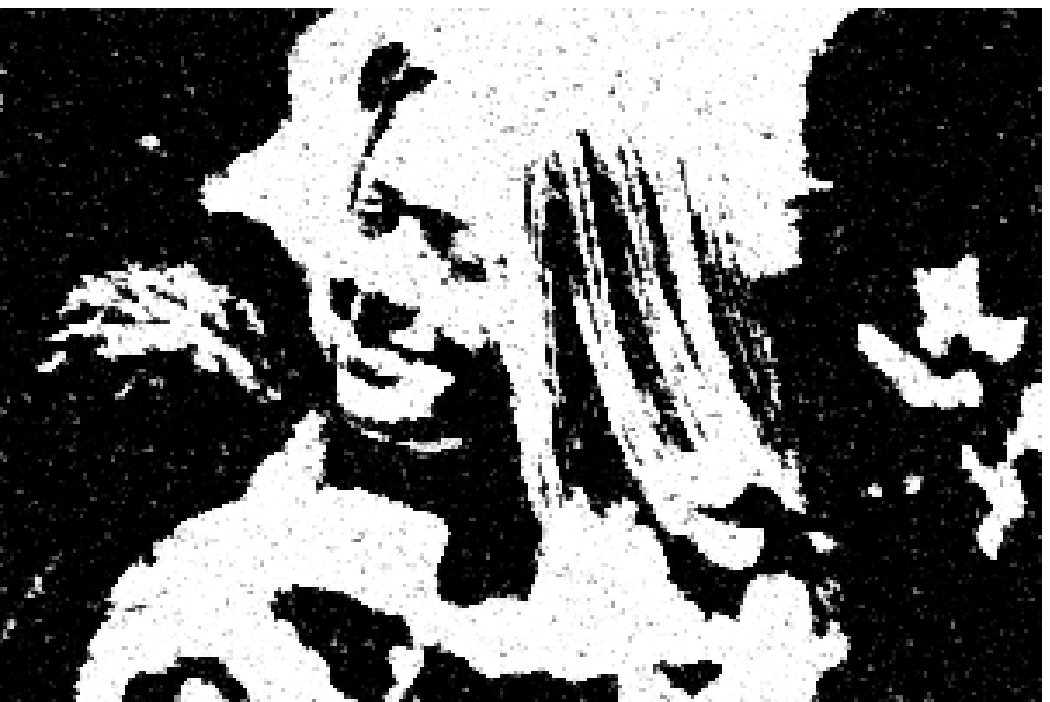}\tabularnewline
{\small inde-pendent} & \includegraphics[width=1\linewidth]{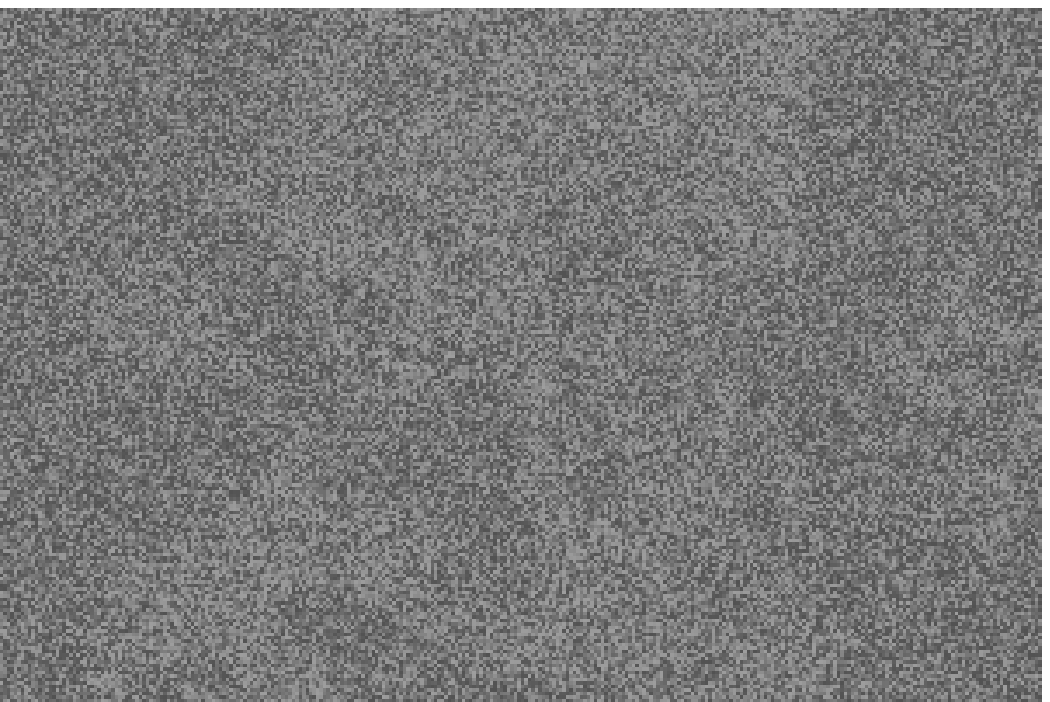} & \includegraphics[width=1\linewidth]{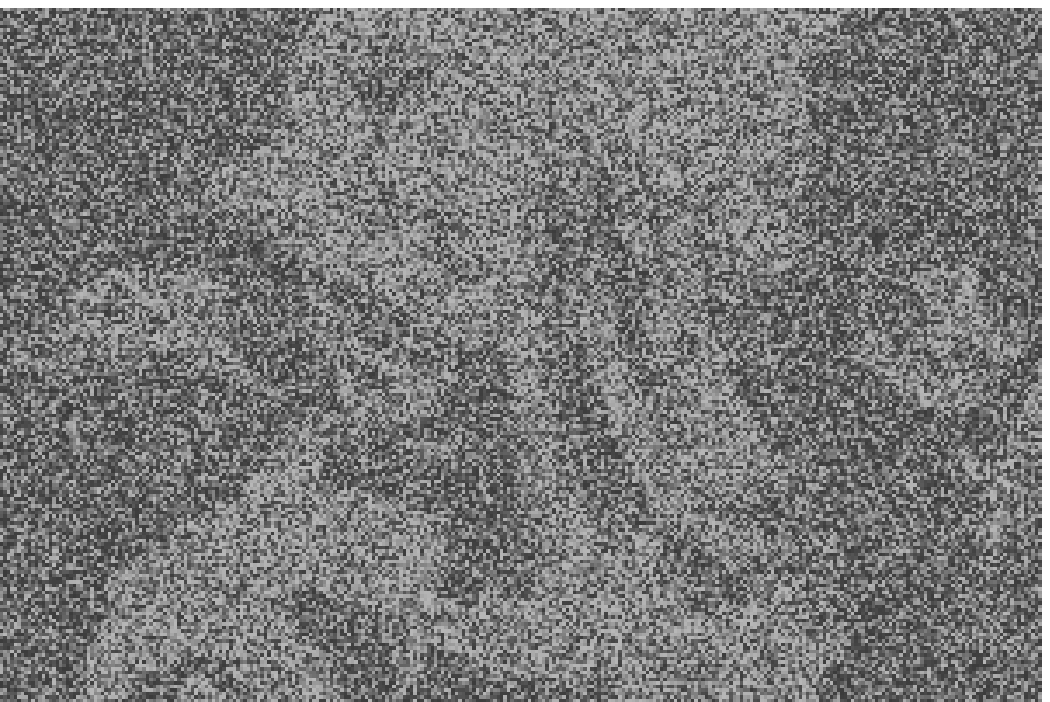} & \includegraphics[width=1\linewidth]{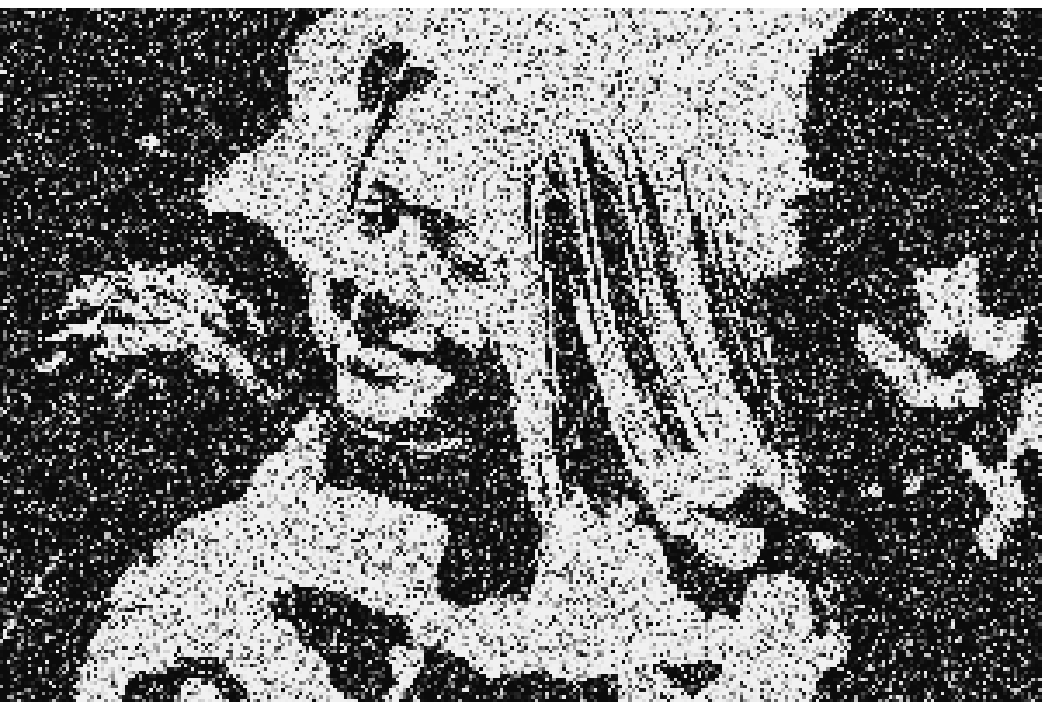}\tabularnewline
\end{tabular}
\par\end{centering}

\caption{Predicted marginals for an example binary denoising test image with
different noise levels $n$.\label{fig:Predicted-marginals-for}}
\end{figure}

In a first experiment, we create a binary denoising problem using
the Berkeley Segmentation Dataset. Here, we took 132 $200\times300$
images from the Berkeley segmentation dataset, binarized them according
to if each pixel is above the image mean. The noisy input values are
then generated as $y_{i}=x_{i}(1\lyxmathsym{\textminus}t_{i}^{n})+(1\lyxmathsym{\textminus}x_{i})t_{i}^{n}$,
where $x_{i}$ is the true binary label, and $t_{i}\in[0,1]$ is random.
Here, $n\in(1,\infty)$ is the noise level, where lower values correspond
to more noise. Thirty-two images were used for training, and 100 for
testing. This is something of a toy problem, but the ability to systematically
vary the noise level is illustrative.

As unary features ${\bf u}({\bf y},i)$, we use only two features:
a constant of 1, and the noisy input value at that pixel.

For edge features ${\bf v}({\bf y},i,j)$, we also use two features:
one indicating that $(i,j)$ is a horizontal edge, and one indicating
that $(i,j)$ is a vertical edge. The effect is that vertical and
horizontal edges have independent parameters.

For learning, we use full back TRW and back mean field (without message-storing
or truncation) for marginal-based loss functions, and the surrogate
likelihood with the gradient computed in the direct form (Eq. \ref{eq:likelihood-gradient}).
In all cases, a threshold on inference of $10^{-4}$ is used.

Error rates are shown in Tab. \ref{tab:Binary-Denoising-error}, while
predicted marginals for an example test image are shown in Fig. \ref{fig:Predicted-marginals-for}.
We compare against an independent model, which can be seen as truncated
fitting with zero iterations or, equivalently, logistic regression
at the pixel level. We see that for low noise levels, all methods
perform well, while for high noise levels, the marginal-based losses
outperform the surrogate likelihood and pseudolikelihood by a considerable
margin. Our interpretation of this is that model mis-specification
is more pronounced with high noise, and other losses are less robust
to this.

\subsection{Horses}

\begin{table}
\begin{centering}
\setlength{\tabcolsep}{2.0pt}%
\begin{tabular}{|c|c|c|cc|cc|cc|}
\hline 
 &  &  & \multicolumn{2}{c|}{10 iters} & \multicolumn{2}{c|}{20 iters} & \multicolumn{2}{c|}{40 iters}\tabularnewline
\hline 
Loss & $\lambda$ & Mode & Train & Test & Train & Test & Train & Test\tabularnewline
\hline 
surrogate likelihood & $10^{-3}$ & TRW & .421 & .416 & .088 & .109 & .088 & .109\tabularnewline
\hline 
univariate logistic & $10^{-3}$ & TRW & .065 & .094 & .064 & .093 & .064 & .093\tabularnewline
\hline 
clique logistic & $10^{-3}$ & TRW & .064 & .094 & .062 & .093 & .062 & .092\tabularnewline
\hline 
sm. class. $\alpha=5$ & $10^{-3}$ & TRW & .068 & .097 & .067 & .097 & .067 & .096\tabularnewline
\hline 
sm. class. $\alpha=15$ & $10^{-3}$ & TRW & .065 & .097 & .064 & .096 & .063 & .096\tabularnewline
\hline 
sm. class. $\alpha=50$ & $10^{-3}$ & TRW & .064 & .096 & .063 & .095 & .062 & .095\tabularnewline
\hline 
surrogate likelihood & $10^{-3}$ & MNF & .405 & .383 & .236 & .226 & .199 & .200\tabularnewline
\hline 
univariate logistic & $10^{-3}$ & MNF & .078 & .108 & .077 & .106 & .078 & .106\tabularnewline
\hline 
clique logistic & $10^{-3}$ & MNF & .073 & .101 & .075 & .105 & .079 & .107\tabularnewline
\hline 
\hline 
pseudolikelihood & $10^{-4}$ & TRW & \multicolumn{4}{c|}{} & .222 & .249\tabularnewline
\cline{1-3} \cline{8-9} 
piecewise & $10^{-4}$ & TRW & \multicolumn{4}{c|}{} & .202 & .236\tabularnewline
\cline{1-3} \cline{8-9} 
independent & $10^{-4}$ & \multicolumn{1}{c}{} &  & \multicolumn{1}{c}{} &  &  & .095 & .125\tabularnewline
\hline 
\end{tabular}
\par\end{centering}

\caption{Training and test errors on the horses dataset, using either TRW on
mean-field (MNF) inference. With too few iterations, the surrogate
likelihood diverges.\label{tab:Training-and-test-horses}}
\end{table}
\begin{figure}
\begin{raggedright}
\vspace{-15pt}
\subfloat[Input]{\includegraphics[width=0.32\columnwidth]{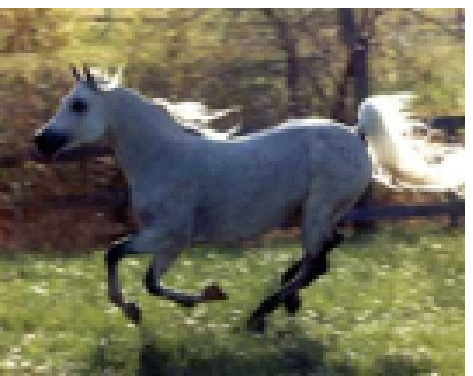}

}\subfloat[True Labels]{\includegraphics[width=0.32\columnwidth]{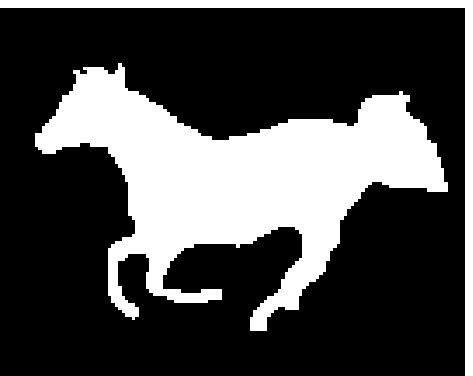}

}
\par\end{raggedright}

\begin{raggedright}
\subfloat[Surr. Like. {\scriptsize TRW}]{\includegraphics[width=0.32\columnwidth]{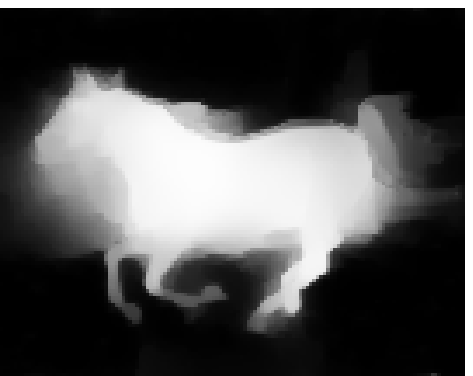}

}\subfloat[U. Logistic {\scriptsize TRW}]{\includegraphics[width=0.32\columnwidth]{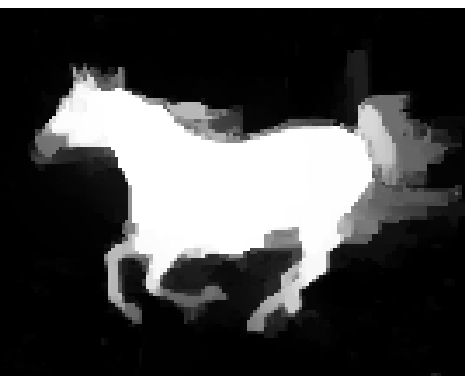}

}\subfloat[Sm. Class {\scriptsize $\lambda$$=$$50$} {\scriptsize TRW}]{\includegraphics[width=0.32\columnwidth]{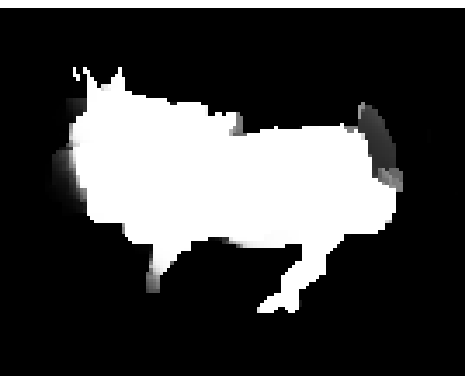}

}
\par\end{raggedright}

\begin{raggedright}
\subfloat[Surr. Like. {\scriptsize MNF}]{\includegraphics[width=0.32\columnwidth]{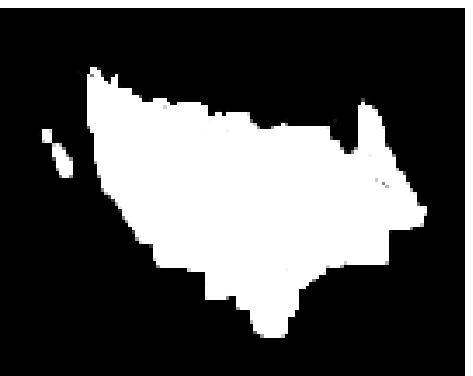}

}\subfloat[U. Logistic {\scriptsize MNF}]{\includegraphics[width=0.32\columnwidth]{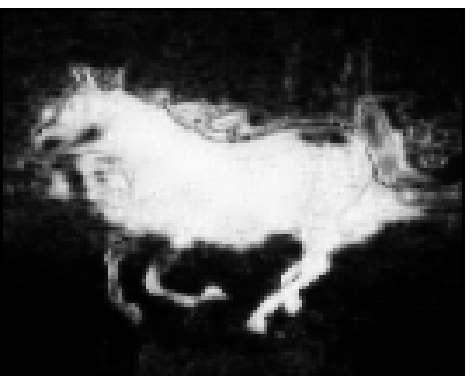}

}\subfloat[Independent]{\includegraphics[width=0.32\columnwidth]{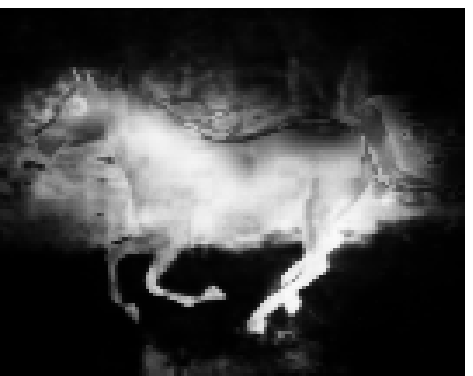}

}
\par\end{raggedright}

\caption{Predicted marginals for a test image from the horses dataset. Truncated
learning uses 40 iterations.\label{fig:Predicted-marginals-horses}}
\end{figure}

Secondly, we use the Weizman horse dataset, consisting of 328 images
of horses at various resolutions. We use 200 for training and 128
for testing. The set of possible labels $x_{i}$ is again binary--
either the pixel is part of a horse or not.

For unary features ${\bf u}({\bf y},i)$, we begin by computing the
RGB intensities of each pixel, along with the normalized vertical
and horizontal positions. We expand these 5 initial features into
a larger set using sinusoidal expansion \cite{ValueFunctionApproximationinReinforcementLearning}.
Specifically, denote the five original features by ${\bf s}$. Then,
we include the features $\sin({\bf c}\cdot{\bf s})$ and $\cos({\bf c}\cdot{\bf s})$
for all binary vectors ${\bf c}$ of the appropriate length. This
results in an expanded set of $64$ features. To these we append a
36-component Histogram of Gradients \cite{HistogramsOfOrientedGradients},
for a total of 100 features.

For edge features between $i$ and $j$, we use a set of 21 ``base''
features: A constant of one, the $l_{2}$ norm of the difference of
the RGB values at $i$ and $j$, discretized as above 10 thresholds,
and the maximum response of a Sobel edge filter at $i$ or $j$, again
discretized using 10 thresholds. To generate the final feature vector
${\bf v}({\bf y},i,j)$, this is increased into a set of 42 features.
If $(i,j)$ is a horizontal edge, the first half of these will contain
the base features, and the other half will be zero. If $(i,j)$ is
a vertical edge, the opposite situation occurs. This essentially allows
for different parametrization of vertical and horizontal edges.

In a first experiment, we train models with truncated fitting with
various numbers of iterations. The pseudolikelihood and piecewise
likelihood use a convergence threshold of $10^{-5}$ for testing.
Several trends are visible in Tab. \ref{tab:Training-and-test-horses}.
First, with less than 20 iterations, the truncated surrogate likelihood
diverges, and produces errors around 0.4. Second, TRW consistently
outperforms mean field. Finally, marginal-based loss functions outperform
the others, both in terms of training and test errors. Fig. \ref{fig:Predicted-marginals-horses}
shows predicted marginals for an example test image. On this dataset,
the pseudolikelihood, piecewise likelihood, and surrogate likelihood
based on mean field are outperformed by an independent model, where
each label is predicted by input features independent of all others.

\subsection{Stanford Backgrounds Data}

\begin{table}[h]
\begin{centering}
\setlength{\tabcolsep}{4.1pt}%
\begin{tabular}{|c|c|cc|cc|cc|}
\hline 
 &  & \multicolumn{2}{c|}{5 iters} & \multicolumn{2}{c|}{10 iters} & \multicolumn{2}{c|}{20 iters}\tabularnewline
\hline 
Loss & $\lambda$ & Train & Test & Train & Test & \multicolumn{1}{c|}{Train} & Test\tabularnewline
\hline 
\hline 
surrogate EM & $10^{-3}$ & .876  & .877 & .239  & .249 & .238 & .248\tabularnewline
\hline 
univariate logistic & $10^{-3}$ & .210  & .228 & .202  & .224 & .201 & .223\tabularnewline
\hline 
clique logistic & $10^{-3}$ & .206 & .226 & .198 & .223 & .195 & .221\tabularnewline
\hline 
\hline 
pseudolikelihood & $10^{-4}$ &  & \multicolumn{1}{c}{} &  &  & .516 & .519\tabularnewline
\cline{1-2} \cline{7-8} 
piecewise & $10^{-4}$ &  & \multicolumn{1}{c}{} &  &  & .335 & .341\tabularnewline
\cline{1-2} \cline{7-8} 
independent & $10^{-4}$ &  & \multicolumn{1}{c}{} &  &  & .293 & .299\tabularnewline
\hline 
\end{tabular}
\par\end{centering}

\centering{}\caption{Test errors on the backgrounds dataset using TRW inference. With too
few iterations, surrogate EM diverges, leading to very high error
rates.}
\end{table}
\begin{figure}[h]
\vspace{-10pt}
\subfloat[Input Image]{\includegraphics[width=0.32\columnwidth]{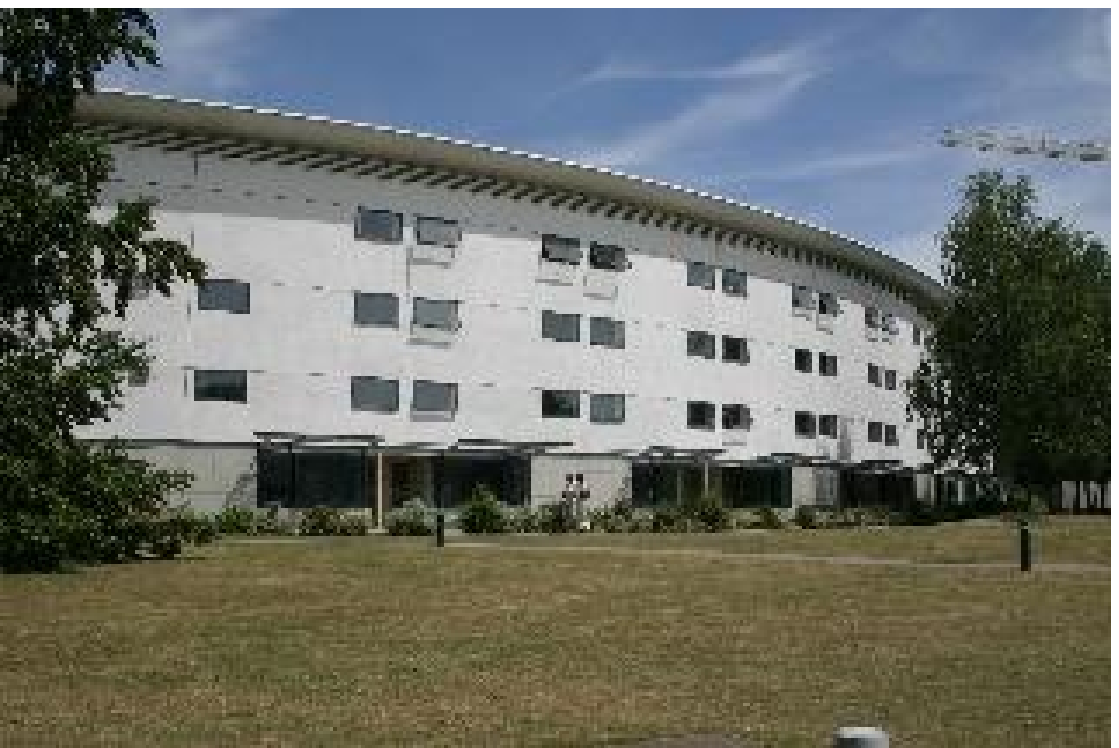}

}\subfloat[True Labels]{\includegraphics[width=0.32\columnwidth]{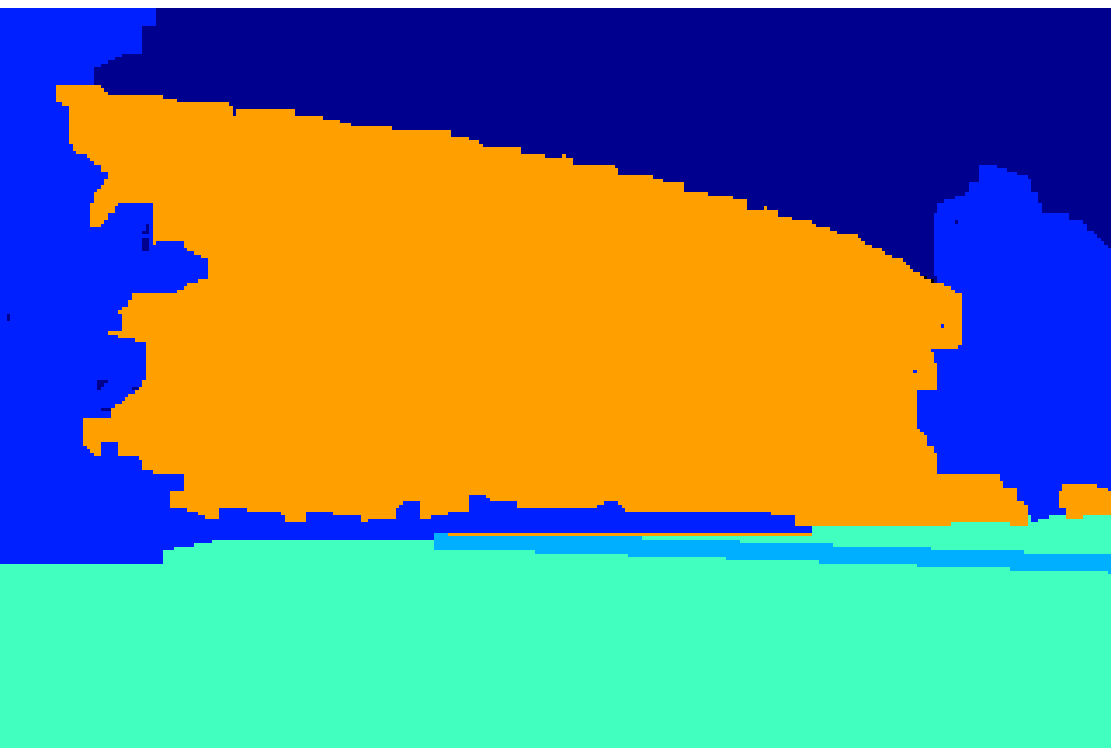}

}

\subfloat[Surrogate EM]{\includegraphics[width=0.32\columnwidth]{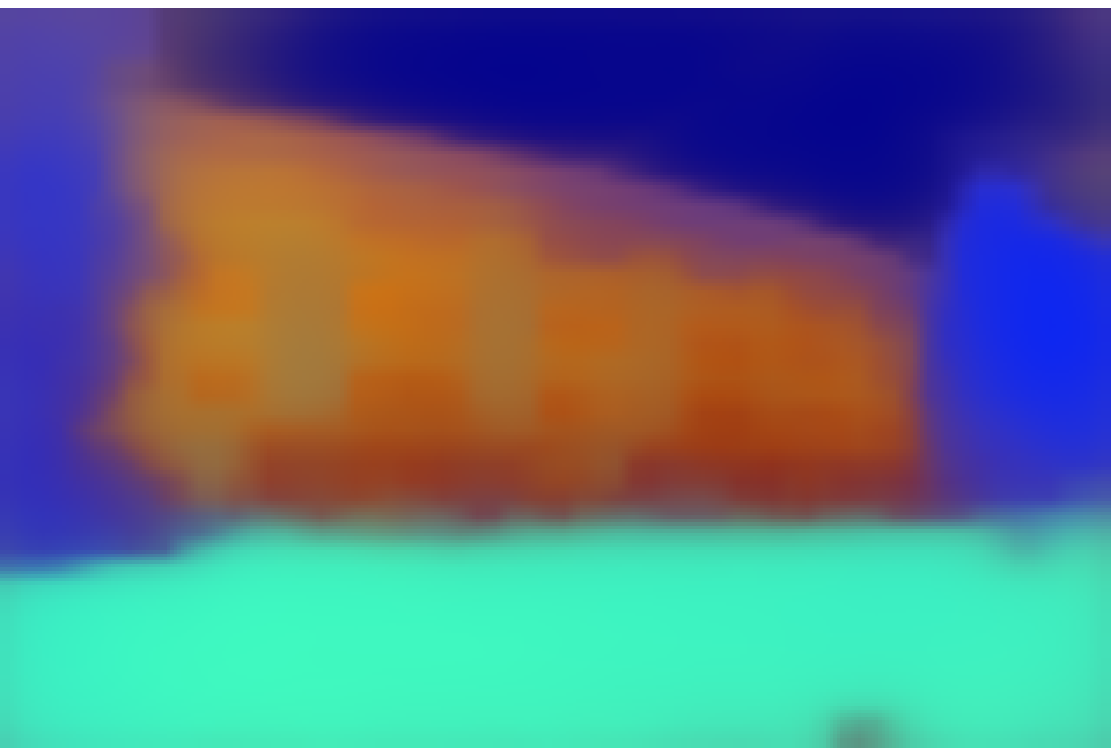}

}\subfloat[Univ. Logistic]{\includegraphics[width=0.32\columnwidth]{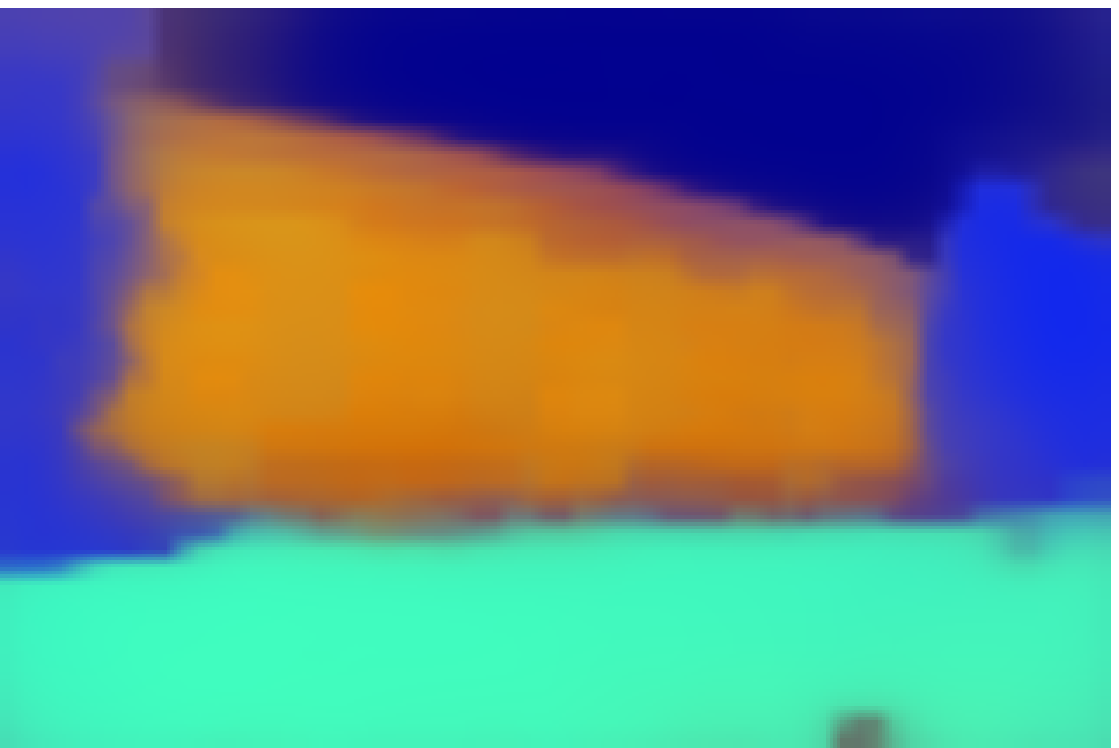}

}\subfloat[Clique Logistic]{\includegraphics[width=0.32\columnwidth]{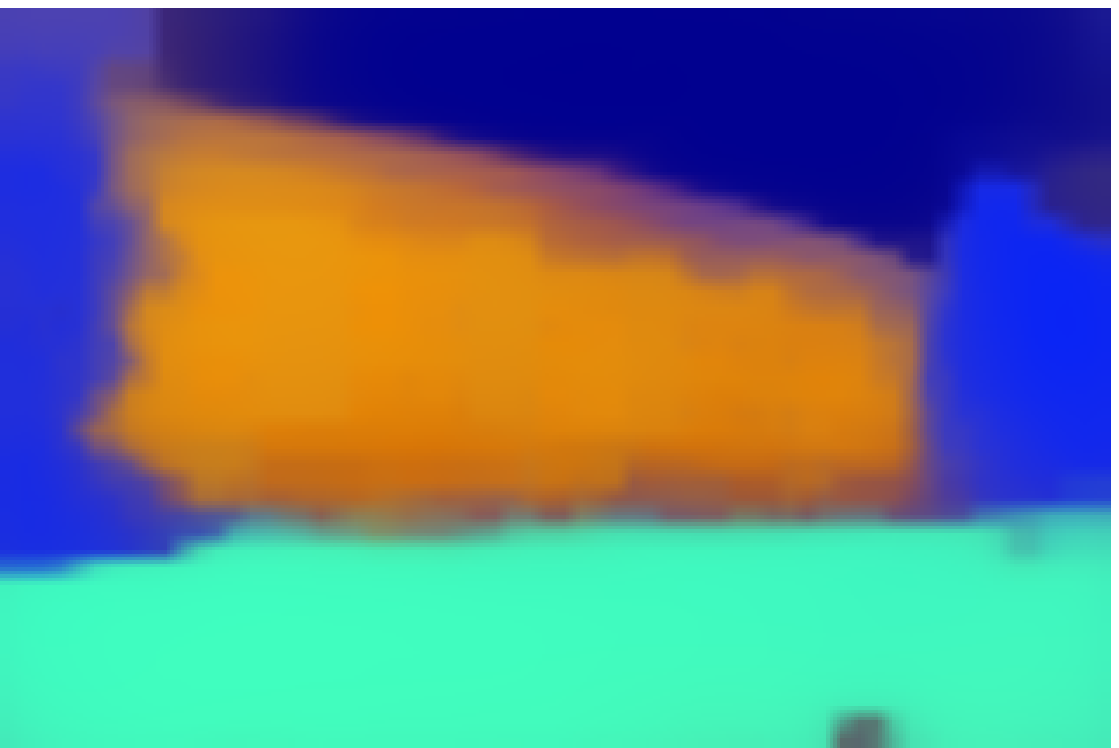}

}

\subfloat[Pseudolikelihood]{\includegraphics[width=0.32\columnwidth]{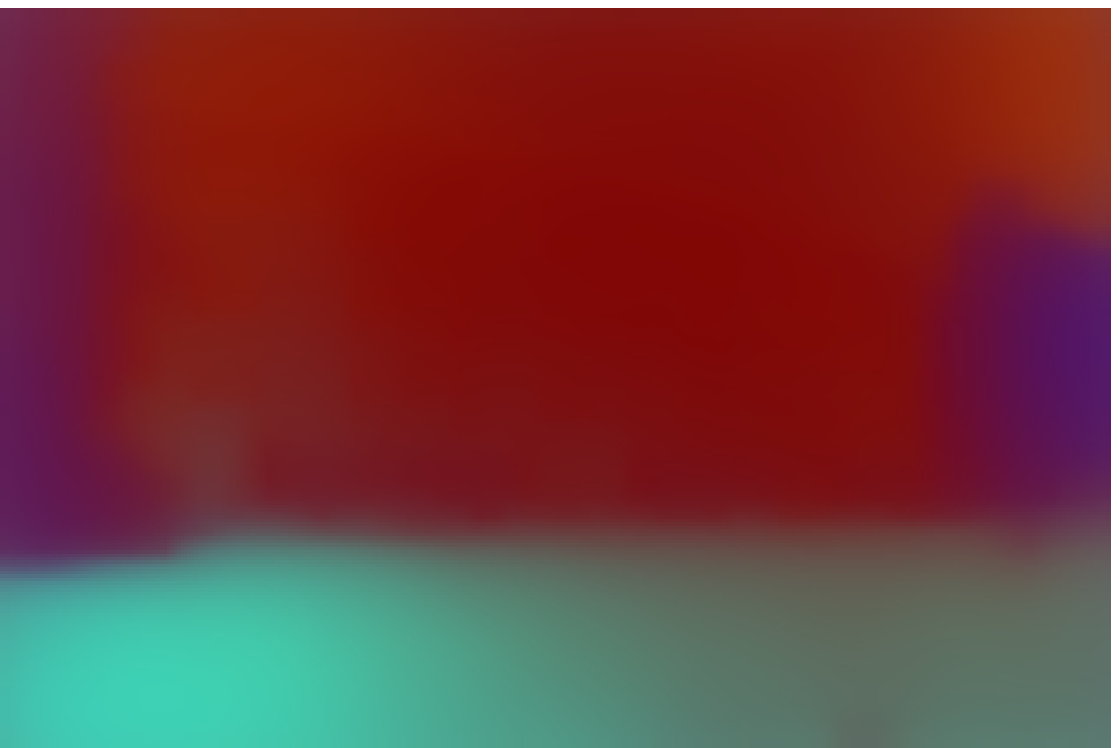}

}\subfloat[Piecewise]{\includegraphics[width=0.32\columnwidth]{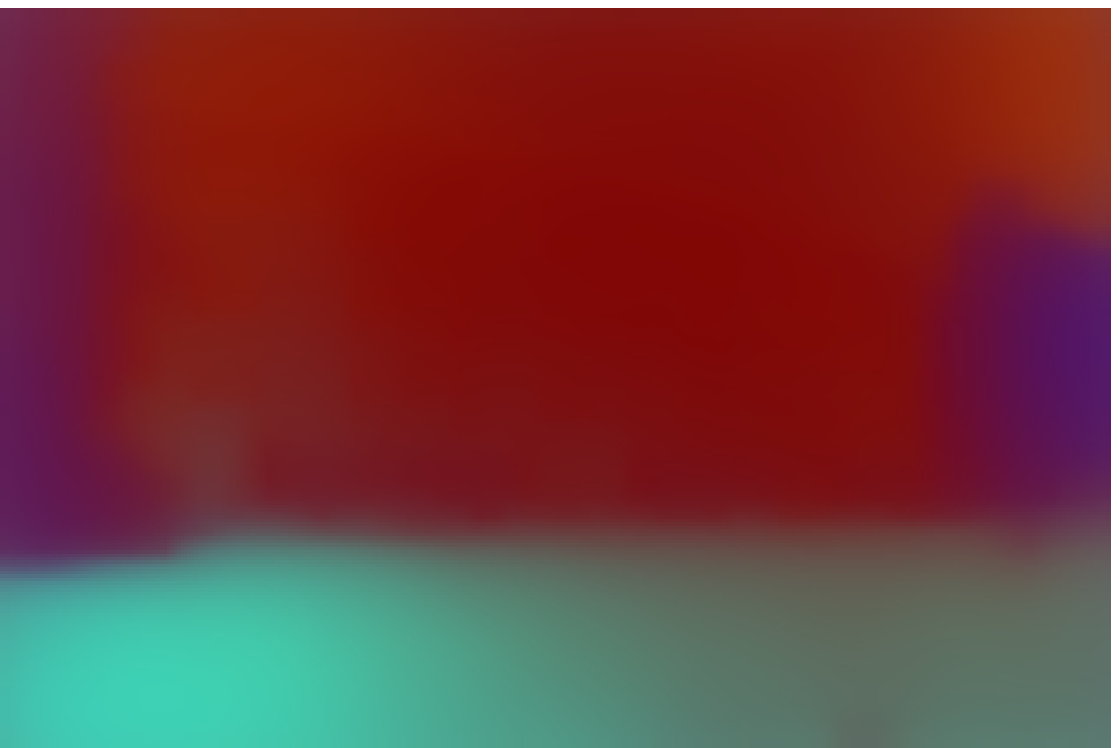}

}\subfloat[Independent]{\includegraphics[width=0.32\columnwidth]{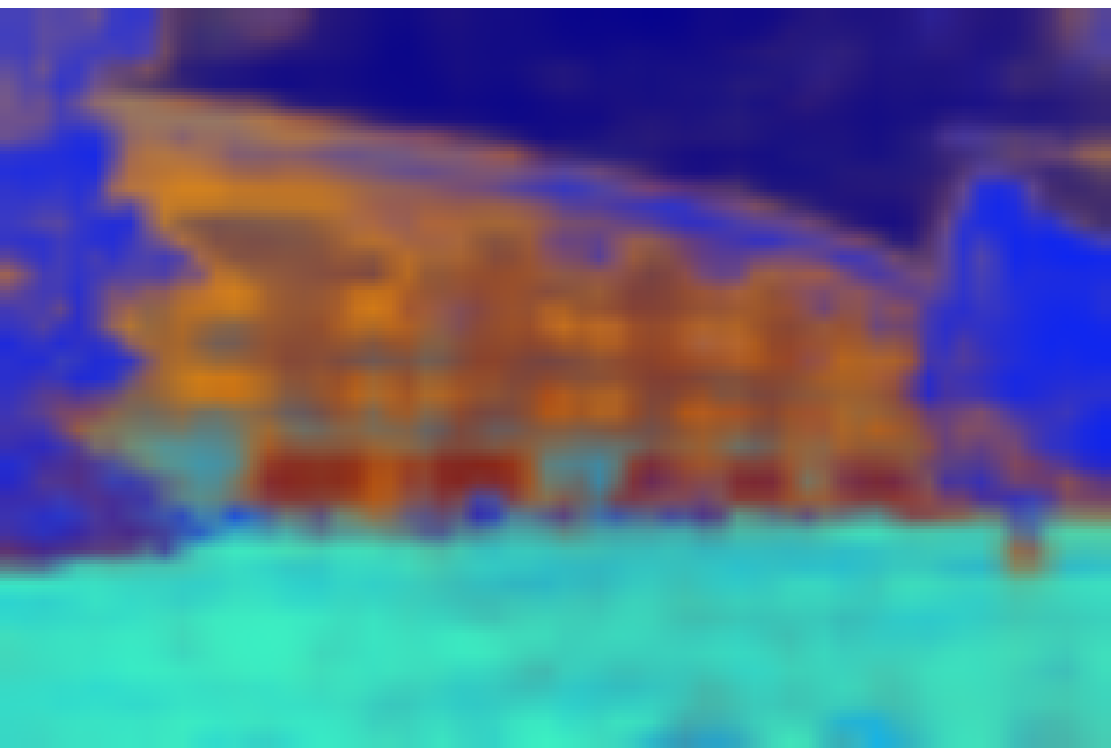}

}

\caption{Example marginals from the backgrounds dataset using $20$ iterations
for truncated fitting. \label{fig:Stanford-Background-Results}}
\end{figure}

Our final experiments consider the Stanford backgrounds dataset. This
consists of 715 images of resolution approximately $240\times320$.
Most pixels are labeled as one of eight classes, with some unlabeled.

The unary features ${\bf u}({\bf y},i)$ we use here are identical
to those for the horses dataset. In preliminary experiments, we tried
training models with various resolutions. We found that reducing resolution
to 20\% of the original after computing features, then upsampling
the predicted marginals yielded significantly better results than
using the original resolution. Hence, this is done for all the following
experiments. Edge features are identical to those for the horses dataset,
except only based on the difference of RGB intensities, meaning 22
total edge features ${\bf v}({\bf y},i,j)$.

In a first experiment, we compare the performance of truncated fitting,
perturbation, and back-propagation, using 100 images from this dataset
for speed. We train with varying thresholds for perturbation and back-propagation,
while for truncated learning, we use various numbers of iterations.
All models are trained with TRW to fit the univariate logistic loss.
If a bad search-direction is encountered, L-BFGS is re-initialized. 
Results are shown in Fig. \ref{fig:pert_bprop-trunc}. We see that
with loose thresholds, perturbation and back-propagation experience
learning failure at sub-optimal solutions. Truncated fitting is far
more successful; using more iterations is slower to fit, but leads
to better performance at convergence.

\begin{figure}
\includegraphics[bb=18bp 180bp 550bp 616bp,clip,scale=0.175]{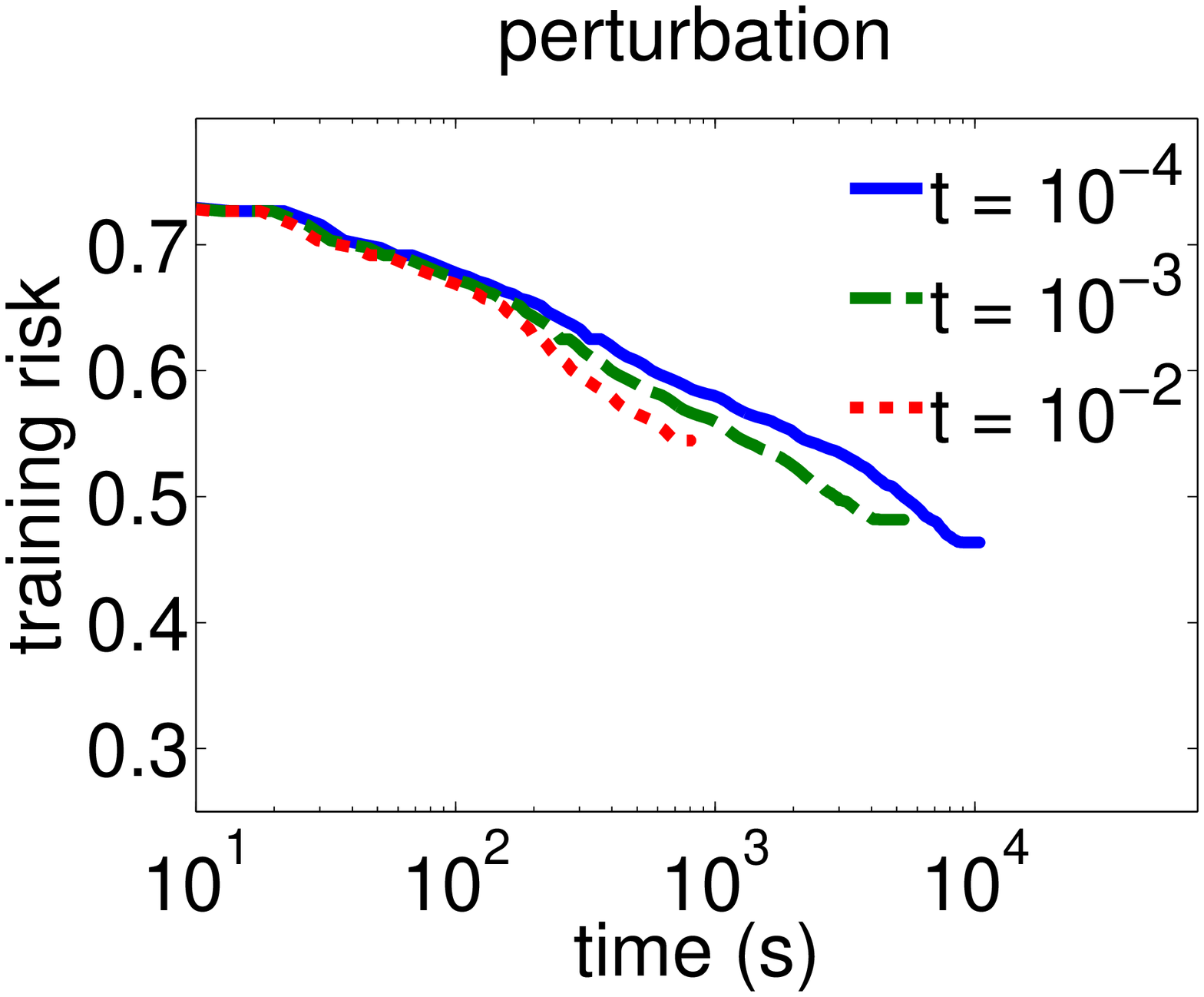}\includegraphics[bb=100bp 180bp 550bp 616bp,clip,scale=0.175]{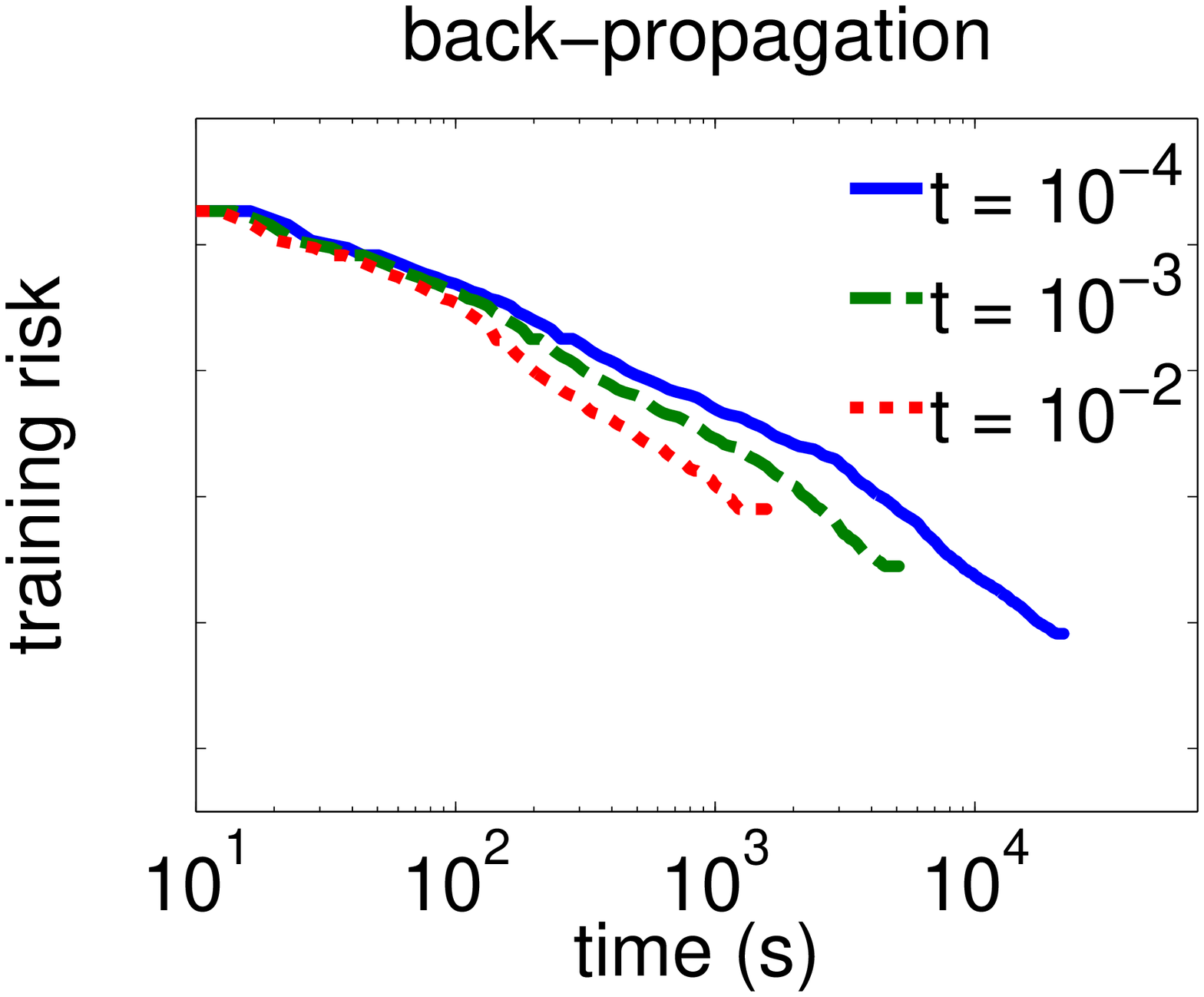}\includegraphics[bb=100bp 180bp 550bp 616bp,clip,scale=0.175]{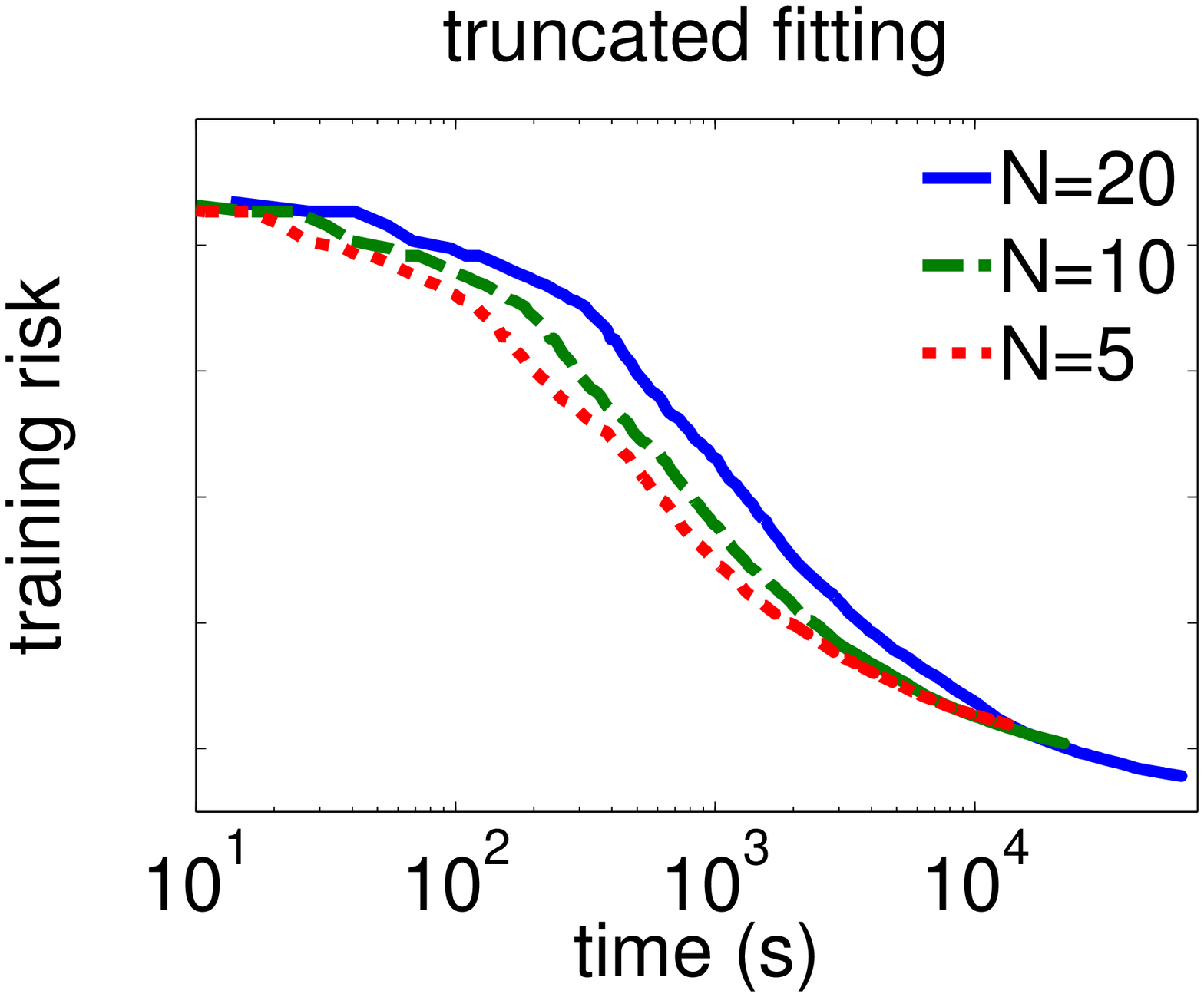}

\caption{Comparison of different learning methods on the backgrounds dataset
with 100 images. All use an 8-core 2.26 Ghz PC.\label{fig:pert_bprop-trunc} }
\vspace{-15pt}
\end{figure}

In a second experiment, we train on the entire dataset, with errors
estimated using five-fold cross validation. Here, an incremental procedure
was used, where first a subset of 32 images was trained on, with 1000
learning iterations. The number of images was repeatedly doubled,
with the number of learning iterations halved. In practice this reduced
training time substantially. Results are shown in Fig. \ref{fig:Stanford-Background-Results}.
These results use a ridge regularization penalty of $\lambda$ on
all parameters. (This is relative to the empirical risk, as measured
per pixel.) For EM, and marginal based loss functions, we set this
as $\lambda=10^{-3}$. We found in preliminary experiments that using
a smaller regularization constant caused truncated EM to diverge even
with 10 iterations. The pseudolikelihood and piecewise benefit from
less regularization, and so we use $\lambda=10^{-4}$ there. Again
the marginal based loss functions outperform others. In particular,
they also perform quite well even with $5$ iterations, where truncated
EM diverges.

\section{Conclusions}

Training parameters of graphical models in a high treewidth setting
involves several challenges. In this paper, we focus on three: model
mis-specification, the necessity of approximate inference, and computational
complexity.

The main technical contribution of this paper is several methods for
training based on the marginals predicted by a given approximate inference
algorithm. These methods take into account model mis-specification
and inference approximations. To combat computational complexity,
we introduce ``truncated'' learning, where the inference algorithm
only needs to be run for a fixed number of iterations. Truncation
can also be applied, somewhat heuristically, to the surrogate likelihood.

Among previous methods, we experimentally find the surrogate likelihood
to outperform the pseudolikelihood or piecewise learning. By more
closely reflecting the test criterion of Hamming loss, marginal-based
loss functions perform still better, particularly on harder problems
(Though the surrogate likelihood generally displays smaller train/test
gaps.) Additionally marginal-based loss functions are more amenable
to truncation, as the surrogate likelihood diverges with too few iterations.

\bibliographystyle{IEEEtran}

\begin{thebibliography}{10}
\providecommand{\url}[1]{#1}
\csname url@samestyle\endcsname
\providecommand{\newblock}{\relax}
\providecommand{\bibinfo}[2]{#2}
\providecommand{\BIBentrySTDinterwordspacing}{\spaceskip=0pt\relax}
\providecommand{\BIBentryALTinterwordstretchfactor}{4}
\providecommand{\BIBentryALTinterwordspacing}{\spaceskip=\fontdimen2\font plus
\BIBentryALTinterwordstretchfactor\fontdimen3\font minus
  \fontdimen4\font\relax}
\providecommand{\BIBforeignlanguage}[2]{{%
\expandafter\ifx\csname l@#1\endcsname\relax
\typeout{** WARNING: IEEEtran.bst: No hyphenation pattern has been}%
\typeout{** loaded for the language `#1'. Using the pattern for}%
\typeout{** the default language instead.}%
\else
\language=\csname l@#1\endcsname
\fi
#2}}
\providecommand{\BIBdecl}{\relax}
\BIBdecl

\bibitem{SpatialInteractionAndTheStatistical}
J.~Besag, ``Spatial interaction and the statistical analysis of lattice
  systems,'' \emph{Journal of the Royal Statistical Society. Series B
  (Methodological)}, vol.~36, no.~2, pp. 192--236, 1974.

\bibitem{ConditionalRandomFields}
J.~Lafferty, A.~McCallum, and F.~Pereira, ``Conditional random fields:
  Probabilistic models for segmenting and labeling sequence data,'' in
  \emph{ICML}, 2001.

\bibitem{ModelDistortionsInBayesian}
M.~Nikolova, ``Model distortions in bayesian {MAP} reconstruction,''
  \emph{Inverse Problems and Imaging}, vol.~1, no.~2, pp. 399--422, 2007.

\bibitem{ProbabilisticSolutionOfIllPosedProblems}
J.~Marroquin, S.~Mitter, and T.~Poggio, ``Probabilistic solution of ill-posed
  problems in computational vision,'' \emph{Journal of the American Statistical
  Association}, vol.~82, no. 397, pp. 76--89, 1987.

\bibitem{TrainingCRFsForMaximumLabelwiseAccuracy}
S.~S. Gross, O.~Russakovsky, C.~B. Do, and S.~Batzoglou, ``Training conditional
  random fields for maximum labelwise accuracy,'' in \emph{NIPS}, 2007.

\bibitem{kumar_exploiting}
S.~Kumar, J.~August, and M.~Hebert, ``Exploiting inference for approximate
  parameter learning in discriminative fields: An empirical study,'' in
  \emph{EMMCVPR}, 2005.

\bibitem{MeasuringUnvertaintyInGraphCutSolutions}
P.~Kohli and P.~Torr, ``{Measuring uncertainty in graph cut solutions},''
  \emph{Computer Vision and Image Understanding}, vol. 112, no.~1, pp. 30--38,
  2008.

\bibitem{WainwrightJordanMonster}
M.~Wainwright and M.~Jordan, ``Graphical models, exponential families, and
  variational inference,'' \emph{Found. Trends Mach. Learn.}, vol.~1, no. 1-2,
  pp. 1--305, 2008.

\bibitem{meltzer_et_al}
T.~Meltzer, A.~Globerson, and Y.~Weiss, ``Convergent message passing algorithms
  - a unifying view,'' in \emph{UAI}, 2009.

\bibitem{StructuredLearningAndPredictionInComputerVision}
S.~Nowozin and C.~H. Lampert, ``Structured learning and prediction in computer
  vision,'' \emph{Foundations and Trends in Computer Graphics and Vision},
  vol.~6, pp. 185--365, 2011.

\bibitem{MathematicalMethodsOfStatistics}
H.~Cram{\'e}r, \emph{Mathematical methods of statistics}.\hskip 1em plus 0.5em
  minus 0.4em\relax Princeton University Press, 1999.

\bibitem{OnParameterLearninginCRFbased}
S.~Nowozin, P.~V. Gehler, and C.~H. Lampert, ``On parameter learning in
  {CRF}-based approaches to object class image segmentation,'' in \emph{ECCV},
  2010.

\bibitem{LearningFlexibleFeatures}
L.~Stewart, X.~He, and R.~S. Zemel, ``Learning flexible features for
  conditional random fields,'' \emph{IEEE Trans. Pattern Anal. Mach. Intell.},
  vol.~30, no.~8, pp. 1415--1426, 2008.

\bibitem{MCMCML}
C.~Geyer, ``Markov chain monte carlo maximum likelihood,'' in \emph{Symposium
  on the Interface}, 1991.

\bibitem{OnConstrastiveDivergenceLearning}
M.~Carreira-Perpinan and G.~Hinton, ``On contrastive divergence learning,'' in
  \emph{AISTATS}, 2005.

\bibitem{FieldsOfExperts}
S.~Roth and M.~J. Black, ``Fields of experts,'' \emph{International Journal of
  Computer Vision}, vol.~82, no.~2, pp. 205--229, 2009.

\bibitem{EstimatingTheWrong}
M.~J. Wainwright, ``Estimating the ``wrong'' graphical model: benefits in the
  computation-limited setting,'' \emph{Journal of Machine Learning Research},
  vol.~7, pp. 1829--1859, 2006.

\bibitem{EfficientlyLearningRandomFields}
J.~J. Weinman, L.~C. Tran, and C.~J. Pal, ``Efficiently learning random fields
  for stereo vision with sparse message passing,'' in \emph{ECCV}, 2008, pp.
  617--630.

\bibitem{RandomFieldModelForIntegration}
T.~Toyoda and O.~Hasegawa, ``Random field model for integration of local
  information and global information,'' \emph{IEEE Trans. Pattern Anal. Mach.
  Intell.}, vol.~30, no.~8, pp. 1483--1489, 2008.

\bibitem{LearningToCombineBottomUpAndTopDown}
A.~Levin and Y.~Weiss, ``Learning to combine bottom-up and top-down
  segmentation,'' \emph{International Journal of Computer Vision}, vol.~81,
  no.~1, pp. 105--118, 2009.

\bibitem{ExploitingInferenceForApproximate}
S.~Kumar, J.~August, and M.~Hebert, ``Exploiting inference for approximate
  parameter learning in discriminative fields: An empirical study,'' in
  \emph{EMMCVPR}, 2005.

\bibitem{FigureGroundAssignment}
X.~Ren, C.~Fowlkes, and J.~Malik, ``Figure/ground assignment in natural
  images,'' in \emph{ECCV}, 2006.

\bibitem{AcceleratedTrainingofCRFs}
S.~V.~N. Vishwanathan, N.~N. Schraudolph, M.~W. Schmidt, and K.~P. Murphy,
  ``Accelerated training of conditional random fields with stochastic gradient
  methods,'' in \emph{ICML}, 2006.

\bibitem{LearningProbabilisticModels}
X.~Ren, C.~Fowlkes, and J.~Malik, ``Learning probabilistic models for contour
  completion in natural images,'' \emph{International Journal of Computer
  Vision}, vol.~77, no. 1-3, pp. 47--63, 2008.

\bibitem{SceneUnderstandingWithDiscriminative}
J.~Yuan, J.~Li, and B.~Zhang, ``Scene understanding with discriminative
  structured prediction,'' in \emph{CVPR}, 2008.

\bibitem{SceneSegmentationWithCRFsLearned}
J.~J. Verbeek and B.~Triggs, ``Scene segmentation with crfs learned from
  partially labeled images,'' in \emph{NIPS}, 2007.

\bibitem{LearningConditionalRandomFieldsForStereo}
D.~Scharstein and C.~Pal, ``Learning conditional random fields for stereo,'' in
  \emph{CVPR}, 2007.

\bibitem{UsingCombinationOfStatisticalModelsAndMultilevel}
P.~Zhong and R.~Wang, ``Using combination of statistical models and multilevel
  structural information for detecting urban areas from a single gray-level
  image,'' \emph{IEEE T. Geoscience and Remote Sensing}, vol.~45, no. 5-2, pp.
  1469--1482, 2007.

\bibitem{StatisticalAnalysis}
J.~Besag, ``Statistical analysis of non-lattice data,'' \emph{Journal of the
  Royal Statistical Society. Series D (The Statistician)}, vol.~24, no.~3, pp.
  179--195, 1975.

\bibitem{MultiscaleConditionalRandomFieldsFor}
X.~He, R.~S. Zemel, and M.~{\'A}. Carreira-Perpi{\~n}{\'a}n, ``Multiscale
  conditional random fields for image labeling,'' in \emph{CVPR}, 2004.

\bibitem{DiscriminativeRandomFields}
S.~Kumar and M.~Hebert, ``Discriminative random fields,'' \emph{International
  Journal of Computer Vision}, vol.~68, no.~2, pp. 179--201, 2006.

\bibitem{LearningInGibbsianFieldsHowAccurate}
S.~C. Zhu and X.~Liu, ``Learning in gibbsian fields: How accurate and how fast
  can it be?'' \emph{IEEE Transactions on Pattern Analysis and Machine
  Intelligence}, vol.~24, pp. 1001--1006, 2002.

\bibitem{PiecewiseTrainingForUndirectedModels}
C.~Sutton and A.~McCallum, ``Piecewise training for undirected models,'' in
  \emph{UAI}, 2005.

\bibitem{RobustModelBasedSceneInterpretation}
S.~Kim and I.-S. Kweon, ``Robust model-based scene interpretation by
  multilayered context information,'' \emph{Computer Vision and Image
  Understanding}, vol. 105, no.~3, pp. 167--187, 2007.

\bibitem{TextonBoostForImageUnderstanding}
J.~Shotton, J.~M. Winn, C.~Rother, and A.~Criminisi, ``Textonboost for image
  understanding: Multi-class object recognition and segmentation by jointly
  modeling texture, layout, and context,'' \emph{Int. J. of Comput. Vision},
  vol.~81, no.~1, pp. 2--23, 2009.

\bibitem{EmpiricalRiskMinimizationofGraphicalModel}
V.~Stoyanov, A.~Ropson, and J.~Eisner, ``Empirical risk minimization of
  graphical model parameters given approximate inference, decoding, and model
  structure,'' in \emph{AISTATS}, 2011.

\bibitem{LearningConvexInference}
J.~Domke, ``Learning convex inference of marginals,'' in \emph{UAI}, 2008.

\bibitem{BahlEtAl}
L.~R. Bahl, P.~F. Bron, P.~V. de~Souza, and R.~L. Mercer, ``A new algorithm for
  the estimation of hidden markov model parameters,'' in \emph{ICASSP}, 1988.

\bibitem{AlternativeObjectiveFunctionFor}
S.~Kakade, Y.~W. Teh, and S.~Roweis, ``An alternate objective function for
  {M}arkovian fields,'' in \emph{ICML}, 2002.

\bibitem{CompositeLikelihoods}
B.~G. Lindsay, ``Composite likelihood methods,'' \emph{Contemporary
  Mathematics}, vol.~80, pp. 221--239, 1988.

\bibitem{Domke}
J.~Domke, ``Learning convex inference of marginals,'' in \emph{UAI}, 2008.

\bibitem{DiscriminativeModelsForMultiClassObjectLayout}
C.~Desai, D.~Ramanan, and C.~C. Fowlkes, ``Discriminative models for
  multi-class object layout,'' \emph{International Journal of Computer Vision},
  vol.~95, no.~1, pp. 1--12, 2011.

\bibitem{LearningCRFsUsingGraphCuts}
M.~Szummer, P.~Kohli, and D.~Hoiem, ``Learning {CRF}s using graph cuts,'' in
  \emph{ECCV}, 2008.

\bibitem{HierarchicalImageRegionLabeling}
J.~J. McAuley, T.~E. de~Campos, G.~Csurka, and F.~Perronnin, ``Hierarchical
  image-region labeling via structured learning,'' in \emph{BMVC}, 2009.

\bibitem{SceneSegmentationViaLowDimensional}
W.~Yang, B.~Triggs, D.~Dai, and G.-S. Xia, ``Scene segmentation via
  low-dimensional semantic representation and conditional random field,'' HAL,
  Tech. Rep., 2009.

\bibitem{ImplicitDifferentiationByPerturbation}
J.~Domke, ``Implicit differentiation by perturbation,'' in \emph{NIPS}, 2010.

\bibitem{ApproximateSolutionMethods}
A.~Boresi and K.~Chong, \emph{Approximate Solution Methods in Engineering
  Mechanics}.\hskip 1em plus 0.5em minus 0.4em\relax Elsevier Science Inc.,
  1991.

\bibitem{AcceleratedConjugateGradientAlgorithm}
N.~Andrei, ``Accelerated conjugate gradient algorithm with finite difference
  hessian/vector product approximation for unconstrained optimization,''
  \emph{J. Comput. Appl. Math.}, vol. 230, no.~2, pp. 570--582, 2009.

\bibitem{NocedalWright}
J.~Nocedal and S.~J. Wright, \emph{Numerical Optimization}, 2nd~ed.\hskip 1em
  plus 0.5em minus 0.4em\relax Springer, 2006.

\bibitem{WellingTeh}
M.~Welling and Y.~W. Teh, ``Linear response algorithms for approximate
  inference in graphical models,'' \emph{Neural Computation}, vol.~16, pp.
  197--221, 2004.

\bibitem{DomkeCVPR2011}
J.~Domke, ``Parameter learning with truncated message-passing,'' in
  \emph{CVPR}, 2011.

\bibitem{stoyanov-eisner-2012-naacl}
V.~Stoyanov and J.~Eisner, ``Minimum-risk training of approximate crf-based nlp
  systems,'' in \emph{Proceedings of NAACL-HLT}.

\bibitem{Eaton_Ghahramani}
F.~Eaton and Z.~Ghahramani, ``Choosing a variable to clamp,'' in
  \emph{AISTATS}, 2009.

\bibitem{ValueFunctionApproximationinReinforcementLearning}
G.~Konidaris, S.~Osentoski, and P.~Thomas, ``Value function approximation in
  reinforcement leanring using the fourier basis,'' in \emph{AAAI}, 2011.

\bibitem{HistogramsOfOrientedGradients}
N.~Dalal and B.~Triggs, ``Histograms of oriented gradients for human
  detection,'' in \emph{CVPR}, 2005.

\end{thebibliography}

\section{Biography}

Justin Domke obtained a PhD degree in Computer Science from the University
of Maryland, College Park in 2009. From 2009 to 2012, he was an Assistant
Professor at Rochester Institute of Technology. Since 2012, he is
a member of the Machine Learning group at NICTA.

\clearpage{}

\newpage{}

\section{Appendix A: Variational Inference}
\begin{thm*}
[Exact variational principle]The log-partition function can also
be represented as

\[
A(\boldsymbol{\theta})=\max_{\boldsymbol{\mu}\in\mathcal{M}}\boldsymbol{\theta}\cdot\boldsymbol{\mu}+H(\boldsymbol{\mu}),
\]
where 
\[
\mathcal{M}=\{\boldsymbol{\mu}':\exists\boldsymbol{\theta},\boldsymbol{\mu}'=\boldsymbol{\mu}(\boldsymbol{\theta})\}
\]
is the marginal polytope, and 
\[
H(\boldsymbol{\mu})=-\sum_{{\bf x}}p({\bf x};\boldsymbol{\theta}(\boldsymbol{\mu}))\log p({\bf x};\boldsymbol{\theta}(\boldsymbol{\mu}))
\]

\noindent is the entropy.\end{thm*}
\begin{proof}
[Proof of the exact variational principle]As $A$ is convex, we have
that
\[
A(\boldsymbol{\theta})=\sup_{\boldsymbol{\mu}}\boldsymbol{\theta}\cdot\boldsymbol{\mu}-A^{*}(\boldsymbol{\mu})
\]
where
\[
A^{*}(\boldsymbol{\mu})=\inf_{\boldsymbol{\theta}}\boldsymbol{\theta}\cdot\boldsymbol{\mu}-A(\boldsymbol{\theta})
\]
is the conjugate dual.

Now, since $dA/d\boldsymbol{\theta}=\boldsymbol{\mu}(\boldsymbol{\theta})$,
if $\boldsymbol{\mu}\not\in\mathcal{M},$ then the infimum for $A^{*}$
is unbounded above. For $\boldsymbol{\mu}\in\mathcal{M}$, the infimum
will be obtained at $\boldsymbol{\theta}(\boldsymbol{\mu})$. Thus
\[
A^{*}(\boldsymbol{\mu})=\begin{cases}
\infty & \boldsymbol{\mu}\not\in\mathcal{M}\\
\boldsymbol{\theta}(\boldsymbol{\mu})\cdot\boldsymbol{\mu}-A(\boldsymbol{\theta}(\boldsymbol{\mu})) & \boldsymbol{\mu}\in\mathcal{M}
\end{cases}.
\]

Now, for $\boldsymbol{\mu}\in\mathcal{M},$ we can see by substitution
that

\begin{align*}
A^{*}(\boldsymbol{\mu})= & \boldsymbol{\theta}(\boldsymbol{\mu})\cdot\sum_{{\bf x}}p({\bf x};\boldsymbol{\theta}(\boldsymbol{\mu})){\bf f}({\bf x})-A(\boldsymbol{\theta}(\boldsymbol{\mu}))\\
= & \sum_{{\bf x}}p({\bf x};\boldsymbol{\theta}(\boldsymbol{\mu}))(\boldsymbol{\theta}(\boldsymbol{\mu})\cdot{\bf f}({\bf x})-A(\boldsymbol{\theta}(\boldsymbol{\mu})))\\
= & \sum_{{\bf x}}p({\bf x};\boldsymbol{\theta}(\boldsymbol{\mu}))\log p({\bf x};\boldsymbol{\theta}(\boldsymbol{\mu}))=-H(\boldsymbol{\mu}).
\end{align*}

And so, finally, 
\begin{equation}
A^{*}(\boldsymbol{\mu})=\begin{cases}
\infty & \boldsymbol{\mu}\not\in\mathcal{M}\\
-H(\boldsymbol{\mu}) & \boldsymbol{\mu}\in\mathcal{M}
\end{cases},\label{eq:Astar_equals_H-1}
\end{equation}

\noindent which is equivalent to the desired result.\end{proof}
\begin{thm*}
[Mean Field Updates]A local maximum of Eq. \ref{eq:mu-TRW} can be
reached by iterating the updates

\[
\mu(x_{j})\leftarrow\frac{1}{Z}\exp\bigl(\theta(x_{j})+\sum_{c:j\in c}\sum_{{\bf x}_{c\backslash j}}\theta({\bf x}_{c})\prod_{i\in c\backslash j}\mu(x_{i})\bigr),
\]

\noindent where $Z$ is a normalizing factor ensuring that $\sum_{x_{j}}\mu(x_{j})=1$. \end{thm*}
\begin{proof}
[Proof of Mean Field Updates]The first thing to note is that for
$\boldsymbol{\mu}\in\mathcal{F},$ several simplifying results hold,
which are easy to verify, namely

\begin{align*}
p({\bf x};\boldsymbol{\theta}(\boldsymbol{\mu}))= & \prod_{i}\mu(x_{i})\\
H(\boldsymbol{\mu})= & -\sum_{i}\sum_{x_{i}}\mu(x_{i})\log\mu(x_{i}).\\
\mu({\bf x}_{c})= & \prod_{i\in c}\mu(x_{i}).
\end{align*}

Now, let $\tilde{A}$ denote the approximate partition function that
results from solving Eq. \ref{eq:A-variational} with the marginal
polytope replaced by $\mathcal{F}$. By substitution of previous results,
we can see that this reduces to an optimization over univariate marginals
only.

\begin{eqnarray}
\tilde{A}(\boldsymbol{\theta}) & = & \max_{\{\mu(x_{i})\}}\sum_{i}\sum_{x_{i}}\theta(x_{i})\mu(x_{i})+\sum_{c}\sum_{{\bf x}_{c}}\theta({\bf x}_{c})\mu({\bf x}_{c})\nonumber \\
 &  & -\sum_{i}\sum_{x_{i}}\mu(x_{i})\log\mu(x_{i}).\label{eq:meanfield_opt}
\end{eqnarray}
Now, form a Lagrangian, enforcing that $\sum_{x_{j}}\mu(x_{j})=1.$

\begin{eqnarray*}
\mathbb{L} & = & \sum_{j}\sum_{x_{j}}\theta(x_{j})\mu(x_{j})+\sum_{c}\sum_{{\bf x}_{c}}\theta({\bf x}_{c})\prod_{i\in c}\mu(x_{i})\\
 &  & -\sum_{j}\sum_{x_{j}}\mu(x_{j})\log\mu(x_{j})+\sum_{j}\lambda_{j}(1-\sum_{x_{j}}\mu(x_{j})).
\end{eqnarray*}

Setting $d\mathbb{L}/d\mu(x_{j})=0$, solving for $\mu(x_{j})$, we
obtain
\[
\mu(x_{j})\propto\exp\bigl(\theta(x_{j})+\sum_{c:j\in c}\sum_{{\bf x}_{c\backslash j}}\theta({\bf x}_{c})\prod_{i\in c\backslash j}\mu(x_{i})\bigr),
\]

\noindent meaning this is a local minimum. Normalizing by $Z$ gives
the result.

Note that only a local maximum results since the mean-field objective
is non-concave\cite[ Sec. 5.4]{WainwrightJordanMonster}.
\end{proof}
Two preliminary results are needed to prove the TRW entropy bound.
\begin{lem*}
Let $\boldsymbol{\mu}^{\mathcal{G}}$ be the ``projection'' of $\boldsymbol{\mu}$
onto a subgraph $\mathcal{G}$, defined by
\[
\boldsymbol{\mu}^{\mathcal{G}}=\{\mu(x_{i})\,\forall i\}\cup\{\mu({\bf x}_{c})\,\forall c\in\mathcal{G}\}.
\]

Then, for $\boldsymbol{\mu}\in\mathcal{M},$
\[
H(\boldsymbol{\mu}^{\mathcal{G}})\geq H(\boldsymbol{\mu}).
\]
\end{lem*}
\begin{proof}
First, note that, by Eq. \ref{eq:Astar_equals_H-1}, for $\boldsymbol{\mu}\in\mathcal{M},$
\[
H(\boldsymbol{\mu})=-A^{*}(\boldsymbol{\mu})=-\inf_{\boldsymbol{\theta}}(\boldsymbol{\theta}\cdot\boldsymbol{\mu}-A(\boldsymbol{\theta})).
\]

Now, the entropy of $\boldsymbol{\mu}^{\mathcal{G}}$ could be defined
as an infimum only over the parameters $\boldsymbol{\theta}$ corresponding
to the cliques in $\mathcal{G}$. Equivalently, however, we can define
it as a constrained optimization 
\[
H(\boldsymbol{\mu}^{\mathcal{G}})=-\inf_{\substack{\boldsymbol{\theta}:\theta_{c}=0\\
\forall c\not\in\mathcal{G}
}
}\bigl(\boldsymbol{\theta}\cdot\boldsymbol{\mu}^{\mathcal{G}}-A(\boldsymbol{\theta})\bigr).
\]

Since the infimum for $H(\boldsymbol{\mu}^{\mathcal{G}})$ takes place
over a more constrained set, but $\boldsymbol{\mu}$ and $\boldsymbol{\mu}^{\mathcal{G}}$
are identical on all the components where $\boldsymbol{\theta}$ may
be nonzero, we have the result.
\end{proof}
Our next result is that the approximate entropy considered in Eq.
\ref{eq:TRW-entropy} is exact for tree-structured graphs, when $\rho_{c}=1$.
\begin{lem*}
For $\boldsymbol{\mu}\in\mathcal{M}$ for a marginal polytope $\mathcal{M}$
corresponding to a tree-structured graph,
\[
H(\boldsymbol{\mu})=\sum_{i}H(\mu_{i})-\sum_{c}I(\mu_{c}).
\]
\end{lem*}
\begin{proof}
First, note that for any any tree structured distribution can be factored
as
\[
p({\bf x})=\prod_{i}p(x_{i})\prod_{c}\frac{p({\bf x}_{c})}{\prod_{i\in c}p(x_{i})}.
\]

\noindent (This is easily shown by induction.) Now, recall our definition
of $H$:
\[
H(\boldsymbol{\mu})=-\sum_{{\bf x}}p({\bf x};\boldsymbol{\theta}(\boldsymbol{\mu}))\log p({\bf x};\boldsymbol{\theta}(\boldsymbol{\mu}))
\]

Substituting the tree-factorized version of $p$ into the equation
yields
\begin{eqnarray*}
H(\boldsymbol{\mu}) & = & -\sum_{{\bf x}}p({\bf x};\boldsymbol{\theta}(\boldsymbol{\mu}))\log p({\bf x};\boldsymbol{\theta}(\boldsymbol{\mu}))\\
 & = & -\sum_{{\bf x}}\sum_{i}p({\bf x};\boldsymbol{\theta}(\boldsymbol{\mu}))\log p(x_{i};\boldsymbol{\theta}(\boldsymbol{\mu}))\\
 &  & -\sum_{{\bf x}}\sum_{c}p({\bf x};\boldsymbol{\theta}(\boldsymbol{\mu}))\log\frac{p({\bf x}_{c};\boldsymbol{\theta}(\boldsymbol{\mu}))}{\prod_{i\in c}p(x_{i};\boldsymbol{\theta}(\boldsymbol{\mu}))}\\
 & = & -\sum_{i}\sum_{x_{i}}\mu(x_{i})\log\mu(x_{i})\\
 &  & -\sum_{c}\sum_{{\bf x}_{c}}\mu({\bf x}_{c})\log\frac{\mu({\bf x}_{c})}{\prod_{i\in c}\mu(x_{i})}
\end{eqnarray*}

\end{proof}
Finally, combining these two lemmas, we can show the main result,
that the TRW entropy is an upper bound.
\begin{thm*}
[TRW Entropy Bound]Let $Pr(\mathcal{G})$ be a distribution over
tree structured graphs, and define $\rho_{c}=Pr(c\in\mathcal{G}).$
Then, with $\tilde{H}$ as defined in Eq. \ref{eq:TRW-entropy}, 
\[
\tilde{H}(\boldsymbol{\mu})\geq H(\boldsymbol{\mu}).
\]
\end{thm*}
\begin{proof}
The previous Lemma shows that for any specific tree $\mathcal{G}$,
\[
H(\boldsymbol{\mu}^{\mathcal{G}})\geq H(\boldsymbol{\mu}).
\]

Thus, it follows that
\begin{eqnarray*}
H(\boldsymbol{\mu}) & \leq & \sum_{\mathcal{G}}Pr(\mathcal{G})H(\boldsymbol{\mu}^{\mathcal{G}})\\
 & = & \sum_{\mathcal{G}}Pr(\mathcal{G})\bigl(\sum_{i}H(\mu_{i})-\sum_{c\in\mathcal{G}}I(\mu_{c})\bigr)\\
 & = & \sum_{i}H(\mu_{i})-\sum_{c}\rho_{c}I(\mu_{c})
\end{eqnarray*}
\end{proof}
\begin{thm*}
[TRW Updates]Let $\rho_{c}$ be as in the previous theorem. Then,
if the updates in Eq. \ref{eq:TRW-msgs} reach a fixed point, the
marginals defined by 
\begin{eqnarray*}
\mu({\bf x}_{c}) & \propto & e^{\frac{1}{\rho_{c}}\theta({\bf x}_{c})}\prod_{i\in c}e^{\theta(x_{i})}\frac{\prod_{d:i\in d}m_{d}(x_{i})^{\rho_{d}}}{m_{c}(x_{i})},\\
\mu(x_{i}) & \propto & e^{\theta(x_{i})}\prod_{d:i\in d}m_{d}(x_{i})^{\rho_{d}}
\end{eqnarray*}

\end{thm*}
\noindent constitute the global optimum of Eq. \ref{eq:A-TRW}.
\begin{proof}
The TRW optimization is defined by
\[
\tilde{A}(\boldsymbol{\theta})=\max_{\boldsymbol{\mu}\in\mathcal{L}}\boldsymbol{\theta}\cdot\boldsymbol{\mu}+\tilde{H}(\boldsymbol{\mu}).
\]

Consider the equivalent optimization 

\begin{eqnarray*}
 & \max_{\boldsymbol{\mu}}\theta\cdot\mu+\tilde{H}(\mu)\\
\text{s.t.} & 1 & =\sum_{x_{i}}\mu(x_{i})\\
 & \mu(x_{i}). & =\sum_{{\bf x}_{c\backslash i}}\mu({\bf x}_{c})
\end{eqnarray*}

\noindent which makes the constraints of the local polytope explicit

First, we form a Lagrangian, and consider derivatives with respect
to $\boldsymbol{\mu}$, for fixed Lagrange multipliers.

\begin{align*}
\mathbb{L}= & \theta\cdot\mu+H(\mu)+\sum_{i}\lambda_{i}(1-\sum_{x_{i}}\mu(x_{i}))\\
 & +\sum_{c}\sum_{x_{i}}\lambda_{c}(x_{i})\bigl(\sum_{{\bf x}_{c\backslash i}}\mu({\bf x}_{c})-\mu(x_{i})\bigr)\\
\frac{d\mathbb{L}}{d\mu({\bf x}_{c})}= & \theta({\bf x}_{c})+\rho_{c}\bigl(\sum_{i\in c}\bigl(1+\log\sum_{{\bf x}'_{i}}\mu(x_{i},{\bf x}'_{-i})\bigr)\\
 & \,\,-1-\log\mu({\bf x}_{c})\bigr)+\sum_{i\in c}\lambda_{c}(x_{i})\\
\frac{d\mathbb{L}}{d\mu(x_{i})}= & \theta(x_{i})-1-\log\mu(x_{i})-\lambda_{i}-\sum_{c:i\in c}\lambda_{c}(x_{i})
\end{align*}

Setting these derivatives equal to zero, we can solve for the log-marginals
in terms of the Lagrange multipliers:

\begin{align*}
\rho_{c}\log\mu({\bf x}_{c})= & \theta({\bf x}_{c})+\rho_{c}\bigl(\sum_{i\in c}\bigl(1+\log\sum_{{\bf x}'_{i}}\mu(x_{i},{\bf x}'_{-i})\bigr)-1\bigr)\\
 & +\sum_{i\in c}\lambda_{c}(x_{i})\\
\log\mu(x_{i})= & \theta(x_{i})-1-\lambda_{i}-\sum_{c:i\in c}\lambda_{c}(x_{i})
\end{align*}

Now, at a solution, we must have that $\mu(x_{i})=\sum_{{\bf x}_{c\backslash i}}\mu({\bf x}_{c})$.
This leads first to the the constraint that

\begin{alignat*}{1}
\log\mu({\bf x}_{c})= & \frac{1}{\rho_{c}}\theta({\bf x}_{c})+\sum_{i\in c}\bigl(1+\log\mu(x_{i})+\frac{1}{\rho_{c}}\lambda_{c}(x_{i})\bigr)-1\\
= & \frac{1}{\rho_{c}}\theta({\bf x}_{c})+\sum_{i\in c}\bigl(1+\theta(x_{i})-1-\lambda_{i}-\sum_{c:i\in c}\lambda_{c}(x_{i})\\
 & +\frac{1}{\rho_{c}}\lambda_{c}(x_{i})\bigr)-1.
\end{alignat*}

Now, define the ``messages'' in terms of the Lagrange multipliers
as 
\[
m_{c}(x_{i})=e^{-\frac{1}{\rho_{c}}\lambda_{c}(x_{i})}.
\]

\noindent If the appropriate values of the messages were known, then
we could solve for the clique-wise marginals as
\begin{eqnarray*}
\mu({\bf x}_{c}) & \propto & e^{\frac{1}{\rho_{c}}\theta({\bf x}_{c})}\prod_{i\in c}e^{\theta(x_{i})}\exp\bigl(\frac{1}{\rho_{c}}\lambda_{c}(x_{i})\bigr)\\
 &  & \times\prod_{d:i\in d}\exp\bigl(-\lambda_{d}(x_{i}))\bigr)\\
 & = & e^{\frac{1}{\rho_{c}}\theta({\bf x}_{c})}\prod_{i\in c}e^{\theta(x_{i})}\frac{\prod_{d:i\in d}m_{d}(x_{i})^{\rho_{d}}}{m_{c}(x_{i})}.
\end{eqnarray*}
 The univiariate marginals are available simply as

\begin{eqnarray*}
\mu(x_{i}) & \propto & \exp\bigl(\theta(x_{i})-\sum_{d:i\in d}\lambda_{d}(x_{i})\bigr)\\
 & = & e^{\theta(x_{i})}\prod_{d:i\in d}\exp(-\lambda_{d}(x_{i})\bigr)\\
 & = & e^{\theta(x_{i})}\prod_{d:i\in d}m_{d}(x_{i})^{\rho_{d}}.
\end{eqnarray*}

We may now derive the actual propagation. At a valid solution, the
Lagrange multipliers (and hence the messages) must be selected so
that the constraints are satisfied. In particular, we must have that
$\mu(x_{i})=\sum_{{\bf x}_{c\backslash i}}\mu({\bf x}_{c}).$ From
the constraint, we can derive constraints on neighboring sets of messages.

\begin{alignat}{1}
\mu(x_{i}) & =\sum_{{\bf x}_{c\backslash i}}\mu({\bf x}_{c})\nonumber \\
e^{\theta(x_{i})}\prod_{d:i\in d}m_{d}(x_{i})^{\rho_{d}} & \propto\sum_{{\bf x}_{c\backslash i}}e^{\frac{1}{\rho_{c}}\theta({\bf x}_{c})}\nonumber \\
 & \times\prod_{i\in c}e^{\theta(x_{i})}\frac{\prod_{d:i\in d}m_{d}(x_{i})^{\rho_{d}}}{m_{c}(x_{i})}\label{eq:TRW_beforecancel}
\end{alignat}

Now, the left hand side of this equation cancels one term from the
product on line \ref{eq:TRW_beforecancel}, except for the denominator
of $m_{c}(x_{i}).$ This leads to the constraint of

\begin{alignat*}{1}
m_{c}(x_{i}) & \propto\sum_{{\bf x}_{c\backslash i}}e^{\frac{1}{\rho_{c}}\theta({\bf x}_{c})}\prod_{j\in c\backslash i}e^{\theta(x_{j})}\frac{\prod_{d:j\in d}m_{d}(x_{j})^{\rho_{d}}}{m_{c}(x_{j})}.
\end{alignat*}
This is exactly the equation used as a fixed-point equation in the
TRW algorithm.
\end{proof}

\section{Appendix B: Implicit Differentiation}
\begin{thm*}
Suppose that
\[
\boldsymbol{\mu}(\boldsymbol{\theta}):=\underset{\boldsymbol{\mu}:B\boldsymbol{\mu}={\bf d}}{\arg\max}\,\,\boldsymbol{\theta}\cdot\boldsymbol{\mu}+H(\boldsymbol{\mu}).
\]
Define $L(\boldsymbol{\theta},{\bf x})=Q(\boldsymbol{\mu}(\boldsymbol{\theta}),{\bf x}).$
Then, letting $D=\frac{d^{2}H}{d\boldsymbol{\mu}d\boldsymbol{\mu}^{T}},$
\[
\frac{dL}{d\boldsymbol{\theta}}=\bigl(D^{-1}B^{T}(BD^{-1}B^{T})^{-1}BD^{-1}-D^{-1}\bigr)\frac{dQ}{d\boldsymbol{\mu}}.
\]
\end{thm*}
\begin{proof}
First, recall the implicit differentiation theorem. If the relationship
between ${\bf a}$ and ${\bf b}$ is implicitly determined by ${\bf f}({\bf a},{\bf b})={\bf 0}$,
then

\[
\frac{d{\bf b}^{T}}{d{\bf a}}=-\frac{d{\bf f}^{T}}{d{\bf a}}\bigl(\frac{d{\bf f}^{T}}{d{\bf b}}\bigr)^{-1}.
\]

In our case, given the Lagrangian

\[
\mathbb{L}=\boldsymbol{\theta}\cdot\boldsymbol{\mu}+H(\boldsymbol{\mu})+\boldsymbol{\lambda}^{T}(B\boldsymbol{\mu}-{\bf d}),
\]

Our implicit function is determined by the constraints that $\frac{dL}{d\boldsymbol{\mu}}={\bf 0}$
and $\frac{d\mathbb{L}}{d\boldsymbol{\lambda}}={\bf 0}$. That is,
it must be true that

\[
\frac{d\mathbb{L}}{d\boldsymbol{\mu}}=\theta+\frac{dH}{d\boldsymbol{\mu}}+B^{T}\boldsymbol{\lambda}={\bf 0}
\]

\[
\frac{d\mathbb{L}}{d\boldsymbol{\lambda}}=B\boldsymbol{\mu}-{\bf d}={\bf 0}.
\]

Thus, our implicit function is

\[
{\bf f}\left(\left[\begin{array}{c}
\boldsymbol{\mu}\\
\boldsymbol{\lambda}
\end{array}\right]\right)=\left[\begin{array}{c}
\theta+\frac{dH}{d\boldsymbol{\mu}}+B^{T}\boldsymbol{\lambda}\\
B\boldsymbol{\mu}-{\bf d}
\end{array}\right]=\left[\begin{array}{c}
{\bf 0}\\
{\bf 0}
\end{array}\right]
\]

Taking derivatives, we have that
\[
\frac{d\left[\begin{array}{c}
\boldsymbol{\mu}\\
\boldsymbol{\lambda}
\end{array}\right]^{T}}{d\boldsymbol{\theta}}=-\bigl(\frac{d{\bf f}^{T}}{d\boldsymbol{\theta}}\bigr)\Bigl(\frac{d{\bf f}^{T}}{d\left[\begin{array}{c}
\boldsymbol{\mu}\\
\boldsymbol{\lambda}
\end{array}\right]}\Bigr)^{-1}
\]

Taking the terms on the right hand side in turn, we have

\[
\frac{d{\bf f}^{T}}{d\boldsymbol{\theta}}=\frac{d\left[\begin{array}{c}
\theta+\frac{dH}{d\boldsymbol{\mu}}+B^{T}\boldsymbol{\lambda}\\
B\boldsymbol{\mu}-{\bf d}
\end{array}\right]^{T}}{d\boldsymbol{\theta}}=\left[\begin{array}{c}
I\\
0
\end{array}\right]^{T}
\]

\[
\frac{d{\bf f}^{T}}{d\left[\begin{array}{c}
\boldsymbol{\mu}\\
\boldsymbol{\lambda}
\end{array}\right]}=\left[\begin{array}{cc}
\frac{d^{2}H}{d\boldsymbol{\mu}d\boldsymbol{\mu}^{T}} & B^{T}\\
B & 0
\end{array}\right]
\]

\begin{equation}
\frac{d\left[\begin{array}{c}
\boldsymbol{\mu}\\
\boldsymbol{\lambda}
\end{array}\right]^{T}}{d\boldsymbol{\theta}}=-\left[\begin{array}{c}
I\\
0
\end{array}\right]^{T}\left[\begin{array}{cc}
\frac{d^{2}H}{d\boldsymbol{\mu}d\boldsymbol{\mu}^{T}} & B^{T}\\
B & 0
\end{array}\right]^{-1}\label{eq:du_lambda_dtheta}
\end{equation}

This means that $-\frac{d\boldsymbol{\mu}^{T}}{d\boldsymbol{\theta}}$
is the upper-left block of the inverse of the matrix on the right
hand side of Eq. \ref{eq:du_lambda_dtheta}. It is well known that
if
\[
M=\left[\begin{array}{cc}
E & F\\
G & H
\end{array}\right],
\]
then the upper-left block of $M^{-1}$ is 
\[
E^{-1}+E^{-1}F(H\text{\textminus}GE^{-1}F)^{-1}GE^{-1}.
\]

So, we have that 
\begin{equation}
\frac{d\boldsymbol{\mu}^{T}}{d\boldsymbol{\theta}}=D^{-1}B^{T}(BD^{-1}B^{T})BD^{-1}-D^{-1},\label{eq:du/dtheta}
\end{equation}
where $D:=\frac{d^{2}H}{d\boldsymbol{\mu}d\boldsymbol{\mu}^{T}}$.

The result follows simply from substituting Eq. \ref{eq:du/dtheta}
into the chain rule 
\[
\frac{dL}{d\boldsymbol{\theta}}=\frac{d\boldsymbol{\mu}^{T}}{d\boldsymbol{\theta}}\frac{dQ}{d\boldsymbol{\mu}}.
\]

\end{proof}

\section{Appendix C: Truncated Fitting}

Several simple lemmas will be useful below. A first one considers
the case where we have a ``product of powers''.
\begin{lem*}
[Products of Powers]Suppose that $b=\prod_{i}a_{i}^{p_{i}}$. Then
\[
\frac{db}{da_{i}}=\frac{p_{i}}{a_{i}}b.
\]

\end{lem*}
Next, both mean-field and TRW involve steps where we first take a
product of a set of terms, and then normalize. The following lemma
is useful in dealing with such operations.
\begin{lem*}
[Normalized Products]Suppose that $b_{i}=\prod_{j}a_{ij}$ and $c_{i}=b_{i}/\sum_{j}a_{ij}.$
Then,
\[
\frac{dc_{i}}{da_{jk}}=\bigl(I_{i=j}-c_{i}\bigr)\frac{c_{j}}{a_{jk}}.
\]
\end{lem*}
\begin{cor*}
Under the same conditions,
\[
\frac{dL}{da_{jk}}=\frac{c_{j}}{a_{jk}}\bigl(\frac{dL}{dc_{j}}-\sum_{i}\frac{dL}{dc_{i}}c_{i}\bigr).
\]

\end{cor*}
Accordingly, we find it useful to define the operator

\[
\text{backnorm}({\bf g},{\bf c})={\bf c}\odot({\bf g}-{\bf g}\cdot{\bf c}).
\]
This can be used as follows. Suppose that we have calculated $\overleftarrow{{\bf c}}=\frac{dL}{d{\bf c}}$.
Then, if we set $\overleftarrow{\boldsymbol{\nu}}=\text{backnorm}(\overleftarrow{{\bf c}},{\bf c})$,
and we have that $\frac{dL}{da_{jk}}=\frac{\overleftarrow{\nu_{j}}}{a_{jk}}$.

An important special case of this is where $a_{jk}=\exp f_{jk}.$
Then, we have simply that $\frac{dL}{df_{jk}}=\overleftarrow{\nu_{j}}.$

Another important special case is where $a_{jk}=f_{jk}^{\rho}.$ Then,
we have that $\frac{da_{jk}}{df_{jk}}=\rho f_{jk}^{\rho-1},$ and
so $\frac{dL}{df_{jk}}=\rho\frac{\overleftarrow{\nu_{j}}}{f_{jk}}$.
\begin{thm*}
After execution of back mean field, 
\[
\overleftarrow{\theta}(x_{i})=\frac{dL}{d\theta(x_{i})}\text{ and }\overleftarrow{\theta}({\bf x}_{c})=\frac{dL}{d\theta({\bf x}_{c})}.
\]
\end{thm*}
\begin{proof}
[Proof sketch]The idea is just to mechanically differentiate each
step of the algorithm, computing the derivative of the computed loss
with respect to each intermediate quantity. First, note that we can
re-write the main mean-field iteration as

\begin{equation}
{\displaystyle \mu(x_{j})\propto\exp\bigl(\theta(x_{j})\bigr)\prod_{c:j\in c}\prod_{{\bf x}_{c\backslash j}}\exp\bigl(\theta({\bf x}_{c})\prod_{i\in c\backslash j}\mu(x_{i})\bigr)\bigr)}.\label{eq:mean-field-iteration}
\end{equation}

Now, suppose we have the derivative of the loss with respect to this
intermediate vector of marginals $\overleftarrow{\boldsymbol{\mu}_{j}}$.
We wish to ``push back'' this derivative on the values affecting
these marginals, namely $\theta(x_{j})$, $\theta({\bf x}_{c})$ (for
all $c$ such that $j\in c$), and $\mu(x_{i})$ (for all $i\not=j$
such that $\exists c:\{i,j\}\in c$). To do this, we take two steps:

1) Calculate the derivative with respect to the value on the righthand
side of Eq. \ref{eq:mean-field-iteration} \emph{before} normalization.

2) Calculate the derivative of this un-normalized quantity with respect
to $\theta(x_{j})$, $\theta({\bf x}_{c})$ and $\mu(x_{i})$.

Now, define $\boldsymbol{\nu}_{j}$ to be the vector of marginals
produced by Eq. \ref{eq:mean-field-iteration} before normalization.
Then, by the discussion above, $\overleftarrow{\boldsymbol{\nu}_{j}}=\text{backnorm}(\overleftarrow{\boldsymbol{\mu}_{j}},\boldsymbol{\mu}_{j}).$
This completes step 1.

Now, with $\overleftarrow{\boldsymbol{\nu}_{j}}$ in hand, we can
immediately calculate the backpropagation of $\overleftarrow{\boldsymbol{\mu}_{j}}$
on $\theta$ as

\[
\overleftarrow{\theta}(x_{j})=\overleftarrow{\nu}(x_{j}).
\]

\noindent This, follows from the fact that $\frac{dL}{da}=\frac{dL}{de^{a}}e^{a}$,
where $\theta(x_{j})$ plays the role of $a$, $ $and $e^{a}$ plays
the role of $\nu(x_{j})$.

Similarly, we can calculate that 
\[
\frac{dL}{d\theta({\bf x}_{c})\prod_{i\in c\backslash j}\mu(x_{i})}=\overleftarrow{\nu}(x_{j}).
\]

\noindent Thus, since
\[
\frac{d\theta({\bf x}_{c})\prod_{i\in c\backslash j}\mu(x_{i})}{d\theta({\bf x}_{c})}=\prod_{i\in c\backslash j}\mu(x_{i}),
\]

\noindent we have that

\[
\overleftarrow{\theta}({\bf x}_{c})=\overleftarrow{\nu}(x_{j})\prod_{i\in c\backslash j}\mu(x_{i}).
\]

Similarly, for any ${\bf x}_{c}$ that ``matches'' $x_{i}$ (in
the sense that the same value $x_{i}$ is present as the appropriate
component of ${\bf x}_{c}$),

\[
\frac{d\theta({\bf x}_{c})\prod_{k\in c\backslash j}\mu(x_{k})}{d\mu(x_{i})}=\theta({\bf x}_{c})\prod_{k\in c\backslash\{i,j\}}\mu(x_{k}).
\]
From which we have
\[
\overleftarrow{\mu}(x_{i})=\sum_{{\bf x}_{c\backslash i}}\overleftarrow{\nu}(x_{j})\theta({\bf x}_{c})\prod_{k\in c\backslash\{i,j\}}\mu(x_{k}),
\]

\noindent meaning this is a local minimum. Normalizing by $Z$ gives
the result.\end{proof}
\begin{thm*}
After execution of back TRW, 
\[
\overleftarrow{\theta}(x_{i})=\frac{dL}{d\theta(x_{i})}\text{ and }\overleftarrow{\theta}({\bf x}_{c})=\frac{dL}{d\theta({\bf x}_{c})}.
\]
\end{thm*}
\begin{proof}
[Proof sketch]Again, the idea is just to mechanically differentiate
each step of the algorithm. Since the marginals are derived in terms
of the messages, we must first take derivatives with respect to the
marginal-producing steps. First, consider step 3, where predicted
clique marginals are computed. Defining $\overleftarrow{\nu}({\bf x}_{c})=\text{backnorm}(\overleftarrow{\mu_{c}},\mu_{c}),$
we have that
\begin{eqnarray*}
\overleftarrow{\theta}({\bf x}_{c}) & = & \frac{1}{\rho_{c}}\overleftarrow{\nu}({\bf x}_{c})\\
\overleftarrow{\theta}(x_{i}) & = & \sum_{{\bf x}_{c\backslash i}}\overleftarrow{\nu}({\bf x}_{c})\\
\overleftarrow{m_{d}}(x_{i}) & = & \frac{\rho_{d}-I[c=d]}{m_{d}(x_{i})}\sum_{{\bf x}_{c\backslash i}}\overleftarrow{\nu}
\end{eqnarray*}

Next, consider step 4, where predicted univariate marginals are computed.
Defining, $\overleftarrow{\nu}(x_{i})=\text{backnorm}(\overleftarrow{\mu_{i}},\mu_{i})$,
we have
\begin{eqnarray*}
\overleftarrow{\theta}(x_{i}) & = & \overleftarrow{\nu}(x_{i})\\
\overleftarrow{m_{d}}(x_{i}) & = & \rho_{d}\frac{\overleftarrow{\nu}(x_{i})}{m_{d}(x_{i})}.
\end{eqnarray*}

Finally, we consider the main propagation, in step 2. Here, we recompute
the intermediate quantity

\[
s({\bf x}_{c})=e^{\frac{1}{\rho_{c}}\theta({\bf x}_{c})}\prod_{j\in c\backslash i}e^{\theta(x_{j})}\frac{\underset{d:j\in d}{\prod}m_{d}(x_{j})^{\rho_{d}}}{m_{c}(x_{j})}.
\]

After this, consider the step where when pair $(c,i)$ is being updated.
We first compute 
\[
\frac{dL}{dm_{c}^{0}(c_{i})}=\frac{\overleftarrow{\nu}(x_{i})}{m_{c}(x_{i})},
\]
where $m_{c}^{0}(x_{i})$ is defined as the value of the marginal
before normalization, and
\[
\overleftarrow{\nu}(x_{i})=\text{backnorm}(\overleftarrow{m_{ci}},m_{ci}).
\]

\noindent (See the Normalized Products Lemma above.) Given this, we
can consider the updates required to gradients of $\theta({\bf x}_{c})$,
$\theta(x_{i})$ and $m_{d}(x_{j})$ in turn.

\noindent First, we have that the update to $\overleftarrow{\theta}({\bf x}_{c})$
should be
\begin{eqnarray*}
 &  & \frac{dL}{dm_{c}^{0}(x_{i})}\frac{dm_{c}^{0}({\bf x}_{i})}{d\theta({\bf x}_{c})}\\
 & = & \frac{\overleftarrow{\nu}(x_{i})}{m_{c}(x_{i})}\frac{1}{\rho_{c}}s({\bf x}_{c}),
\end{eqnarray*}
which is the update present in the algorithm.

Next, the update to $\overleftarrow{\theta}(x_{j})$ should be 
\begin{eqnarray*}
 &  & \sum_{x_{i}}\frac{dL}{dm_{c}^{0}(x_{i})}\frac{dm_{c}^{0}({\bf x}_{i})}{d\theta(x_{j})}\\
 & = & \sum_{x_{i}}\frac{\overleftarrow{\nu}(x_{i})}{m_{c}(x_{i})}\sum_{x_{c\backslash\{i,j\}}}s({\bf x}_{c})\\
 & = & \sum_{{\bf x}_{c\backslash j}}\frac{\overleftarrow{\nu}(x_{i})}{m_{c}(x_{i})}s({\bf x}_{c}).
\end{eqnarray*}

In terms of the incoming messages, consider the update to $\overleftarrow{m_{d}}(x_{j})$,
where $j\not=i$, $j\in d$, and $d\not=c$. This will be
\begin{eqnarray*}
 &  & \sum_{x_{i}}\frac{dL}{dm_{c}^{0}(x_{i})}\frac{dm_{c}^{0}(x_{i})}{dm_{d}(x_{j})}\\
 & = & \sum_{x_{i}}\frac{\overleftarrow{\nu}(x_{i})}{m_{c}(x_{i})}\sum_{{\bf x}_{c\backslash\{i,j\}}}\frac{\rho_{d}}{m_{d}(x_{j})}s({\bf x}_{c})\\
 & = & \frac{\rho_{d}}{m_{d}(x_{j})}\sum_{{\bf x}_{c\backslash j}}s({\bf x}_{c})\frac{\overleftarrow{\nu}(x_{i})}{m_{c}(x_{i})}.
\end{eqnarray*}

Finally, consider the update to $\overleftarrow{m_{c}}(x_{j})$, where
$j\not=i$. This will have the previous update, plus the additional
term, considering the presence of $m_{c}(x_{j})$ in the denominator
of the main TRW update, of 
\begin{eqnarray*}
 &  & \sum_{x_{i}}\frac{\overleftarrow{\nu}(x_{i})}{m_{c}(x_{i})}\sum_{{\bf x}_{c\backslash\{i,j\}}}\frac{\rho_{d}}{m_{d}(x_{j})}s({\bf x}_{c})\\
 & = & \frac{\rho_{d}}{m_{d}(x_{j})}\sum_{{\bf x}_{c\backslash j}}\frac{\overleftarrow{\nu}(x_{i})}{m_{c}(x_{i})}s({\bf x}_{c}).
\end{eqnarray*}

Now, after the update has taken place, the messages $m_{c}(x_{i})$
are reverted to their previous values. As these values have not (yet)
influenced any other variables, they are initialized with $\overleftarrow{m_{c}}(x_{i})=0$.
\end{proof}
\clearpage{}

\newpage{}

\begin{figure}[p]
\begin{centering}
\vspace{-10pt}
\renewcommand{\tabcolsep}{1pt}%
\begin{tabular}{>{\centering}m{0.57in}>{\centering}m{0.9in}>{\centering}m{0.9in}>{\centering}m{0.9in}}
 & $n=1.25$ & $n=1.5$ & $n=5$\tabularnewline
{\small input} & \includegraphics[width=1\linewidth]{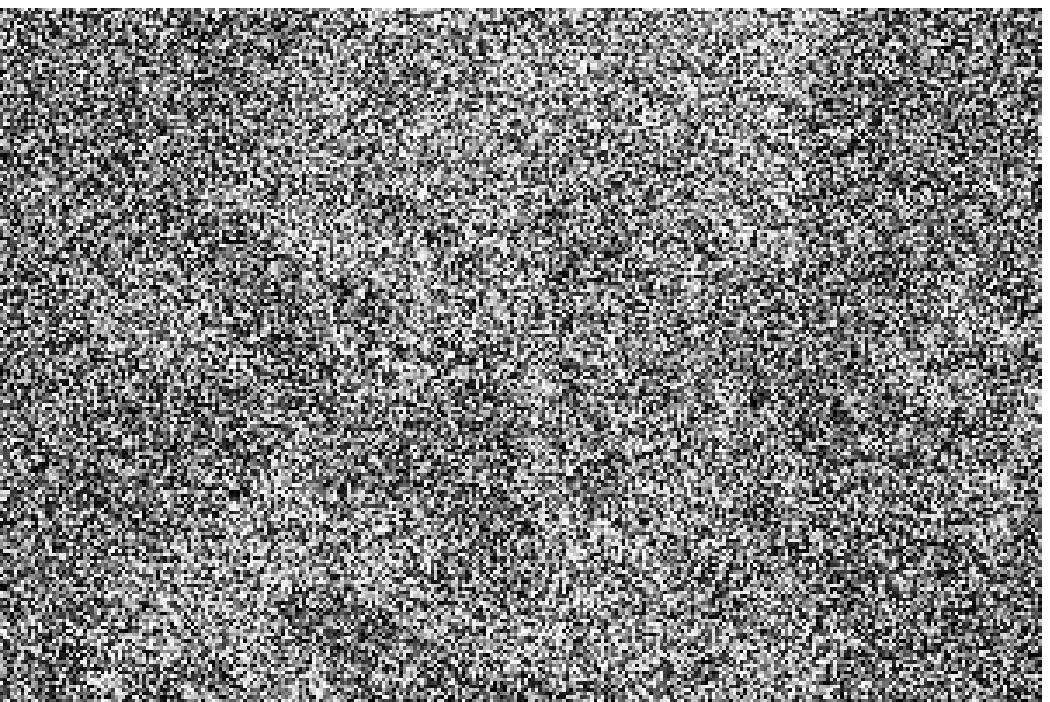} & \includegraphics[width=1\linewidth]{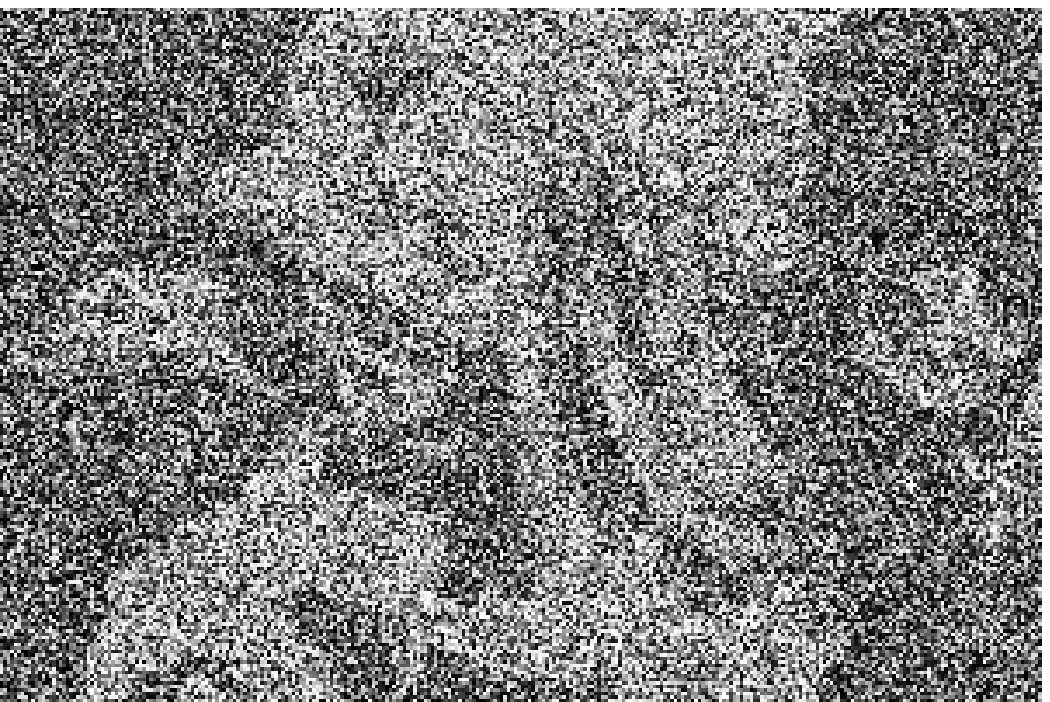} & \includegraphics[width=1\linewidth]{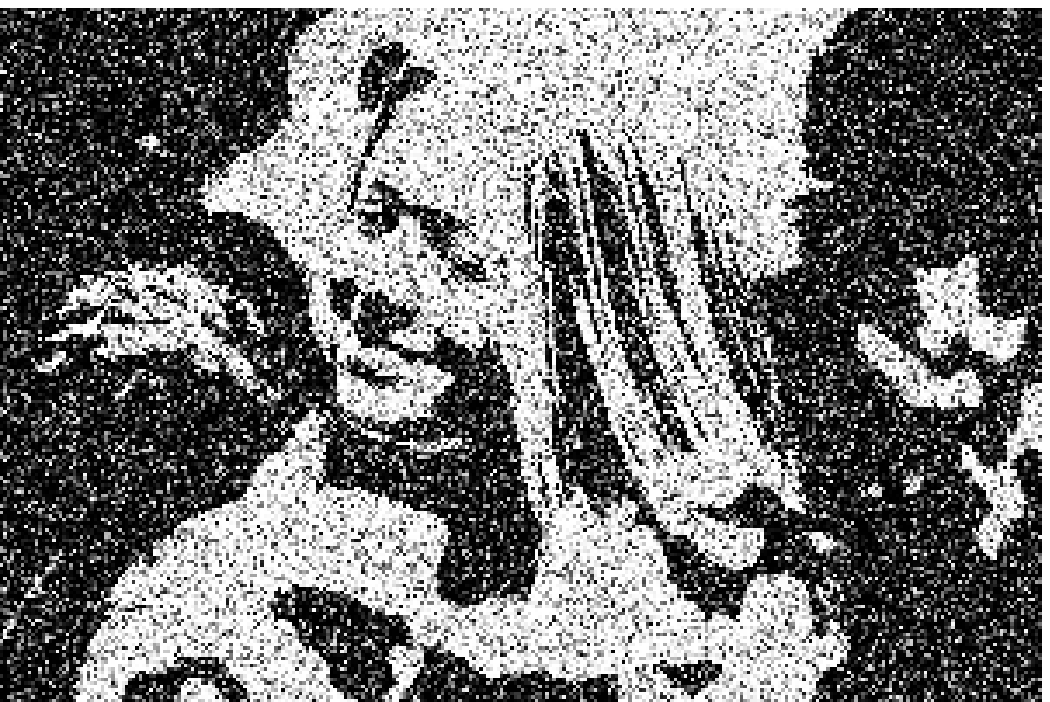}\tabularnewline
{\small surrogate likelihood} & \includegraphics[width=1\linewidth]{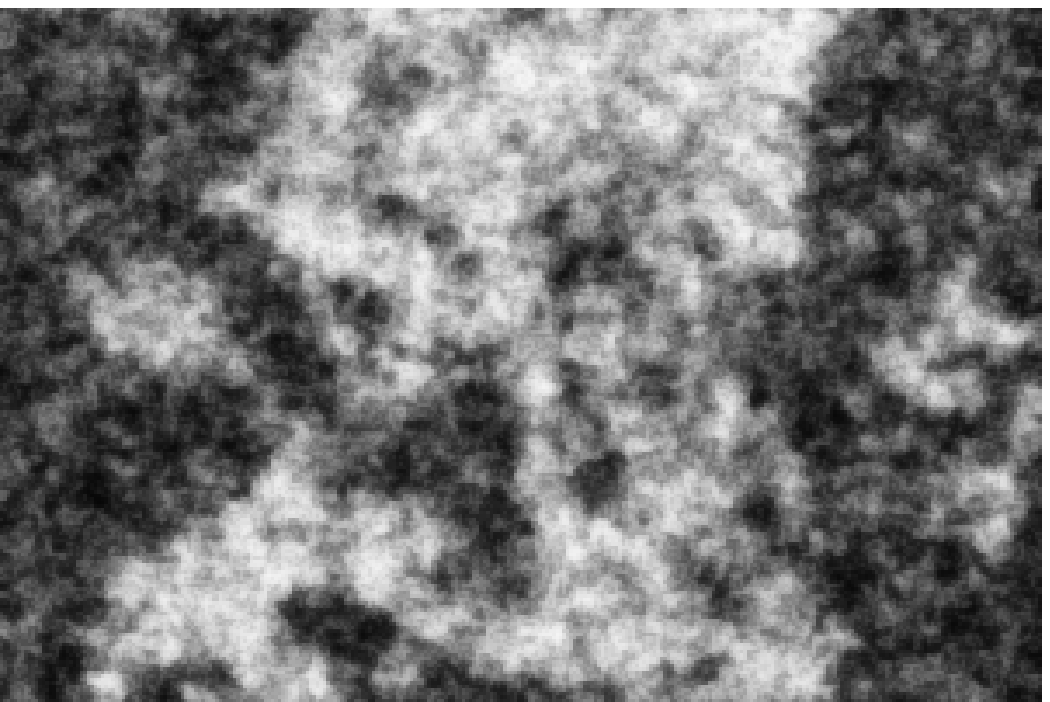} & \includegraphics[width=1\linewidth]{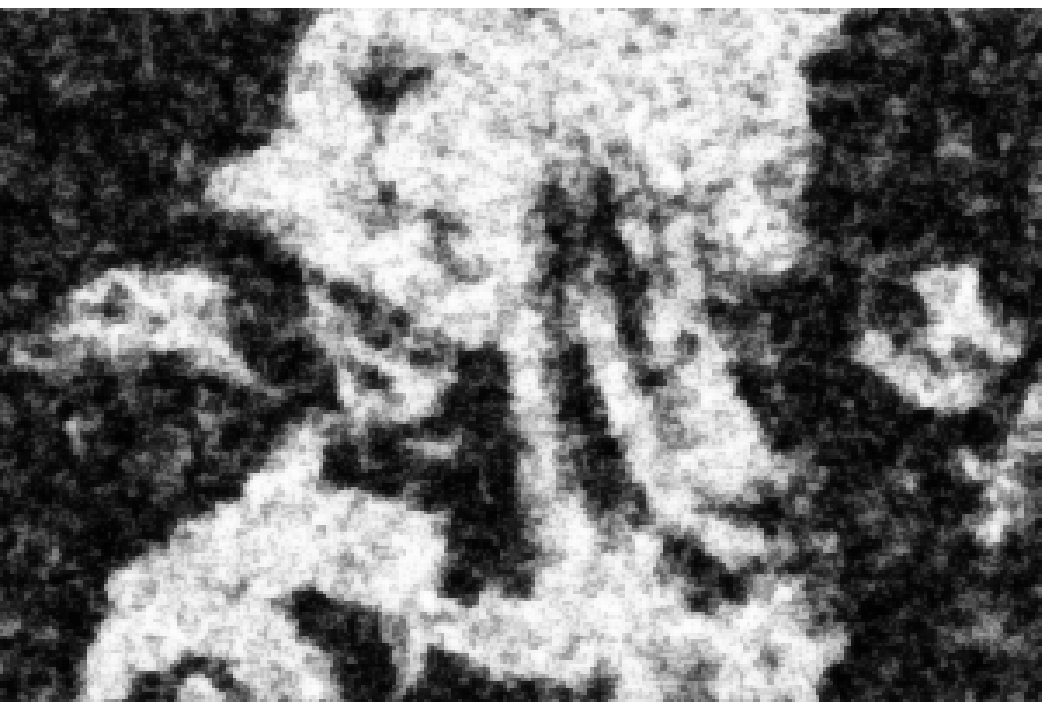} & \includegraphics[width=1\linewidth]{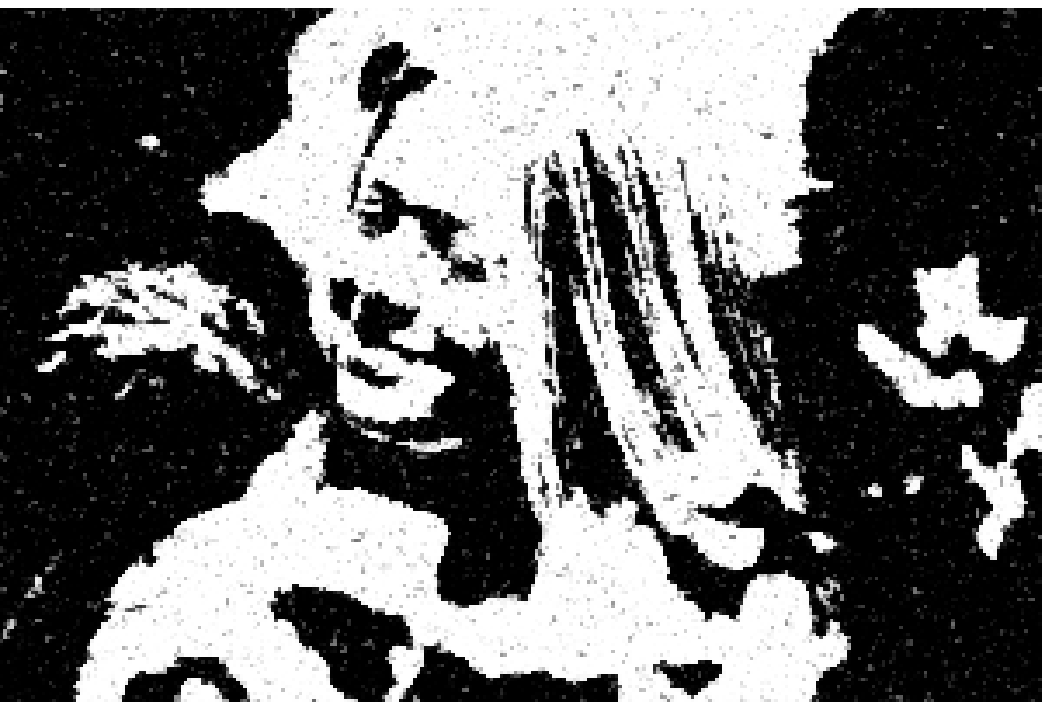}\tabularnewline
{\small univariate logistic} & \includegraphics[width=1\linewidth]{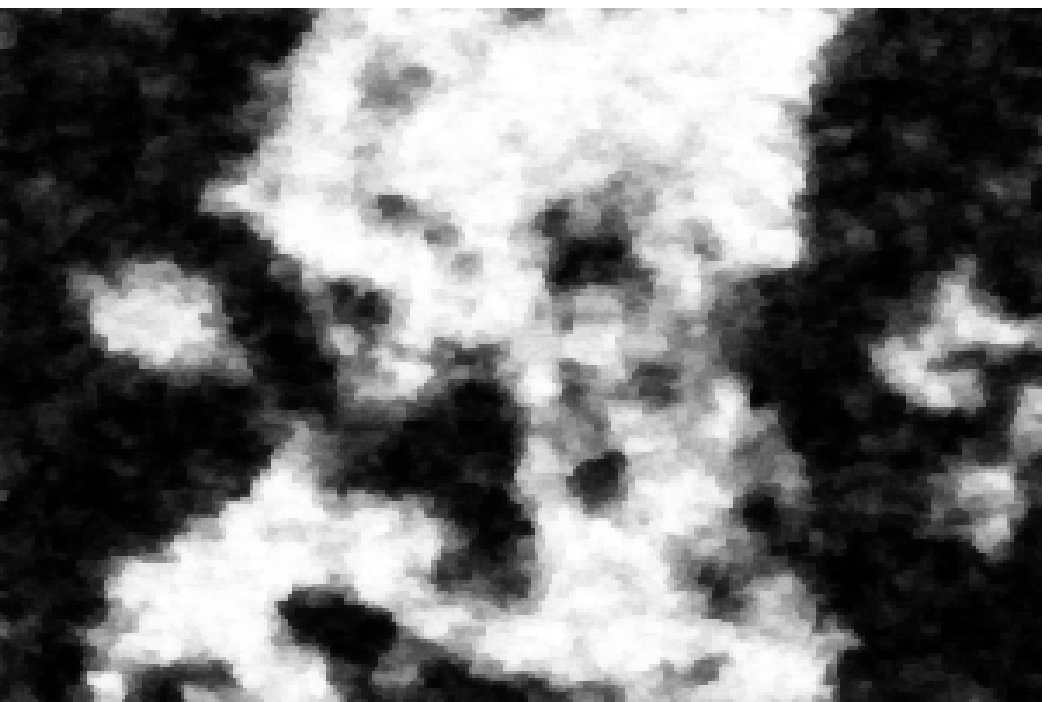} & \includegraphics[width=1\linewidth]{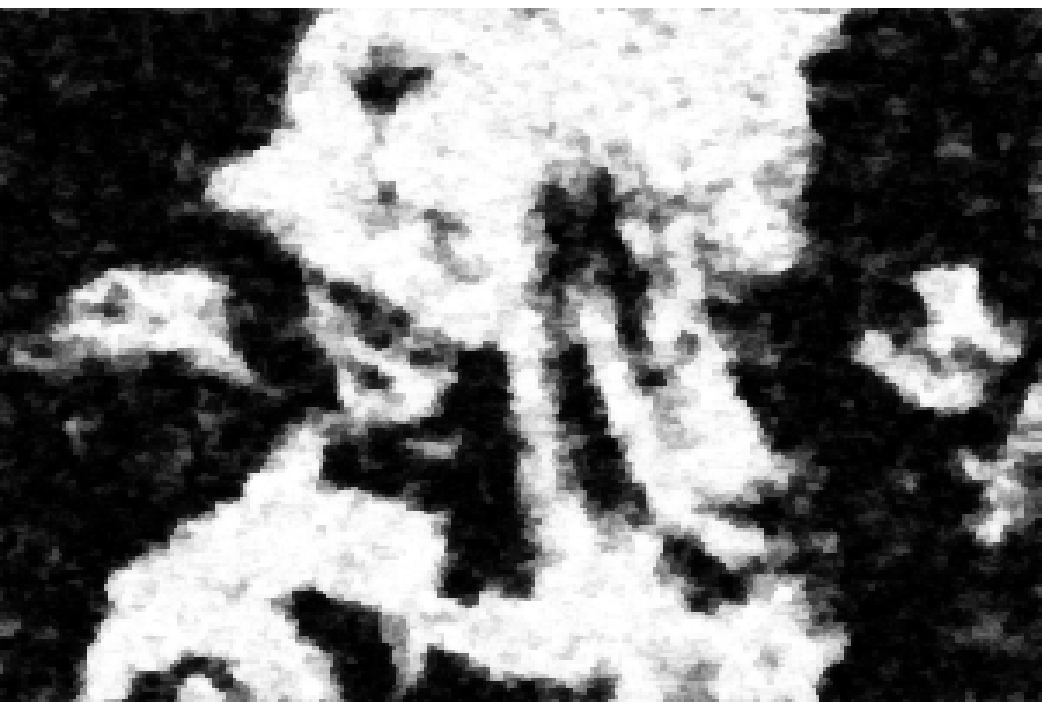} & \includegraphics[width=1\linewidth]{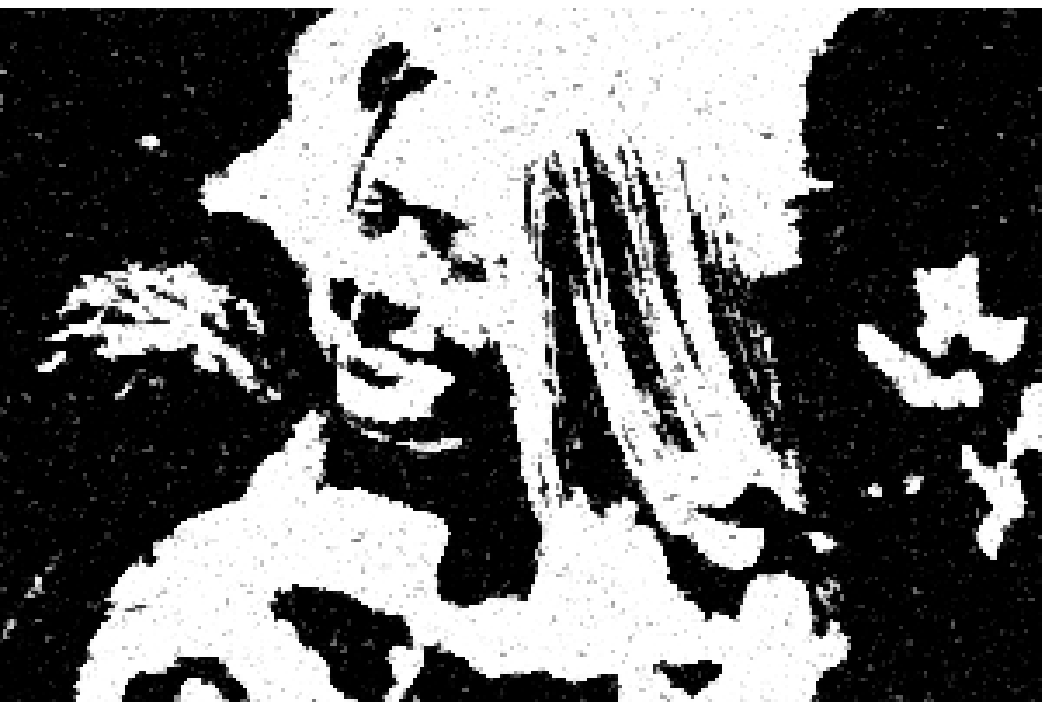}\tabularnewline
{\small clique logistic} & \includegraphics[width=1\linewidth]{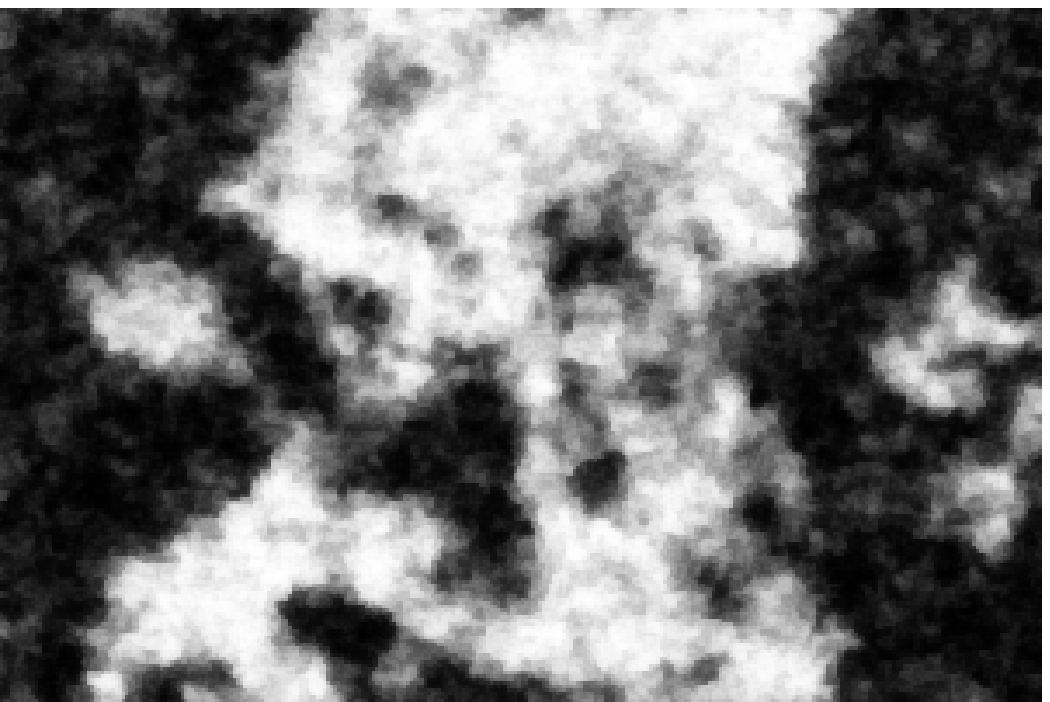} & \includegraphics[width=1\linewidth]{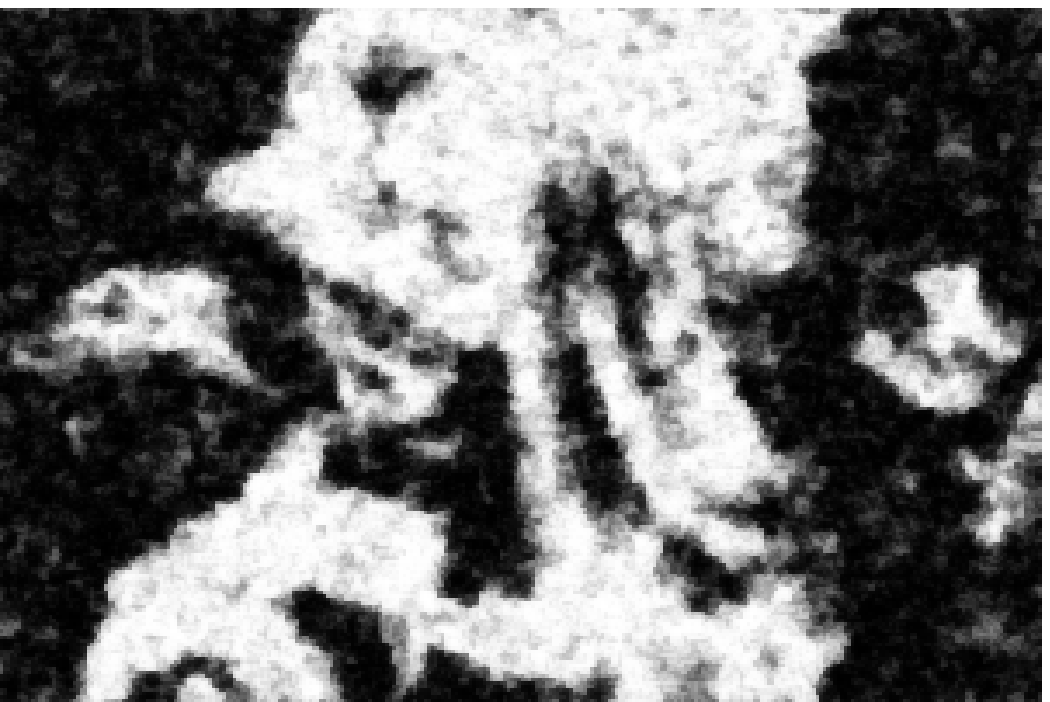} & \includegraphics[width=1\linewidth]{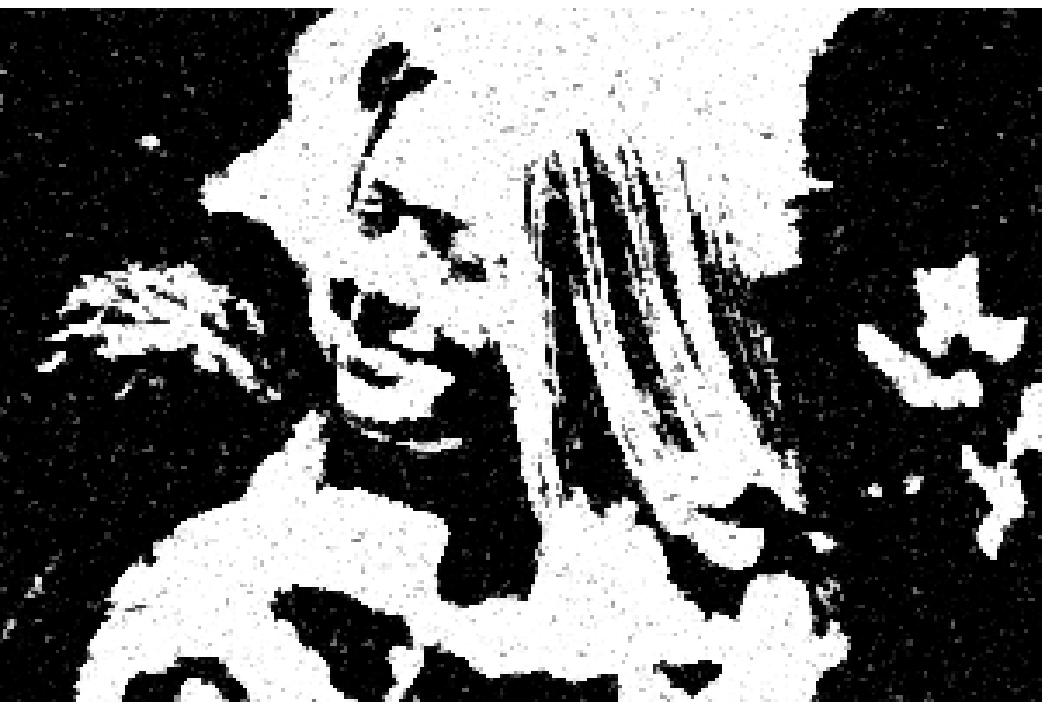}\tabularnewline
sm. class $\lambda=5$ & \includegraphics[width=1\linewidth]{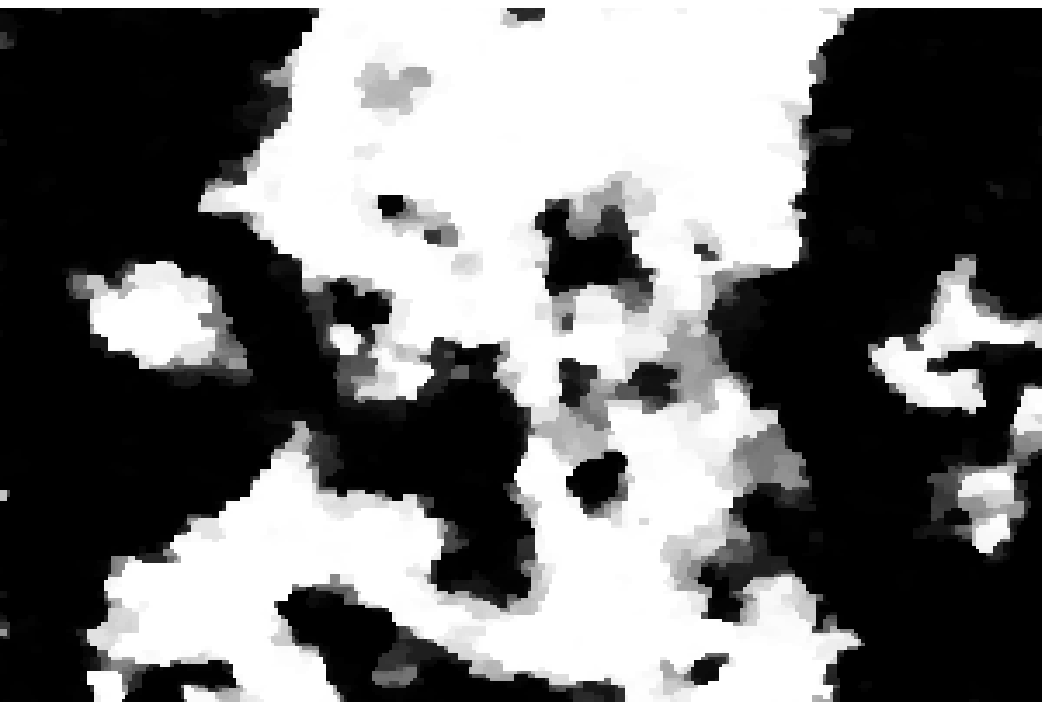} & \includegraphics[width=1\linewidth]{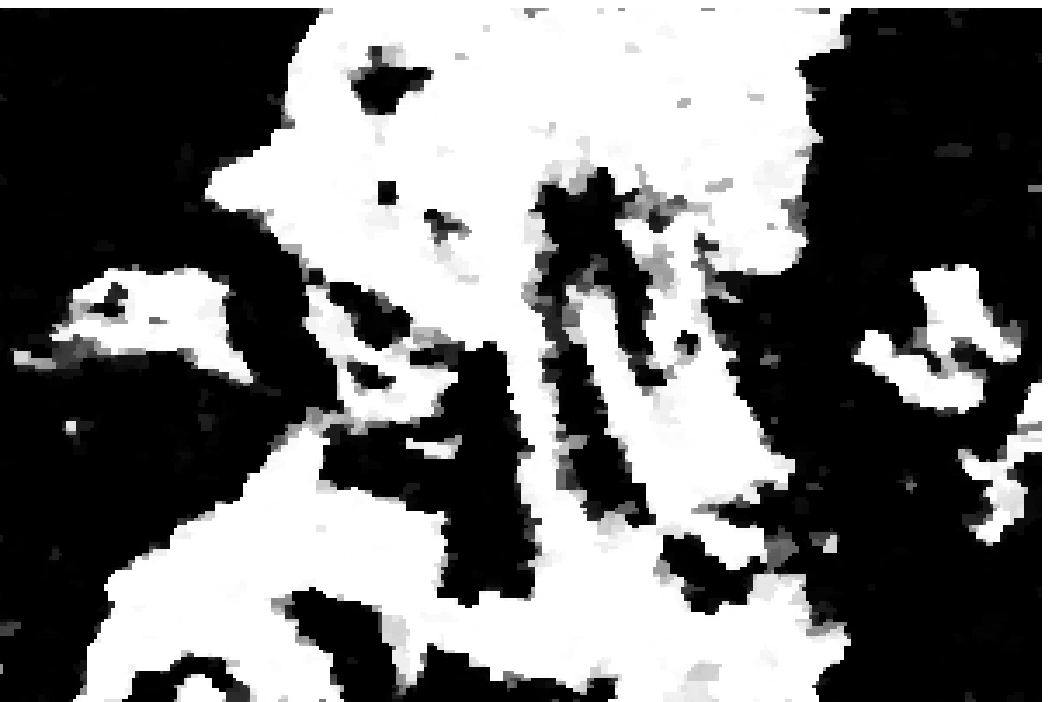} & \includegraphics[width=1\linewidth]{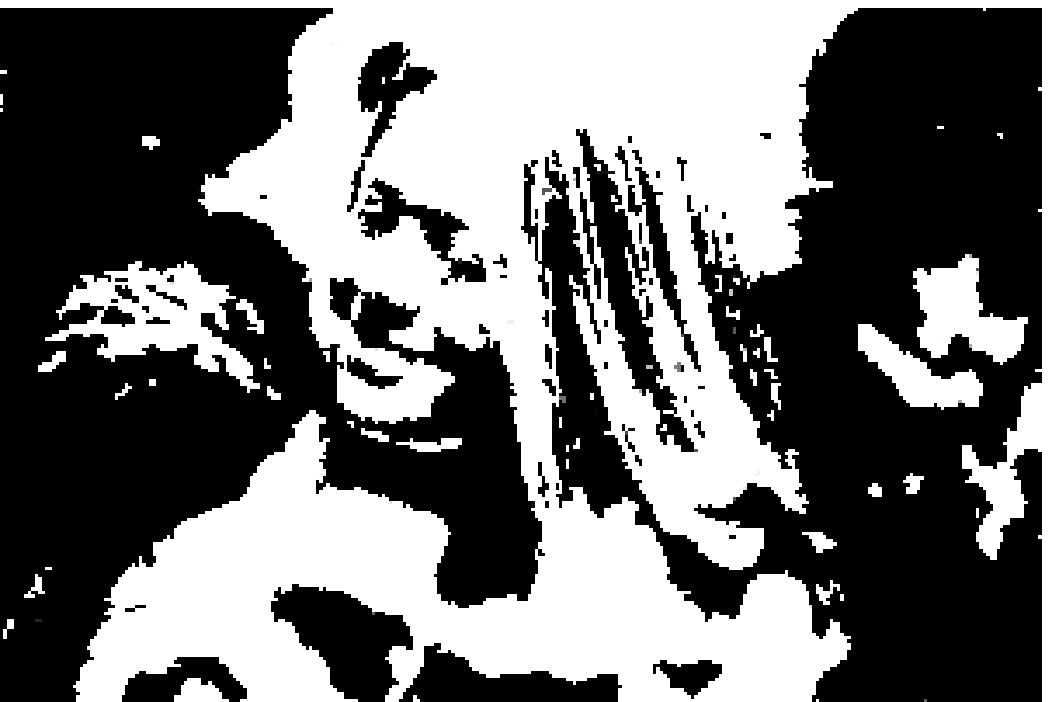}\tabularnewline
sm. class $\lambda=15$ & \includegraphics[width=1\linewidth]{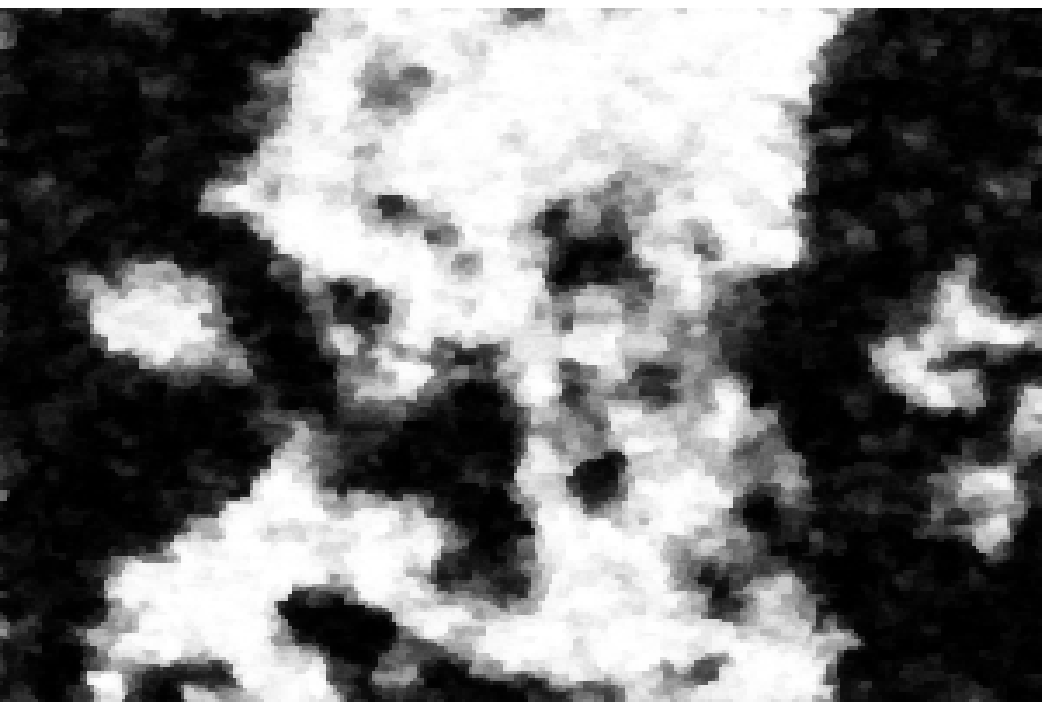} & \includegraphics[width=1\linewidth]{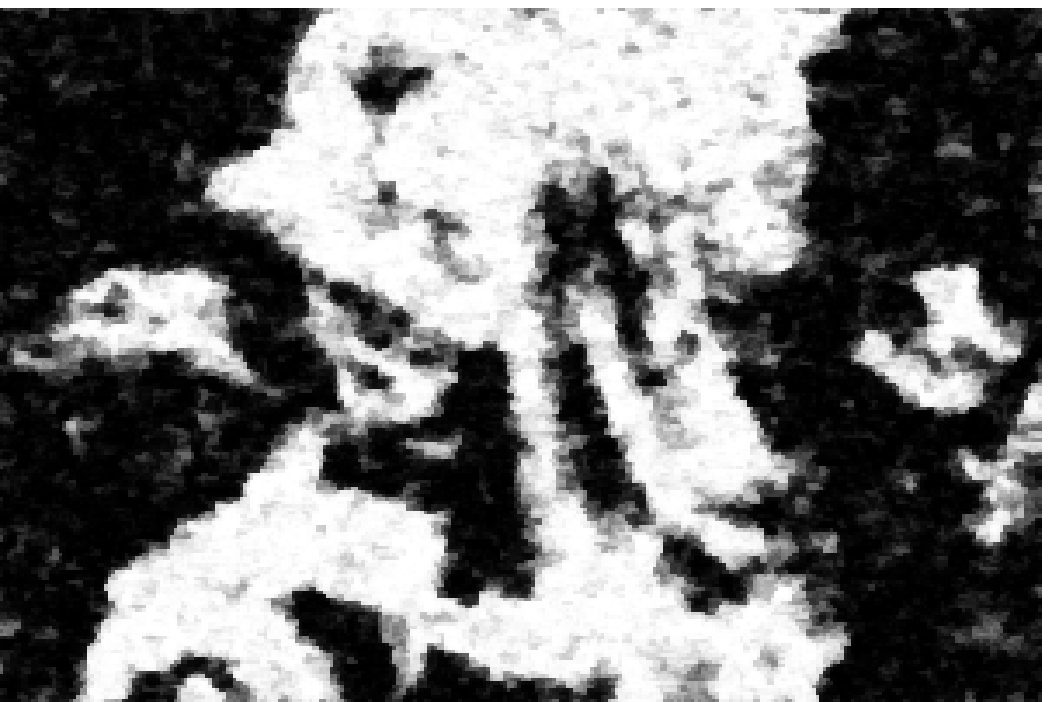} & \includegraphics[width=1\linewidth]{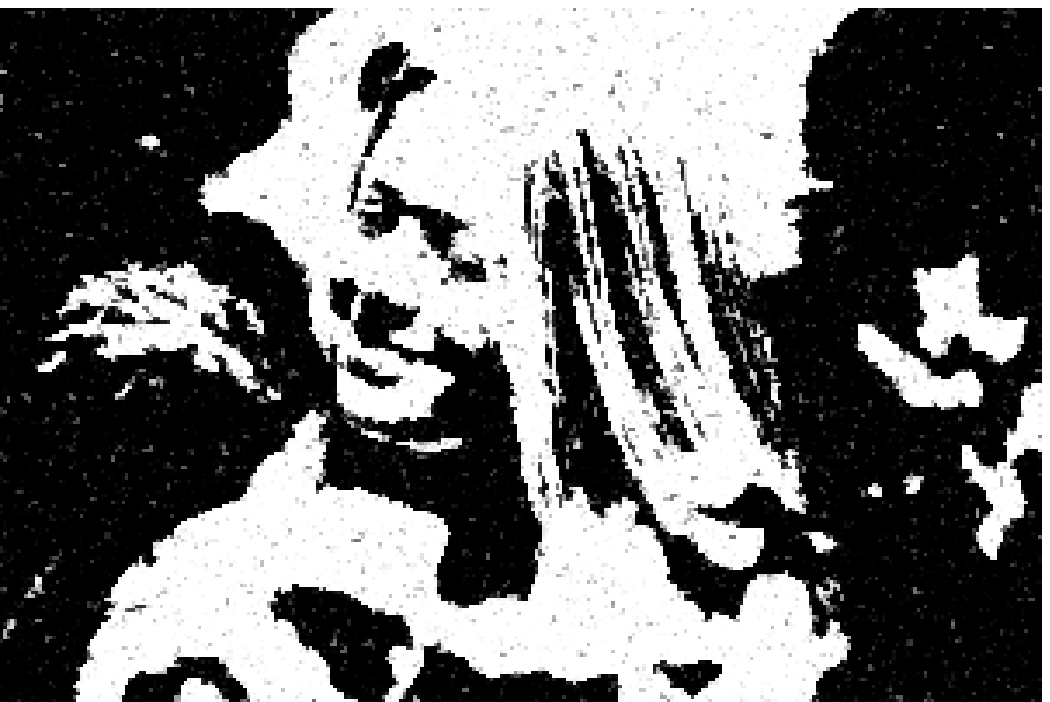}\tabularnewline
sm. class $\lambda=50$ & \includegraphics[width=1\linewidth]{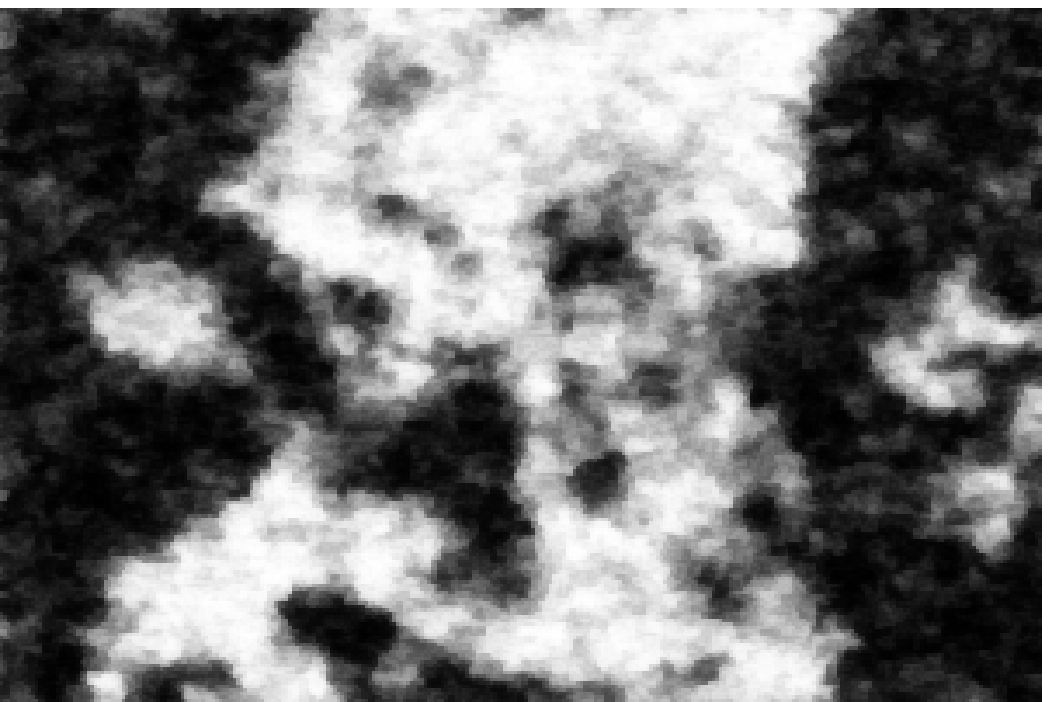} & \includegraphics[width=1\linewidth]{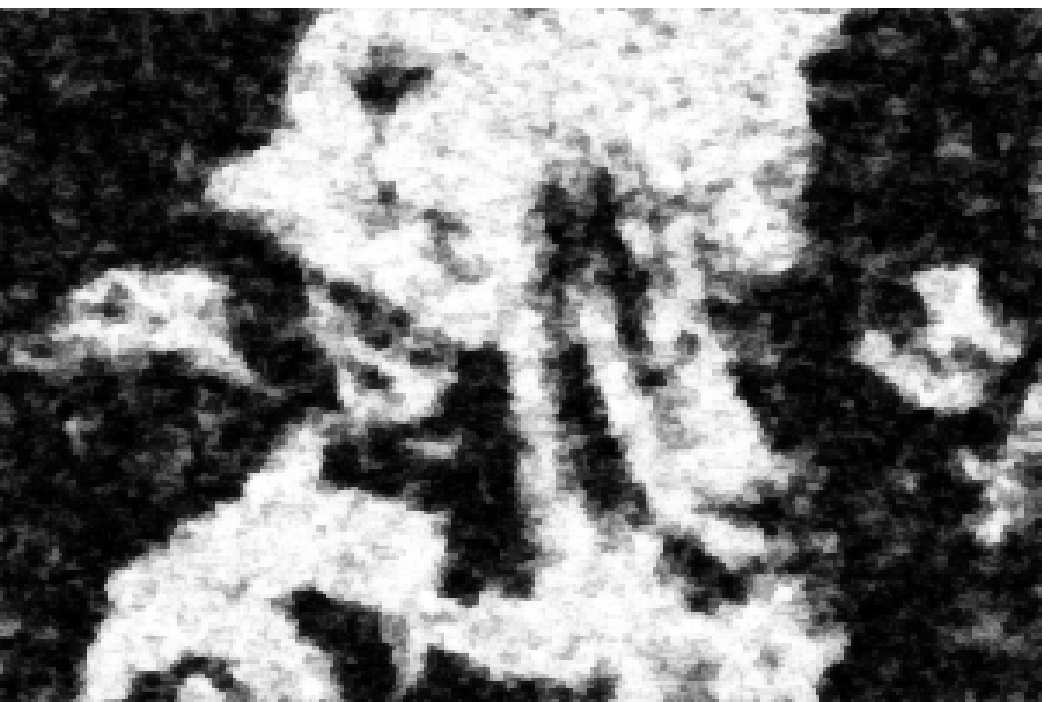} & \includegraphics[width=1\linewidth]{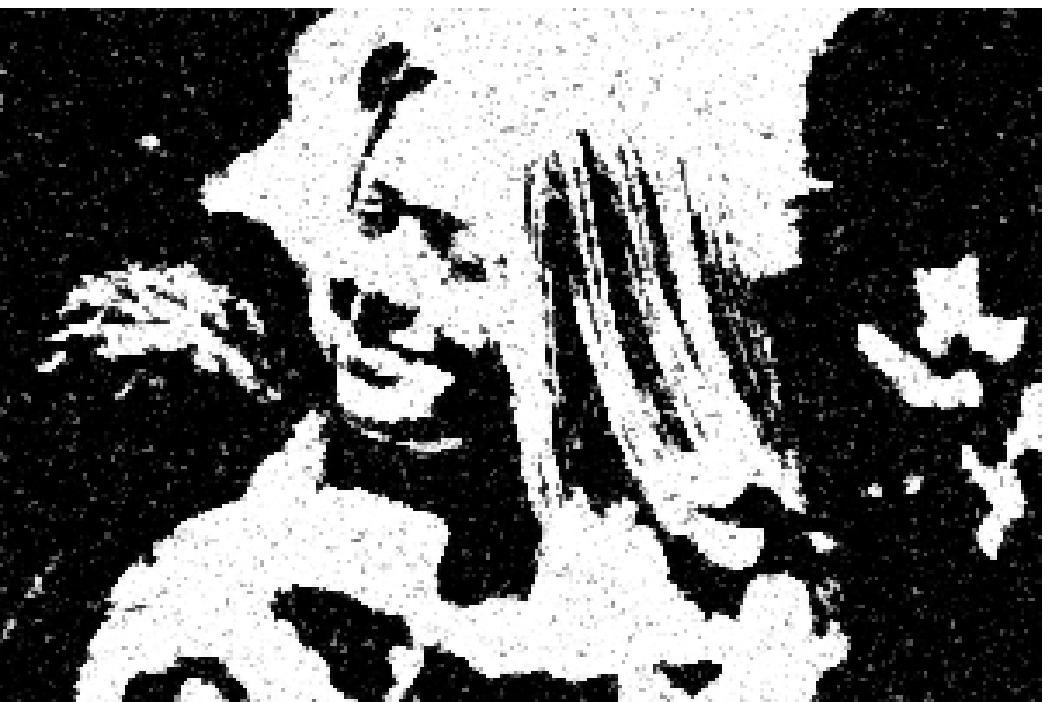}\tabularnewline
{\small pseudo-likelihood} & \includegraphics[width=1\linewidth]{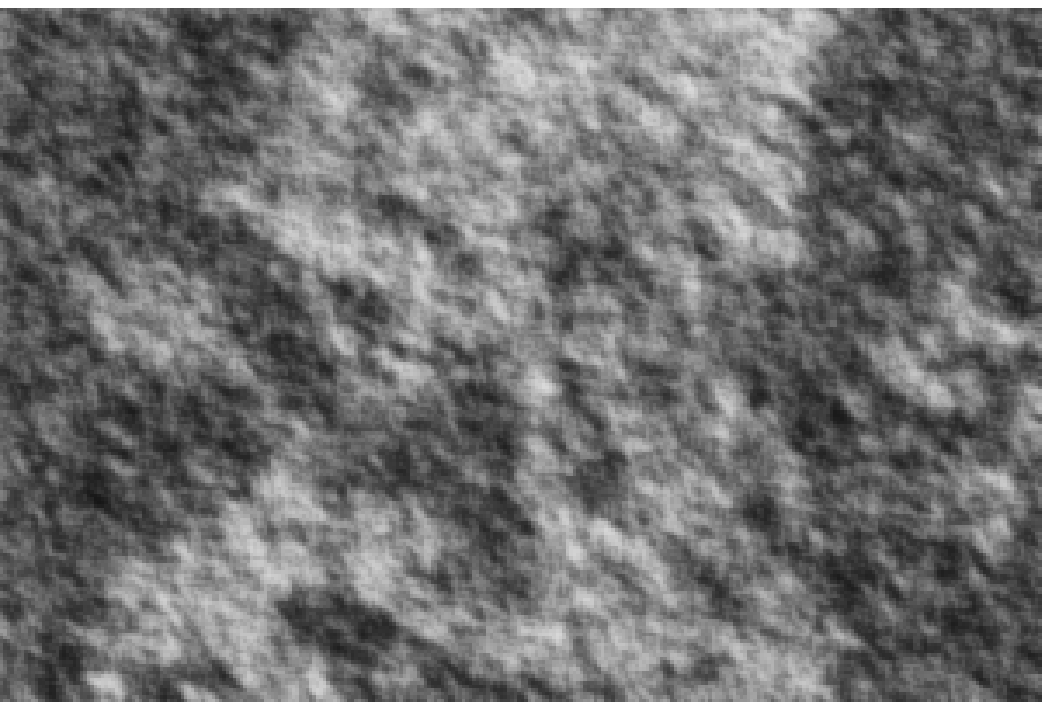} & \includegraphics[width=1\linewidth]{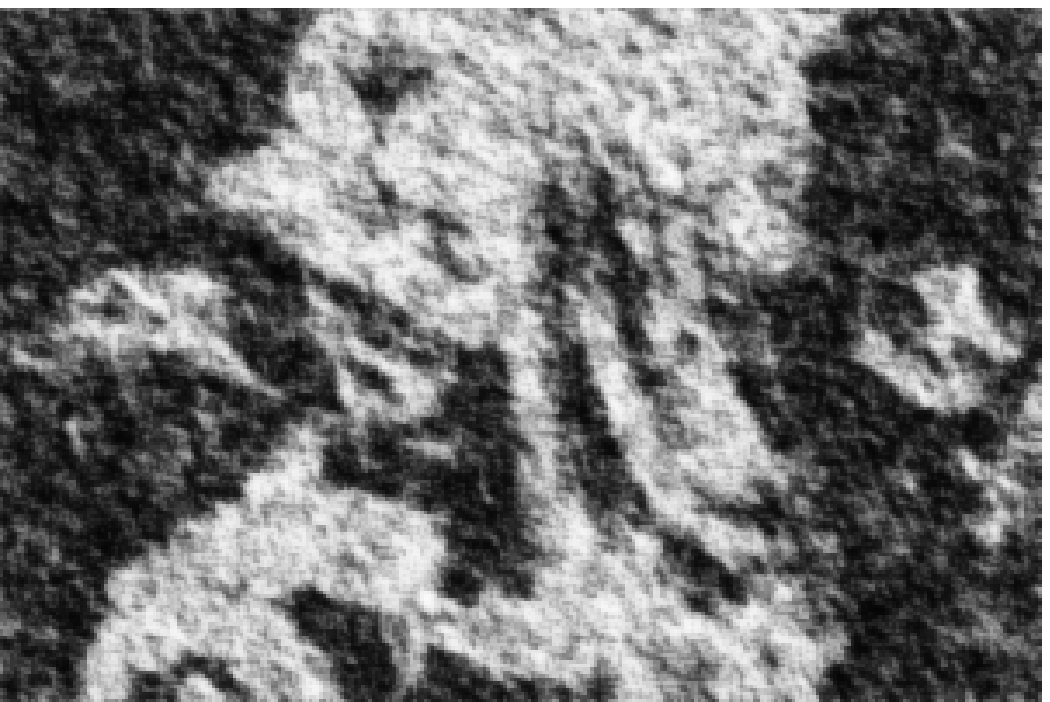} & \includegraphics[width=1\linewidth]{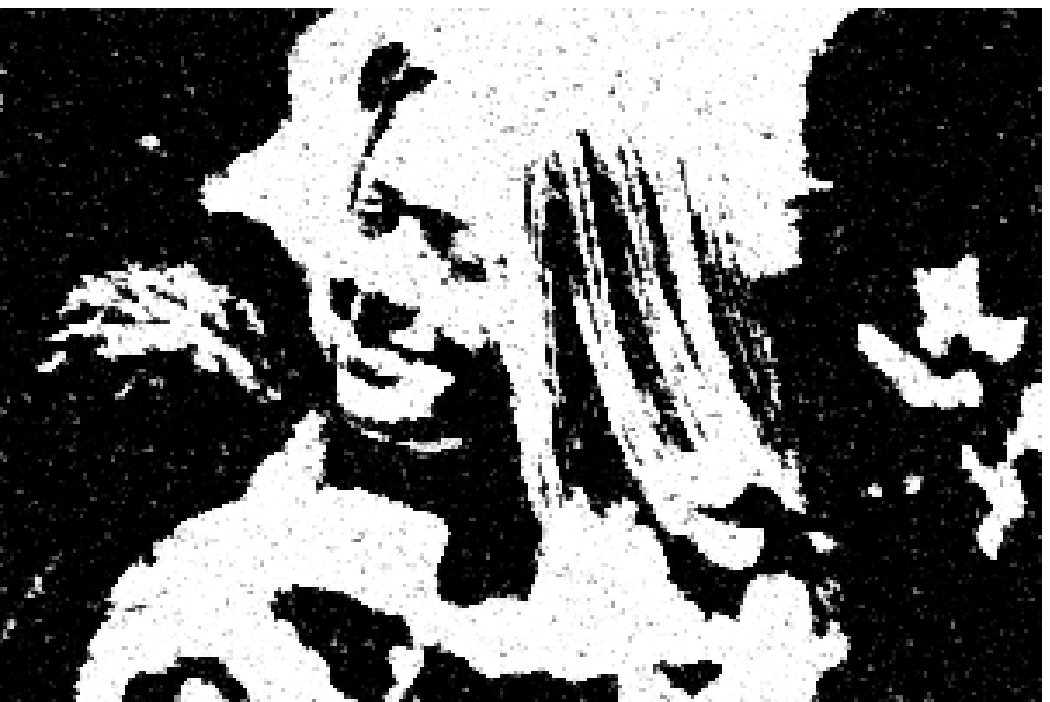}\tabularnewline
{\small piecewise} & \includegraphics[width=1\linewidth]{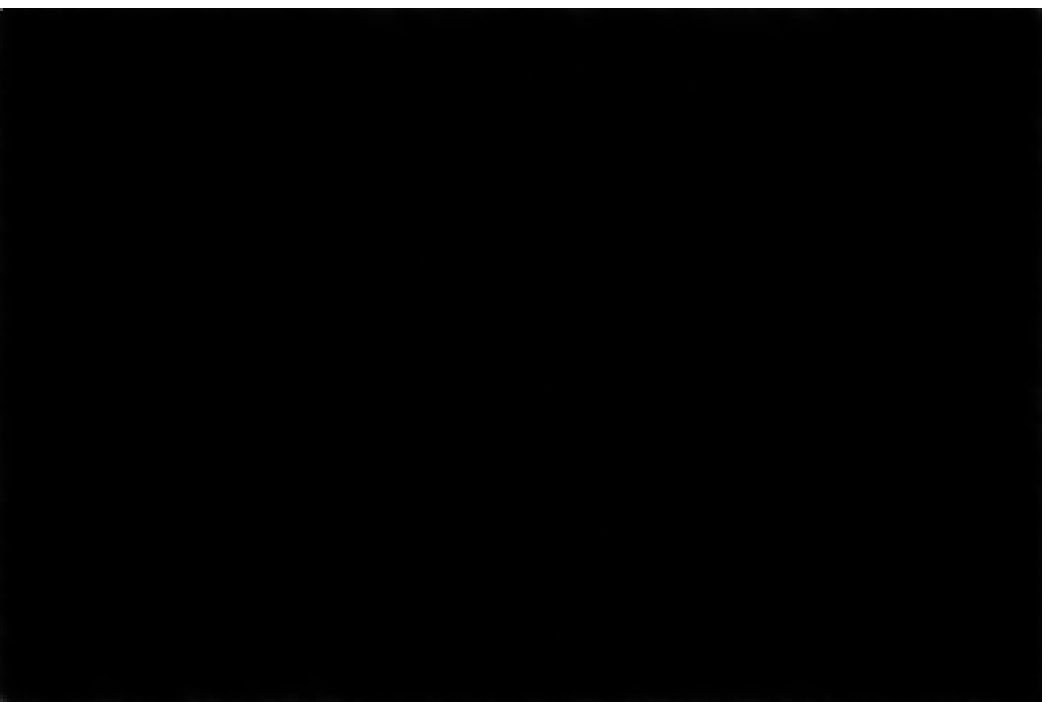} & \includegraphics[width=1\linewidth]{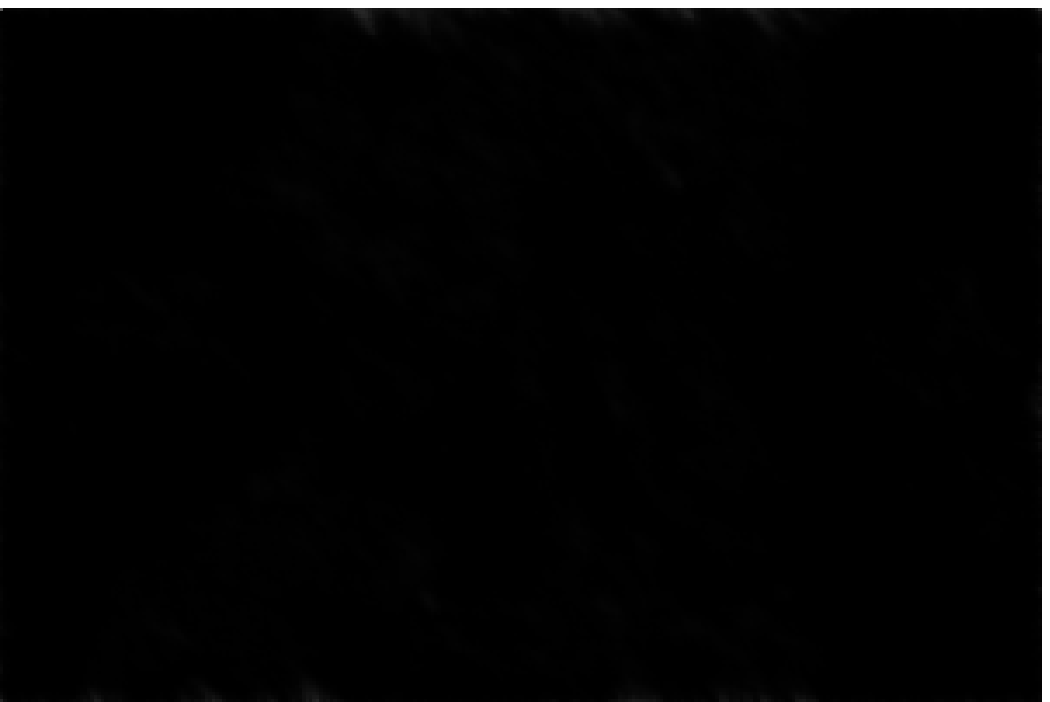} & \includegraphics[width=1\linewidth]{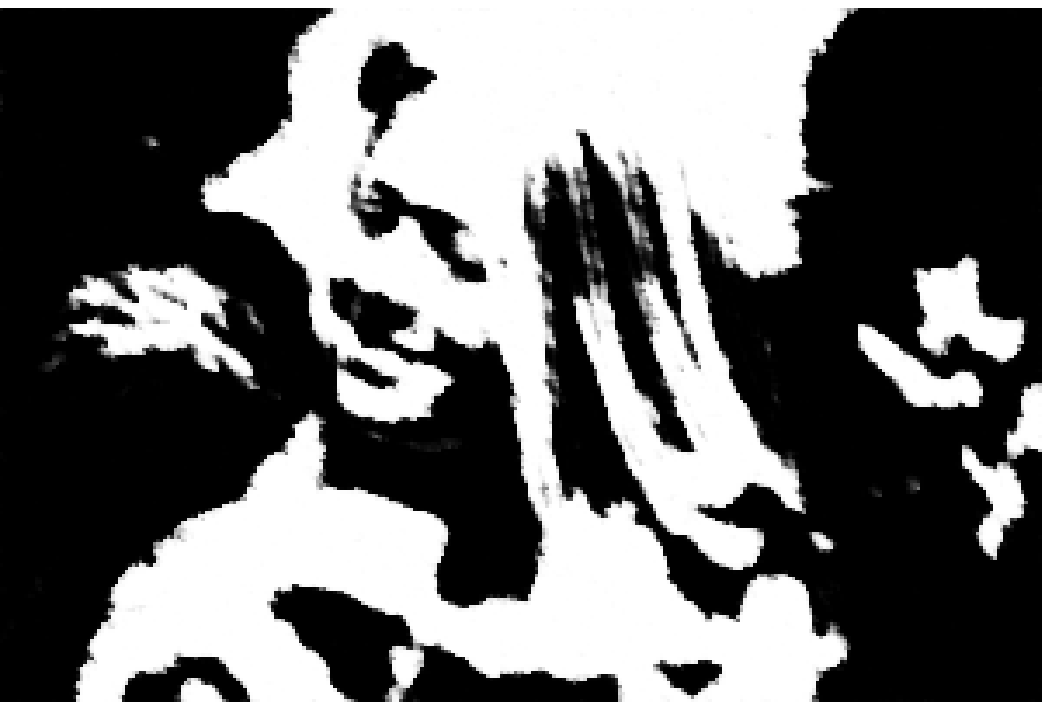}\tabularnewline
{\small inde-pendent} & \includegraphics[width=1\linewidth]{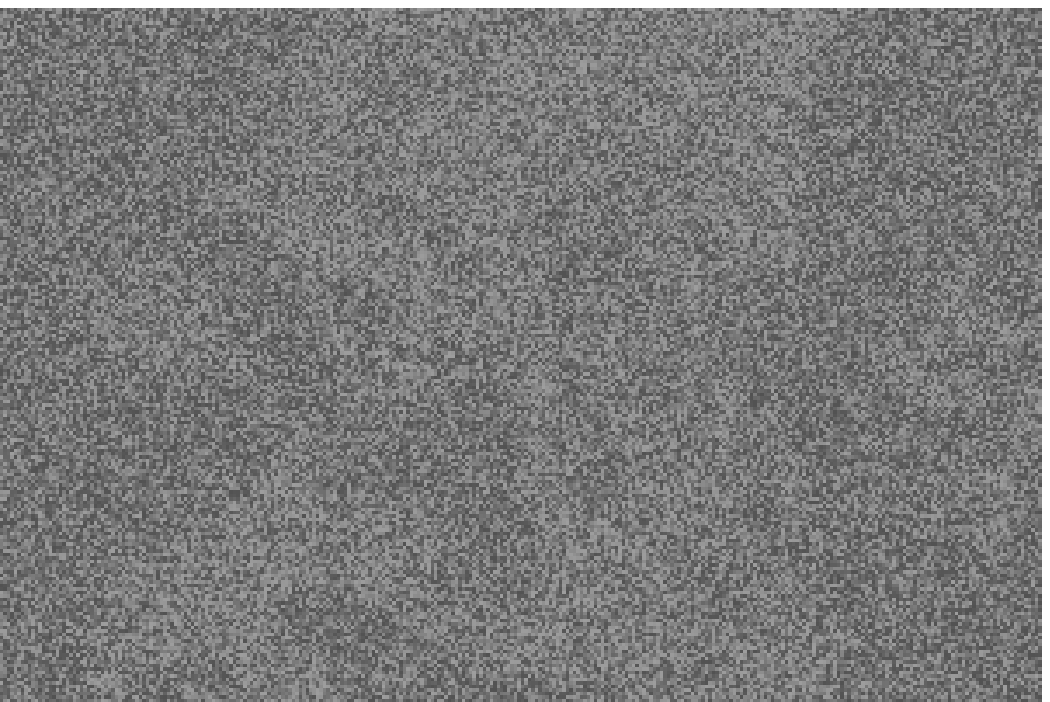} & \includegraphics[width=1\linewidth]{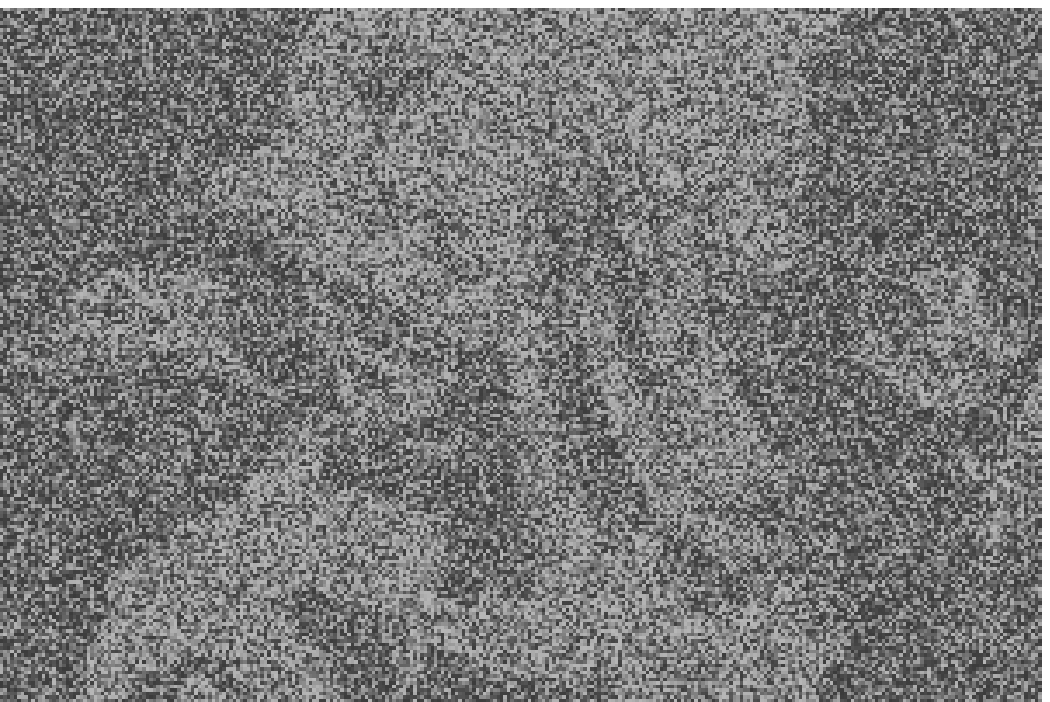} & \includegraphics[width=1\linewidth]{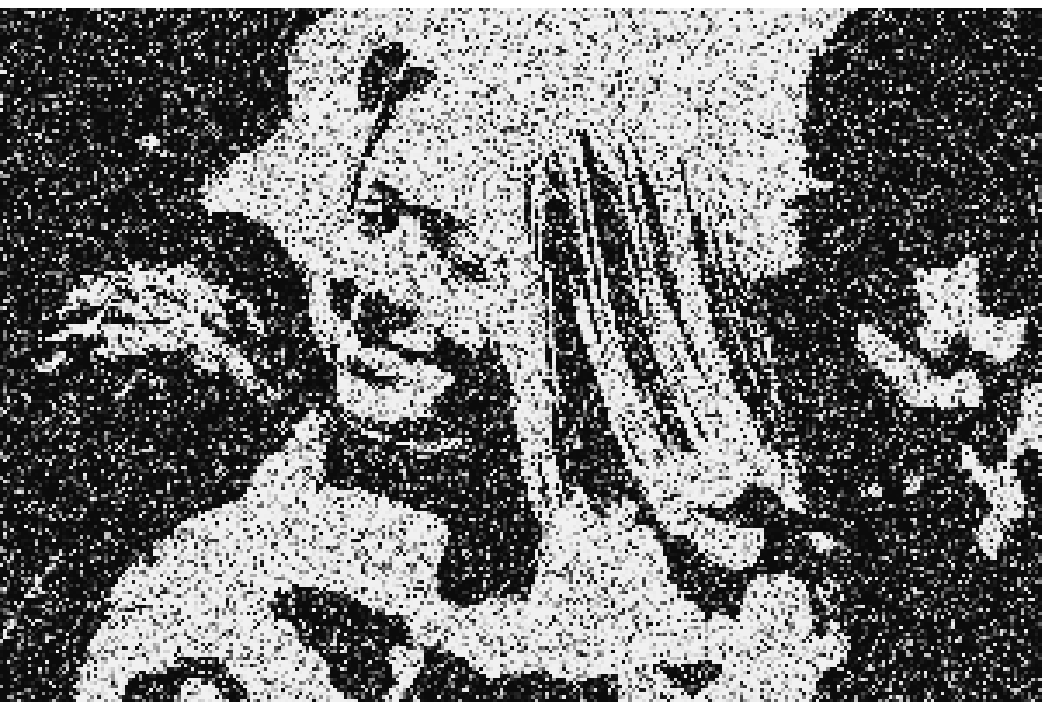}\tabularnewline
\end{tabular}
\par\end{centering}

\caption{Predicted marginals for an example binary denoising test image with
different noise levels $n$.}
\end{figure}
\begin{figure}[p]
\begin{centering}
\vspace{-10pt}
\renewcommand{\tabcolsep}{1pt}%
\begin{tabular}{>{\centering}m{0.57in}>{\centering}m{0.9in}>{\centering}m{0.9in}>{\centering}m{0.9in}}
 & $n=1.25$ & $n=1.5$ & $n=5$\tabularnewline
{\small input} & \includegraphics[width=1\linewidth]{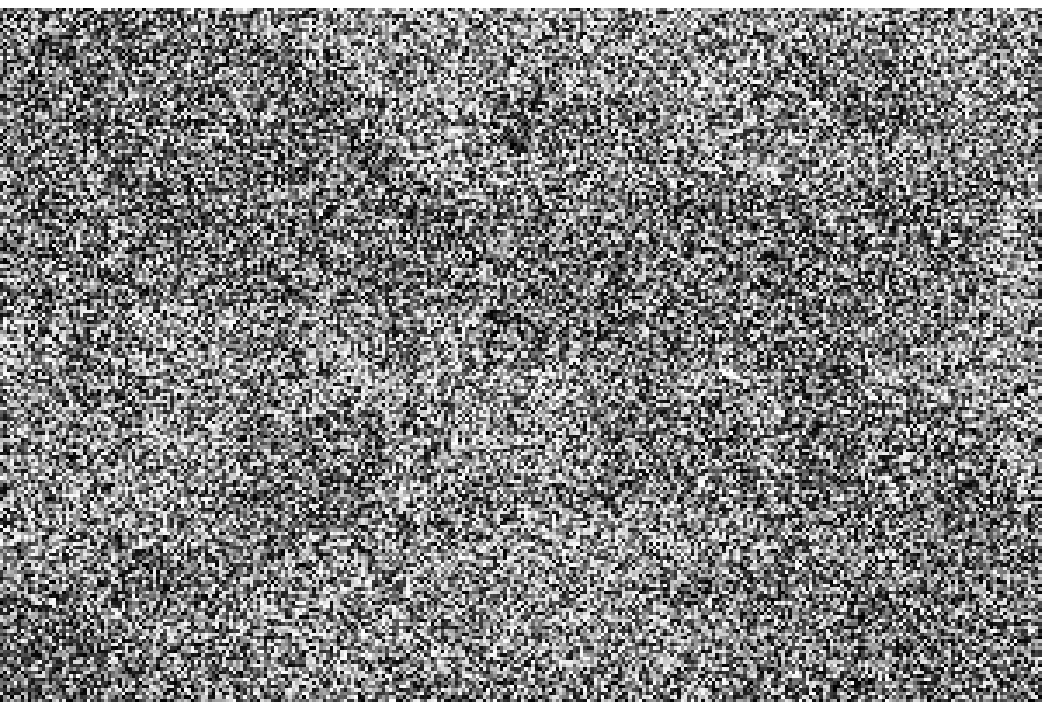} & \includegraphics[width=1\linewidth]{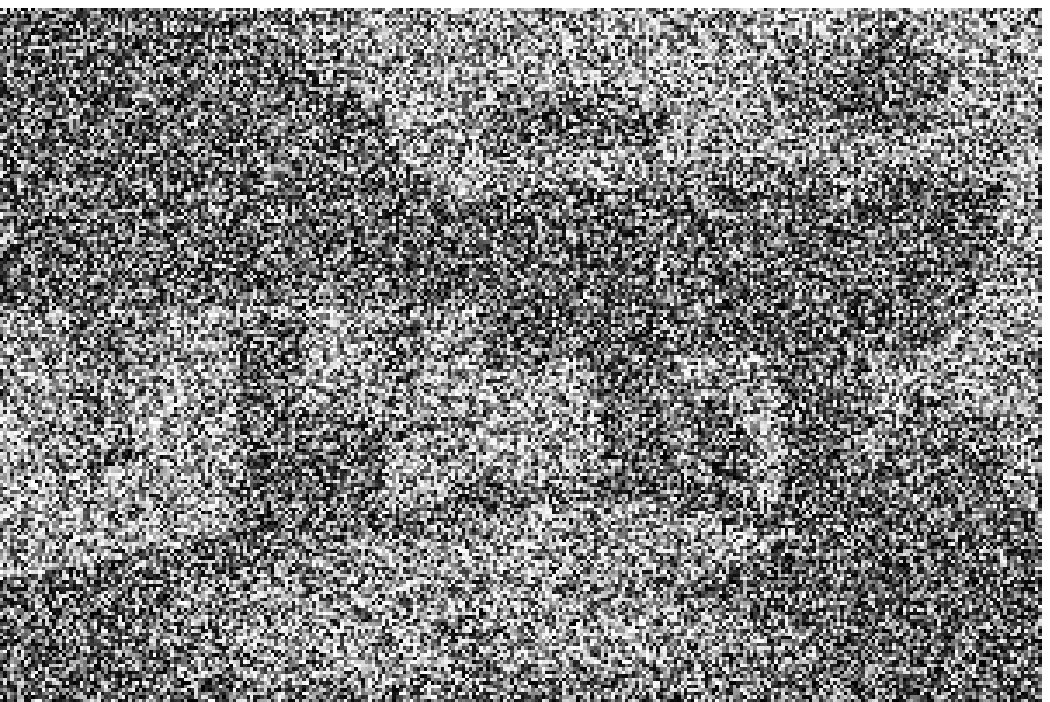} & \includegraphics[width=1\linewidth]{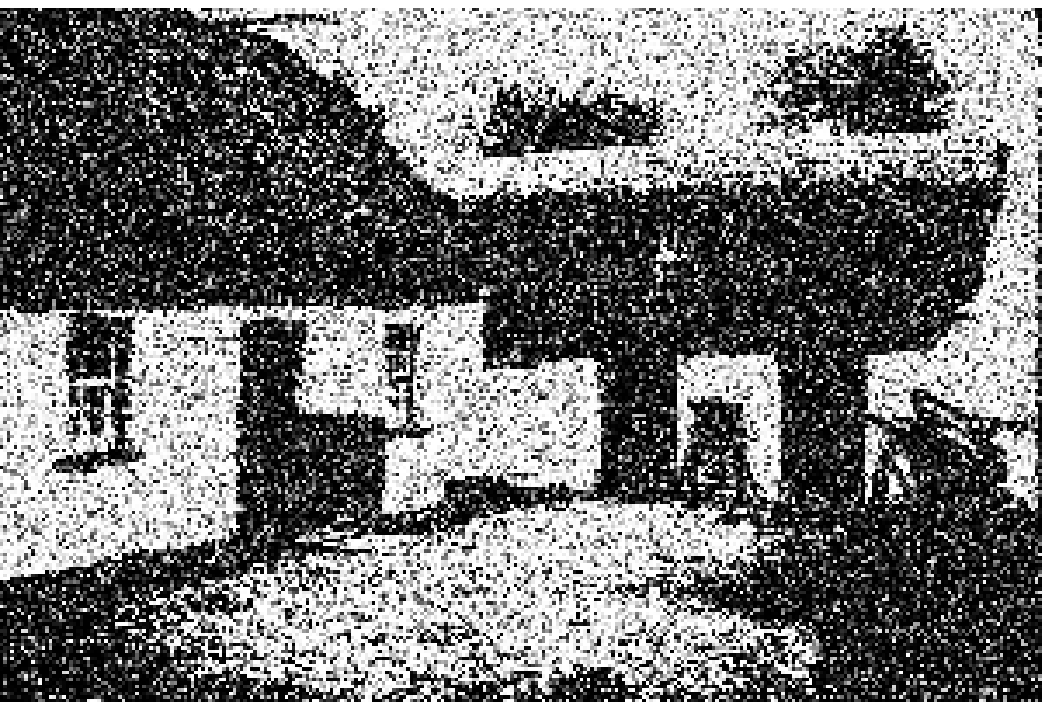}\tabularnewline
{\small surrogate likelihood} & \includegraphics[width=1\linewidth]{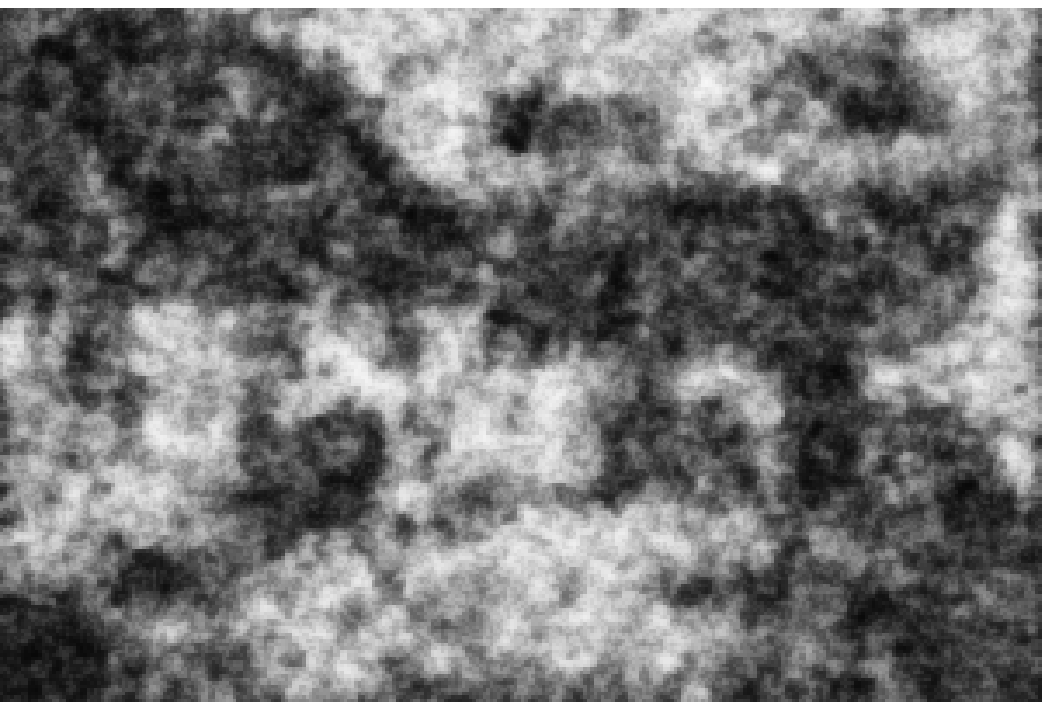} & \includegraphics[width=1\linewidth]{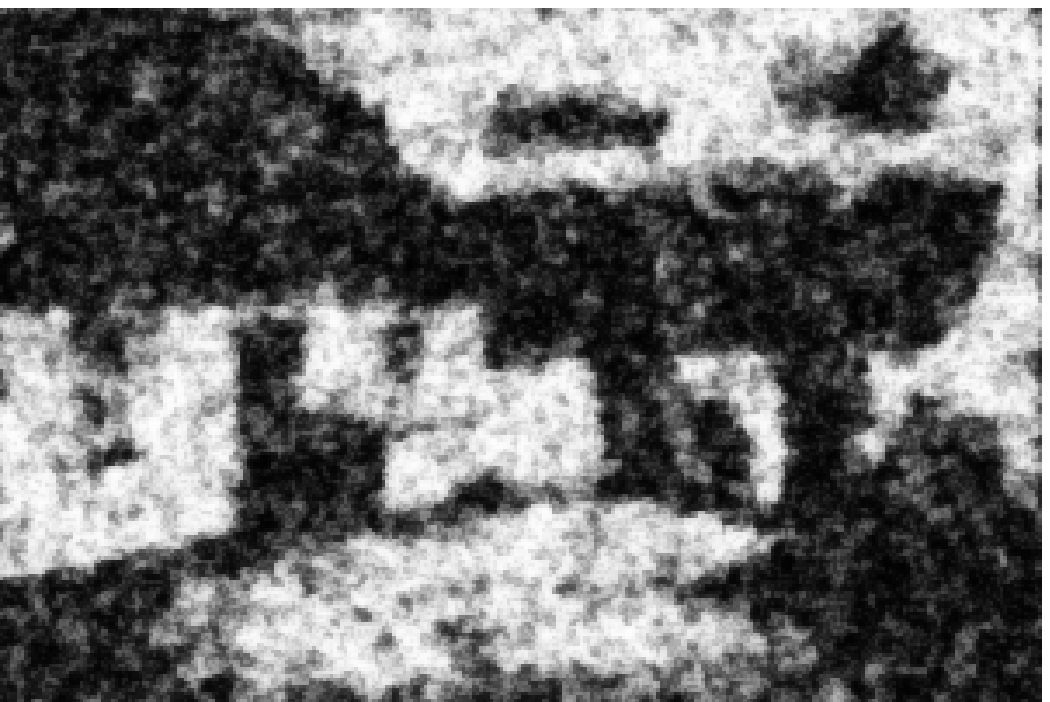} & \includegraphics[width=1\linewidth]{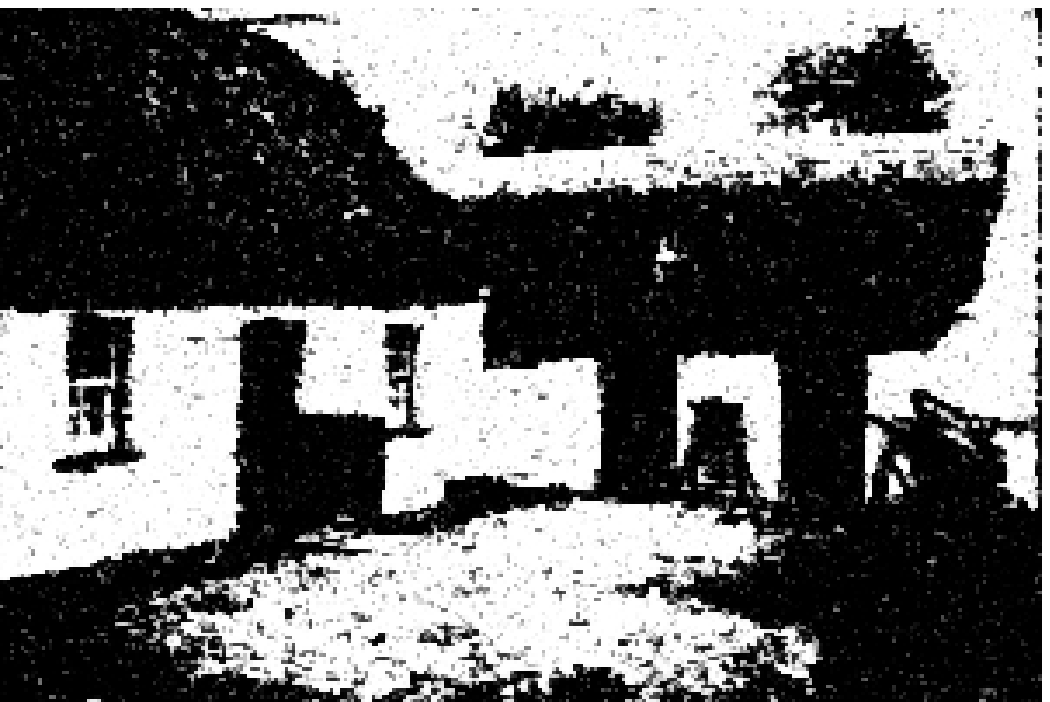}\tabularnewline
{\small univariate logistic} & \includegraphics[width=1\linewidth]{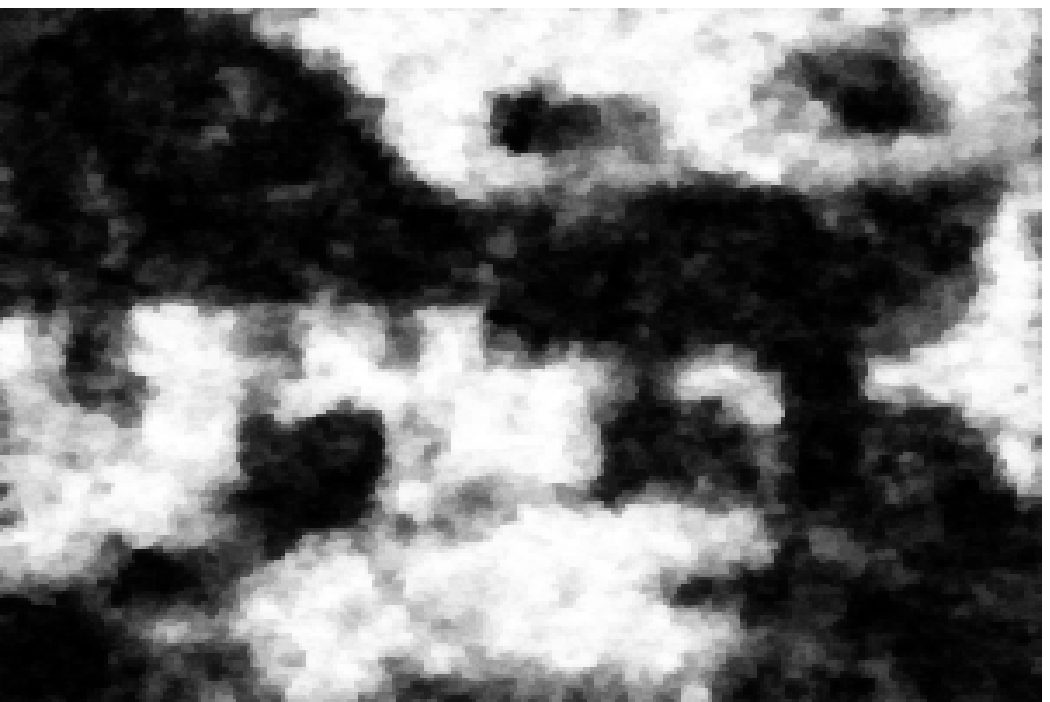} & \includegraphics[width=1\linewidth]{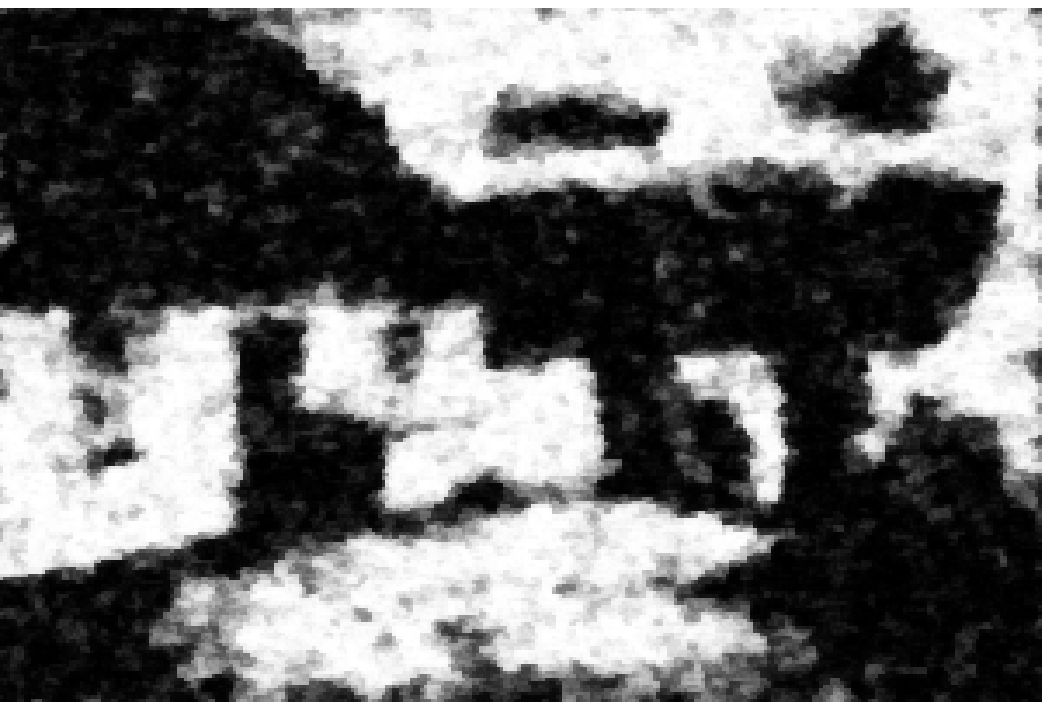} & \includegraphics[width=1\linewidth]{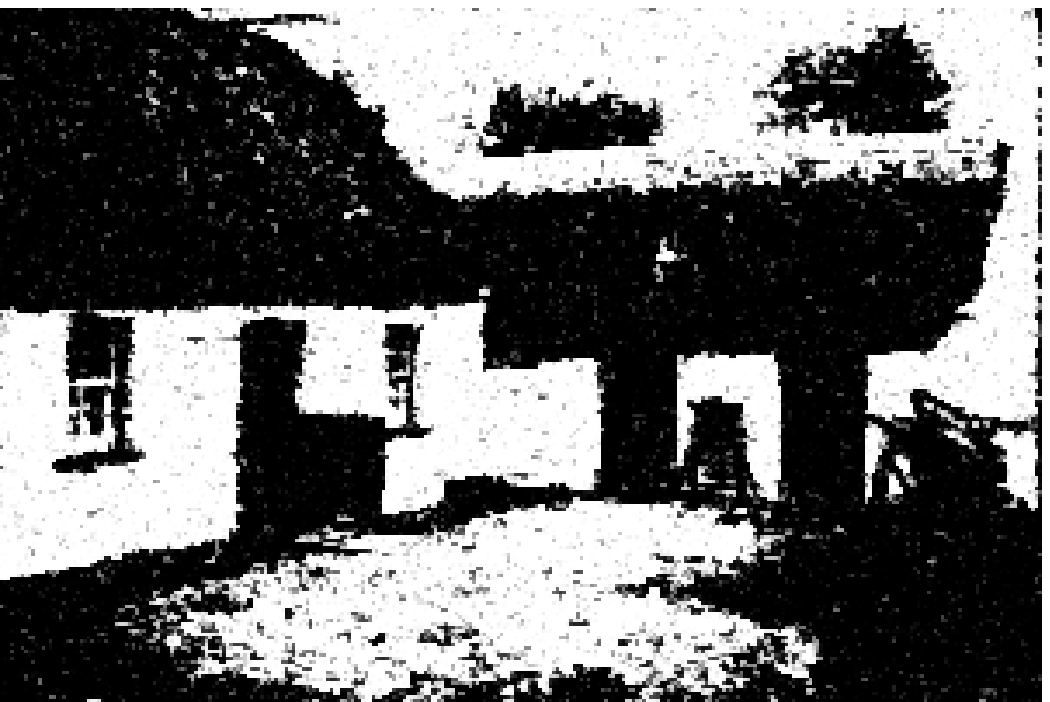}\tabularnewline
{\small clique logistic} & \includegraphics[width=1\linewidth]{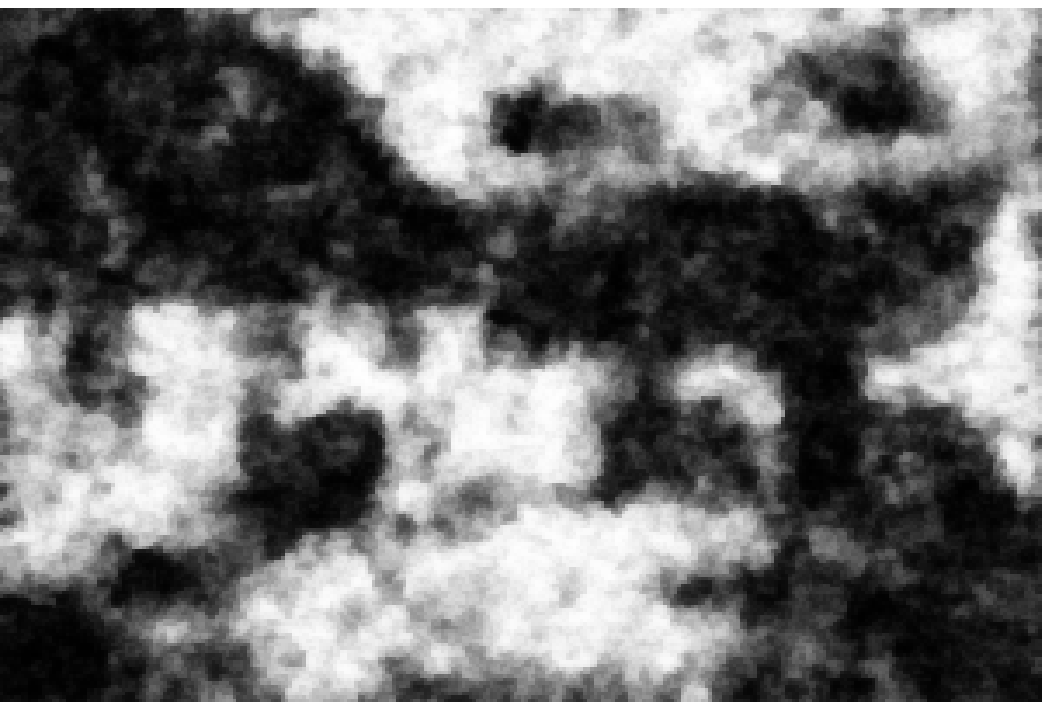} & \includegraphics[width=1\linewidth]{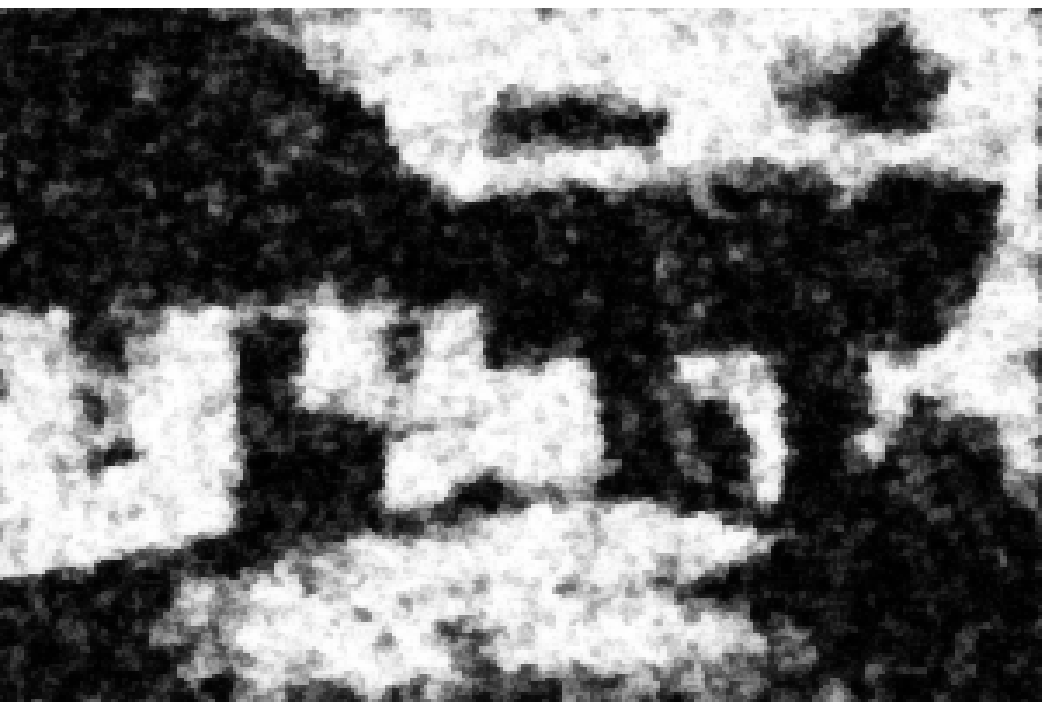} & \includegraphics[width=1\linewidth]{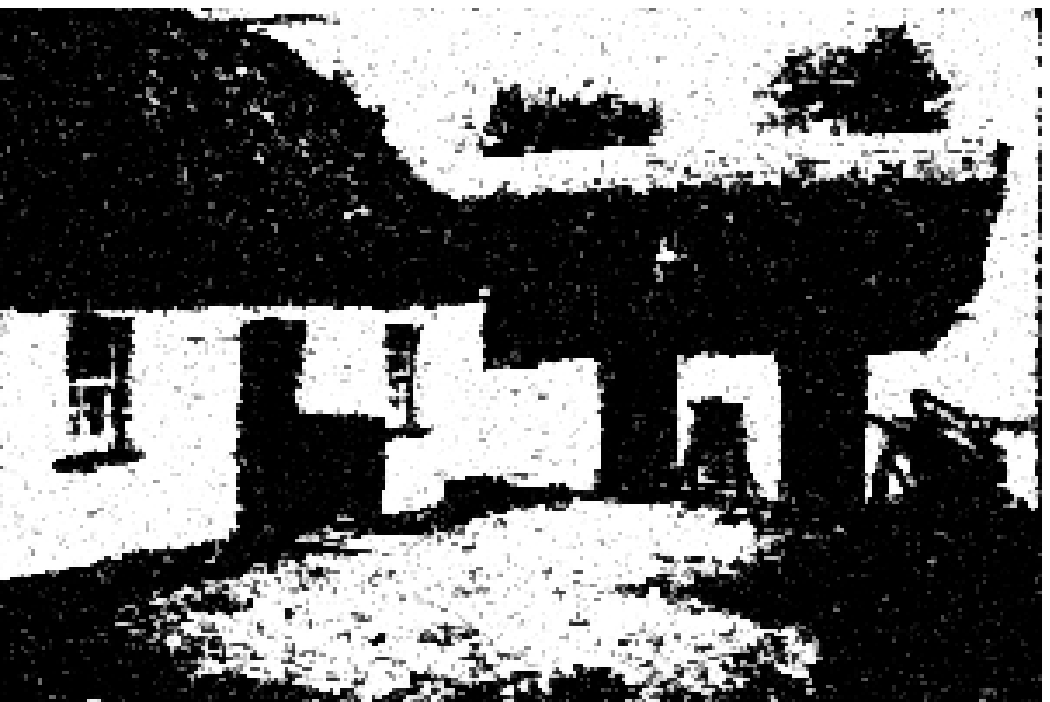}\tabularnewline
sm. class $\lambda=5$ & \includegraphics[width=1\linewidth]{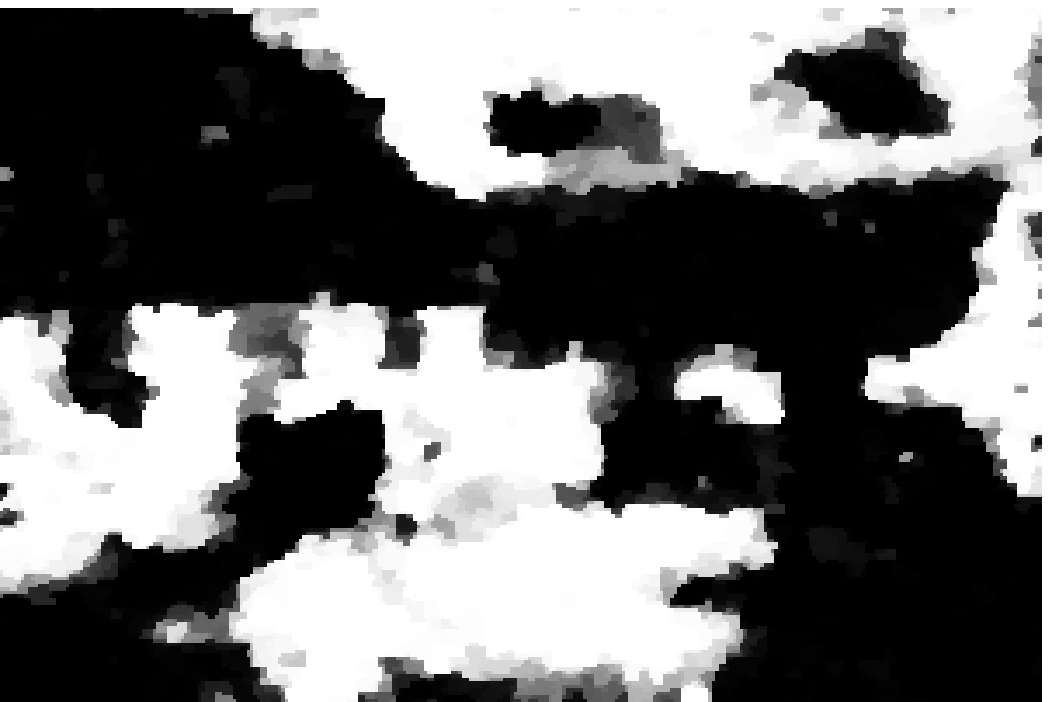} & \includegraphics[width=1\linewidth]{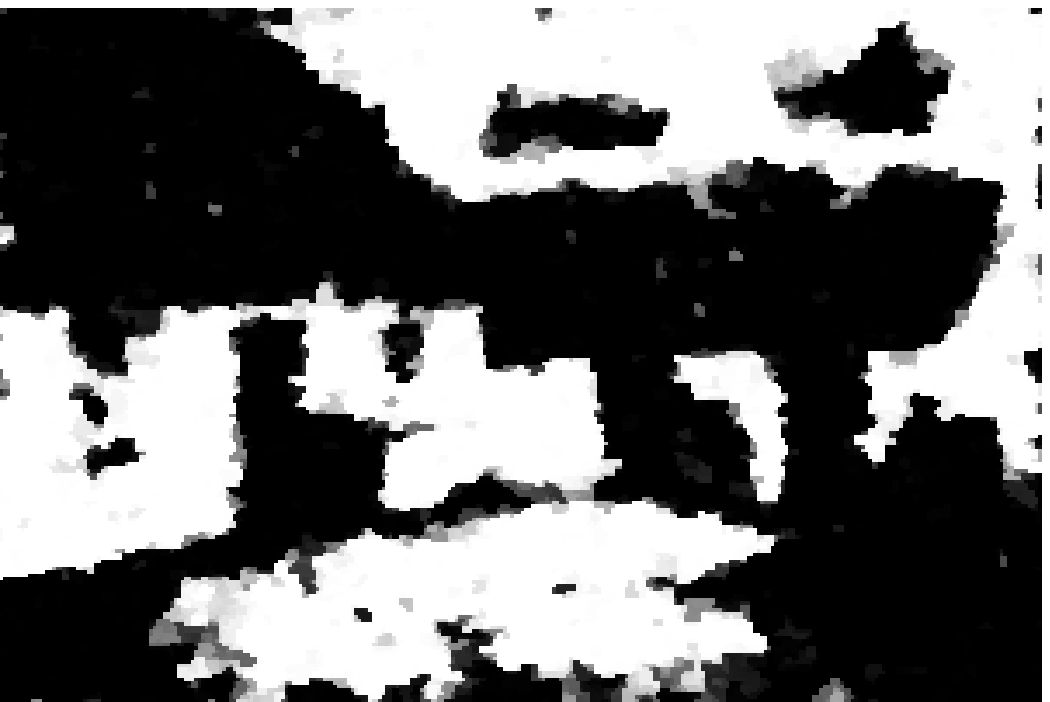} & \includegraphics[width=1\linewidth]{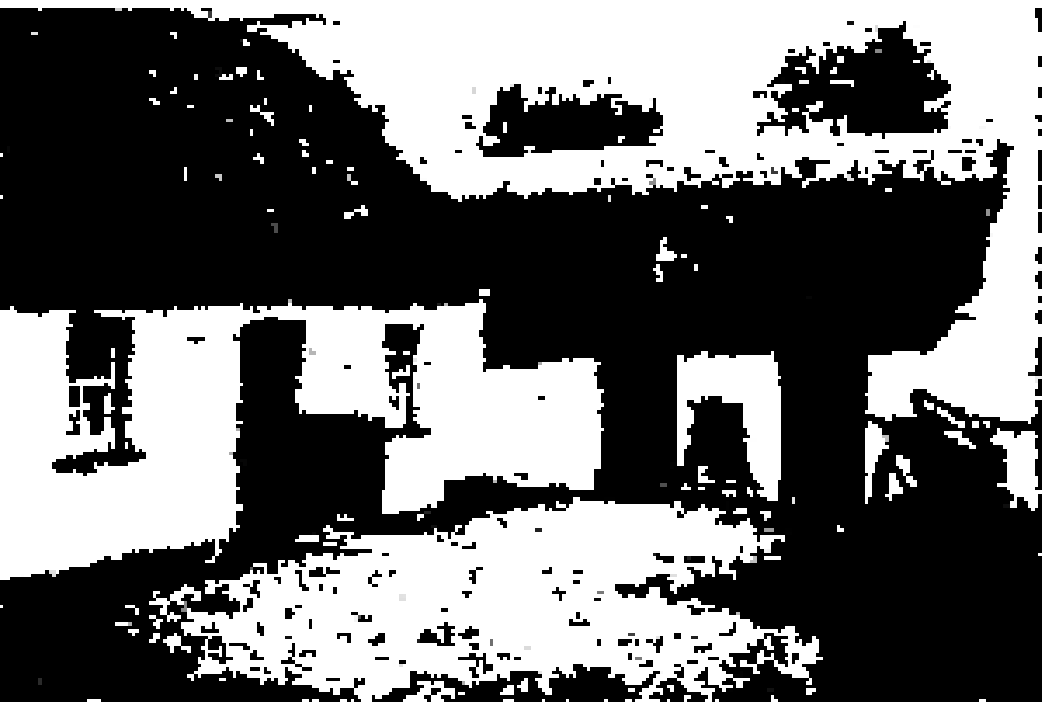}\tabularnewline
sm. class $\lambda=15$ & \includegraphics[width=1\linewidth]{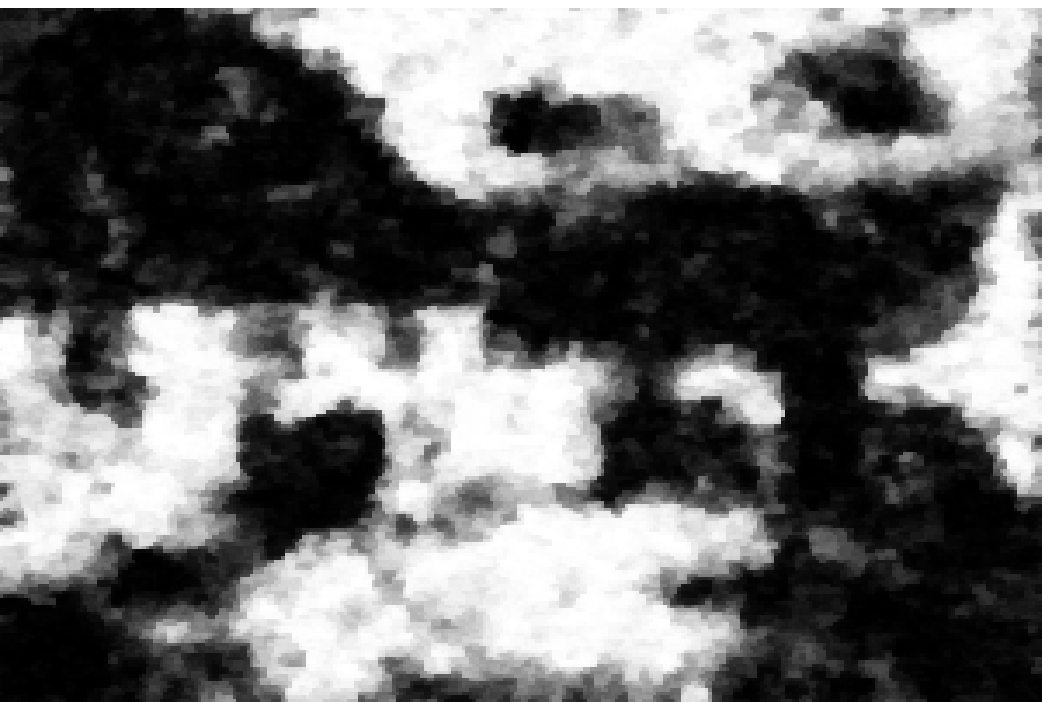} & \includegraphics[width=1\linewidth]{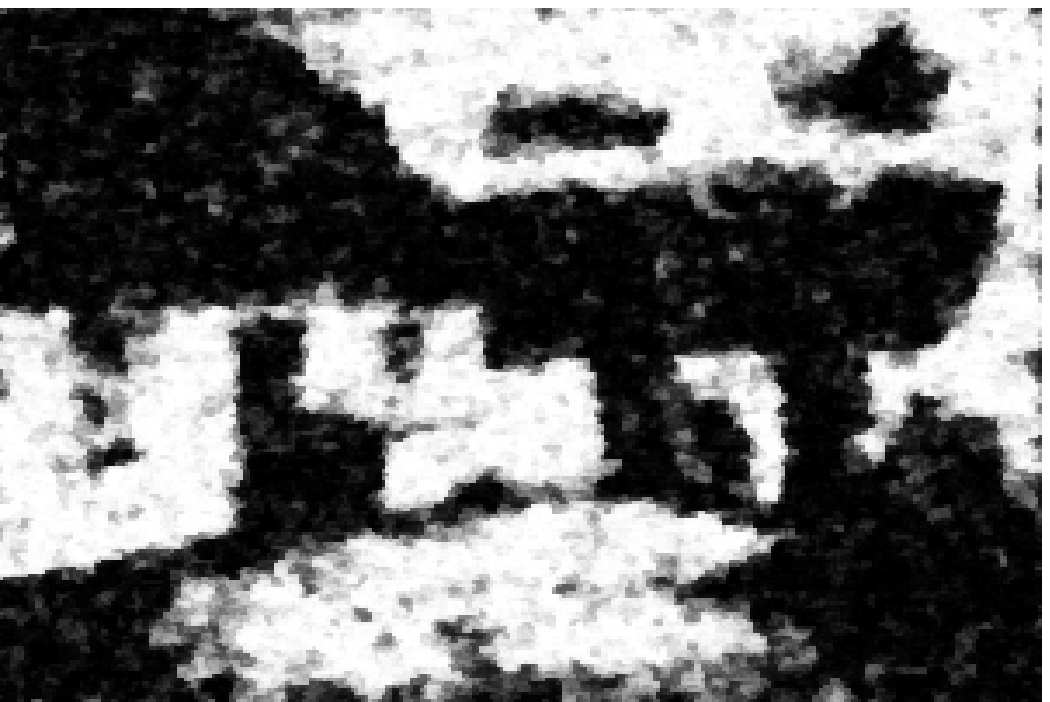} & \includegraphics[width=1\linewidth]{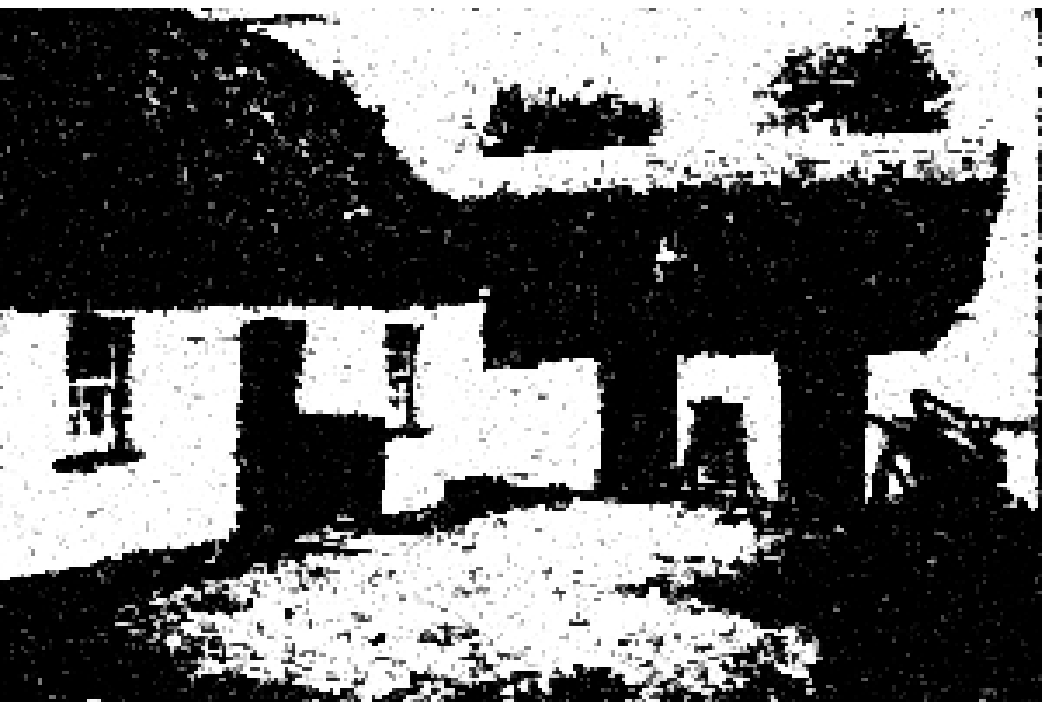}\tabularnewline
sm. class $\lambda=50$ & \includegraphics[width=1\linewidth]{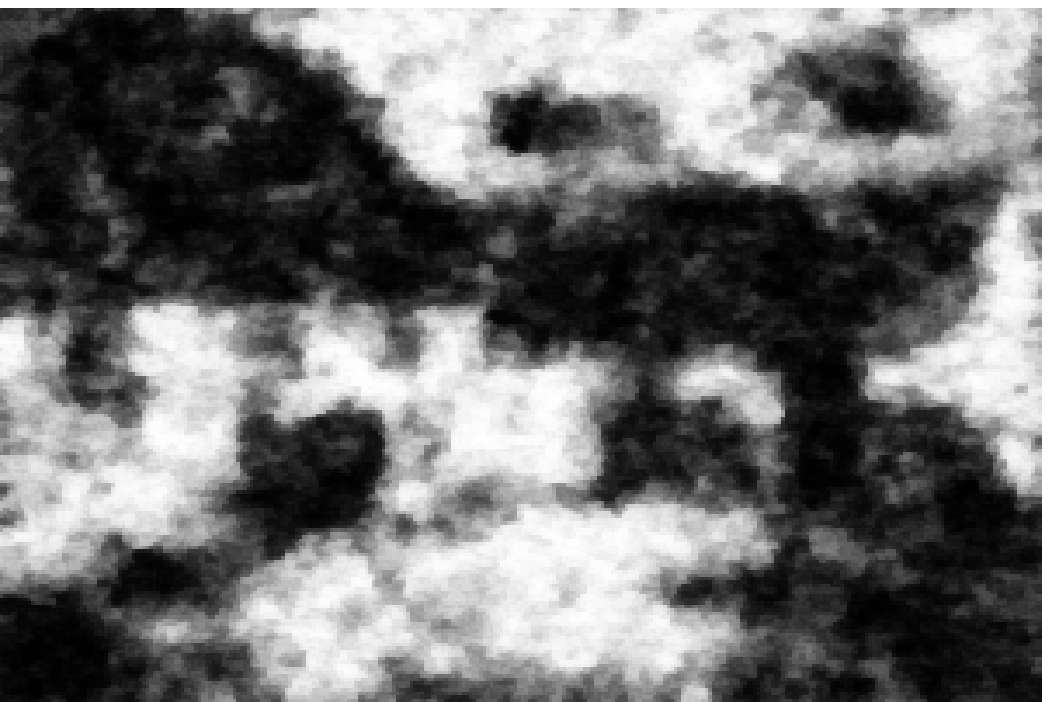} & \includegraphics[width=1\linewidth]{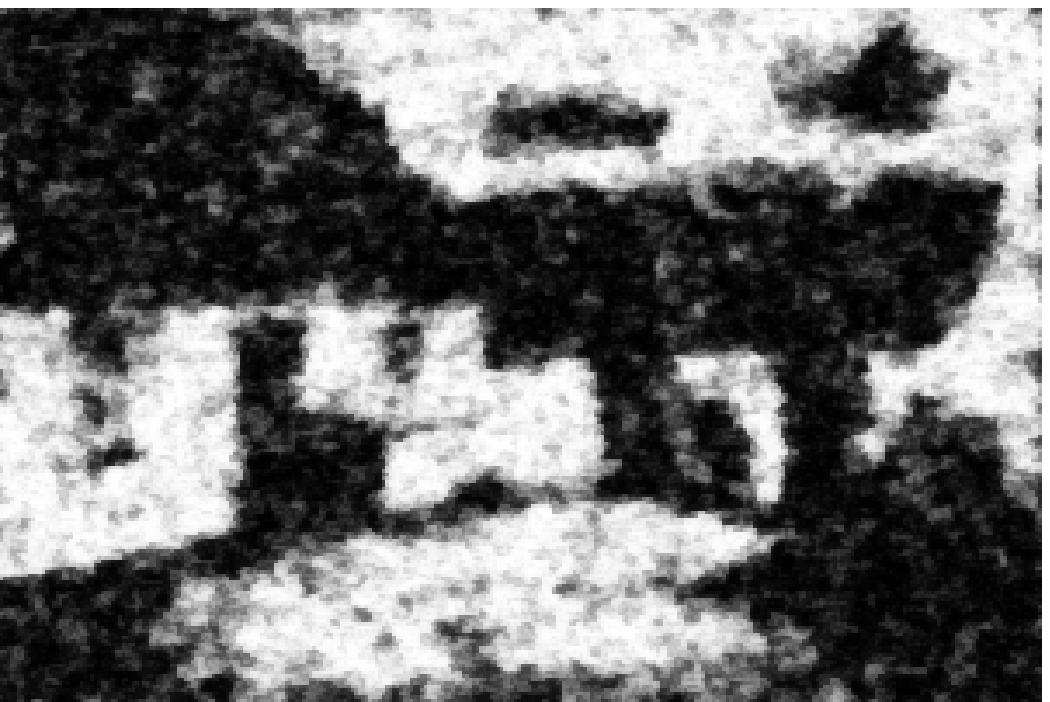} & \includegraphics[width=1\linewidth]{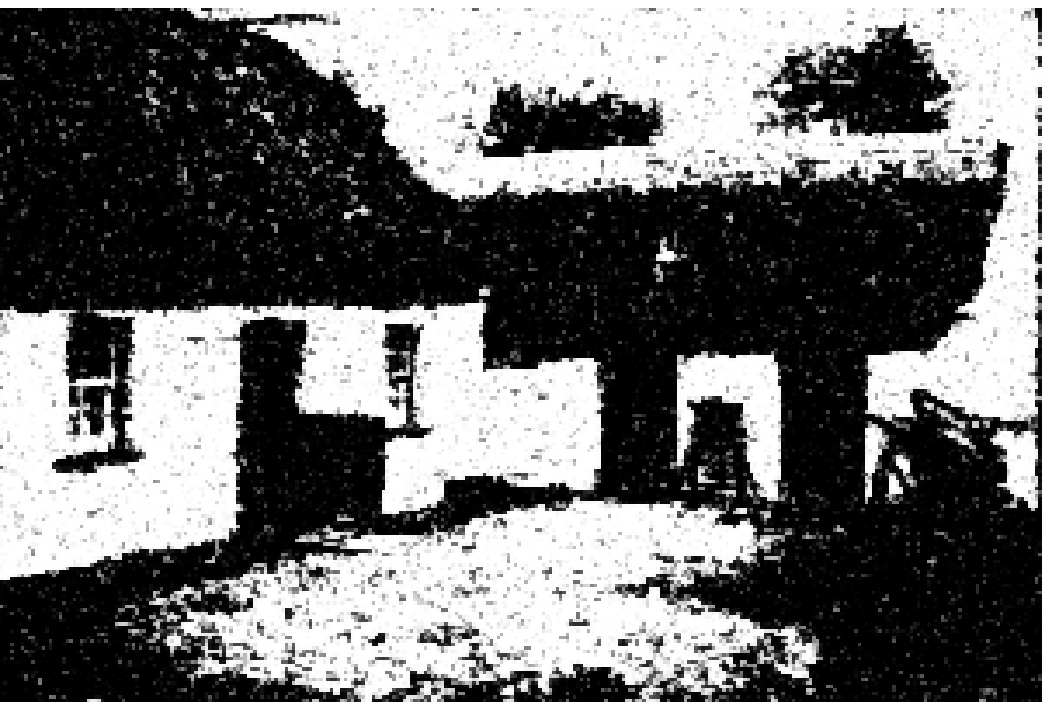}\tabularnewline
{\small pseudo-likelihood} & \includegraphics[width=1\linewidth]{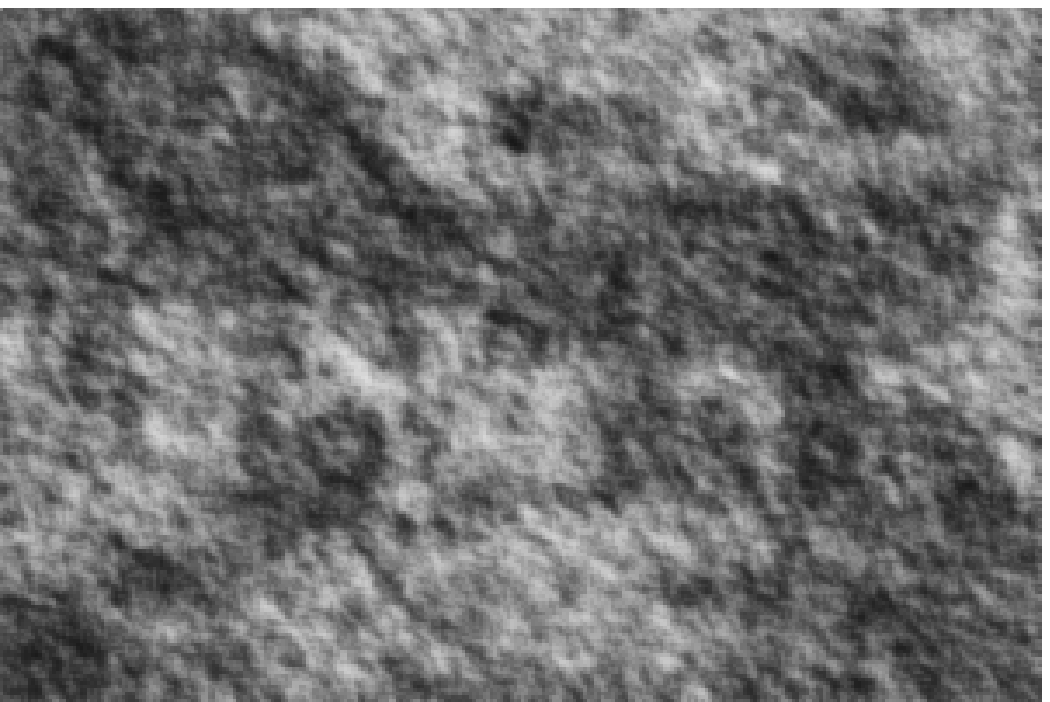} & \includegraphics[width=1\linewidth]{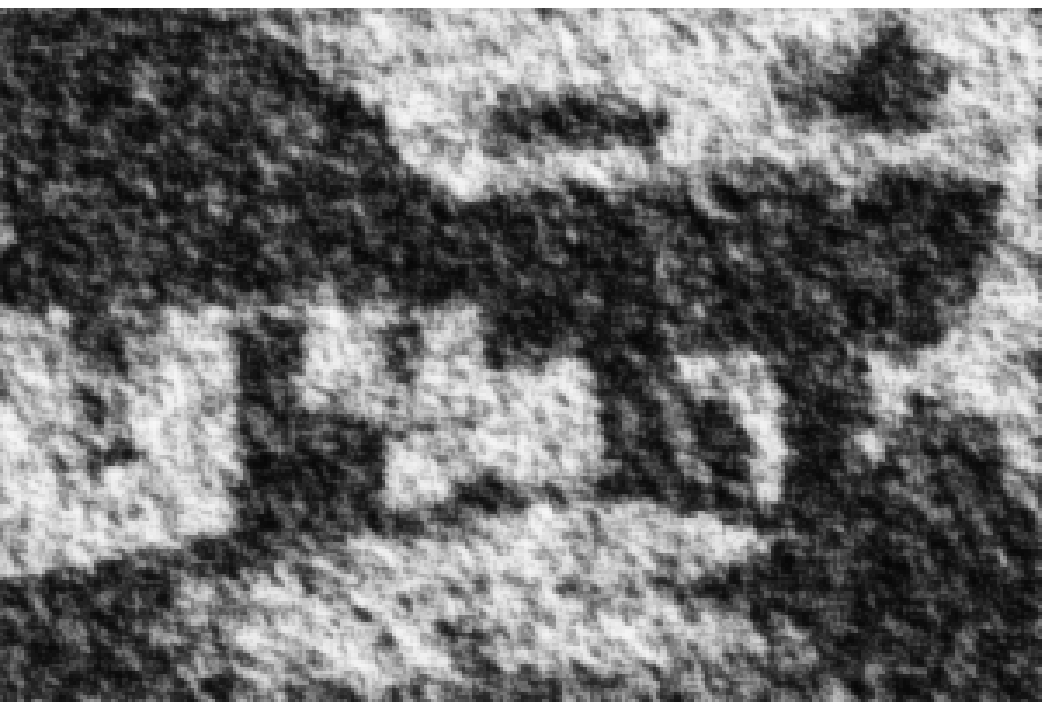} & \includegraphics[width=1\linewidth]{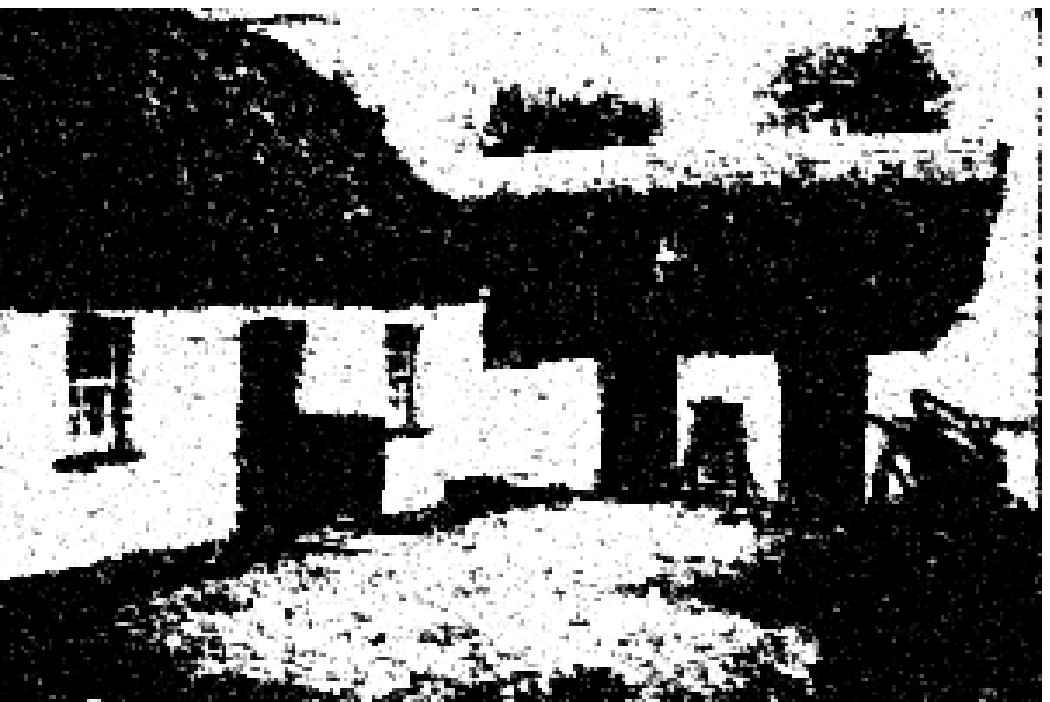}\tabularnewline
{\small piecewise} & \includegraphics[width=1\linewidth]{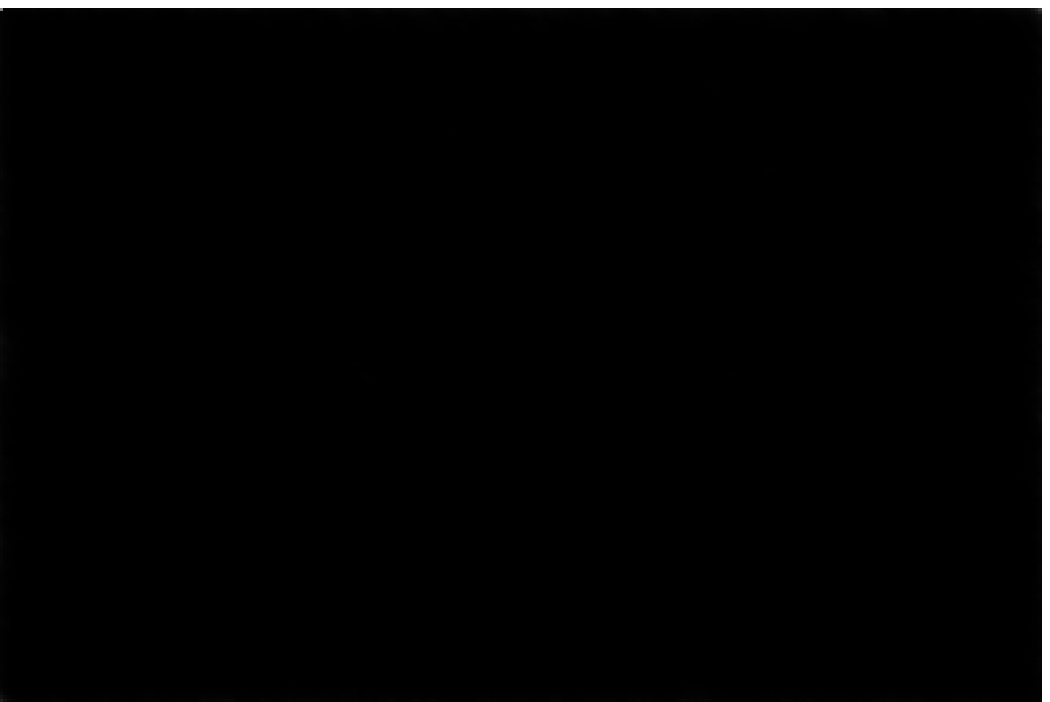} & \includegraphics[width=1\linewidth]{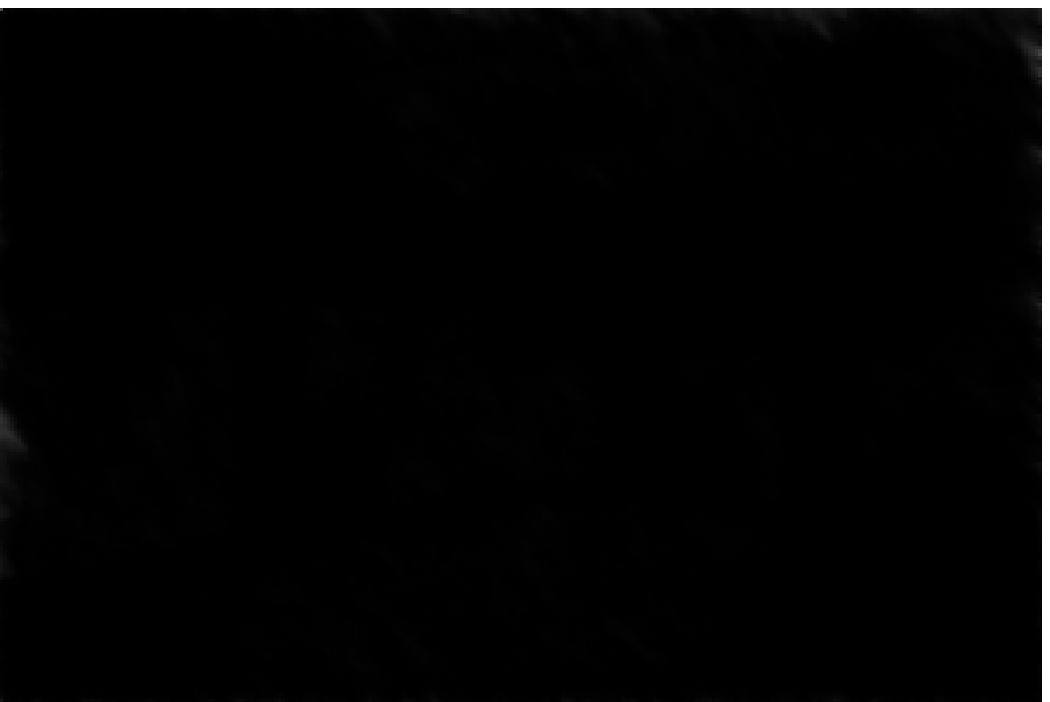} & \includegraphics[width=1\linewidth]{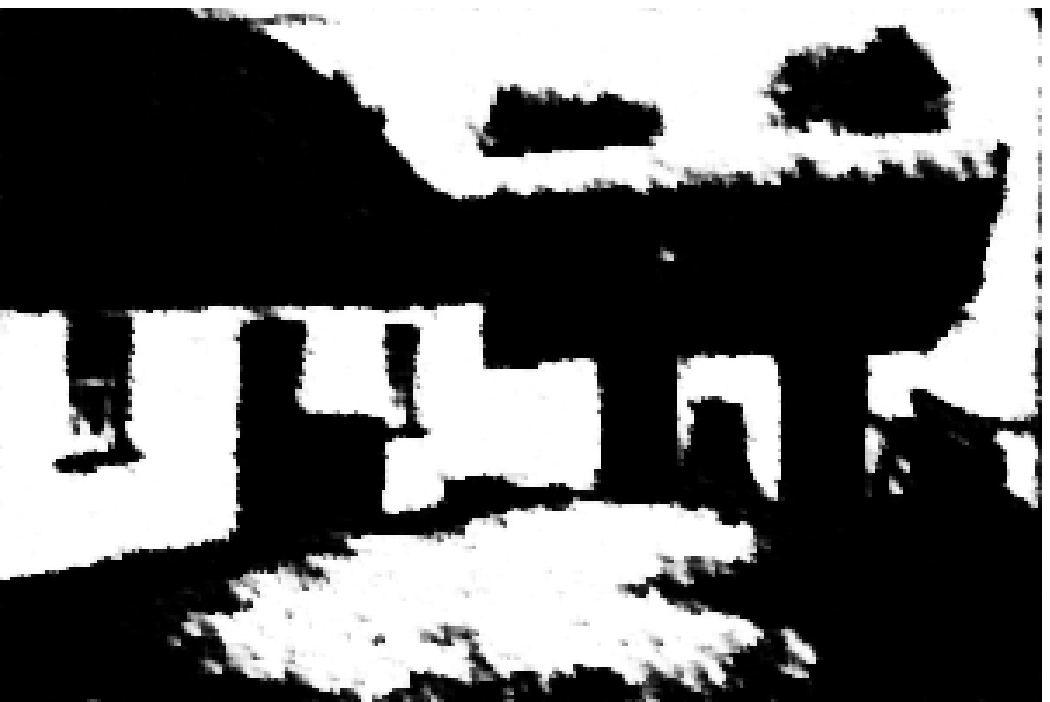}\tabularnewline
{\small inde-pendent} & \includegraphics[width=1\linewidth]{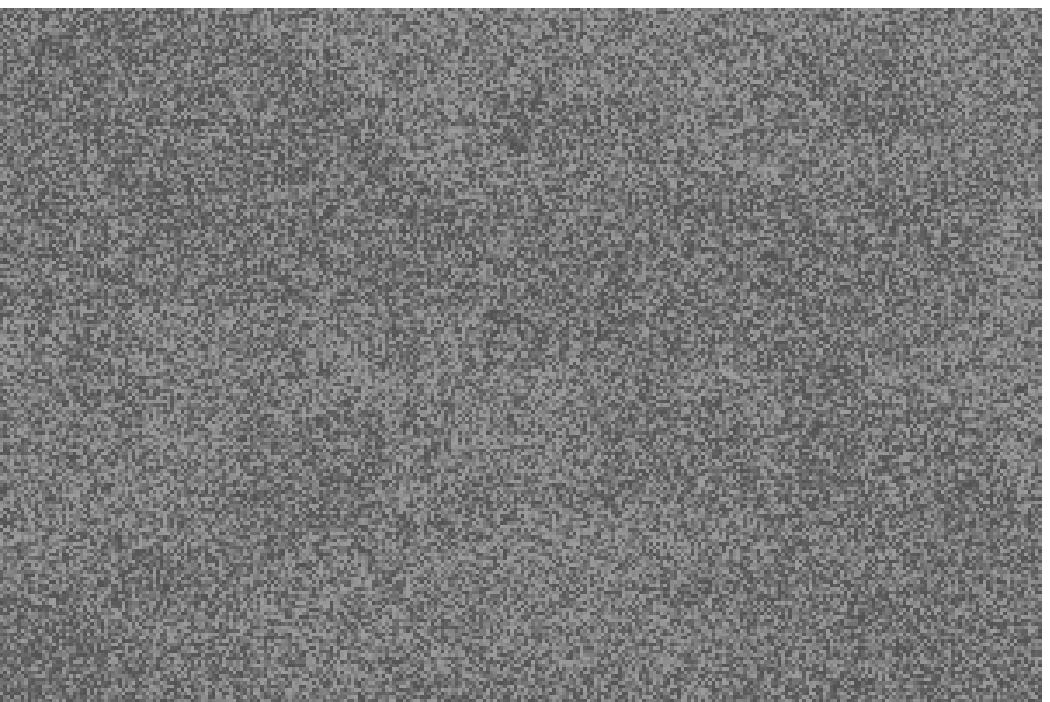} & \includegraphics[width=1\linewidth]{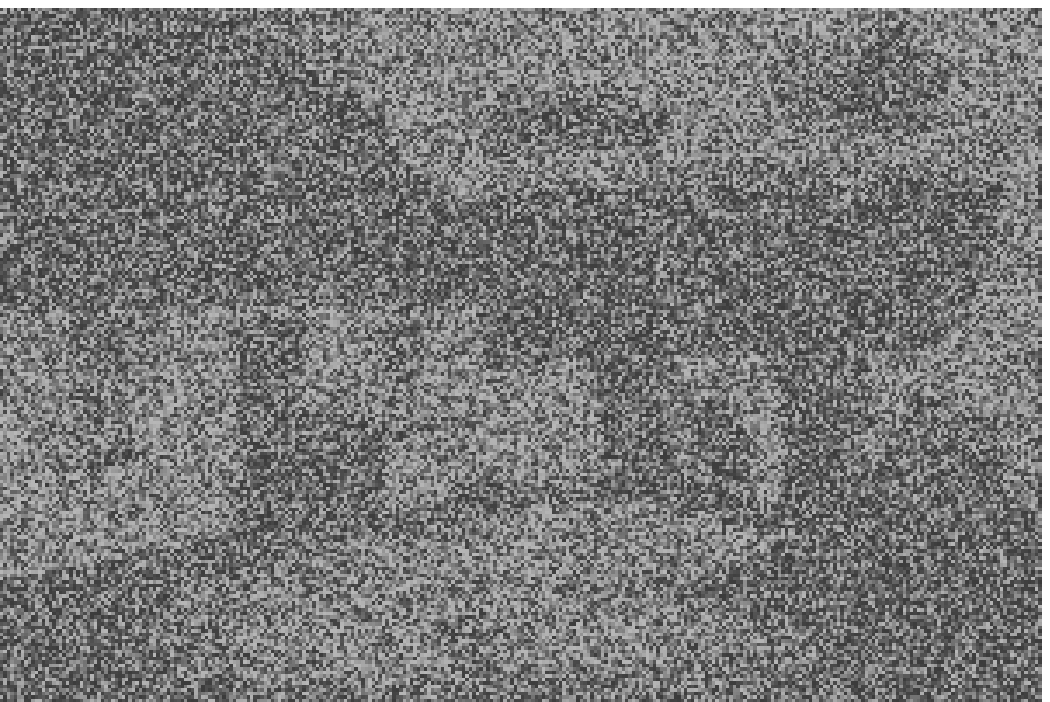} & \includegraphics[width=1\linewidth]{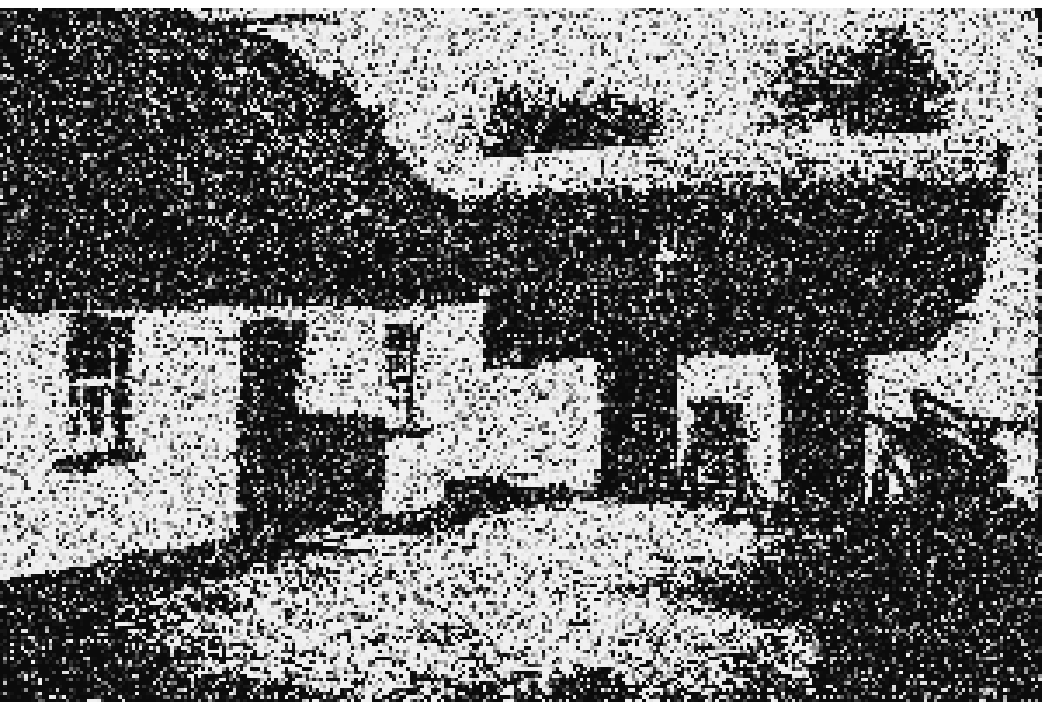}\tabularnewline
\end{tabular}
\par\end{centering}

\caption{Predicted marginals for an example binary denoising test image with
different noise levels $n$.}
\end{figure}
\begin{figure}[p]
\begin{centering}
\vspace{-10pt}
\renewcommand{\tabcolsep}{1pt}%
\begin{tabular}{>{\centering}m{0.57in}>{\centering}m{0.9in}>{\centering}m{0.9in}>{\centering}m{0.9in}}
 & $n=1.25$ & $n=1.5$ & $n=5$\tabularnewline
{\small input} & \includegraphics[width=1\linewidth]{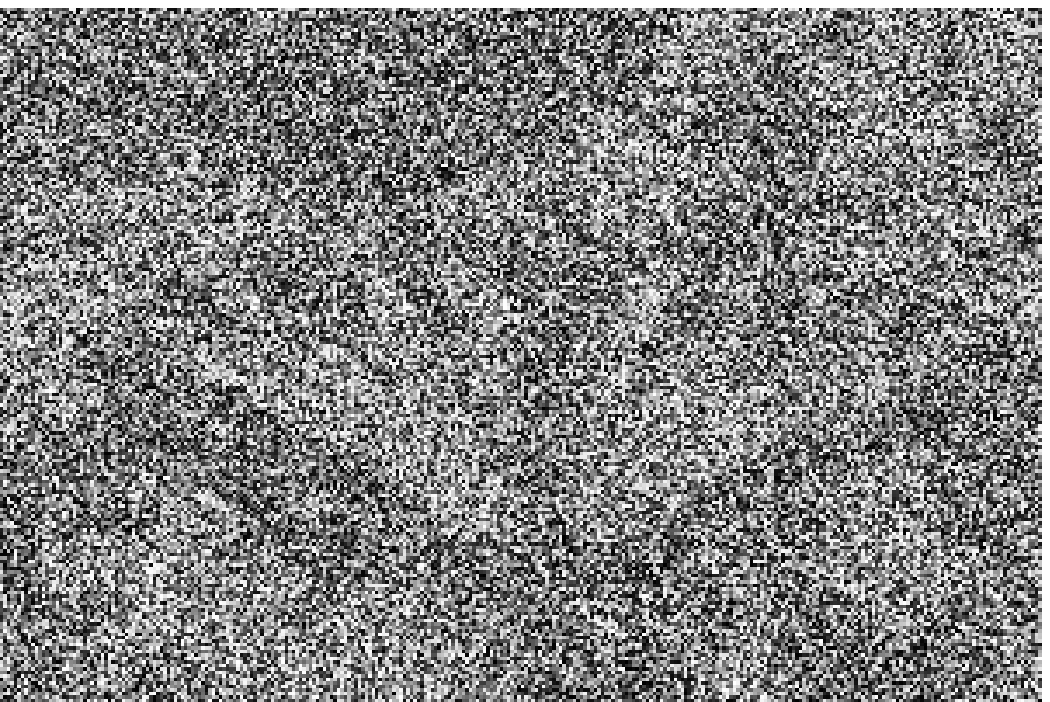} & \includegraphics[width=1\linewidth]{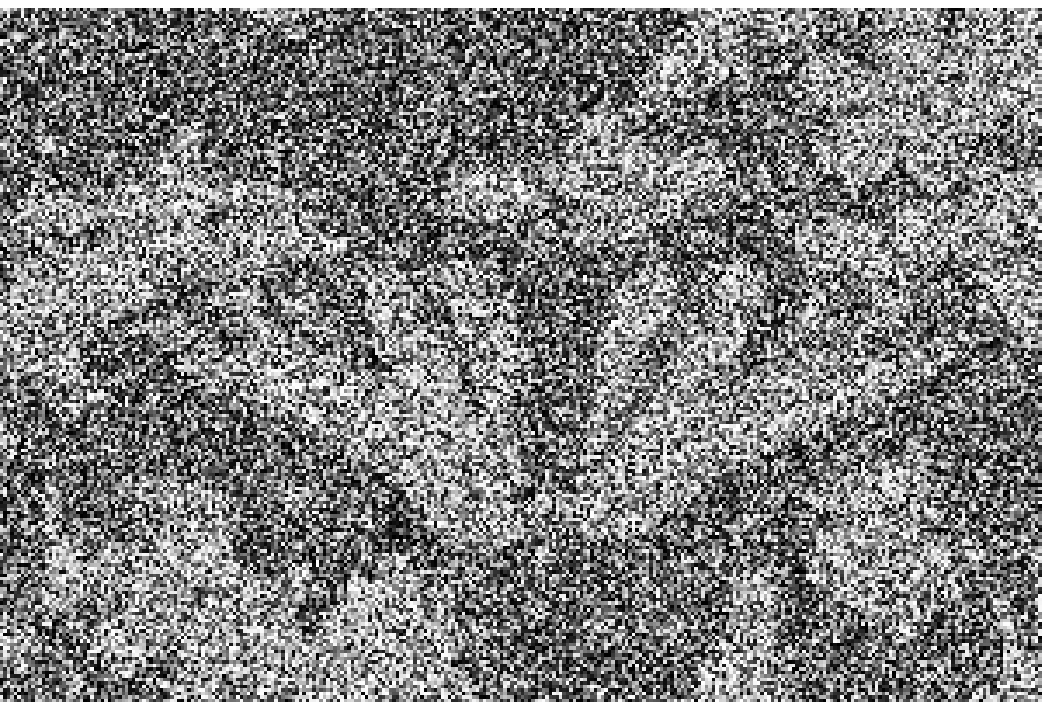} & \includegraphics[width=1\linewidth]{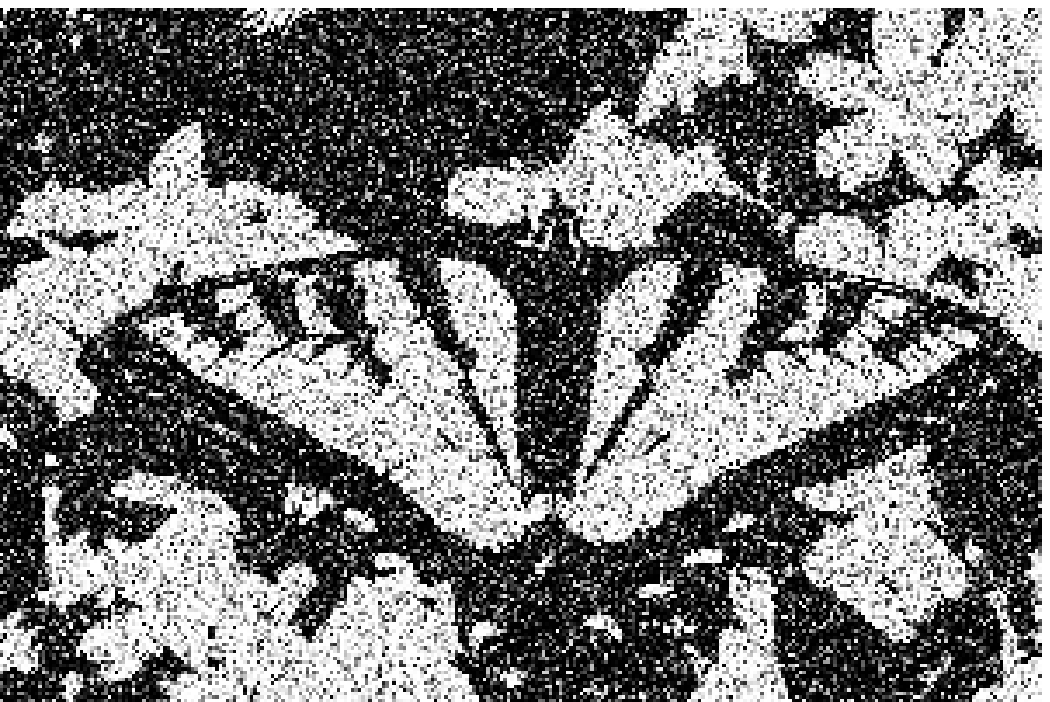}\tabularnewline
{\small surrogate likelihood} & \includegraphics[width=1\linewidth]{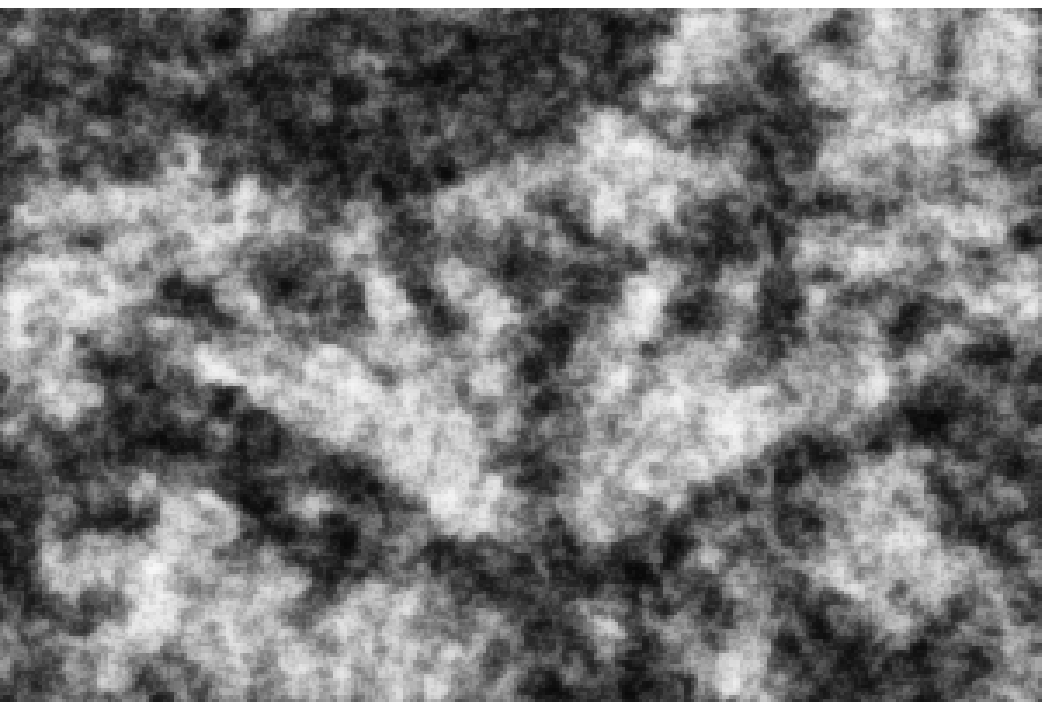} & \includegraphics[width=1\linewidth]{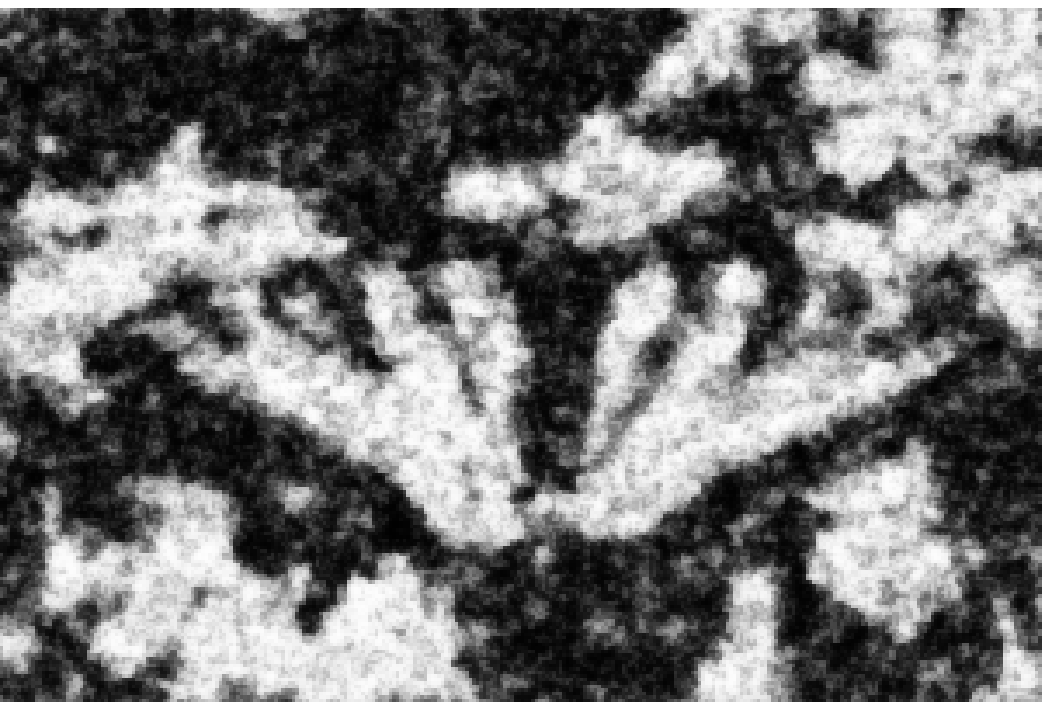} & \includegraphics[width=1\linewidth]{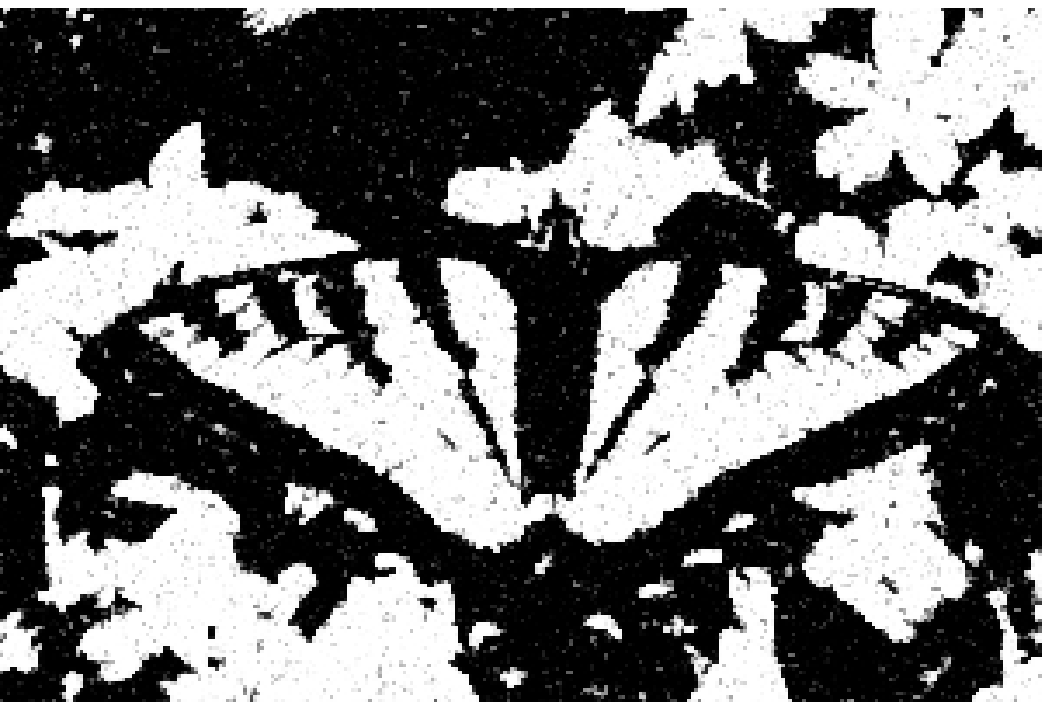}\tabularnewline
{\small univariate logistic} & \includegraphics[width=1\linewidth]{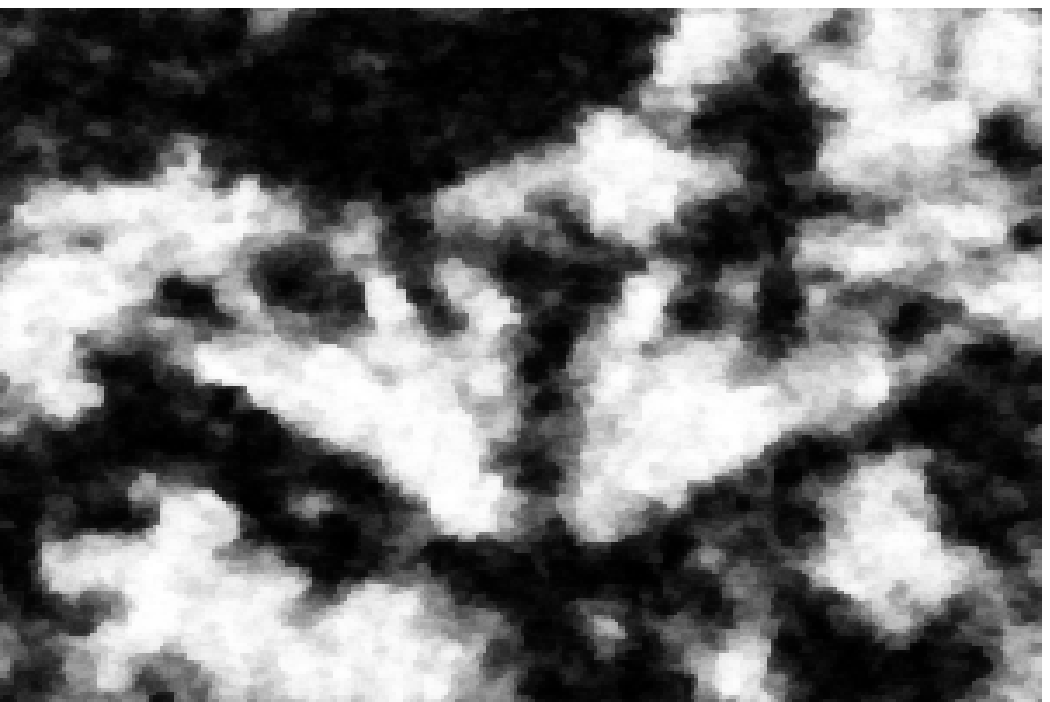} & \includegraphics[width=1\linewidth]{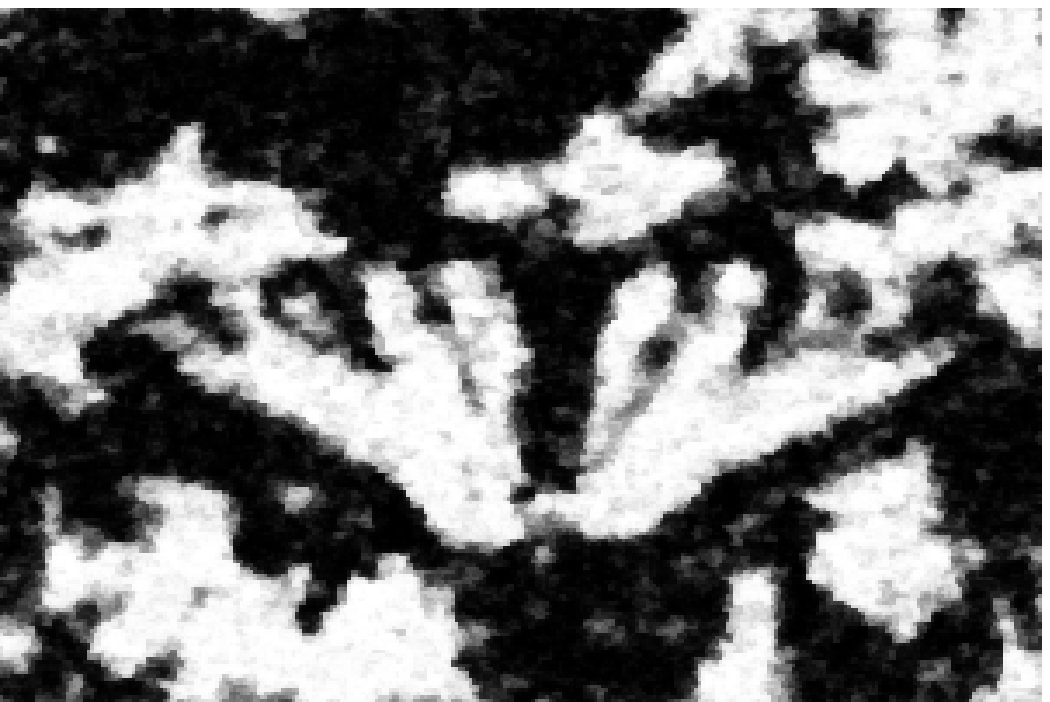} & \includegraphics[width=1\linewidth]{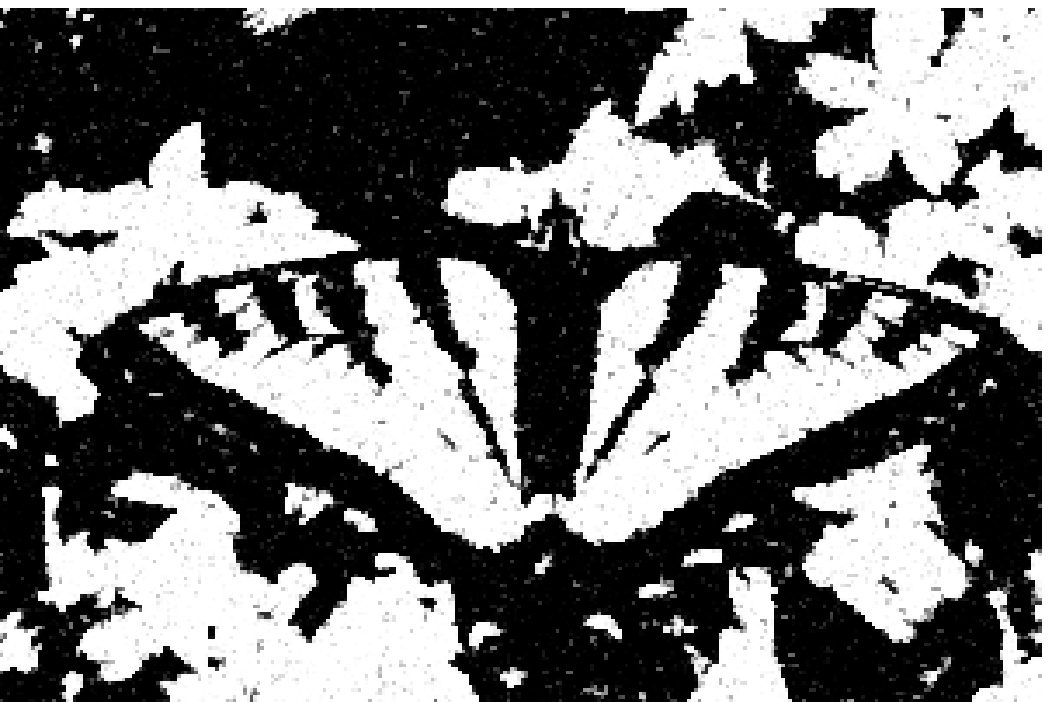}\tabularnewline
{\small clique logistic} & \includegraphics[width=1\linewidth]{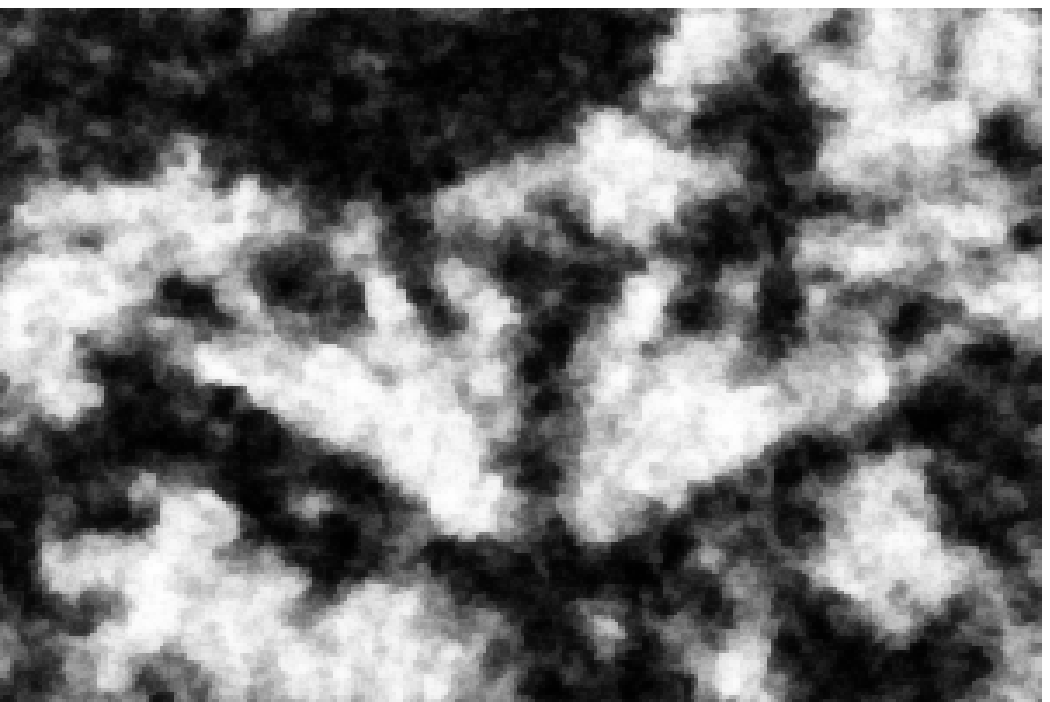} & \includegraphics[width=1\linewidth]{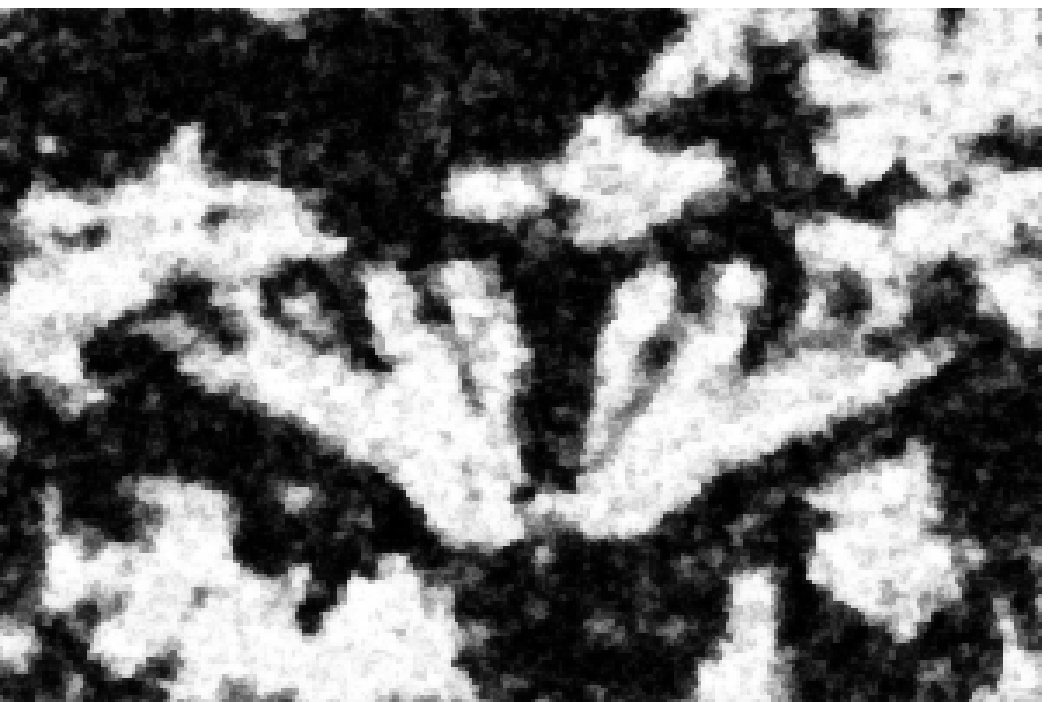} & \includegraphics[width=1\linewidth]{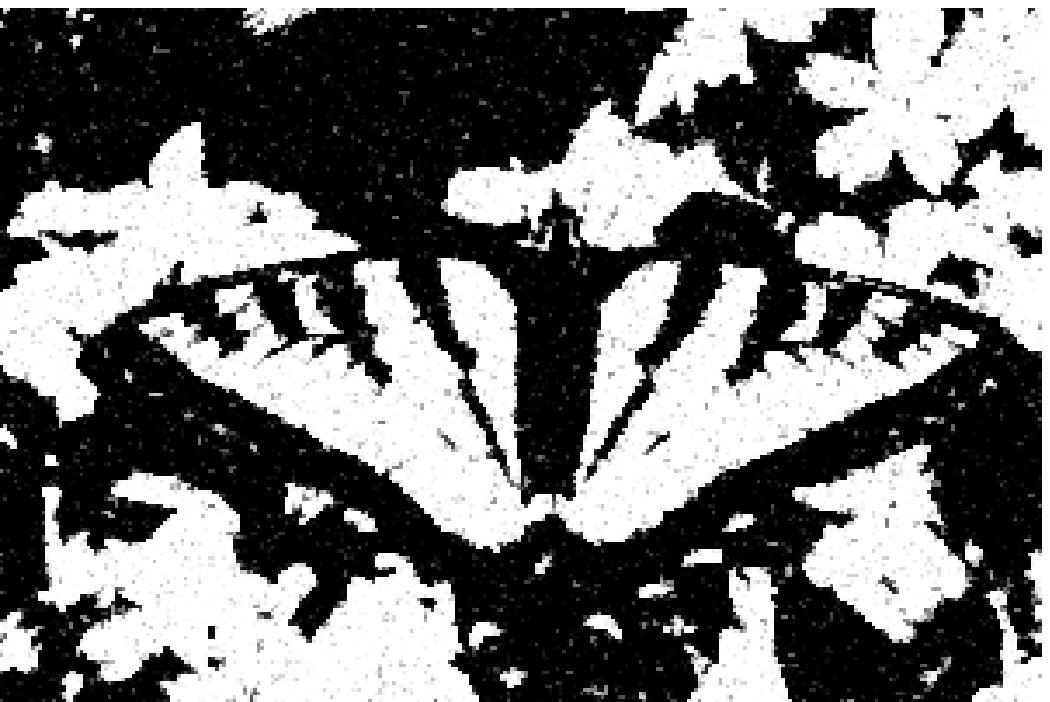}\tabularnewline
sm. class $\lambda=5$ & \includegraphics[width=1\linewidth]{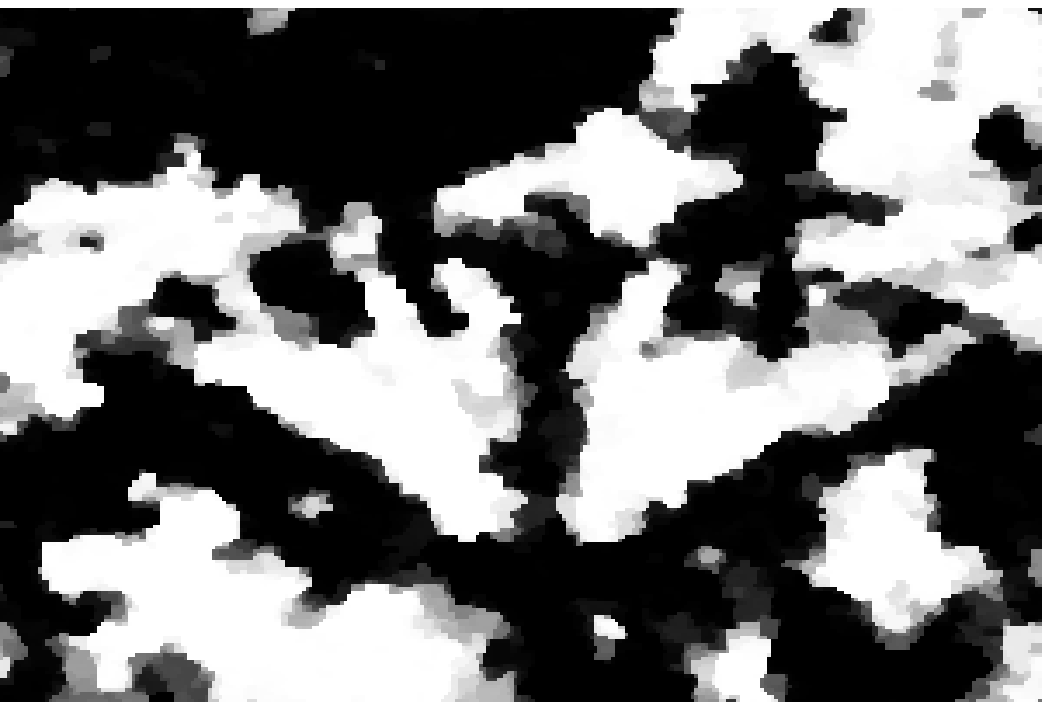} & \includegraphics[width=1\linewidth]{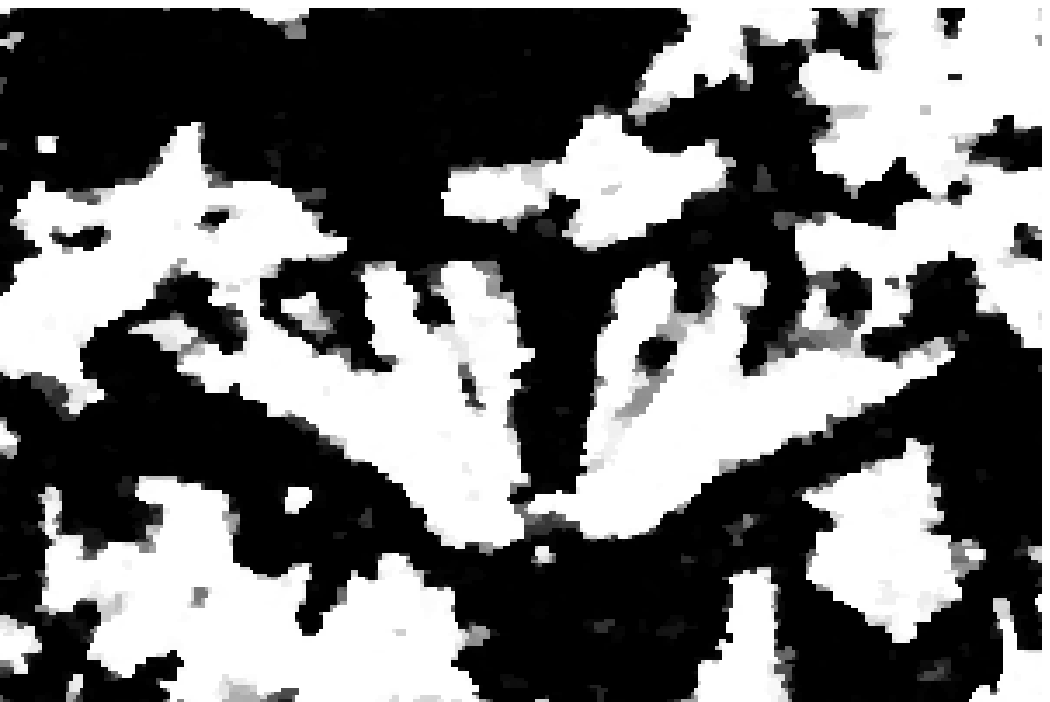} & \includegraphics[width=1\linewidth]{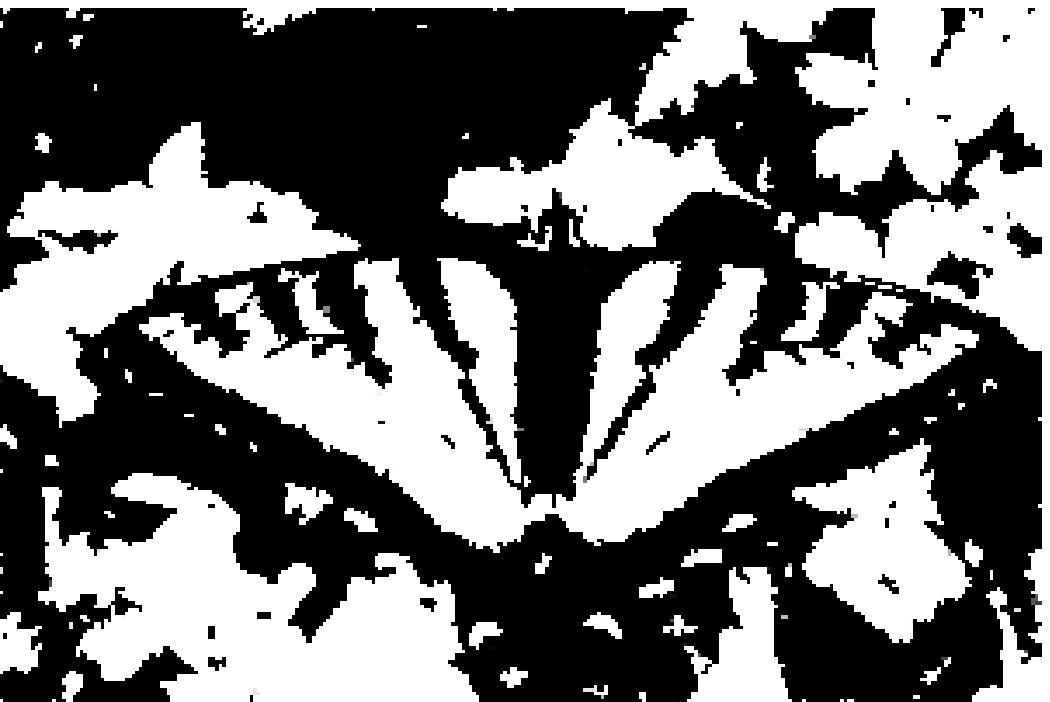}\tabularnewline
sm. class $\lambda=15$ & \includegraphics[width=1\linewidth]{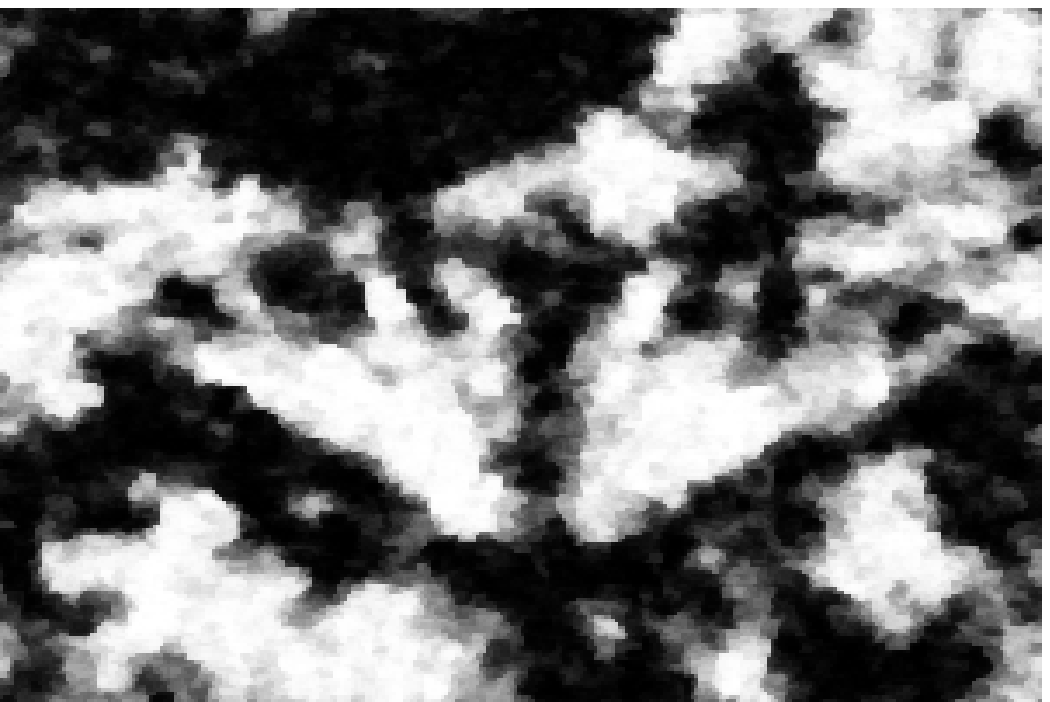} & \includegraphics[width=1\linewidth]{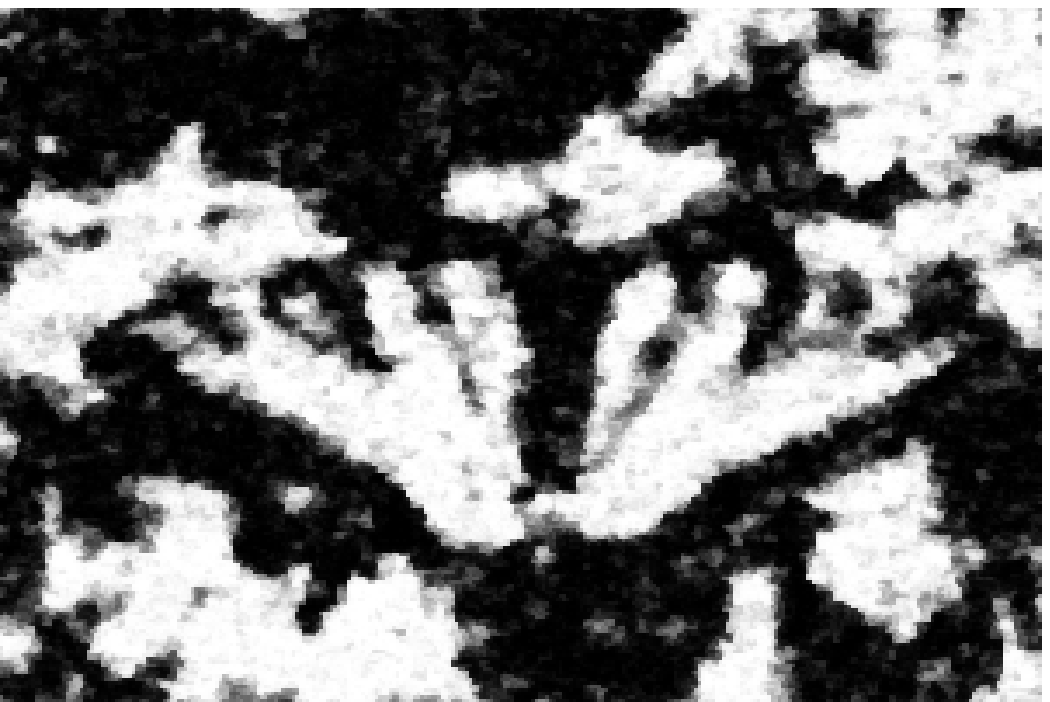} & \includegraphics[width=1\linewidth]{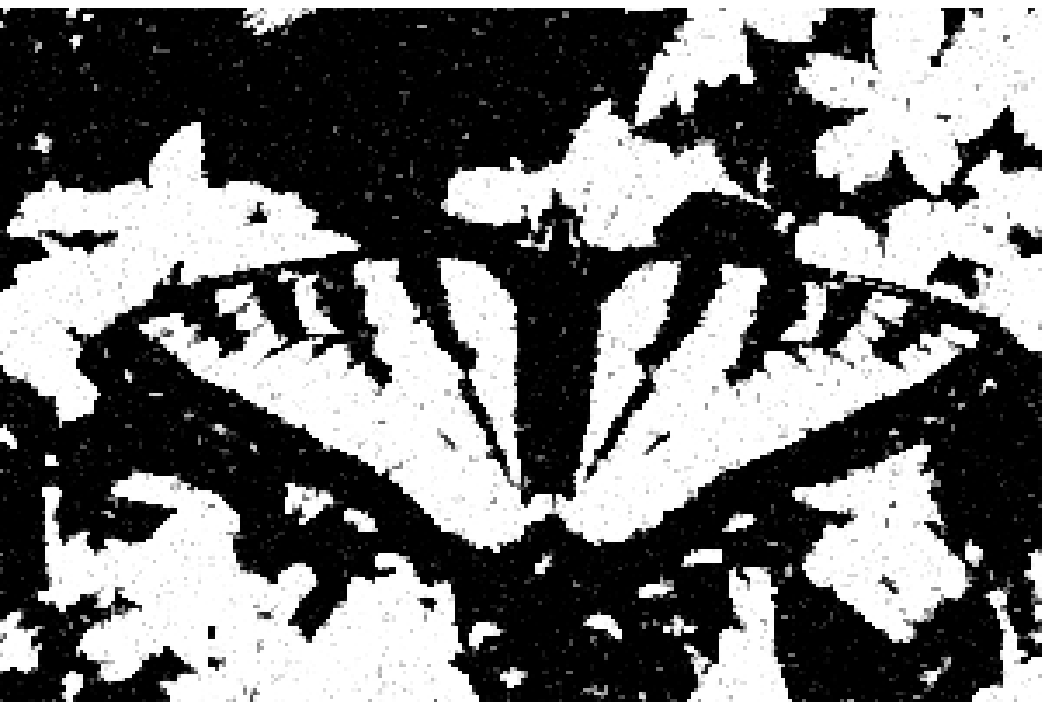}\tabularnewline
sm. class $\lambda=50$ & \includegraphics[width=1\linewidth]{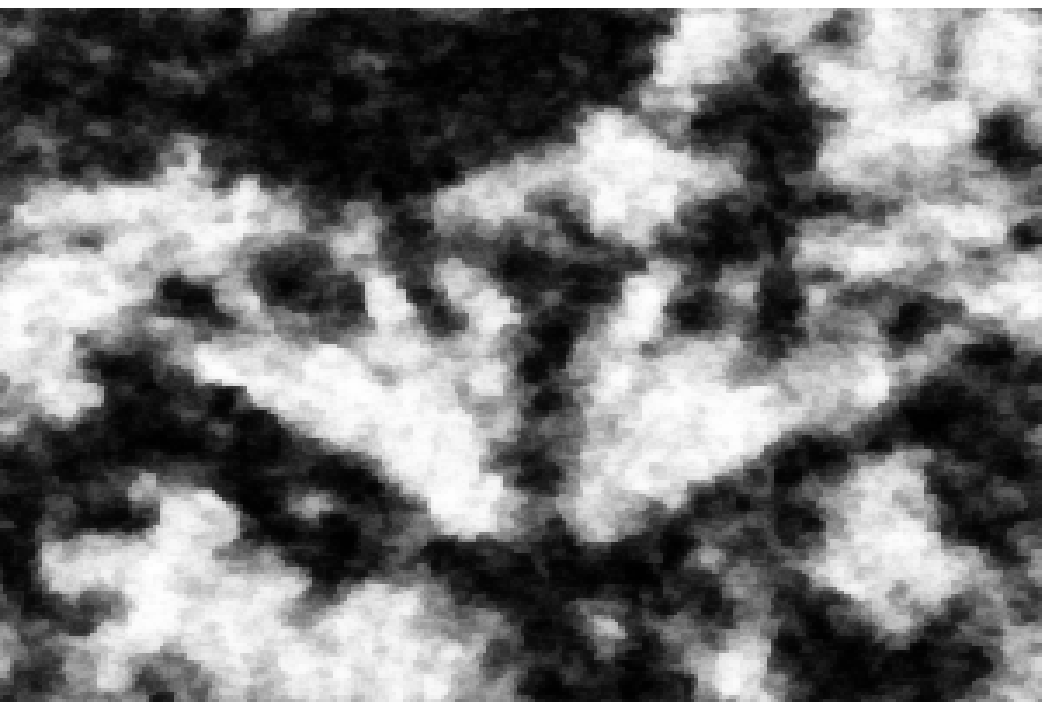} & \includegraphics[width=1\linewidth]{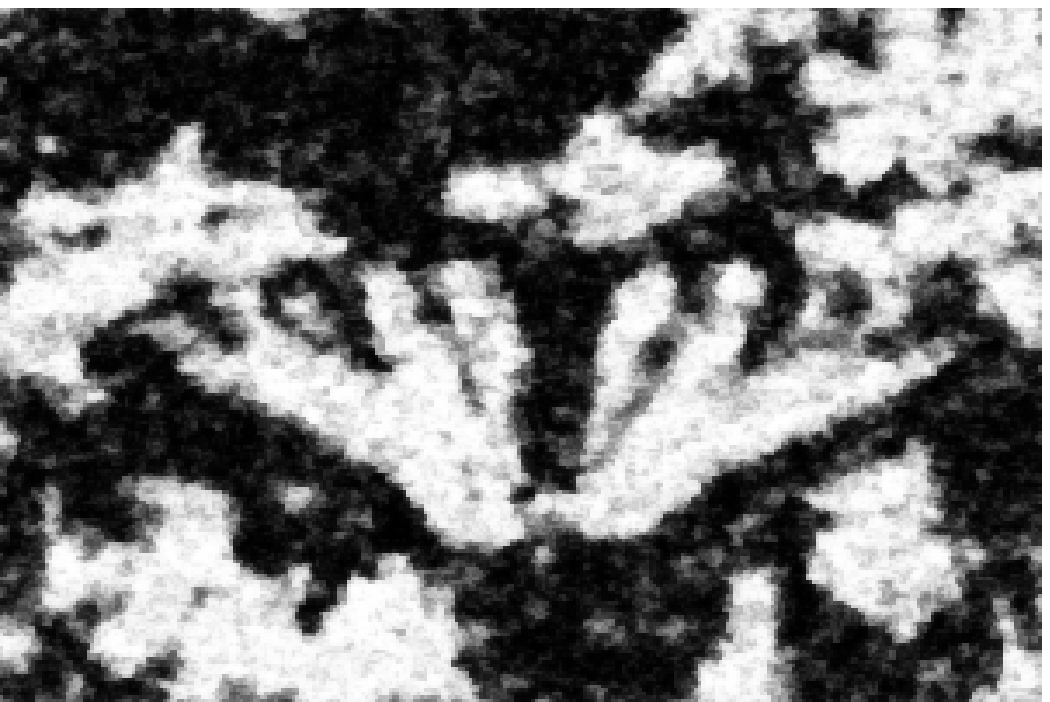} & \includegraphics[width=1\linewidth]{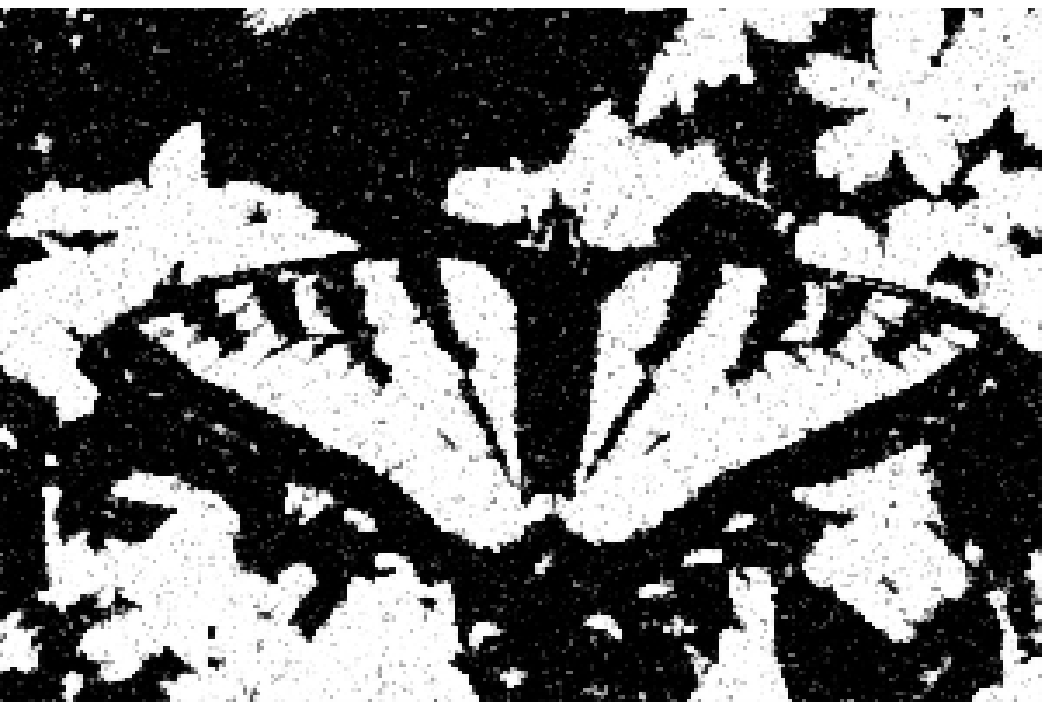}\tabularnewline
{\small pseudo-likelihood} & \includegraphics[width=1\linewidth]{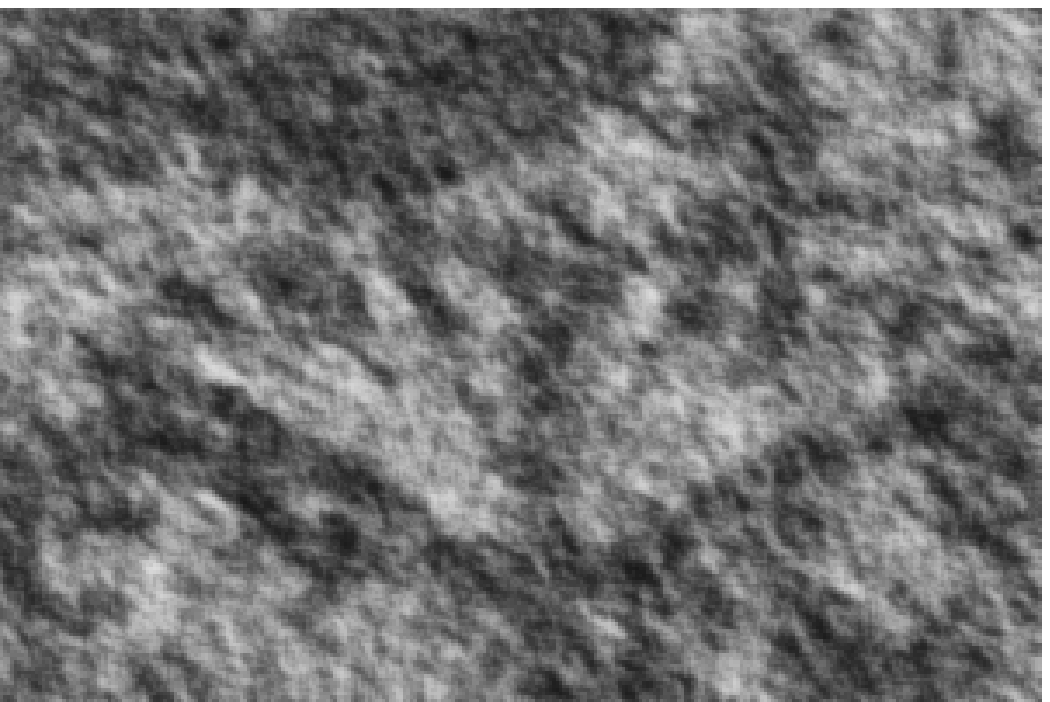} & \includegraphics[width=1\linewidth]{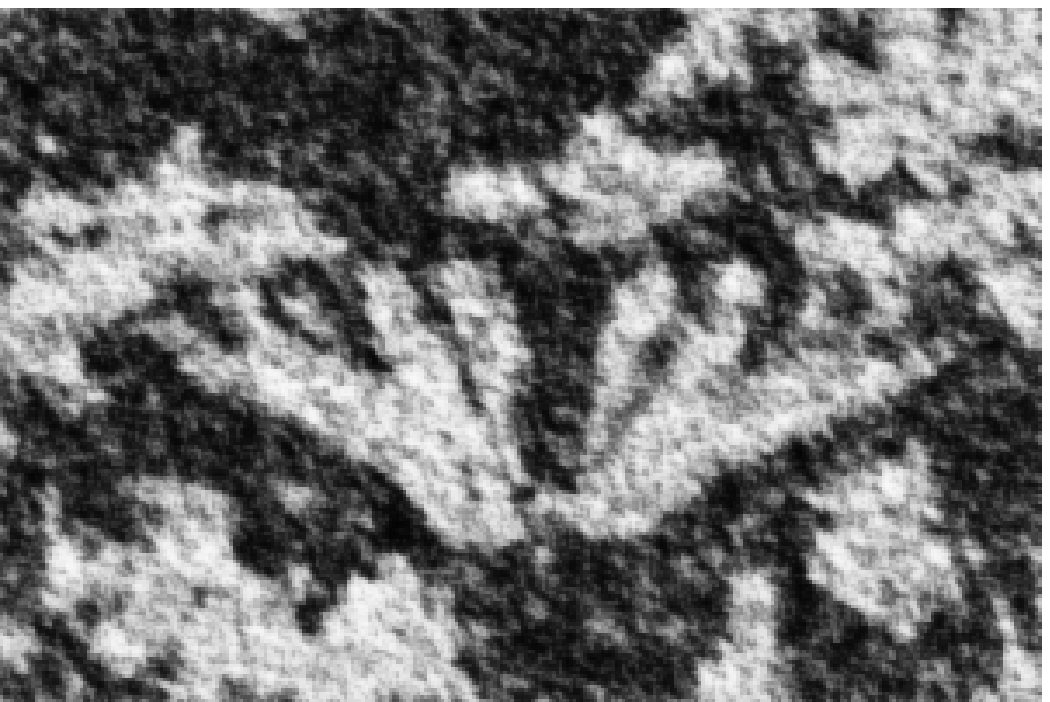} & \includegraphics[width=1\linewidth]{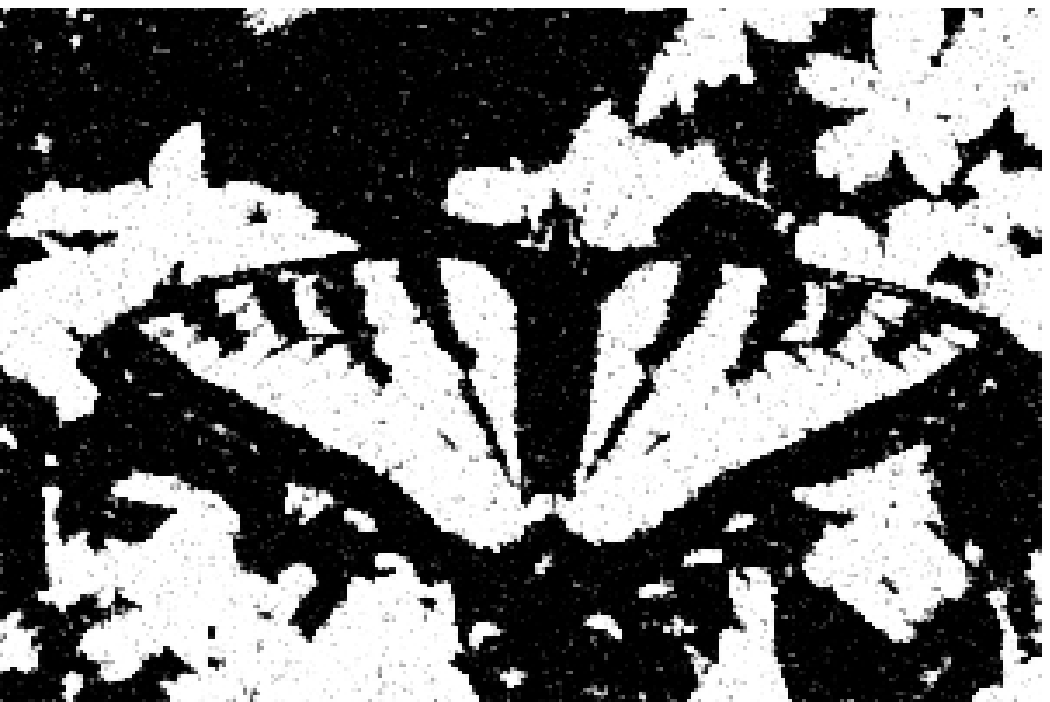}\tabularnewline
{\small piecewise} & \includegraphics[width=1\linewidth]{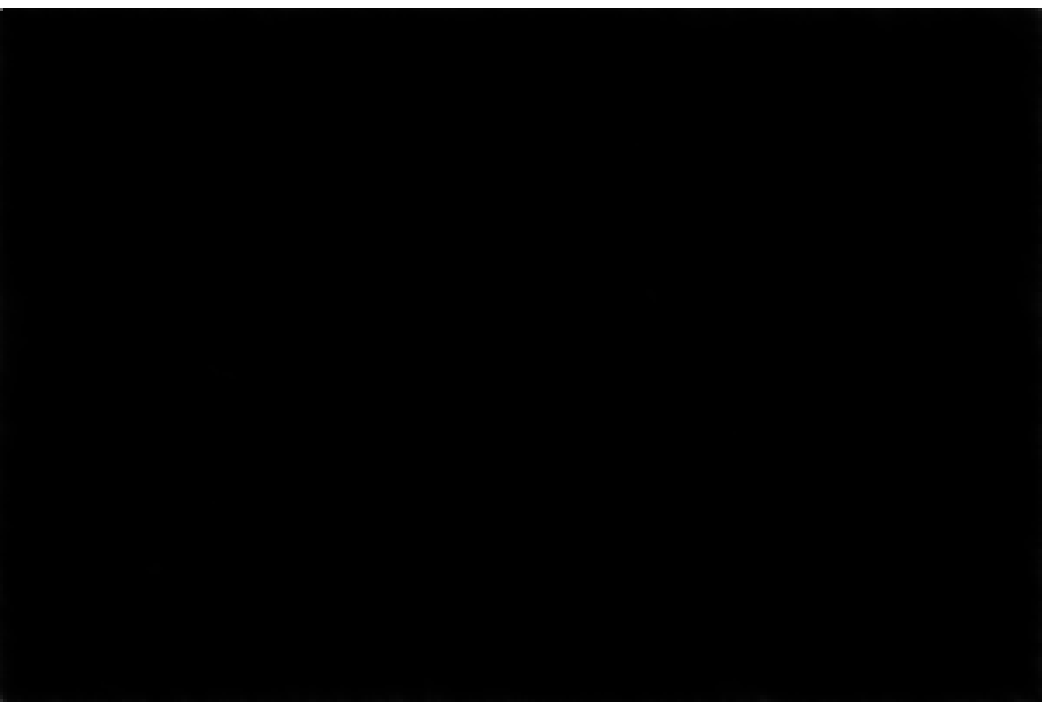} & \includegraphics[width=1\linewidth]{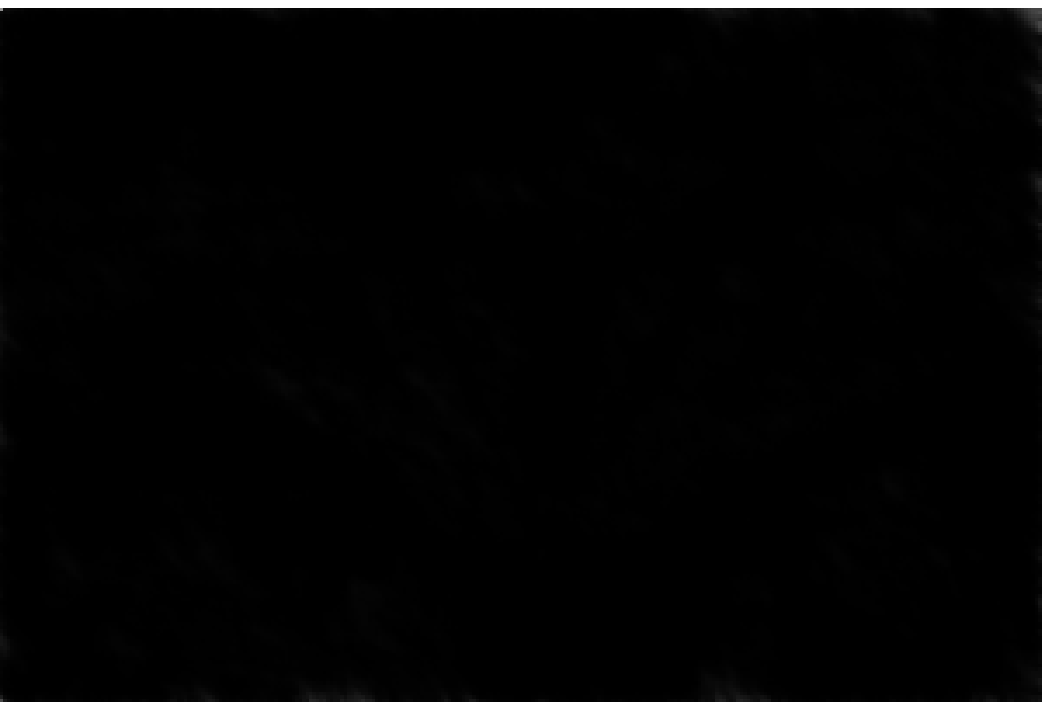} & \includegraphics[width=1\linewidth]{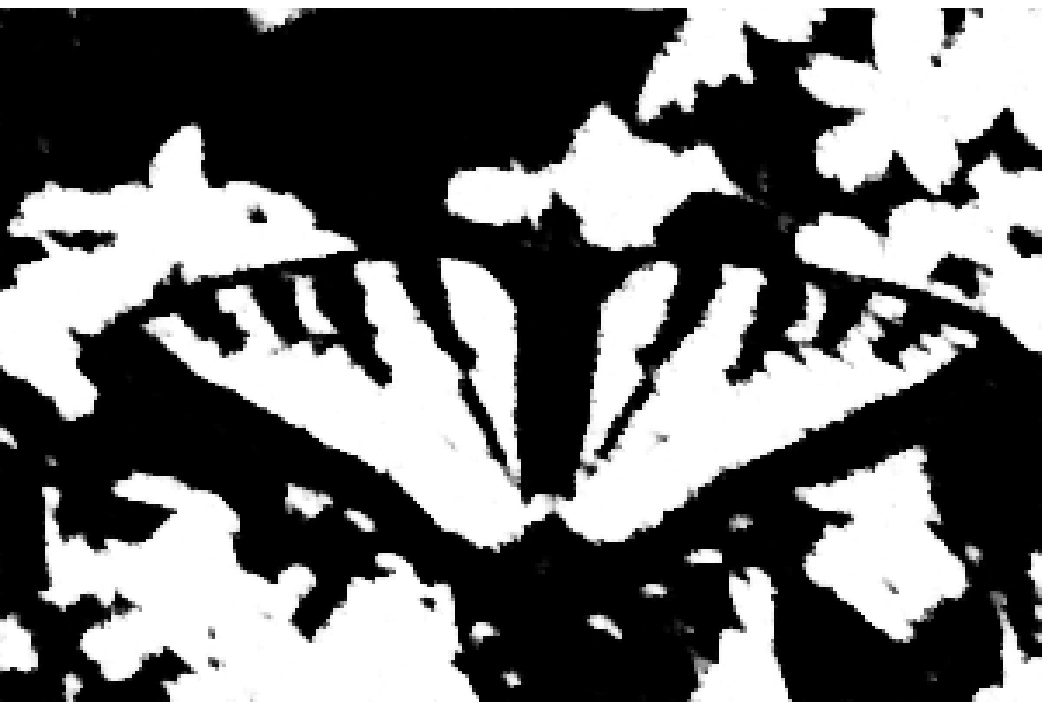}\tabularnewline
{\small inde-pendent} & \includegraphics[width=1\linewidth]{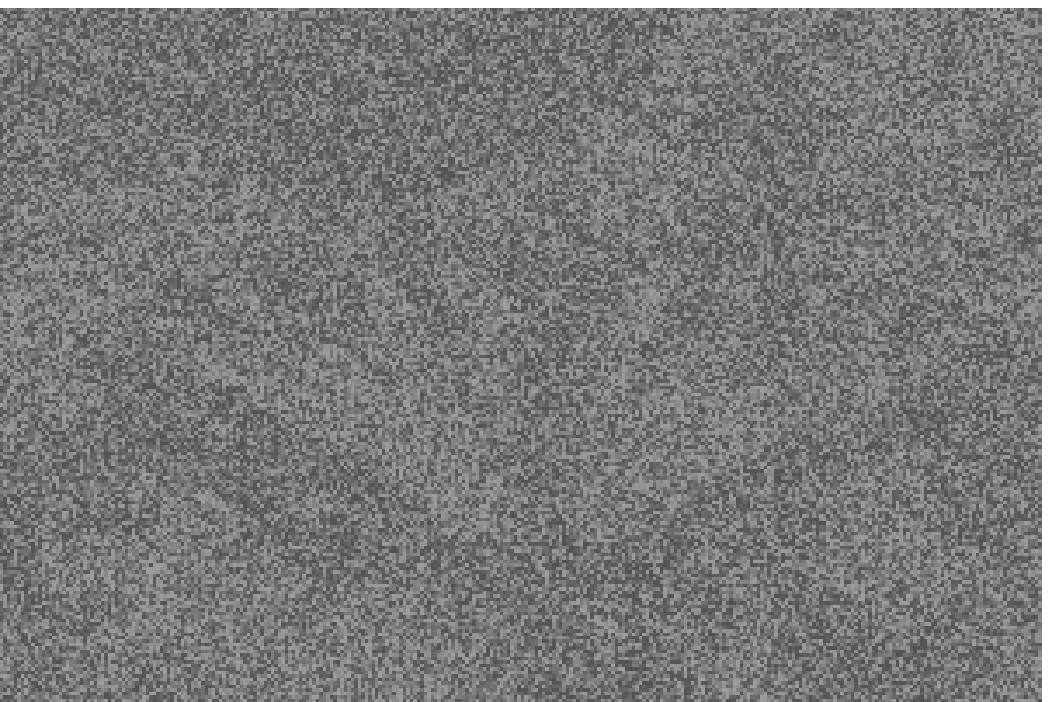} & \includegraphics[width=1\linewidth]{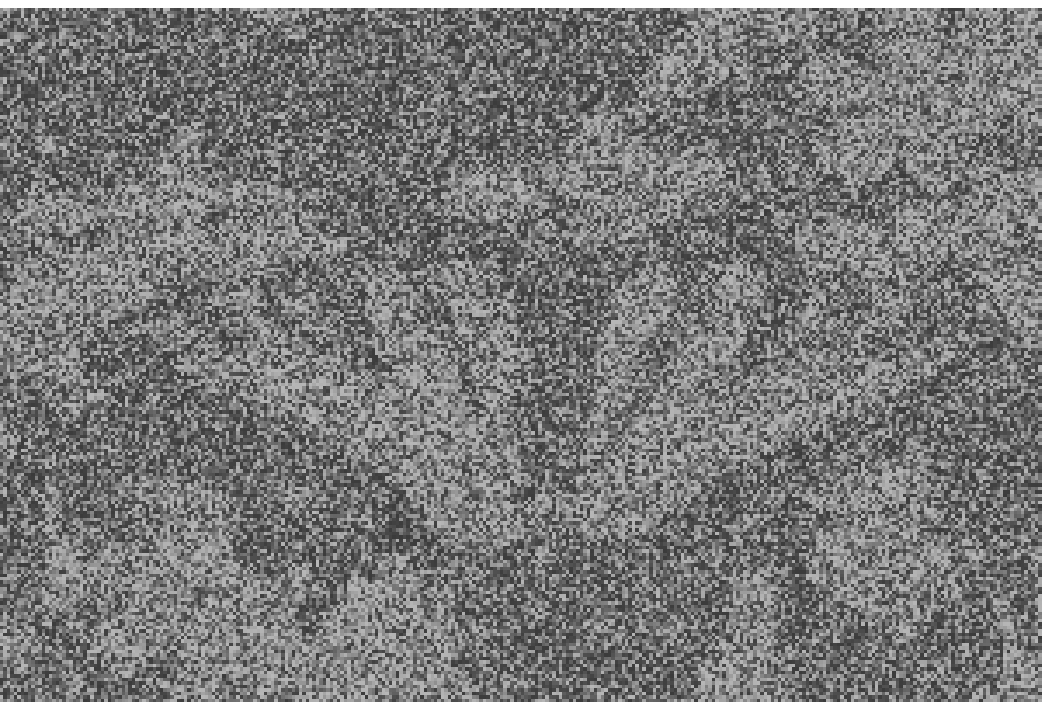} & \includegraphics[width=1\linewidth]{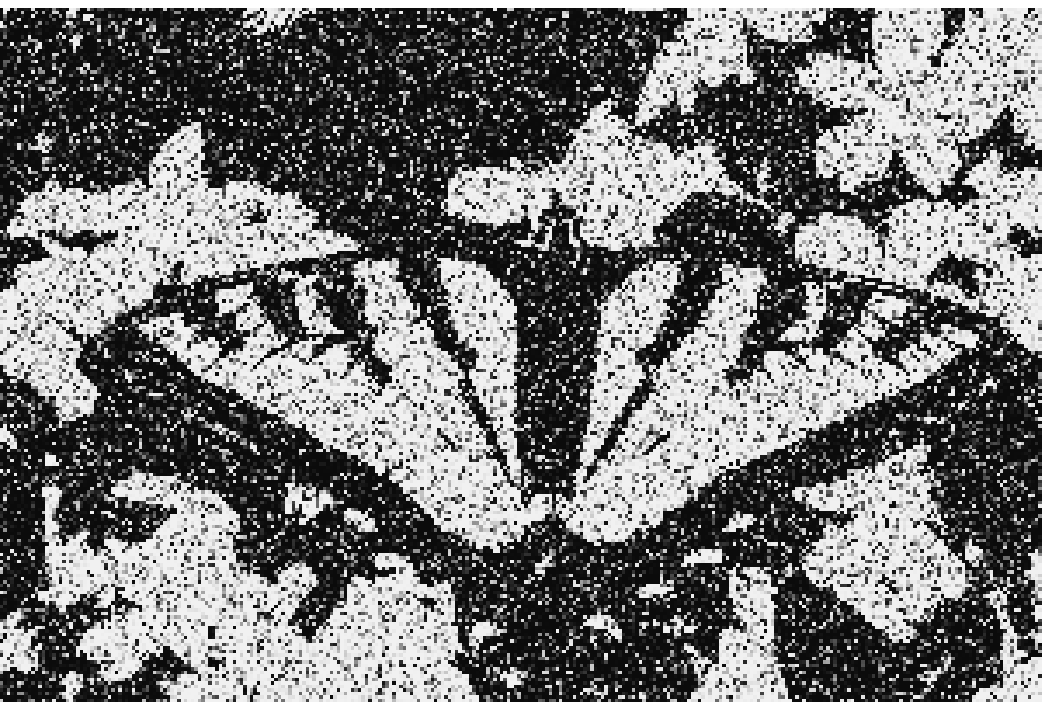}\tabularnewline
\end{tabular}
\par\end{centering}

\caption{Predicted marginals for an example binary denoising test image with
different noise levels $n$.}
\end{figure}
\begin{figure}[p]
\begin{raggedright}
\vspace{-15pt}

\par\end{raggedright}

\begin{raggedright}
\subfloat[Input]{\includegraphics[width=0.32\columnwidth]{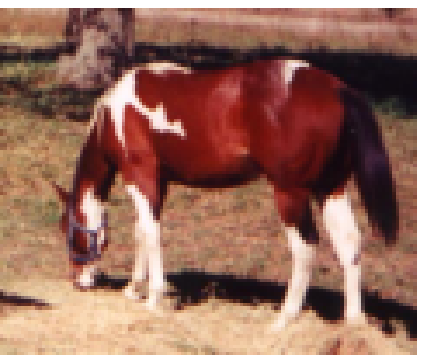}

}\subfloat[True Labels]{\includegraphics[width=0.32\columnwidth]{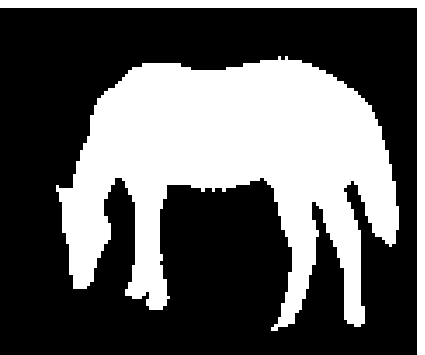}

}
\par\end{raggedright}

\begin{raggedright}
\subfloat[Surr. Like. {\scriptsize TRW}]{\includegraphics[width=0.32\columnwidth]{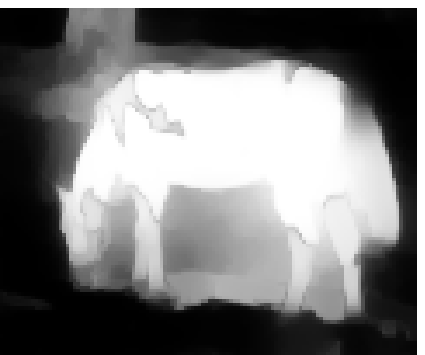}

}\subfloat[U. Logistic {\scriptsize TRW}]{\includegraphics[width=0.32\columnwidth]{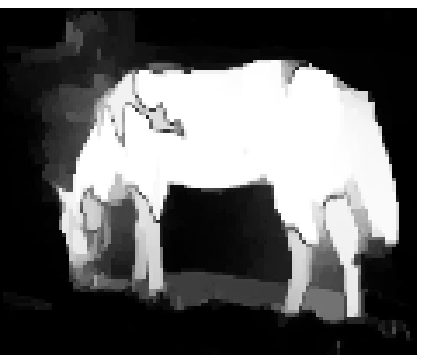}

}\subfloat[Sm. Class {\scriptsize $\lambda$$=$$50$} {\scriptsize TRW}]{\includegraphics[width=0.32\columnwidth]{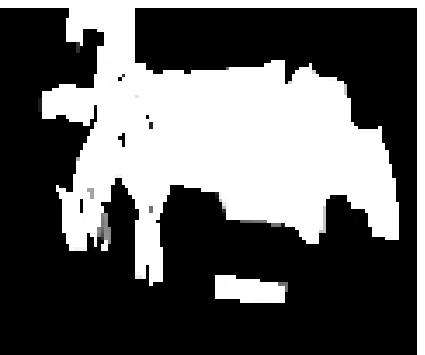}

}
\par\end{raggedright}

\begin{raggedright}
\subfloat[Surr. Like. {\scriptsize MNF}]{\includegraphics[width=0.32\columnwidth]{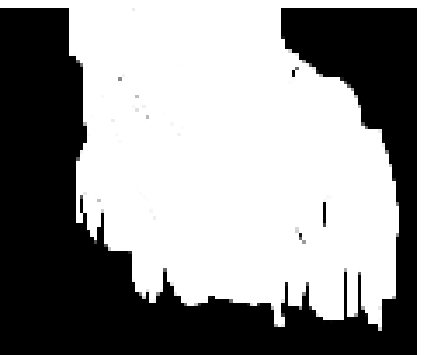}

}\subfloat[U. Logistic {\scriptsize MNF}]{\includegraphics[width=0.32\columnwidth]{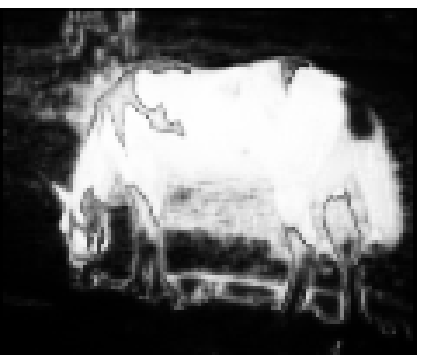}

}\subfloat[Independent]{\includegraphics[width=0.32\columnwidth]{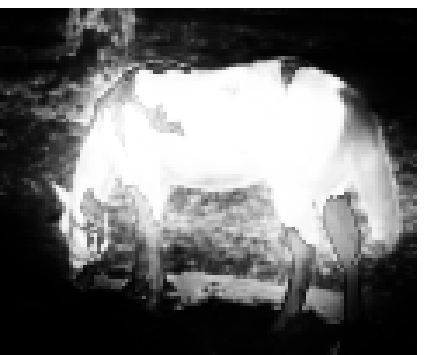}

}
\par\end{raggedright}

\caption{Predicted marginals for a test image from the horses dataset. Truncated
learning uses 40 iterations.}
\end{figure}
\begin{figure}[p]
\begin{raggedright}
\vspace{-15pt}

\par\end{raggedright}

\begin{raggedright}
\subfloat[Input]{\includegraphics[width=0.32\columnwidth]{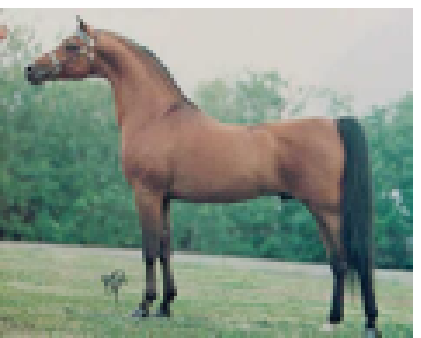}

}\subfloat[True Labels]{\includegraphics[width=0.32\columnwidth]{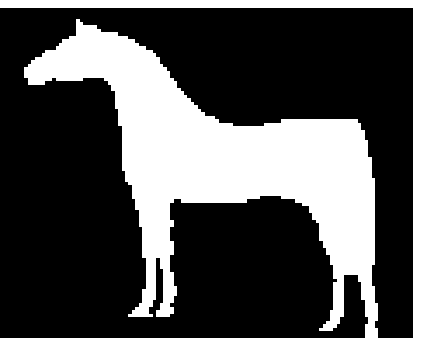}

}
\par\end{raggedright}

\begin{raggedright}
\subfloat[Surr. Like. {\scriptsize TRW}]{\includegraphics[width=0.32\columnwidth]{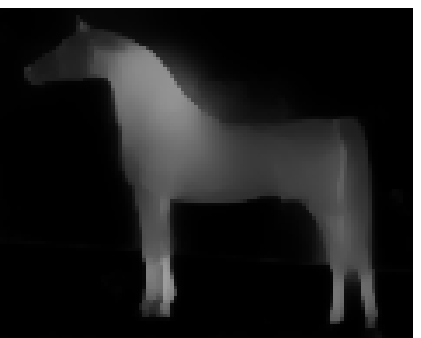}

}\subfloat[U. Logistic {\scriptsize TRW}]{\includegraphics[width=0.32\columnwidth]{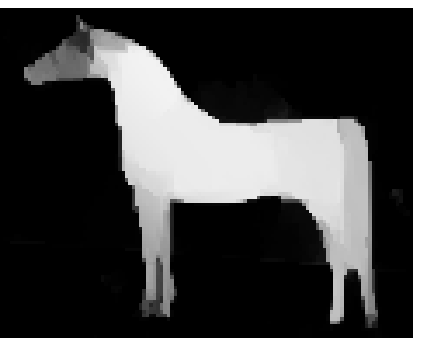}

}\subfloat[Sm. Class {\scriptsize $\lambda$$=$$50$} {\scriptsize TRW}]{\includegraphics[width=0.32\columnwidth]{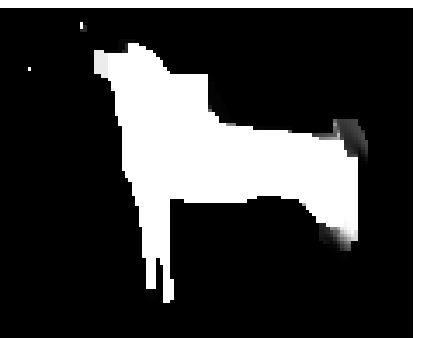}

}
\par\end{raggedright}

\begin{raggedright}
\subfloat[Surr. Like. {\scriptsize MNF}]{\includegraphics[width=0.32\columnwidth]{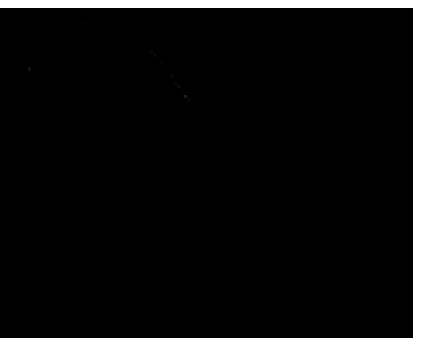}

}\subfloat[U. Logistic {\scriptsize MNF}]{\includegraphics[width=0.32\columnwidth]{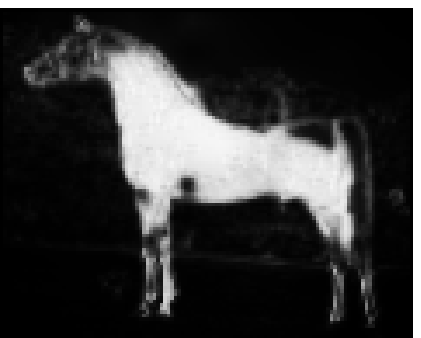}

}\subfloat[Independent]{\includegraphics[width=0.32\columnwidth]{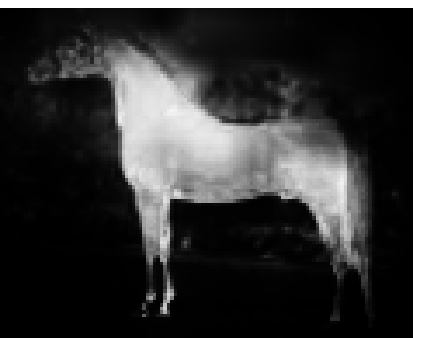}

}
\par\end{raggedright}

\caption{Predicted marginals for a test image from the horses dataset. Truncated
learning uses 40 iterations.}
\end{figure}
\begin{figure}[p]
\vspace{-10pt}
\subfloat[Input Image]{\includegraphics[width=0.32\columnwidth]{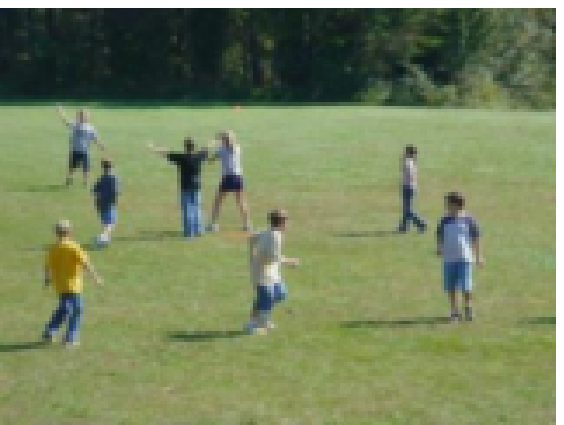}

}\subfloat[True Labels]{\includegraphics[width=0.32\columnwidth]{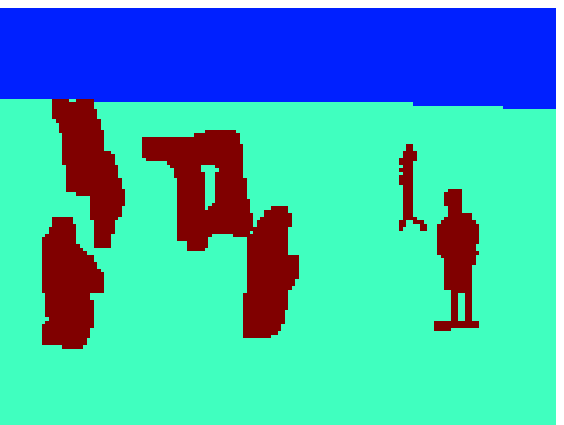}

}

\subfloat[Surrogate EM]{\includegraphics[width=0.32\columnwidth]{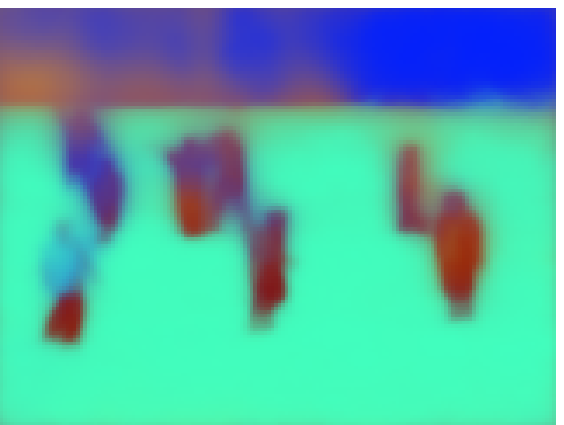}

}\subfloat[Univ. Logistic]{\includegraphics[width=0.32\columnwidth]{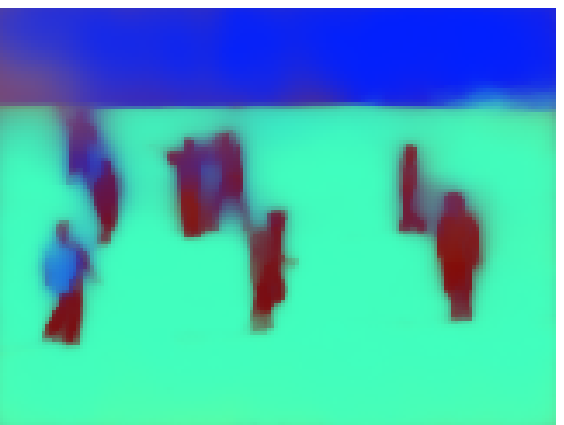}

}\subfloat[Clique Logistic]{\includegraphics[width=0.32\columnwidth]{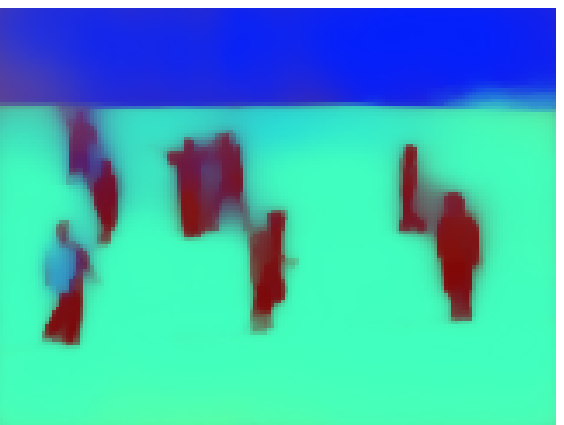}

}

\subfloat[Pseudolikelihood]{\includegraphics[width=0.32\columnwidth]{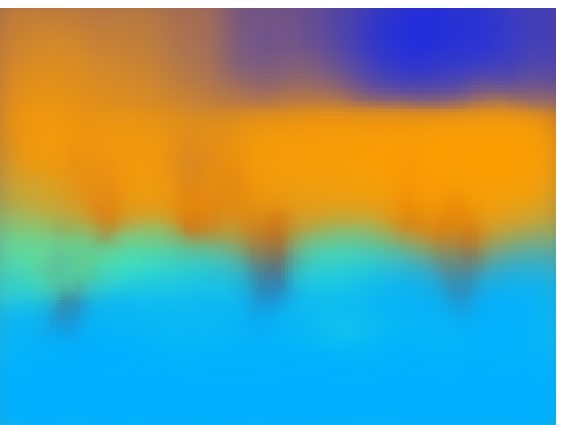}

}\subfloat[Piecewise]{\includegraphics[width=0.32\columnwidth]{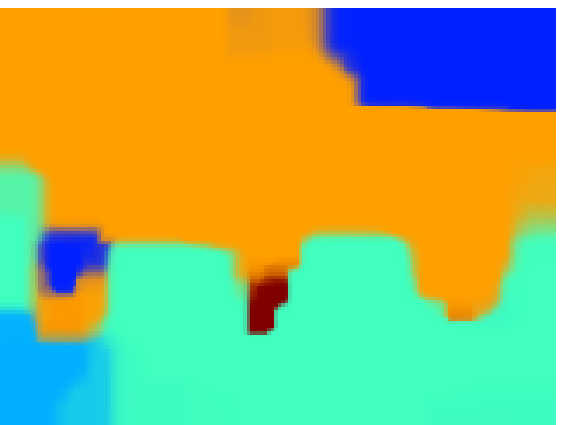}

}\subfloat[Independent]{\includegraphics[width=0.32\columnwidth]{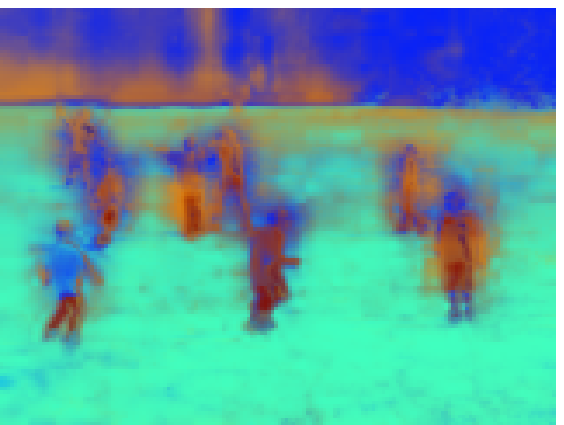}

}

\caption{Example marginals from the backgrounds dataset using $20$ iterations
for truncated fitting.}
\end{figure}
\begin{figure}[p]
\vspace{-10pt}
\subfloat[Input Image]{\includegraphics[width=0.32\columnwidth]{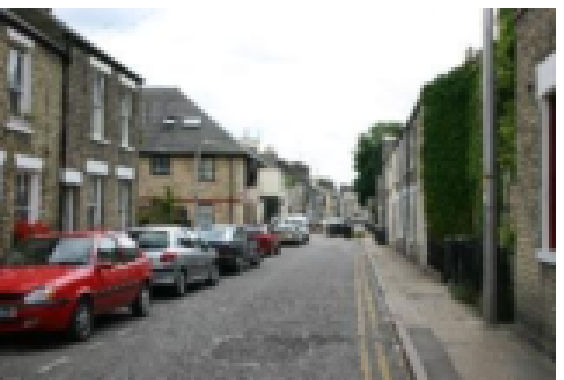}

}\subfloat[True Labels]{\includegraphics[width=0.32\columnwidth]{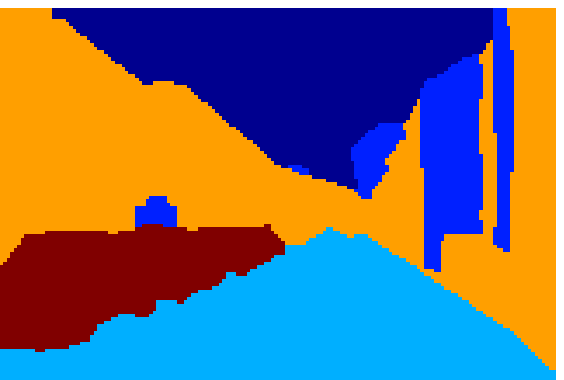}

}

\subfloat[Surrogate EM]{\includegraphics[width=0.32\columnwidth]{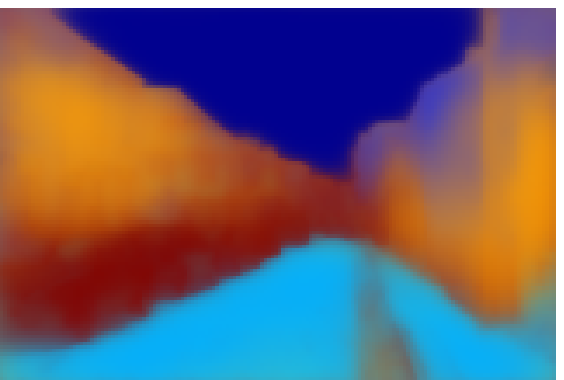}

}\subfloat[Univ. Logistic]{\includegraphics[width=0.32\columnwidth]{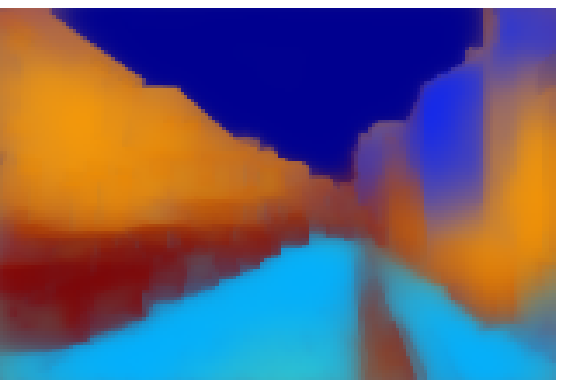}

}\subfloat[Clique Logistic]{\includegraphics[width=0.32\columnwidth]{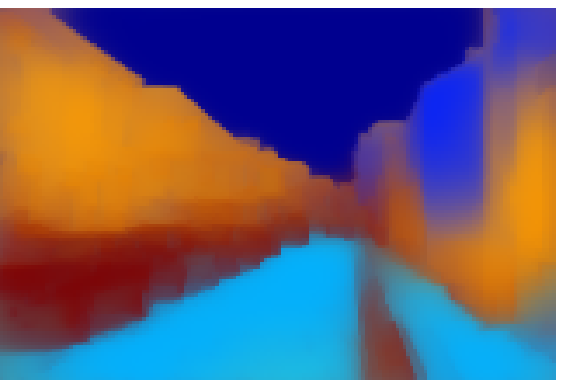}

}

\subfloat[Pseudolikelihood]{\includegraphics[width=0.32\columnwidth]{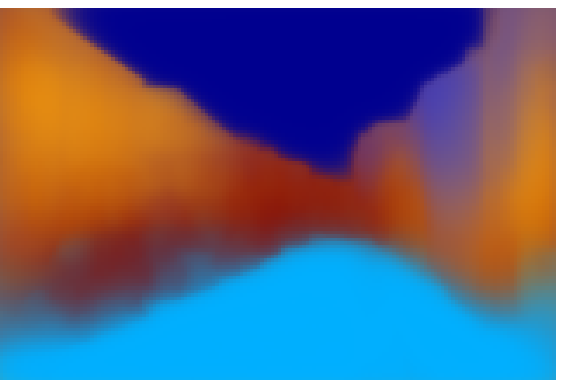}

}\subfloat[Piecewise]{\includegraphics[width=0.32\columnwidth]{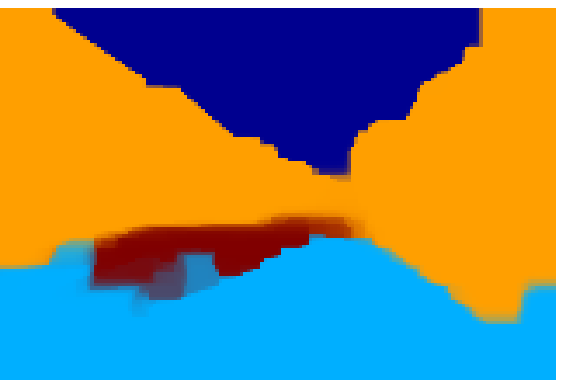}

}\subfloat[Independent]{\includegraphics[width=0.32\columnwidth]{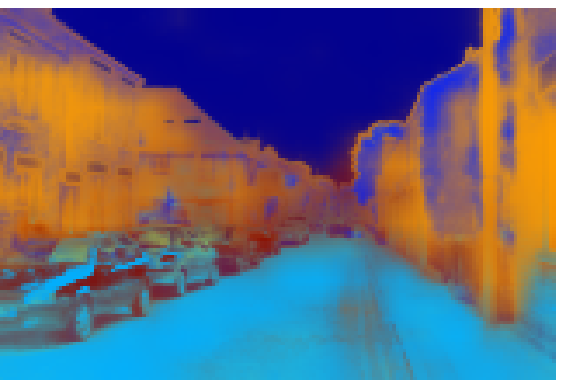}

}

\caption{Example marginals from the backgrounds dataset using $20$ iterations
for truncated fitting.}
\end{figure}
\begin{figure}[p]
\vspace{-10pt}
\subfloat[Input Image]{\includegraphics[width=0.32\columnwidth]{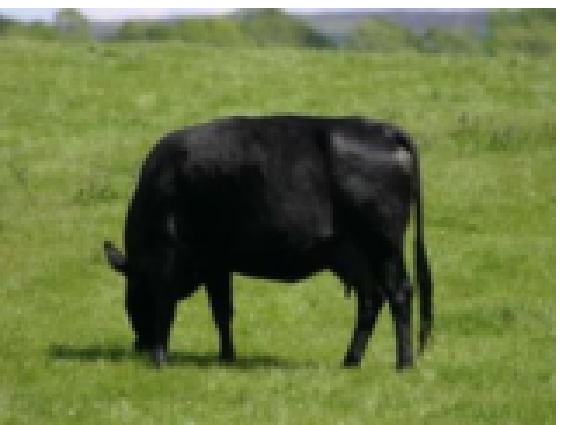}

}\subfloat[True Labels]{\includegraphics[width=0.32\columnwidth]{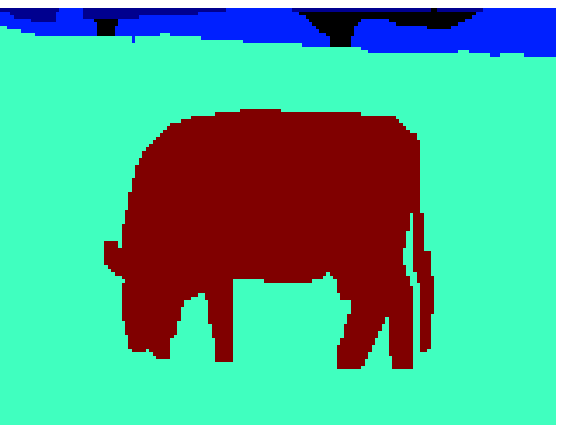}

}

\subfloat[Surrogate EM]{\includegraphics[width=0.32\columnwidth]{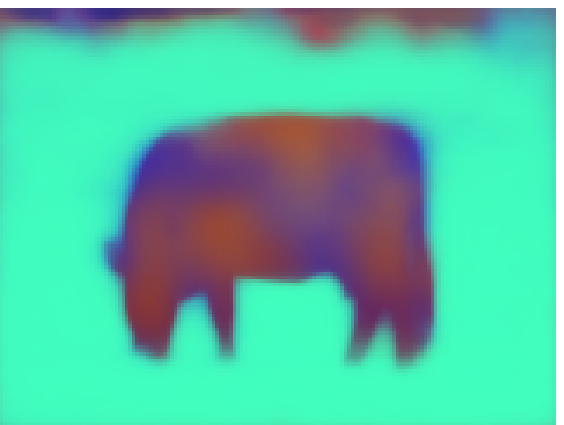}

}\subfloat[Univ. Logistic]{\includegraphics[width=0.32\columnwidth]{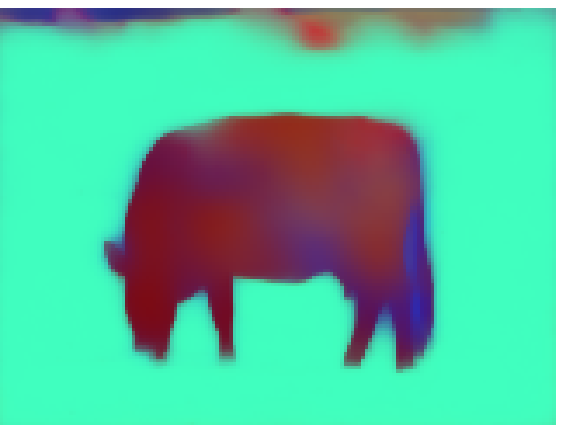}

}\subfloat[Clique Logistic]{\includegraphics[width=0.32\columnwidth]{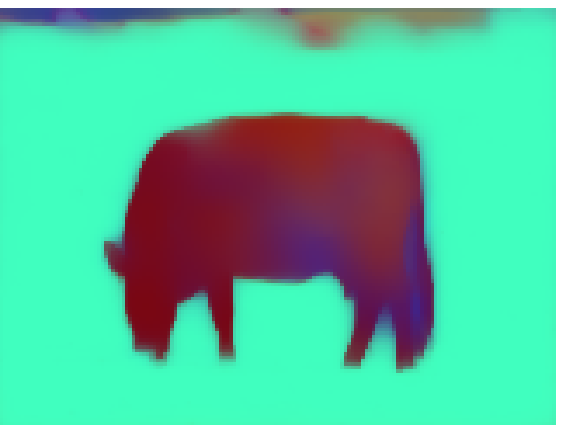}

}

\subfloat[Pseudolikelihood]{\includegraphics[width=0.32\columnwidth]{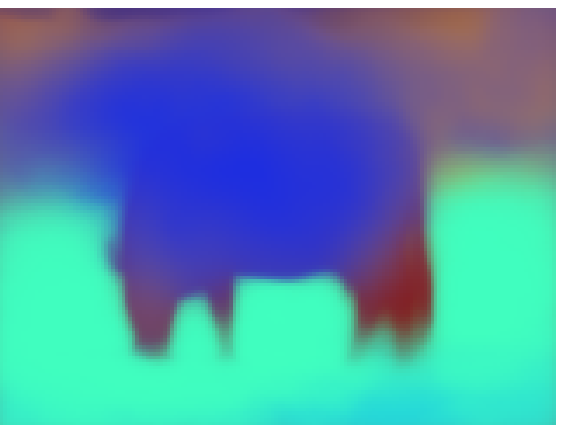}

}\subfloat[Piecewise]{\includegraphics[width=0.32\columnwidth]{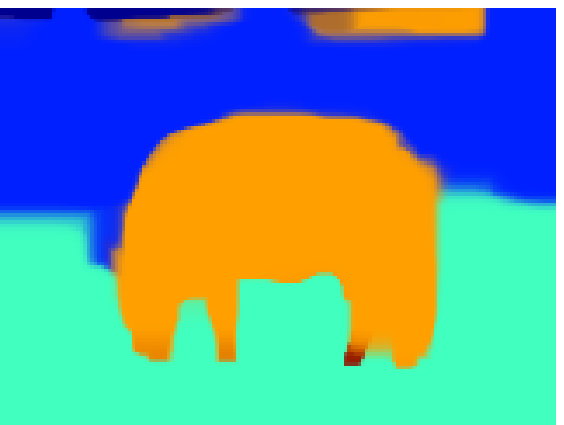}

}\subfloat[Independent]{\includegraphics[width=0.32\columnwidth]{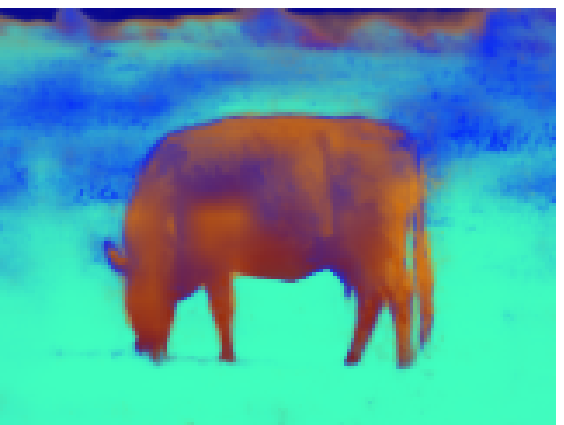}

}

\caption{Example marginals from the backgrounds dataset using $20$ iterations
for truncated fitting.}
\end{figure}

\end{document}